\documentclass[sigconf]{aamas} 

\usepackage{balance} 

\setcopyright{ifaamas}
\acmConference[AAMAS '22]{Proc.\@ of the 21st International Conference
on Autonomous Agents and Multiagent Systems (AAMAS 2022)}{May 9--13, 2022}
{Online}{P.~Faliszewski, V.~Mascardi, C.~Pelachaud,
M.E.~Taylor (eds.)}
\copyrightyear{2022}
\acmYear{2022}
\acmDOI{}
\acmPrice{}
\acmISBN{}

\acmSubmissionID{506}

\title[Robust LfO with Model Misspecification]{Robust Learning from Observation with Model Misspecification}

\author{Luca Viano}
\affiliation{
  \institution{LIONS, EPFL}
  \city{Lausanne}
  \country{Switzerland}}
\email{luca.viano@epfl.ch}

\author{Yu-Ting Huang}
\affiliation{
  \institution{EPFL}
  \city{Lausanne}
  \country{Switzerland}}
\email{y.t.huang.tp@gmail.com}

\author{Parameswaran Kamalaruban}
\affiliation{
  \institution{The Alan Turing Institute}
  \city{London}
  \country{United Kingdom}}
\email{kparameswaran@turing.ac.uk}

\author{Craig Innes}
\affiliation{
  \institution{The University of Edinburgh}
  \city{Edinburgh}
  \country{United Kingdom}}
\email{craig.innes@ed.ac.uk}

\author{Subramanian Ramamoorthy}
\affiliation{
  \institution{The University of Edinburgh}
  \city{Edinburgh}
  \country{United Kingdom}}
\email{s.ramamoorthy@ed.ac.uk	}

\author{Adrian Weller}
\affiliation{
  \institution{University of Cambridge,}
  \institution{The Alan Turing Institute}
  \country{United Kingdom}}
\email{aw665@cam.ac.uk}

\begin{abstract}
Imitation learning (IL) is a popular paradigm for training policies in robotic systems when specifying the reward function is difficult. However, despite the success of IL algorithms, they impose the somewhat unrealistic requirement that the expert demonstrations must come from the same domain in which a new imitator policy is to be learned. We consider a practical setting, where (i) state-only expert demonstrations from the real (deployment) environment are given to the learner, (ii) the imitation learner policy is trained in a simulation (training) environment whose transition dynamics is slightly different from the real environment, and (iii) the learner does not have any access to the real environment during the training phase beyond the batch of demonstrations given. Most of the current IL methods, such as generative adversarial imitation learning and its state-only variants, fail to imitate the optimal expert behavior under the above setting. By leveraging insights from the Robust reinforcement learning (RL) literature and building on recent adversarial imitation approaches, we propose a robust IL algorithm to learn policies that can effectively transfer to the real environment without fine-tuning. Furthermore, we empirically demonstrate on continuous-control benchmarks that our method outperforms the state-of-the-art state-only IL method in terms of the zero-shot transfer performance in the real environment and robust performance under different testing conditions.
\end{abstract}


\newcommand{\BibTeX}{\rm B\kern-.05em{\sc i\kern-.025em b}\kern-.08em\TeX}

\usepackage[utf8]{inputenc} 
\usepackage[T1]{fontenc}    
\usepackage{hyperref}       
\usepackage{url}            
\usepackage{booktabs}       
\usepackage{amsfonts}       
\usepackage{nicefrac}       
\usepackage{microtype}      
\usepackage{xcolor}         

\usepackage{wrapfig}

\usepackage{enumitem}

\RequirePackage{algorithm}
\RequirePackage{algorithmic}

\usepackage{graphicx}
\usepackage{subfig}

\usepackage{amsmath}
\usepackage{amsthm}
\usepackage{stmaryrd}
\usepackage{comment}
\usepackage{setspace}
\newcommand{\bc}[1]{\left\{{#1}\right\}}
\newcommand{\br}[1]{\left({#1}\right)}
\newcommand{\bs}[1]{\left[{#1}\right]}
\newcommand{\abs}[1]{\left| {#1} \right|}

\newcommand{\norm}[1]{\left\| {#1} \right\|}

\renewcommand{\P}[1]{\mathbb{P}\bs{{#1}}}

\newcommand{\E}[1]{\mathbb{E}\bs{{#1}}}

\newcommand{\Eee}[2]{\mathbb{E}_{#1}\bs{{#2}}}

\newtheorem{theorem}{Theorem}

\DeclareMathOperator*{\argmin}{arg\,min}
\DeclareMathOperator*{\argmax}{arg\,max}

\allowdisplaybreaks

\begin{document}


\pagestyle{fancy}
\fancyhead{}


\maketitle 



\section{Introduction}
\label{sec:intro} 

Deep Reinforcement Learning (RL)~\citep{sutton1999policy,silver2014deterministic,schulman2015trust} methods have demonstrated impressive performance in continuous control~\citep{lillicrap2015continuous}, and robotics~\citep{levine2016end}. However, a broader application of these methods in real-world domains is impeded by the challenges in designing a proper reward function~\citep{schaal1999imitation,amodei2016concrete,everitt2016avoiding}. Imitation Learning (IL) algorithms~\citep{ng2000algorithms,ziebart2008maximum,ho2016generative} address this issue by replacing reward functions with expert demonstrations, which are easier to collect in most scenarios. However, despite the success of IL algorithms, they typically impose the somewhat unrealistic requirement that the state-action demonstrations must be collected from the same environment as the one in which the imitator is trained. In this work, we focus on a more realistic setting for imitation learning, where: 
\begin{enumerate}[leftmargin=*]
\item the expert demonstrations collected from the real (deployment) environment by executing an expert policy only contain states, 
\item the learner is trained in a simulation (training) environment, and does not have access to the real environment during the training phase beyond the batch of demonstrations given, and
\item the simulation environment does not model the real environment exactly, i.e., there exists a transition dynamics mismatch between these environments. 
\end{enumerate}
The learned policy under the above setting is transferred to the real environment on which its final performance is evaluated. Existing IL methods either do not apply under the above setting or result in poor transfer performance. 

\looseness-1A large body of work in IL, such as Generative Adversarial Imitation Learning (GAIL~\citep{ho2016generative}) and its variants, has focused on the setting with demonstrations that contain both states and actions, which are difficult to obtain for real-world settings such as learning from videos~\citep{handa2020dexpilot}. Further, closely following the state-action demonstrations limits the the ability to generalize across environments~\citep{radosavovic2020state}. Training agents in simulation environments not only provides data at low-cost, but also alleviates safety concerns related to the trial-and-error process with real robots. However, building a high-fidelity simulator that perfectly models the real environment would require a large computational budget. Low-fidelity simulations are feasible, due to their speed, but the gap between the simulated and real environments degrades the performance of the policies when transferred to real robots~\citep{zhao2020sim}. To this end, we consider the following research question: \emph{how to train an imitator policy in an offline manner with state-only expert demonstrations and a misspecified simulator such that the policy performs well in the real environment?}

\begin{table*}[h]
\caption{Comparison of our method with the existing imitation learning methods that also consider dynamics mismatch. However, the existing methods do not fit under the specific setting that we study. The expert, training, and deployment are denoted by $M^\mathrm{exp}$, $M^\mathrm{tr}$, and $M^\mathrm{dep}$ respectively. The corresponding transition dynamics are denoted by $T^\mathrm{exp}$, $T^\mathrm{tr}$, and $T^\mathrm{dep}$ respectively. Note that the expert demonstrations are collected from $M^\mathrm{exp}$, the imitation learning agent is trained on $M^\mathrm{tr}$, and the trained policy is finally evaluated on $M^\mathrm{dep}$. Our Robust-GAILfO method has minimal access to $M^\mathrm{dep}$ to select an appropriate $\alpha$. Note that our robust GAILfO method is applicable in both: (i) $T^{\mathrm{dep}} = T^{\mathrm{exp}} \neq T^{\mathrm{tr}}$ setting, and (ii) $T^{\mathrm{dep}} \neq T^{\mathrm{exp}} \neq T^{\mathrm{tr}}$ setting. In setting (i), our primary motivation is that accessing the deployment environment is costly, e.g., interacting with a remote deployment environment is costly due to communication constraints. In setting (ii), after the deployment, the agent has to be robust against potential environmental changes during the test time.}
\label{table:related-work}
\begin{tabular}{llll}\toprule
\textit{IL Methods} & \textit{Type of Demonstrations} & \textit{Access to $M^{\mathrm{dep}}$ during training} & \textit{Dynamics mismatch}  \\ \midrule
GAIL~\citep{ho2016generative} & state-action & yes & $T^{\mathrm{exp}} = T^{\mathrm{tr}} = T^{\mathrm{dep}}$ \\
GAILfO~\citep{torabi2018generative} & state-only & yes & $T^{\mathrm{exp}} = T^{\mathrm{tr}} = T^{\mathrm{dep}}$ \\
AIRL~\citep{fu2017learning} & state-action & yes & $T^{\mathrm{exp}} = T^{\mathrm{tr}} \neq T^{\mathrm{dep}}$ \\
I2L~\citep{gangwani2020stateonly} & state-only & yes & $T^{\mathrm{exp}} \neq T^{\mathrm{tr}} = T^{\mathrm{dep}}$ \\
SAIL~\citep{liu2019state} & state-only & yes & $T^{\mathrm{exp}} \neq T^{\mathrm{tr}} = T^{\mathrm{dep}}$ \\
GARAT~\citep{desai2020imitation} & state-only & yes & $T^{\mathrm{dep}} = T^{\mathrm{exp}} \neq T^{\mathrm{tr}}$ \\
HIDIL~\citep{jiang2020offline} & state-action & no & $T^{\mathrm{dep}} = T^{\mathrm{exp}} \neq T^{\mathrm{tr}}$ \\
IDDM~\citep{yang2019imitation} & state-only & yes & $T^{\mathrm{exp}} = T^{\mathrm{tr}} = T^{\mathrm{dep}}$ \\
ILPO~\citep{edwards2019imitating} & state-only & yes & $T^{\mathrm{exp}} = T^{\mathrm{tr}} =  T^{\mathrm{dep}}$ \\
Robust-GAILfO (ours) & state-only & no & $T^{\mathrm{dep}} = T^{\mathrm{exp}} \neq T^{\mathrm{tr}}$ and $T^{\mathrm{dep}} \neq T^{\mathrm{exp}} \neq T^{\mathrm{tr}}$ \\ \bottomrule
\end{tabular}
\end{table*}

\looseness-1The Adversarial Inverse Reinforcement Learning (AIRL) method from~\citep{fu2017learning} recovers reward functions that can be used to transfer behaviors across changes in dynamics. However, one needs to retrain a policy in the deployment environment with the recovered reward function, whereas we consider a zero-shot transfer setting. In addition, AIRL depends on state-action demonstrations. Recently,~\citep{gangwani2020stateonly,liu2019state} have studied the imitation learning problem under the transition dynamics mismatch between the expert and the learner environments. However, they do not aim to learn policies that are transferable to the expert (real) environment; instead, they optimize the performance in the learner (simulation) environment. In~\citep{desai2020imitation}, the authors attempt to match the simulation environment closer to the real environment by interacting with the real environment during the training phase. A setting very close to ours is considered in~\citep{jiang2020offline}; their method involves learning an inverse dynamics model of the real environment based on the state-action expert demonstrations. None of these methods are directly applicable under our setting (see Table~\ref{table:related-work}). 

\looseness-1We propose a robust IL method for learning robust policies under the above-discussed setting that can be effectively transferred to the real environment without further fine-tuning during deployment. Our method is built upon the robust RL literature~\citep{iyengar2005robust,nilim2005robust,pinto2017robust,tessler2019action} and the IL literature inspired by GAN-based adversarial learning~\citep{ho2016generative,torabi2018generative}. In particular, our algorithm is a robust variant of the Generative Adversarial Imitation Learning from Observation (GAILfO~\citep{torabi2018generative}) algorithm, a state-only IL method based on GAIL. We discuss how our method addresses the dynamics mismatch issue by exploiting the equivalence between the robust MDP formulation and the two-player Markov game~\citep{pinto2017robust,tessler2019action}. In the finite MDP setting,~\citep{viano2020robust} have proposed a robust inverse reinforcement learning method to address the transition dynamics mismatch between the expert and the learner. Our Markov game formulation in Section~\ref{sec:robust-lfo-mg} closely follows that of~\citep{viano2020robust}, and in Section~\ref{sec:robust-gailfo-full}, we scale it high-dimensional continuous control setting using the techniques from GAIL literature. On the empirical side, we are interested in the sim-to-real transfer performance, whereas~\citep{viano2020robust} have considered the performance in the learner environment itself.

\looseness-1We evaluate the efficacy of our method on the continuous control MuJoCo environments. In our experiments, we consider different sources of dynamics mismatch such as joint-friction, and agent-mass. An expert policy is trained under the default dynamics (acting as the real environment). The imitator policy is learned under a modified dynamics (acting as the simulation environment), where one of the mass and friction configurations is changed. The experimental results show that, with appropriate choice of the level of adversarial perturbation, the robustly trained IL policies in the simulator transfer successfully to the real environment compared to the standard GAILfO. We also empirically show that the policies learned by our method are robust to environmental shift during testing.

\section{Related Work}
\label{sec:relatedwork}


\paragraph{Imitation Learning}
\looseness-1Ho and Ermon~\citep{ho2016generative} propose a framework, called Generative Adversarial Imitation Learning (GAIL), for directly extracting a policy from trajectories without recovering a reward function as an intermediate step. GAIL utilizes a discriminator to distinguish between the state-action pairs induced by the expert and the learner policy. GAIL was further extended by Fu et al.~\citep{fu2017learning} to produce a scalable inverse reinforcement learning algorithm based on adversarial reward learning. This approach gives a policy as well as a reward function. Our work is closely related to the state-only IL methods that do not require actions in the expert demonstrations~\citep{torabi2018generative,yang2019imitation}. Inspired by GAIL,~\citep{torabi2018generative} have proposed the Generative Adversarial Imitation Learning from Observation (GAILfO) algorithm for state-only IL. GAILfO tries to minimize the divergence between the state transition occupancy measures of the learner and the expert.

\paragraph{Robust Reinforcement Learning}
\looseness-1In the robust MDP formulation~\citep{iyengar2005robust,nilim2005robust}, the policy is evaluated by the worst-case performance in a class of MDPs centered around a reference environment. In the context of forward RL, some works build on the robust MDP framework, such as~\citep{rajeswaran2016epopt,peng2018sim,mankowitz2019robust}. However, our work is closer to the line of work that leverages the equivalence between action-robust and robust MDPs. In~\citep{morimoto2005robust}, the authors have introduced the notion of worst-case disturbance in the $H_\infty$-control literature to the reinforcement learning paradigm. They consider an adversarial game where an adversary tries to make the worst possible disturbance while an agent tries to make the best control input. Recent literature in RL has proposed a range of robust algorithms based on this game-theoretic perspective~\citep{doyle2013feedback,pinto2017robust,tessler2019action,kamalaruban2020robust}.

\section{Problem Setup and Background}
\label{sec:Setup}

This section formalizes the learning from observation (LfO) problem with model misspecification. 


\paragraph{Environment and Policy}
The environment is formally represented by a Markov decision process (MDP) $M_c := \br{\mathcal{S},\mathcal{A},T,\gamma,P_0,c}$. The state and action spaces are denoted by $\mathcal{S}$ and $\mathcal{A}$, respectively. $T: \mathcal{S} \times \mathcal{S} \times \mathcal{A} \rightarrow \bs{0,1}$ captures the state transition dynamics, i.e., $T\br{s' \mid s,a}$ denotes the probability of landing in state $s'$ by taking action $a$ from state $s$. Here, $c: \mathcal{S} \times \mathcal{S} \to \mathbb{R}$ is the cost function, $\gamma \in \br{0,1}$ is the discounting factor, and $P_0: \mathcal{S} \rightarrow \bs{0,1}$ is an initial distribution over the state space $\mathcal{S}$. We denote an MDP without a cost function by $M = M_c \backslash c = \bc{\mathcal{S}, \mathcal{A}, T, \gamma, P_0}$. We denote a policy $\pi: \mathcal{S} \times \mathcal{A} \rightarrow \bs{0,1}$ as a mapping from a state to a probability distribution over the action space. The set of all stationary stochastic policies is denoted by $\Pi$. For any policy $\pi$ in the MDP $M$, we define the state transition occupancy measure as follows: $\rho^\pi_M \br{s,s'} := \sum_a T\br{s' \mid s,a} \cdot \pi\br{a \mid s} \cdot \sum_{t=0}^\infty \gamma^t \P{S_t = s \mid \pi, M}$. Here, $\P{S_t = s \mid \pi, M}$ denotes the probability of visiting the state $s$ after $t$ steps by following the policy $\pi$ in $M$. The total expected cost of any policy $\pi$ in the MDP $M_c$ is defined as follows: $\Eee{\rho^\pi_{M}}{c\br{s,s'}} := \E{\sum_{t = 0}^\infty \gamma^t c\br{s_t,s_{t+1}}}$, where $s_0 \sim P_0$, $a_t \sim \pi\br{\cdot | s_t}$, $s_{t+1} \sim T \br{\cdot | s_t, a_t}$. A policy $\pi$ is \emph{optimal} for the MDP $M_c$ if $\pi \in \argmin_{\pi'} \Eee{\rho^{\pi'}_{M}}{c\br{s,s'}}$, and we denote an optimal policy by $\pi^*_{M_c}$.

\paragraph{Learner and Expert}
\looseness-1We have two entities: an imitation learner, and an expert. We consider two MDPs, $M^{\mathrm{sim}} = \bc{\mathcal{S}, \mathcal{A}, T^{\mathrm{sim}}, \gamma, P_0}$ and $M^{\mathrm{real}} = \bc{\mathcal{S}, \mathcal{A}, T^{\mathrm{real}}, \gamma, P_0}$, that differ only in the transition dynamics. The true cost function $c^*: \mathcal{S} \times \mathcal{S} \to \mathbb{R}$ is known only to the expert. The learner is trained in the MDP $M^{\mathrm{sim}}$ and is not aware of the true cost function, i.e., it only has access to $M_{c^*}^{\mathrm{sim}} \backslash c^*$. The expert provides demonstrations to the learner  by following the optimal policy $\pi^{*}_{M_{c^*}^{\mathrm{real}}}$ in the expert MDP $M^{\mathrm{real}}$. Typically, in the imitation learning literature, it is assumed that $T^{\mathrm{sim}} = T^{\mathrm{real}}$.  In this work, we consider the setting where there is a transition dynamics mismatch between the learner and the expert, i.e., $T^{\mathrm{sim}} \neq T^{\mathrm{real}}$. The learner tries to recover a policy that closely matches the intention of the expert, based on the occupancy measure $\rho_E \br{s,s'} := \rho^{\pi^*_{M^{\mathrm{real}}_{c^*}}}_{M^{\mathrm{real}}}\br{s,s'}$ (or the demonstrations drawn according to it) received from the expert. The learned policy is evaluated in the expert environment w.r.t. the true cost function, i.e., $M_{c^*}^{\mathrm{real}}$.


\paragraph{Imitation Learning}
\looseness-1We consider the imitation learner model that matches the expert's state transition occupancy measure $\rho_E$~\citep{ziebart2008maximum,ho2016generative,torabi2018generative}. In particular, the learner policy is obtained via solving the following primal problem:
\begin{align}
\min_{\pi \in \Pi} \quad& - H_{\rho^\pi_{M^{\mathrm{sim}}}}\br{\pi} \label{opt_start}\\
\text{subject to} \quad& \rho^{\pi}_{M^{\mathrm{sim}}}\br{s,s'} ~=~ \rho_E \br{s,s'} , \quad \forall s,s' \in \mathcal{S} ,
\label{opt_end} 
\end{align}
where $H_{\rho^\pi_{M^{\mathrm{sim}}}}\br{\pi} :=  \E{\sum_{t = 0}^\infty - \gamma^t \log \pi(a_t | s_t)}$ is the $\gamma$-discounted causal entropy of $\pi$. The corresponding dual problem is given by:
\[
\max_{c \in \mathbb{R}^{\mathcal{S} \times \mathcal{S}}} \br{\min_{\pi \in \Pi} ~ - H_{\rho^\pi_{M^{\mathrm{sim}}}}\br{\pi} + \Eee{\rho^{\pi}_{M^{\mathrm{sim}}}}{c\br{s,s'}}} - \Eee{\rho_E}{c\br{s,s'}} ,
\]
where the costs $c\br{s,s'}$ serve as dual variables for the equality constraints. 

\paragraph{Maximum Causal Entropy (MCE) Inverse Reinforcement Learning (IRL)}
MCE-IRL algorithm~\citep{ziebart2008maximum,ziebart2010modeling} involves a two-step procedure. First, it looks for a cost function $c \in \mathcal{C}$ that assigns low cost to the expert policy and high cost to other policies. Then, it learns a policy by solving a certain reinforcement learning problem with the found cost function. Formally, given a convex cost function regularizer\footnote{$\overline{\mathbb{R}}$ denotes the extended real numbers $\mathbb{R} \cup \bc{+ \infty}$} $\psi: \mathbb{R}^{\mathcal{S} \times \mathcal{S}} \to \overline{\mathbb{R}}$, first, we recover a cost function $\tilde{c}$ by solving the following $\psi$-regularized problem:
\begin{align*}
\textsc{IRL}_\psi \br{\rho_E} ~=~& \argmax_{c \in \mathbb{R}^{\mathcal{S} \times \mathcal{S}}} ~ - \psi \br{c} - \Eee{\rho_E}{c\br{s,s'}} \\
&\quad \quad \quad \quad + \br{\min_{\pi \in \Pi} ~ - \lambda H_{\rho^\pi_{M^{\mathrm{sim}}}}\br{\pi} + \Eee{\rho^{\pi}_{M^{\mathrm{sim}}}}{c\br{s,s'}}}
\end{align*}
Then, we input the learned cost function $\tilde{c} \in \textsc{IRL}_\psi \br{\rho_E}$ into an entropy-regularized reinforcement learning problem:
\[
\textsc{RL} \br{c} ~=~ \argmin_{\pi \in \Pi} ~ - \lambda H_{\rho^\pi_{M^{\mathrm{sim}}}}\br{\pi} + \Eee{\rho^{\pi}_{M^{\mathrm{sim}}}}{c\br{s,s'}} ,
\]
which aims to find a policy that minimizes the cost function and maximizes the entropy.

\paragraph{Generative Adversarial Imitation Learning from Observation (GAILfO)}
Recently,~\citep{ho2016generative,torabi2018generative} have shown that, for a specific choice of the regularizer $\psi$, the two-step procedure $\textsc{RL} \circ \textsc{IRL}_\psi \br{\rho_E}$ of the MCE-IRL algorithm can be reduced to the following optimization problem using GAN discriminator:
\begin{align*}
\min_{\pi \in \Pi} ~ \max_{D \in \br{0,1}^{\mathcal{S} \times \mathcal{S}}} & - \lambda H_{\rho^\pi_{M^{\mathrm{sim}}}}\br{\pi} + \Eee{\rho^{\pi}_{M^{\mathrm{sim}}}}{\log{D\br{s,s'}}} \\
& \quad \quad \quad \quad \quad \quad \quad + \Eee{\rho_E}{\log\br{1-D\br{s,s'}}} ,
\end{align*}
where $D:\mathcal{S} \times \mathcal{S} \to (0,1)$ is a classifier trained to discriminate between the state-next state pairs that arise from the
expert and the imitator. Excluding the entropy term,
the above loss function is similar to the loss of generative
adversarial networks~\citep{goodfellow2014generative}. Even though the occupancy measure matching methods were shown to scale well to high-dimensional problems, they are not robust against dynamics mismatch~\citep{gangwani2020stateonly}.

\section{Robust Learning from Observation via Markov Game}
\label{sec:robust_il_algo}

\subsection{Markov Game}
\label{sec:robust-lfo-mg}

In this section, we focus on recovering a learner policy via imitation learning framework in a robust manner, under the setting described in Section~\ref{sec:intro}. To this end, we consider a class of transition matrices such that it contains both $T^\mathrm{sim}$ and $T^\mathrm{real}$. In particular, for a given $\alpha > 0$, we define the class $\mathcal{T}^{\alpha}$ as follows: 
\begin{equation}
\mathcal{T}^{\alpha} := \bc{\alpha T^{\mathrm{sim}} (s' | s,a) + \bar{\alpha} \sum_b \pi (b | s) \cdot T^{\mathrm{sim}} (s' | s,b), \forall \pi \in \Pi} ,
\label{learner_unc_set}
\end{equation}
where $\bar{\alpha} = (1 - \alpha)$. We define the corresponding class of MDPs as follows: $\mathcal{M}^{\alpha} := \bc{\bc{\mathcal{S}, \mathcal{A}, T^{\alpha}, \gamma, P_0}, \, \forall T^{\alpha} \in \mathcal{T}^{\alpha}}$. We need to choose $\alpha$ such that $M^\mathrm{real} \in \mathcal{M}^\alpha$. 

Our aim is to find a learner policy that performs well in the MDP $M_{c^*}^{\mathrm{real}}$ by using the state-only demonstrations from $\rho_E$, without knowing or interacting with $M^\mathrm{real}$ during training. Thus, we try to learn a robust policy over the class $\mathcal{M}^\alpha$, while aligning with the expert's state transition occupancy measure $\rho_E$, and acting only in $M^\mathrm{sim}$. By doing this, we ensure that the resulting policy performs reasonably well on any MDP $M \in \mathcal{M}^\alpha$ including $M^\mathrm{real}$ w.r.t. the true cost function $c^*$. Based on this idea, we propose the following robust learning from observation (LfO) problem:  
\begin{align}
\min_{\pi^{\mathrm{pl}} \in \Pi} \max_{M \in \mathcal{M}^{\alpha}} \quad& - H_{\rho^{\pi^{\mathrm{pl}}}_{M}}\br{\pi^{\mathrm{pl}}} \label{mdp_opt_start}\\
\text{subject to} \quad& \rho^{\pi^{\mathrm{pl}}}_{M}\br{s,s'} ~=~ \rho_E \br{s,s'} , \, \forall s,s' \in \mathcal{S} ,
\label{mdp_opt_end} 
\end{align}
where our learner policy matches the expert's state transition occupancy measure $\rho_E$ under the most adversarial MDP belonging to the set $\mathcal{M}^\alpha$. The corresponding dual problem is given by:
\begin{align}
&\max_{c \in \mathbb{R}^{\mathcal{S} \times \mathcal{S}}} \br{\min_{\pi^{\mathrm{pl}} \in \Pi} \max_{M \in \mathcal{M}^{\alpha}} ~ - H_{\rho^{\pi^{\mathrm{pl}}}_{M}}\br{\pi^{\mathrm{pl}}} + \Eee{\rho^{\pi^{\mathrm{pl}}}_{M}}{c\br{s,s'}}} \nonumber \\
& \quad \quad \quad \quad  \quad \quad \quad \quad - \Eee{\rho_E}{c\br{s,s'}} . \label{eq:robust-mdp-il-dual-form}
\end{align}
\looseness-1In the dual problem, for any $c$, we attempt to learn a robust policy over the class $\mathcal{M}^{\alpha}$ with respect to the entropy regularized reward function. The parameter $c$ plays the role of aligning the learner's policy with the expert's occupancy measure via constraint satisfaction.

\looseness-1For any given $c$, we need to solve the inner min-max problem of~\eqref{eq:robust-mdp-il-dual-form}. However, during training, we only have access to the MDP $M^\mathrm{sim}$. To this end, we utilize the equivalence between the robust MDP~\citep{iyengar2005robust,nilim2005robust} formulation and the action-robust MDP~\citep{pinto2017robust,tessler2019action} formulation shown in~\citep{tessler2019action}. We can interpret the minimization over the environment class as the minimization over a set of opponent policies that with probability $1 - \alpha$ take control of the agent and perform the worst possible move from the current agent state. We can write:
\begin{align}
& \min_{\pi^{\mathrm{pl}} \in \Pi} \max_{M \in \mathcal{M}^{\alpha}} - H_{\rho^{\pi^{\mathrm{pl}}}_{M}}\br{\pi^{\mathrm{pl}}} + \Eee{\rho^{\pi^{\mathrm{pl}}}_{M}}{c\br{s,s'}} \nonumber \\
~=~& \min_{\pi^{\mathrm{pl}} \in \Pi} \max_{T^\alpha \in \mathcal{T}^{\alpha}}  \E{G_c \bigm| \pi^{\mathrm{pl}}, P_0, T^\alpha} \nonumber \\
~=~& \min_{\pi^{\mathrm{pl}} \in \Pi} \max_{\pi^{\mathrm{op}} \in \Pi} \E{G_c \bigm| \alpha \pi^{\mathrm{pl}} + (1 - \alpha) \pi^{\mathrm{op}}, M^{\mathrm{sim}}} , \label{equivalence_new}
\end{align}
where $G_c := \sum_{t=0}^{\infty} \gamma^t \bc{c\br{s_t, s_{t+1}} - H^{\pi^{pl}}(A|S=s_t)}$. The above equality holds due to the derivation in section 3.1 of~\citep{tessler2019action}. We can formulate the problem~\eqref{equivalence_new} as a two-player zero-sum Markov game~\citep{littman1994markov} with transition dynamics given by
\begin{align*}
T^{\mathrm{two},\alpha}(s' | s, a^{\mathrm{pl}}, a^{\mathrm{op}}) ~=~& \alpha T^{\mathrm{sim}}(s' | s, a^{\mathrm{pl}}) + (1 - \alpha) T^{\mathrm{sim}}(s' | s, a^{\mathrm{op}}) ,
\end{align*}
where $a^{\mathrm{pl}}$ is an action chosen according to the player policy and $a^{\mathrm{op}}$ according to the opponent policy. As a result, we reach a two-player Markov game with a regularization term for the player as follows:
\begin{equation}
\argmin_{\pi^{\mathrm{pl}} \in \Pi} \max_{\pi^{\mathrm{op}} \in \Pi} \E{G_{c} \bigm| \pi^{\mathrm{pl}}, \pi^{\mathrm{op}}, M^{\mathrm{two},\alpha}} , 
\label{objective}
\end{equation}
where $M^{\mathrm{two},\alpha} = \bc{\mathcal{S}, \mathcal{A}, \mathcal{A}, T^{\mathrm{two},\alpha}, \gamma, P_0}$ is the two-player MDP associated with the above game. 

\begin{algorithm}[t]
	\caption{Robust GAILfO}
	\label{alg:robust-gailfo}
	\begin{algorithmic}
	    \STATE \textbf{Input:} state-only expert demonstrations $\mathcal{D}^E$, opponent strength parameter $\alpha$.
		\STATE \textbf{Initialize:} discriminator $D_w$, actor policy $\pi_\theta$, and adversary policy $\pi_\phi$.
        \FOR{$n \in \bc{1,2,\dots, N}$}
		\STATE \looseness-1collect trajectories $\tau_i$ by executing the policies $\pi^{\mathrm{pl}}_\theta$ and $ \pi^{\mathrm{op}}_\phi$ (see Algorithm~\ref{alg:collect_trajs}), and store them in the demonstrations buffer $\mathcal{D}$.
		\STATE update the discriminator $D_w$ to classify the expert demonstrations $\tau_E \in \mathcal{D}^E$ from the samples $\tau_i \in \mathcal{D}$, i.e., update $w$ via gradient ascent with the following gradient:
		\begin{align*}
		    \widehat{\mathbb{E}}_{\tau_i\in \mathcal{D}}[\nabla_w \log D_w (s, s^\prime)] + \widehat{\mathbb{E}}_{\tau_E \in \mathcal{D}^E}[\nabla_w \log (1 - D_w (s, s^\prime))] .
		\end{align*}
		\STATE update the reward function $R_w(s, s^\prime) \gets - \log D_w(s, s^\prime)$.
		\STATE compute the following gradient estimates:
		\begin{align*}	    
		\widehat{\nabla}_{\theta} J(\theta, \phi) =& \frac{1}{\abs{\mathcal{D}}}  \sum_{\tau_i \in \mathcal{D}} \sum_{t}\gamma^t \nabla_{\theta} \log \pi_{\theta, \phi}^\mathrm{mix}(a^i_t|s^i_t) \bs{G^i_t + \lambda G^{\mathrm{log}, i}_t} \\
		\widehat{\nabla}_{\phi} J(\theta, \phi) =& \frac{1}{\abs{\mathcal{D}}} \sum_{\tau_i \in \mathcal{D}} \sum_{t}\gamma^t \nabla_{\phi} \log \pi_{\theta, \phi}^\mathrm{mix}(a^i_t|s^i_t) \bs{G^i_t + \lambda G^{\mathrm{log}, i}_t} ,
		\end{align*}
		where $G^i_t = \sum^{T}_{k=t+1} \gamma^{k - t - 1}R_w(s^i_k, s^i_{k+1})$ and $G^{\mathrm{log},i}_t = \sum^{T}_{k=t+1} - \gamma^{k - t - 1} H^{\pi^{\mathrm{pl}}_{\theta}}(A|S=s^i_k)$
		\STATE update the policies $\pi^{\mathrm{pl}}_\theta$ and $\pi^{\mathrm{op}}_\phi$ using PPO with the gradient estimates above. 
		\ENDFOR
	\end{algorithmic}
\end{algorithm}

\begin{algorithm}[t]
	\caption{Collecting Trajectories}
	\label{alg:collect_trajs}
	\begin{algorithmic}
	\STATE \textbf{Input:} total number of trajectories $N_{\mathrm{traj}}$, reward function $R_w$.
	\FOR{$n \in \bc{1,2, \dots, N_{\mathrm{traj}}}$}
	    \STATE $t \gets 0$
	    \STATE initialize an empty trajectory $\tau$.
		\WHILE{not $\mathrm{done}$}
			\STATE observe state $s_t$.
			\STATE sample actions $a^{\mathrm{pl}}_t \sim \pi^{\mathrm{pl}}_\theta(\cdot | s_t)$, $a^{\mathrm{op}}_t = \pi^{\mathrm{op}}_\phi(\cdot |s_t)$.
			\STATE execute $a^{\mathrm{op}}_t$ with probability $\bar{\alpha}$, or $a^{\mathrm{pl}}_t$ with probability $\alpha$.
		    \STATE observe $r_{t+1} = R_w(s_t, s_{t+1})$, next state $s_{t+1}$, and $\mathrm{done}$.
			\STATE store the tuple $(s_t, a^{\mathrm{pl}}_t, a^{\mathrm{op}}_t, s_{t+1}, r_{t+1})$ in the trajectory $\tau$.
	    \ENDWHILE	
        \STATE $\mathcal{D} \leftarrow \mathcal{D} \cup \bc{\tau}$.
    \ENDFOR
    \STATE \textbf{Output:} $\mathcal{D}$
	\end{algorithmic}
\end{algorithm}

\subsection{Robust GAILfO}
\label{sec:robust-gailfo-full}

In this section, we present our robust Generative Adversarial Imitation Learning from Observation (robust GAILfO) algorithm based on the discussions in Section~\ref{sec:robust-lfo-mg}. We begin with the robust variant of the two-step procedure $\textsc{RL} \circ \textsc{IRL}_\psi \br{\rho_E}$ of the MCE-IRL algorithm:
\begin{align*}
\textsc{IRL}_\psi \br{\rho_E} =& \argmax_{c \in \mathbb{R}^{\mathcal{S} \times \mathcal{S}}} - \psi \br{c} - \Eee{\rho_E}{c\br{s,s'}} \\
& \quad \quad + \min_{\pi^{\mathrm{pl}} \in \Pi} \max_{\pi^{\mathrm{op}} \in \Pi} - \lambda H_{\rho^{\pi^\mathrm{mix}}_{M^{\mathrm{sim}}}}\br{\pi^\mathrm{pl}} + \Eee{\rho^{\pi^\mathrm{mix}}_{M^{\mathrm{sim}}}}{c\br{s,s'}} \\
\textsc{RL} \br{c} =& \argmin_{\pi^{\mathrm{pl}} \in \Pi}  \max_{\pi^{\mathrm{op}} \in \Pi} - \lambda H_{\rho^{\pi^\mathrm{mix}}_{M^{\mathrm{sim}}}}\br{\pi^\mathrm{pl}} + \Eee{\rho^{\pi^\mathrm{mix}}_{M^{\mathrm{sim}}}}{c\br{s,s'}} ,
\end{align*}
where $\pi^\mathrm{mix} = \alpha \pi^{\mathrm{pl}} + (1-\alpha) \pi^{\mathrm{op}}$. Then, similar to~\citep{ho2016generative,torabi2018generative}, the above two step procedure  can be reduced to the following optimization problem using the discriminator $D:\mathcal{S} \times \mathcal{S} \to (0,1)$:
\begin{align*}
\min_{\pi^{\mathrm{pl}} \in \Pi} \max_{\pi^{\mathrm{op}} \in \Pi} \max_{D \in \br{0,1}^{\mathcal{S} \times \mathcal{S}}} & - \lambda H_{\rho^{\pi^\mathrm{mix}}_{M^{\mathrm{sim}}}}\br{\pi^\mathrm{pl}} + \Eee{\rho^{\pi^\mathrm{mix}}_{M^{\mathrm{sim}}}}{\log{D\br{s,s'}}} \\
& \quad \quad + \Eee{\rho_E}{\log\br{1-D\br{s,s'}}} . 
\end{align*}
We parameterize the policies and the discriminator as $\pi^\mathrm{pl}_\theta$, $\pi^\mathrm{op}_\phi$, and $D_w$ (with parameters $\theta$, $\phi$, and $w$), and rewrite the above problem as follows:
\begin{align*}
\min_{\theta} \max_{\phi} \max_{w} & - \lambda H_{\rho^{\pi^\mathrm{mix}_{\theta,\phi}}_{M^{\mathrm{sim}}}}\br{\pi^\mathrm{pl}_\theta} + \Eee{\rho^{\pi^\mathrm{mix}_{\theta,\phi}}_{M^{\mathrm{sim}}}}{\log{D_w\br{s,s'}}} \\
& \quad \quad + \Eee{\rho_E}{\log\br{1-D_w\br{s,s'}}} ,
\end{align*}
where $\pi^\mathrm{mix}_{\theta,\phi} = \alpha \pi^{\mathrm{pl}}_\theta + (1-\alpha) \pi^{\mathrm{op}}_\phi$. We solve the above problem by taking gradient steps alternatively w.r.t. $\theta$, $\phi$, and $w$. The calculation for the gradient estimates are given in Appendix~\ref{app:robust-lfo-details}. Following~\citep{ho2016generative,torabi2018generative}, we use the proximal policy optimization (PPO~\citep{schulman2017proximal}) to update the policies parameters. Our complete algorithm is given in Algorithm~\ref{alg:robust-gailfo}.

We also note that one could use any robust RL approach (including domain randomization) to solve the inner min-max problem of~\eqref{eq:robust-mdp-il-dual-form}. In our work, we used the action-robustness approach since: (i) in the robust RL literature, the equivalence between the domain randomization approach and the action-robustness approach is already established~\citep{tessler2019action}, and (ii) compared to the domain randomization approach, the action-robustness approach only requires access to a single simulation environment and creates a range of environments via action perturbations.

\section{Experiments}
\label{sec:experiments}

\begin{figure}[!h] 
\centering
\begin{tabular}{ccc}
\subfloat[HalfCheetah]{%
       \includegraphics[width=0.27\linewidth]{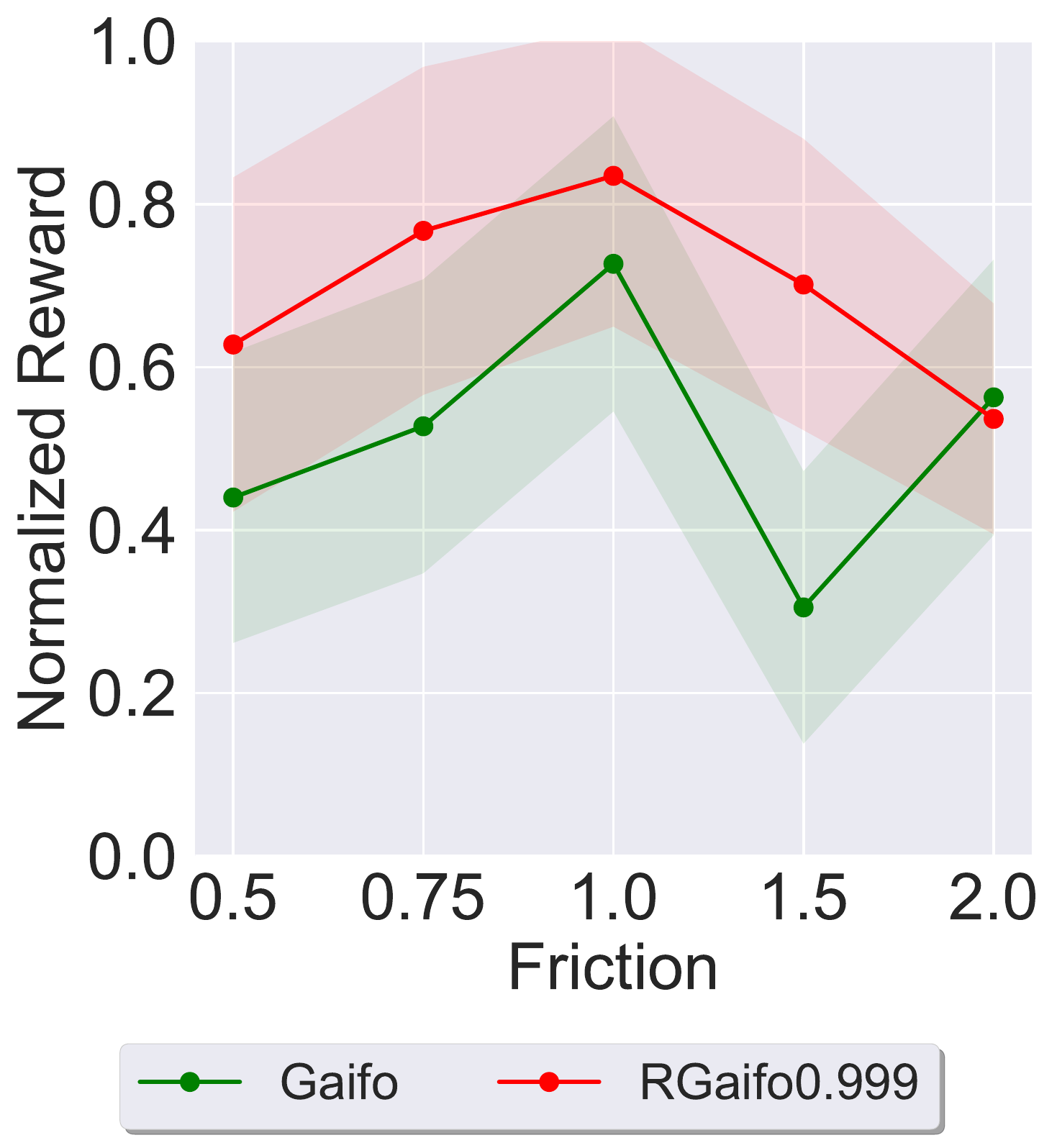}
     } &
\subfloat[Walker]{%
       \includegraphics[width=0.27\linewidth]{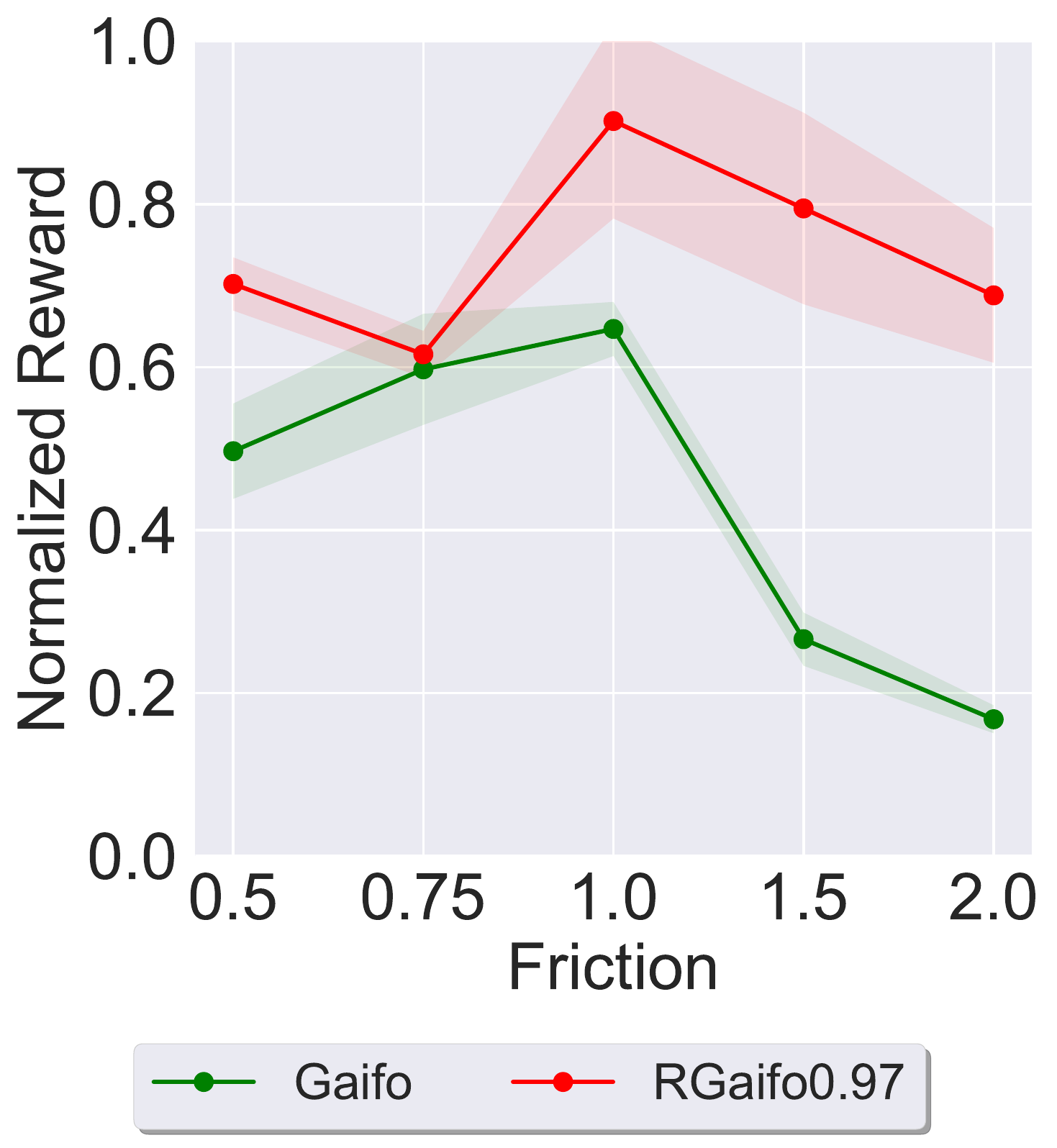}
     } &
\subfloat[Hopper]{%
       \includegraphics[width=0.27\linewidth]{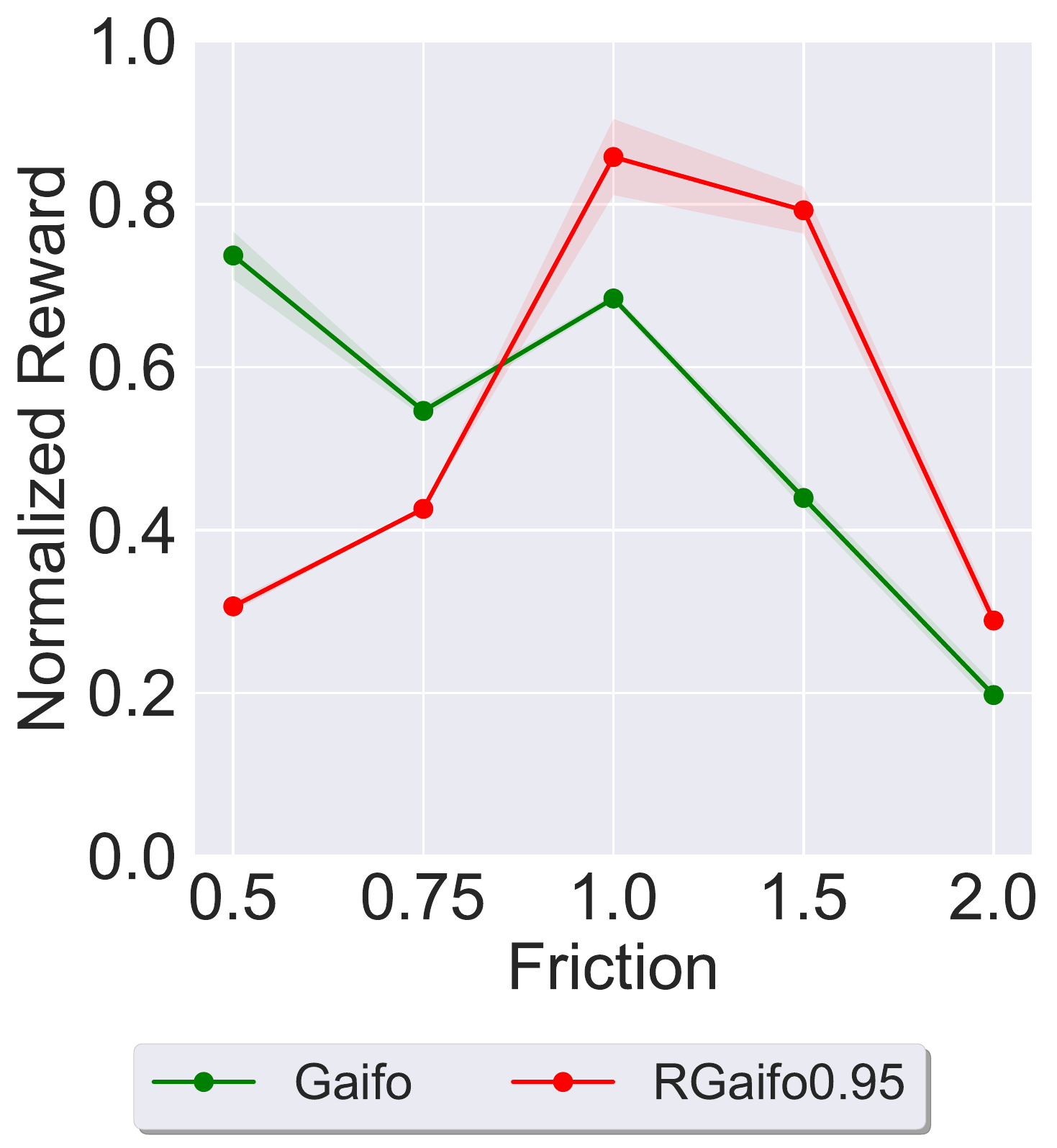}
     } \\
\end{tabular}
\caption{The average (over $3$ seeds) transfer performance of Algorithm~\ref{alg:robust-gailfo} with different values of $\alpha$ for each MuJoCo task as reported in the legend of each plot. The x-axis denotes the relative friction of the learner environment $M^\mathrm{sim}$. The policies are evaluated in $M^\mathrm{real}_{c^*}$ over $1e5$ steps truncating the last episode if it does not terminate.}
\label{fig:TransferFrictionFixedAlpha}
\end{figure}

\begin{figure}[!h] 
\centering
\begin{tabular}{ccc}
\subfloat[HalfCheetah]{%
       \includegraphics[width=0.27\linewidth]{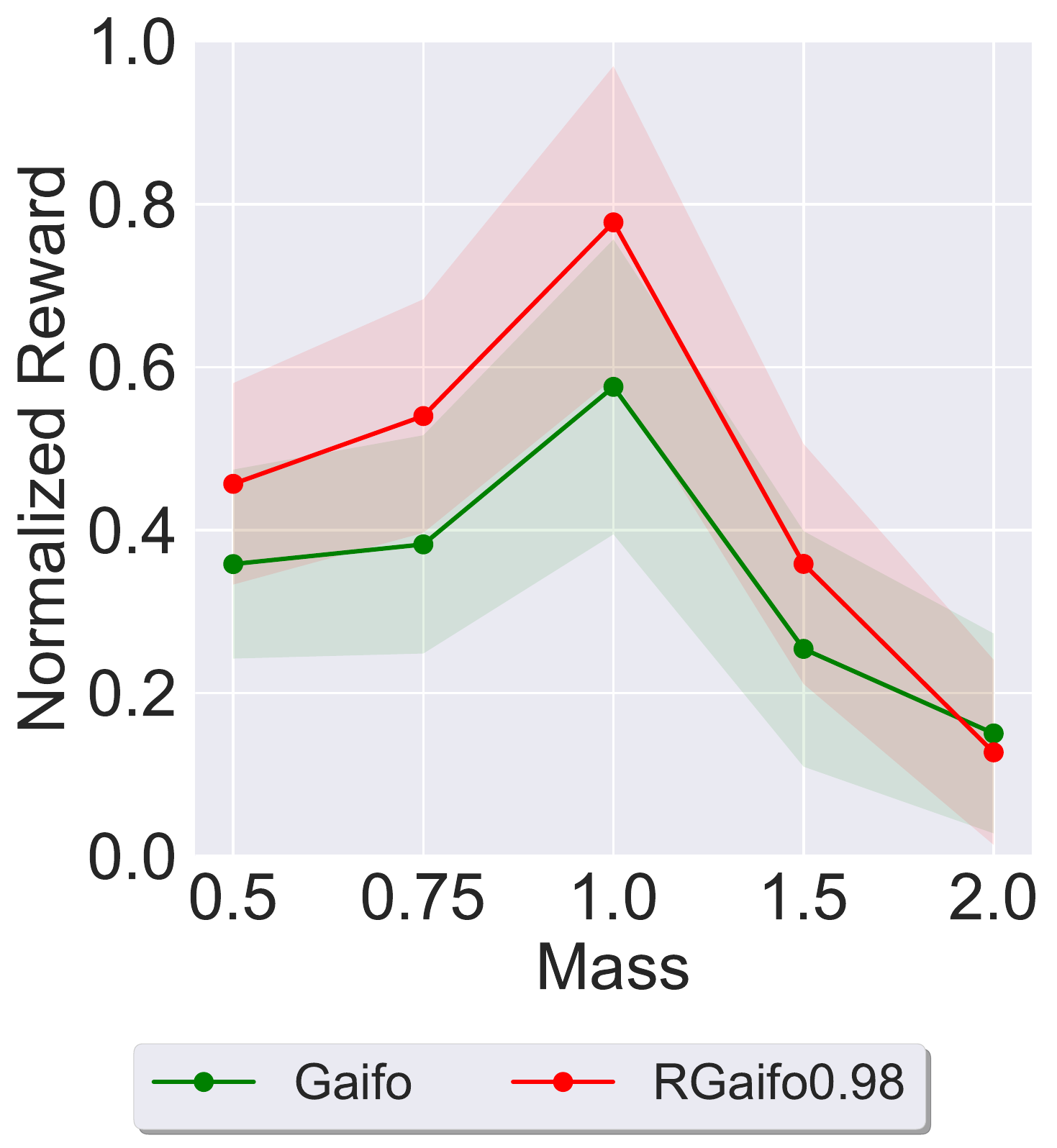}
     } &
\subfloat[Walker]{%
       \includegraphics[width=0.27\linewidth]{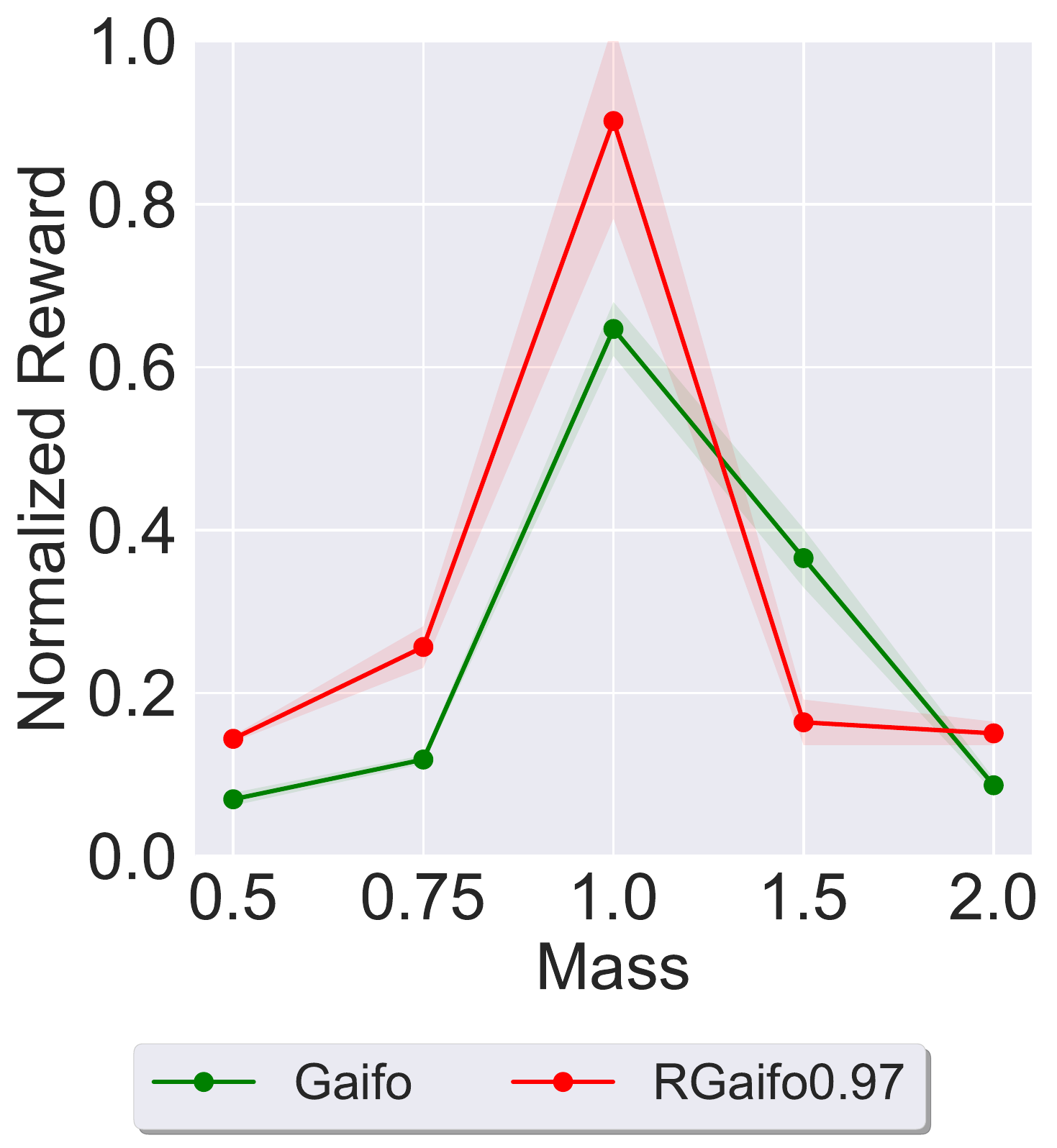}
     } &
\subfloat[Hopper]{%
       \includegraphics[width=0.27\linewidth]{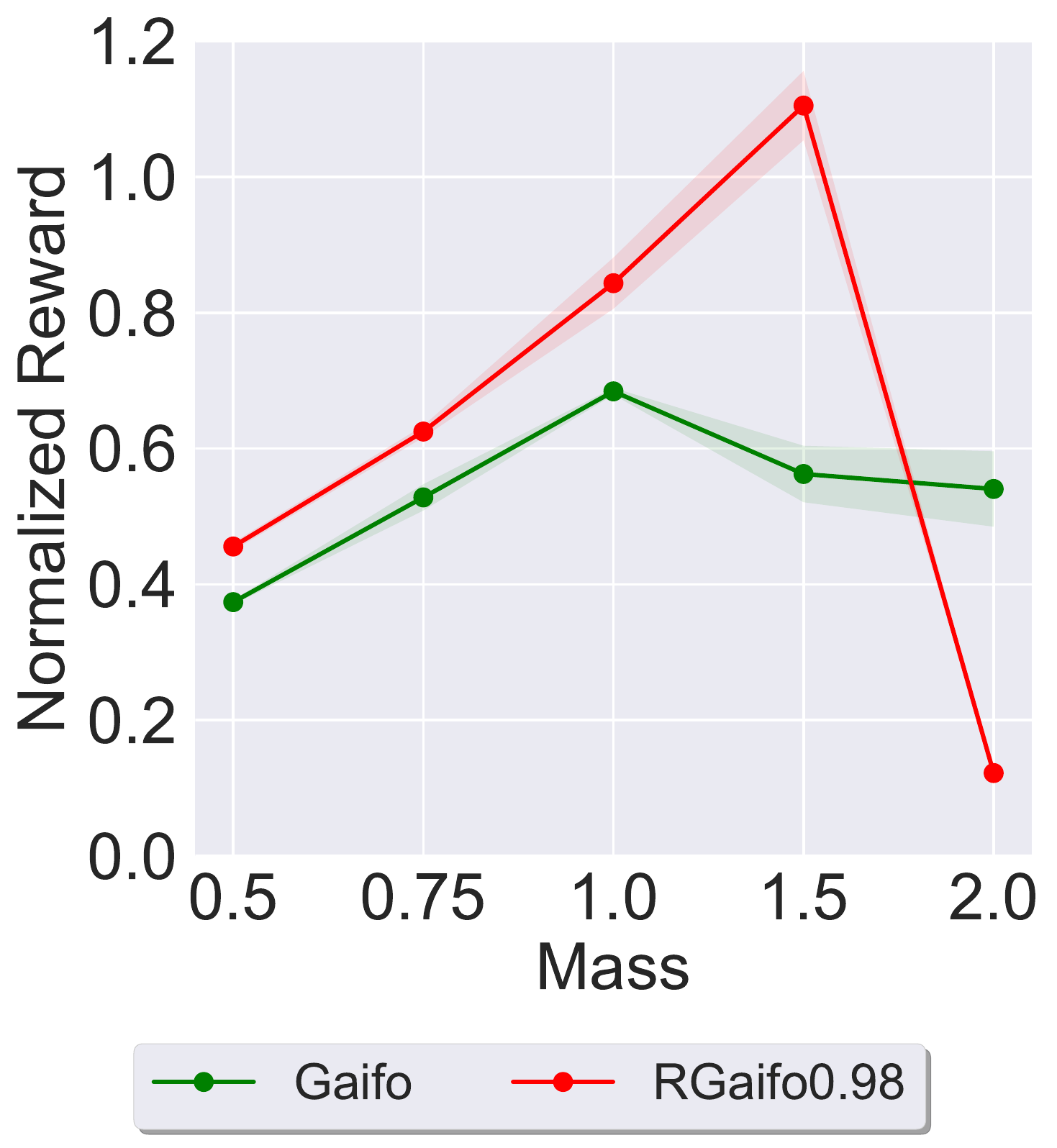}
     } \\
\subfloat[InvDoublePend]{%
       \includegraphics[width=0.27\linewidth]{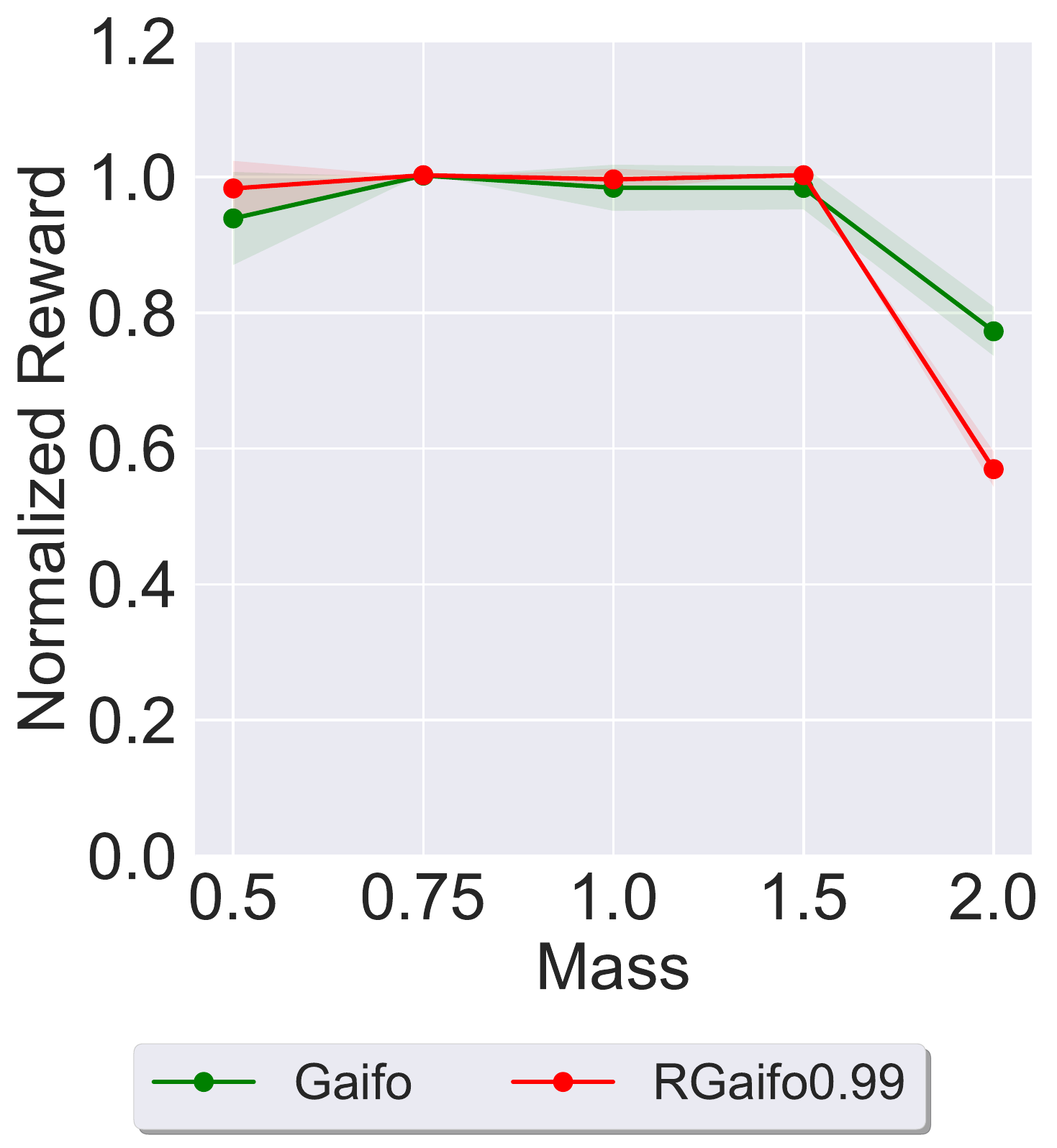}
     } &
\subfloat[Swimmer]{%
       \includegraphics[width=0.27\linewidth]{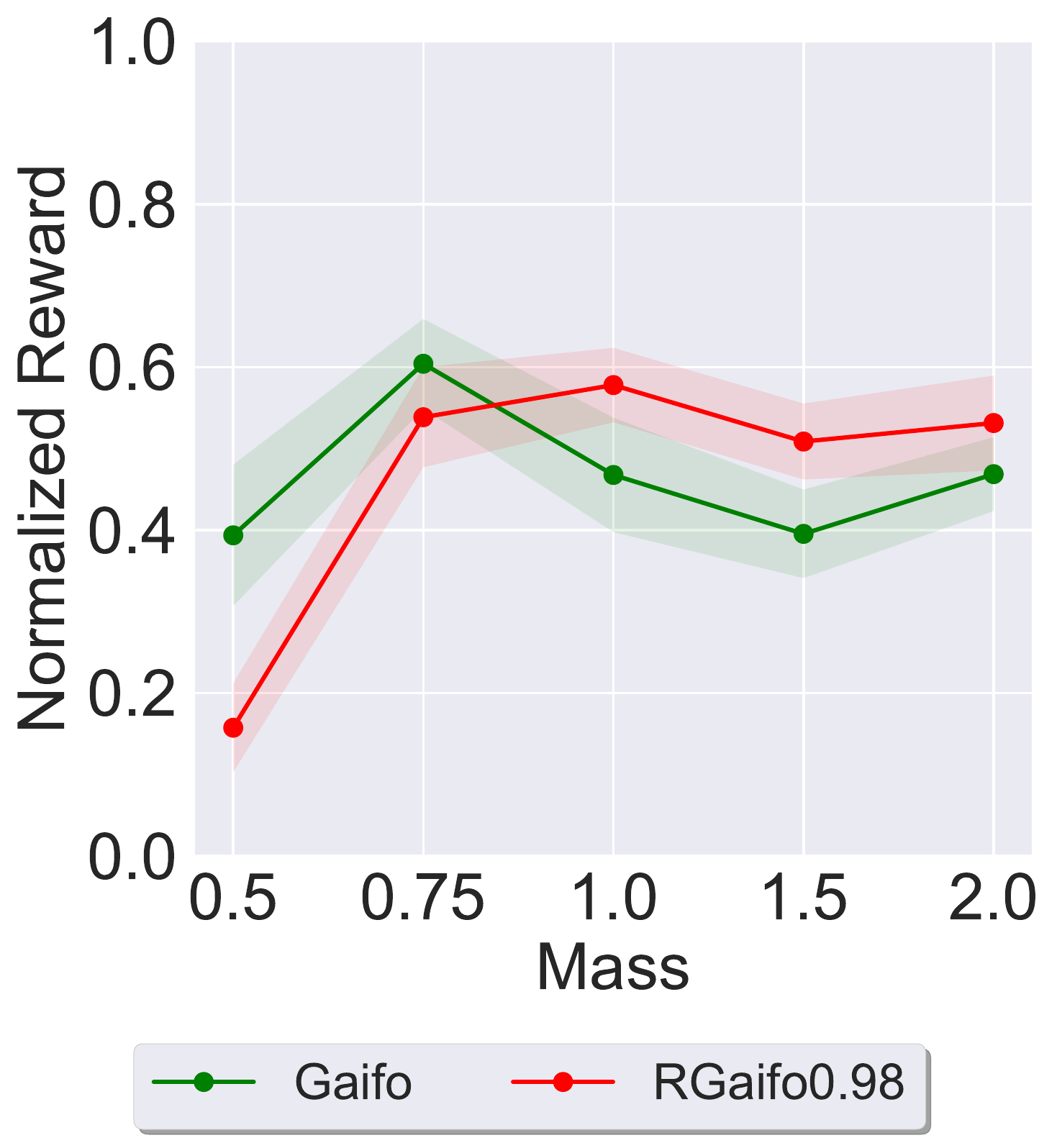}
     } &
     \\
\end{tabular}
\caption{The average (over $3$ seeds) transfer performance of Algorithm~\ref{alg:robust-gailfo} with different values of $\alpha$ for each MuJoCo task as reported in the legend of each plot. The x-axis denotes the relative mass of the learner environment $M^\mathrm{sim}$. The policies are evaluated in $M^\mathrm{real}_{c^*}$ over $1e5$ steps truncating the last episode if it does not terminate.}
\label{fig:TransferMassFixedAlpha}
\end{figure}

\begin{figure*}[!h] 
\centering
\begin{tabular}{ccccc}
\subfloat[HalfCheetah]{%
       \includegraphics[width=0.16\linewidth]{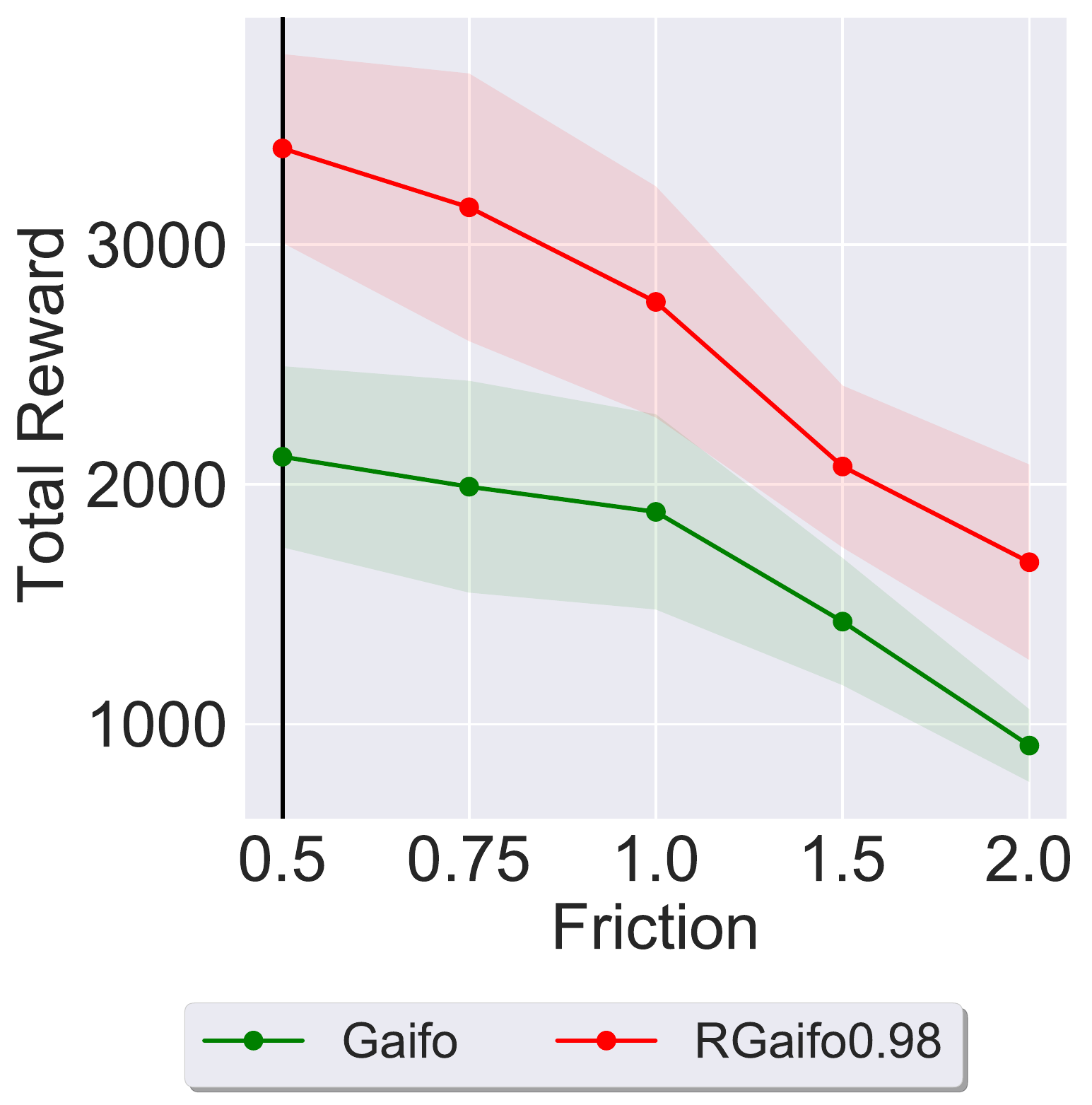}
     } &
\subfloat[HalfCheetah]{%
       \includegraphics[width=0.16\linewidth]{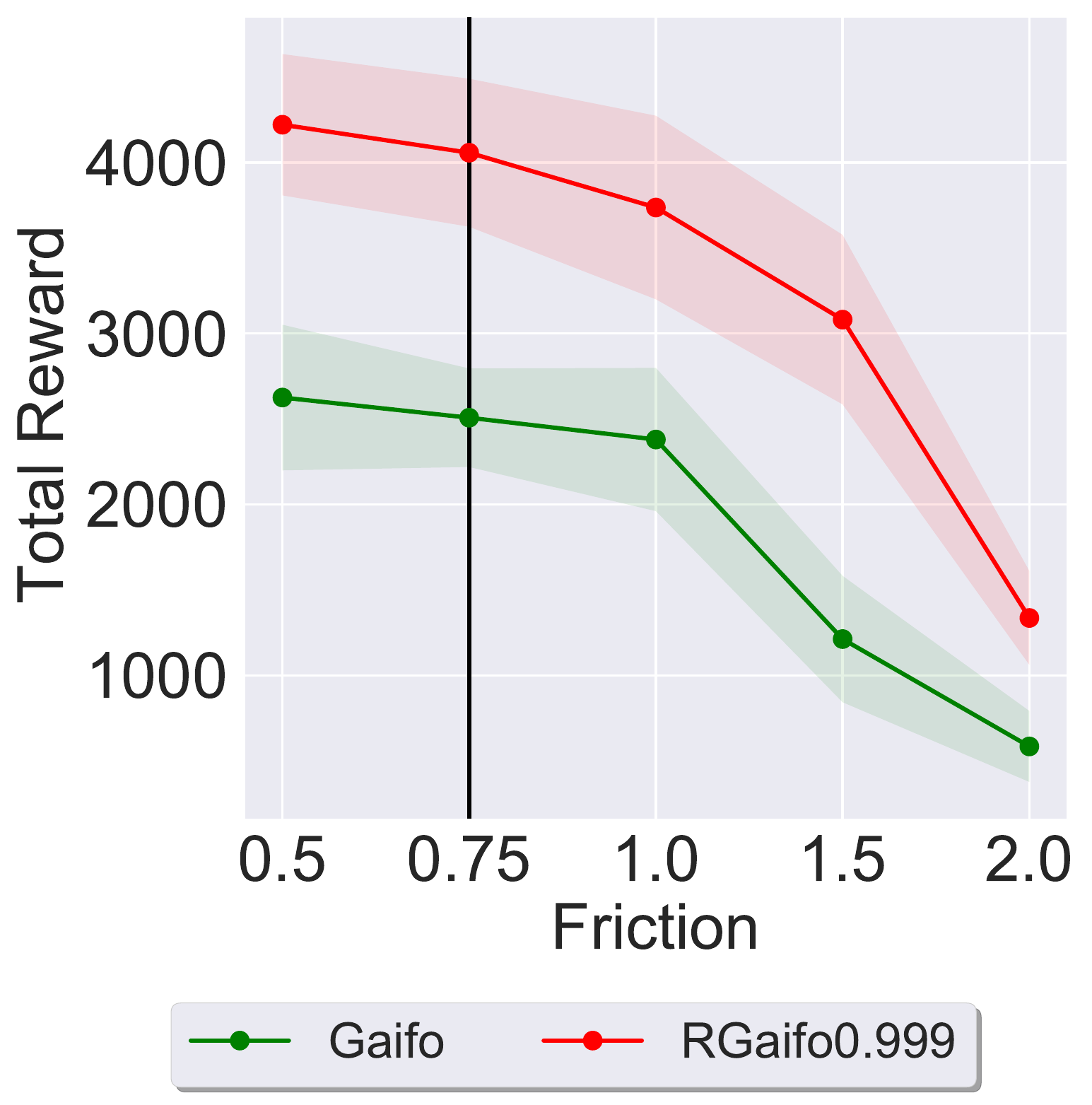}
     } &
\subfloat[HalfCheetah]{%
       \includegraphics[width=0.16\linewidth]{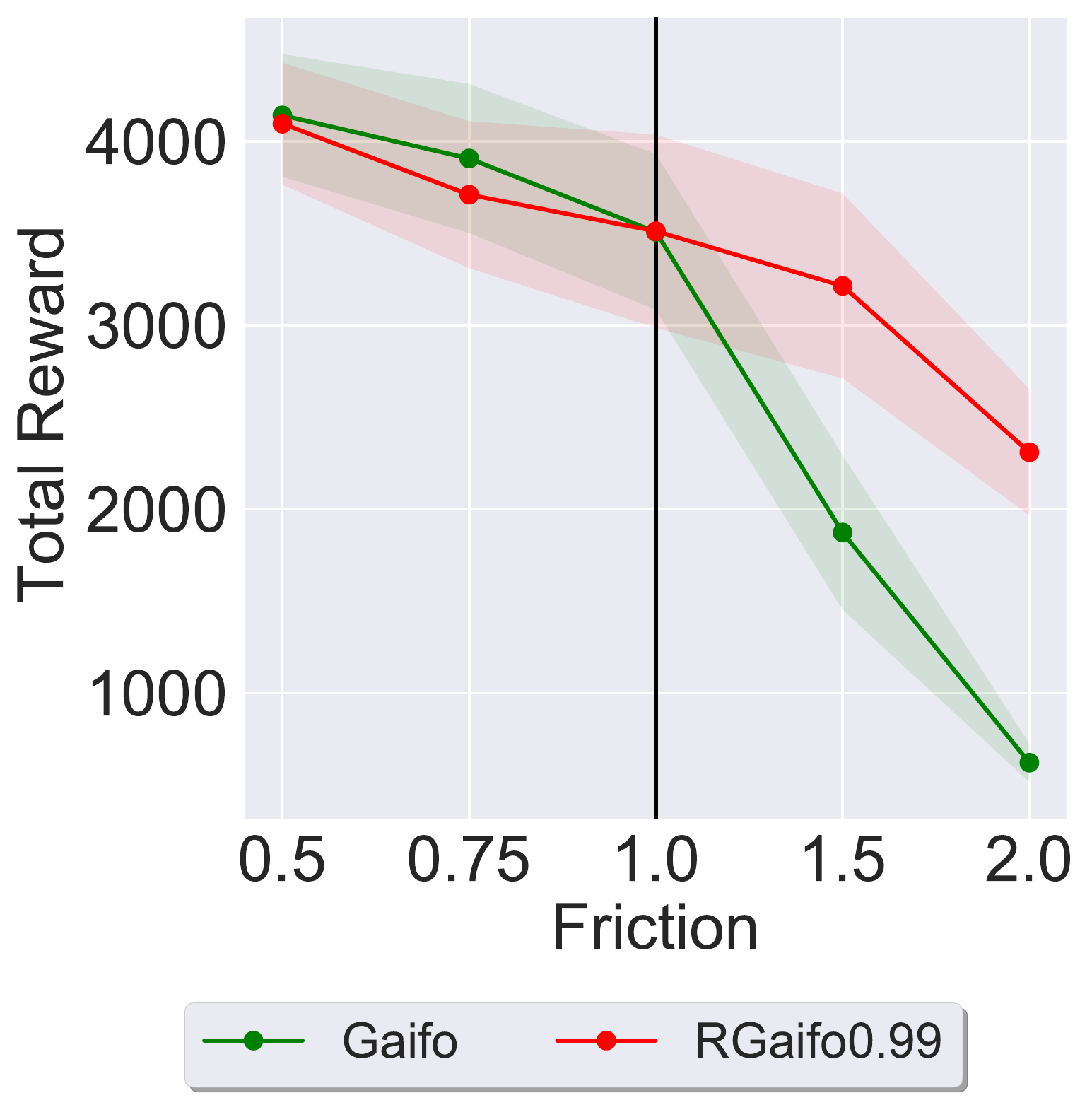}
     } &
\subfloat[HalfCheetah]{%
       \includegraphics[width=0.16\linewidth]{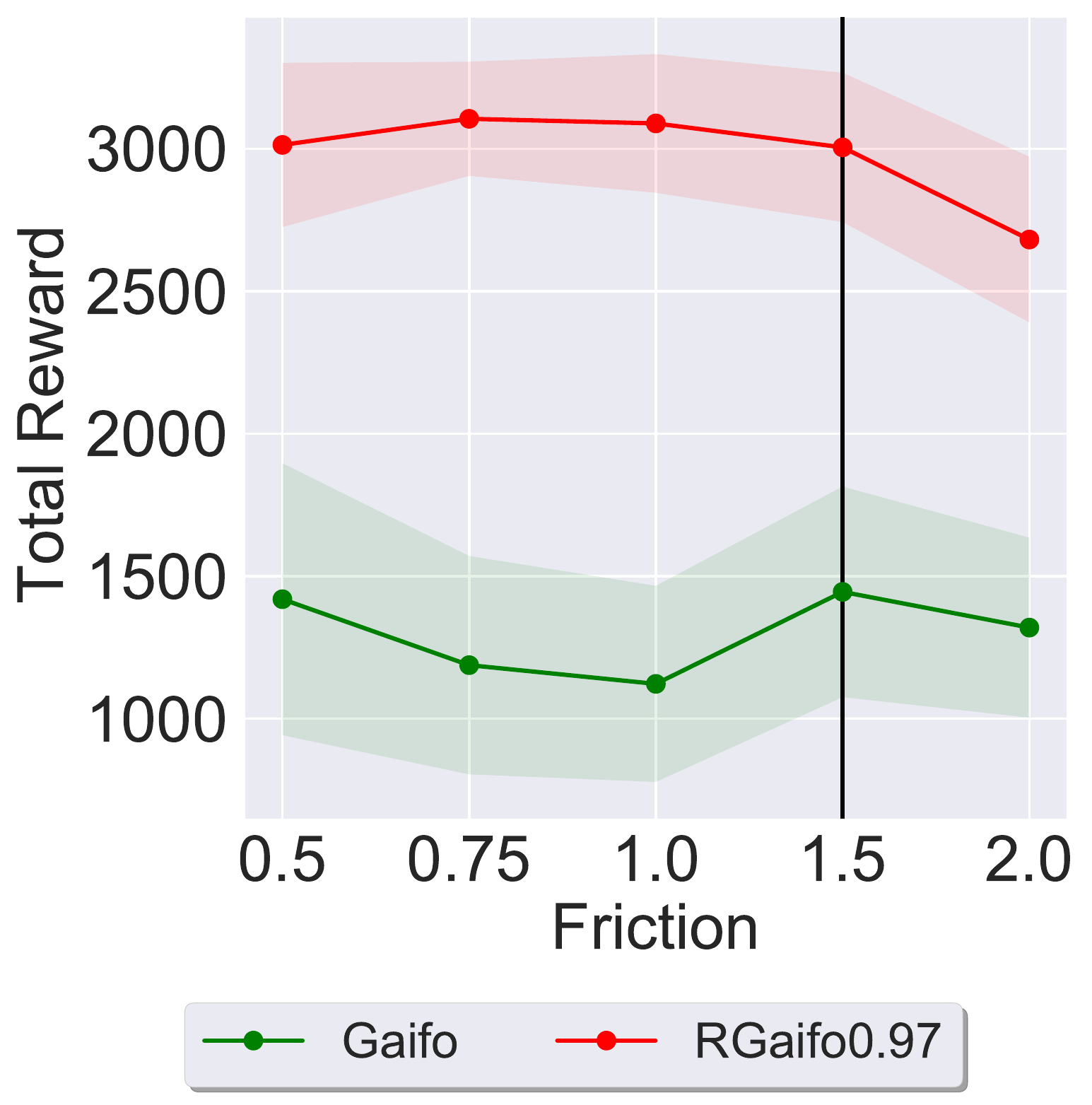}
     } &
\subfloat[HalfCheetah]{%
       \includegraphics[width=0.16\linewidth]{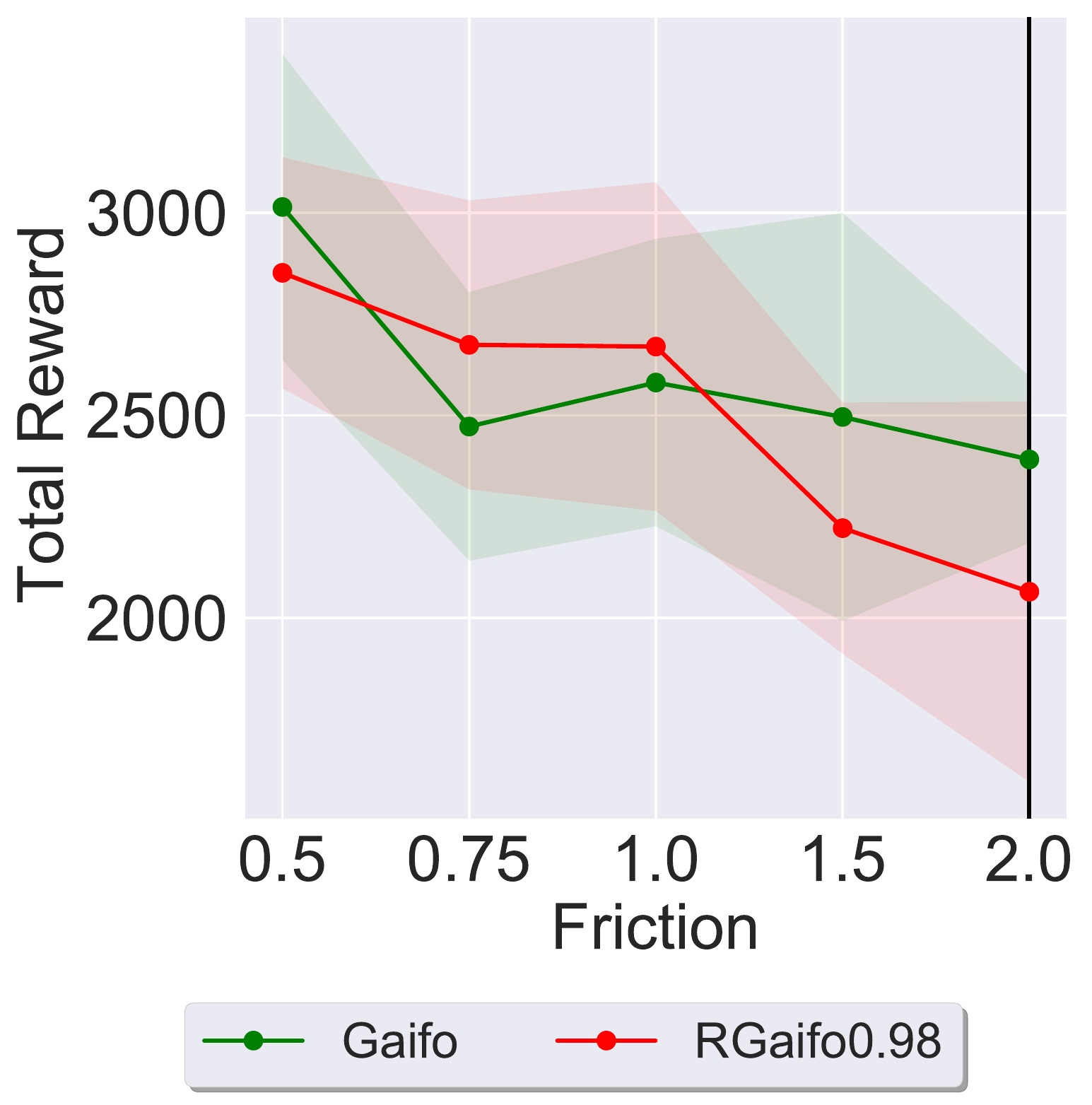}
     } \\
\subfloat[Walker]{%
       \includegraphics[width=0.16\linewidth]{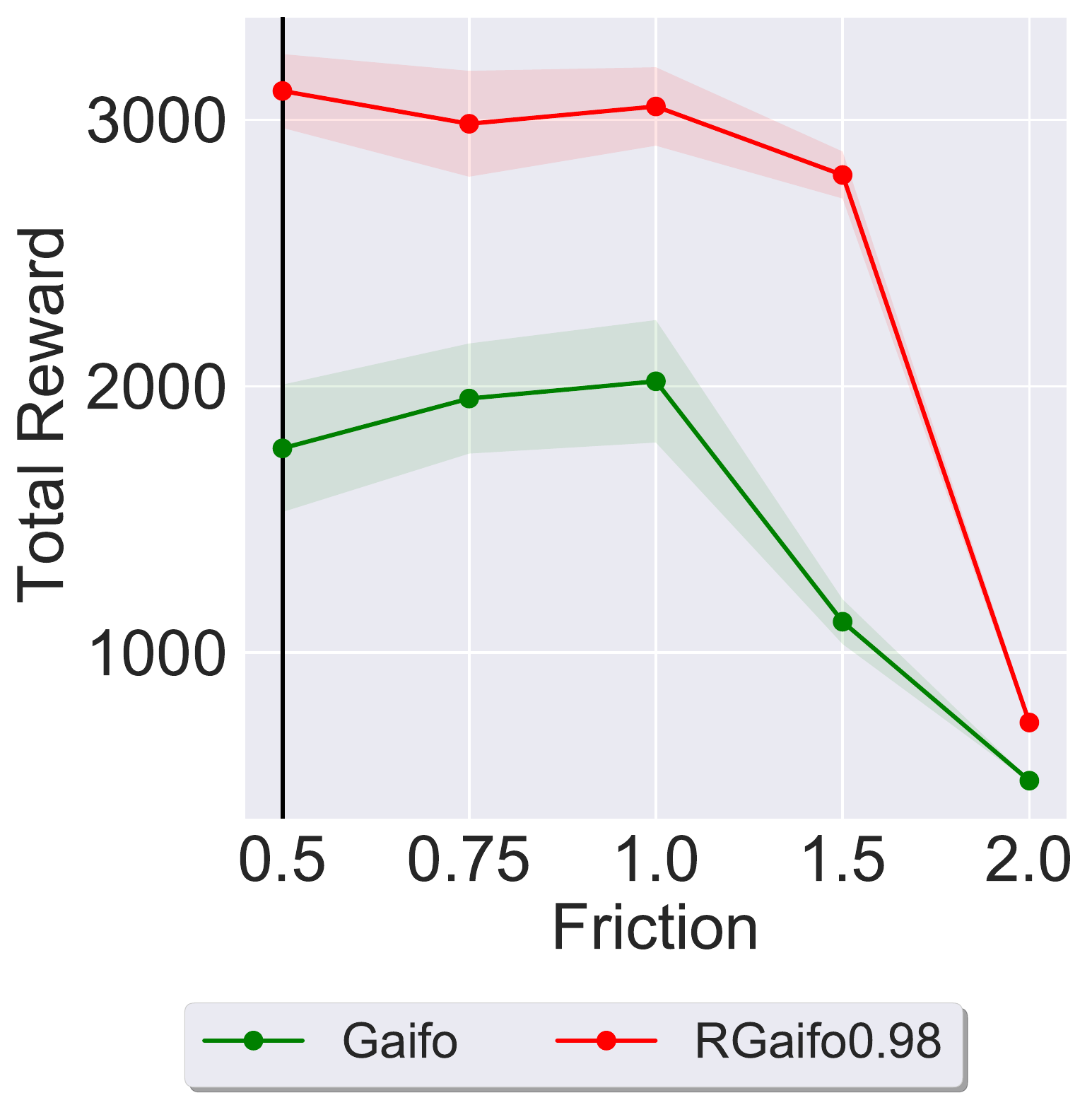}
     } &
\subfloat[Walker]{%
       \includegraphics[width=0.16\linewidth]{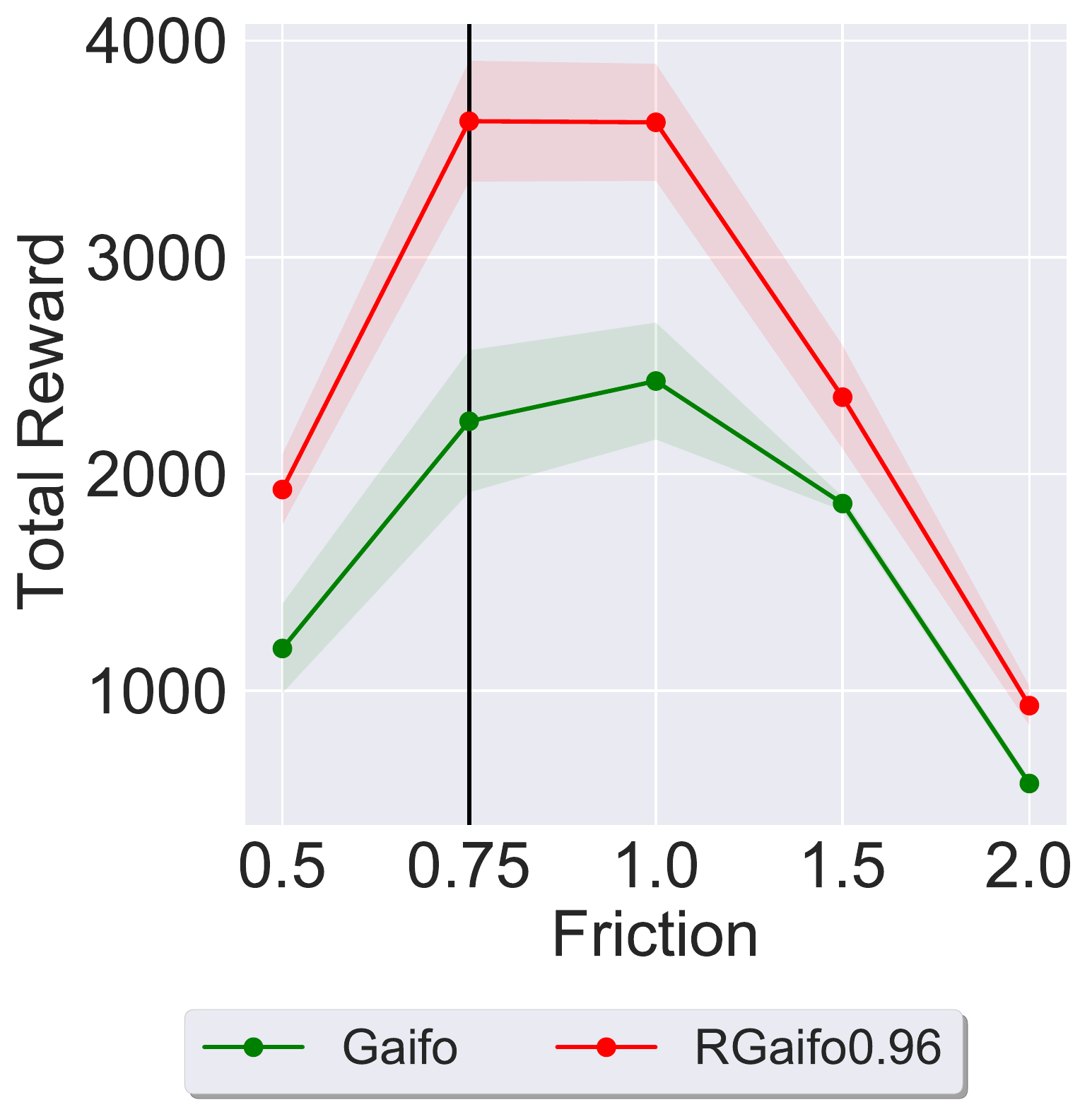}
     } &
\subfloat[Walker]{%
       \includegraphics[width=0.16\linewidth]{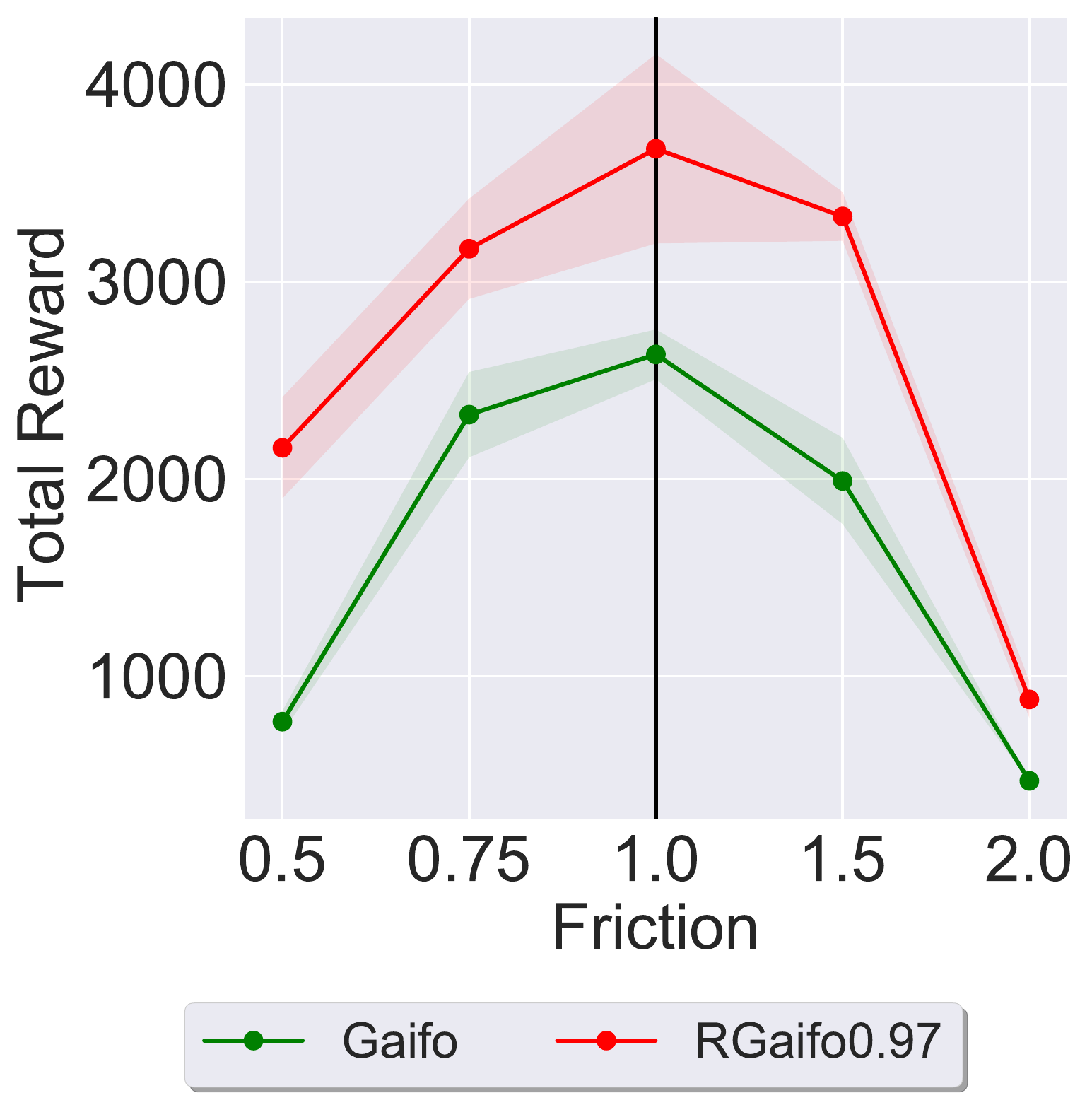}
     } &
\subfloat[Walker]{%
       \includegraphics[width=0.16\linewidth]{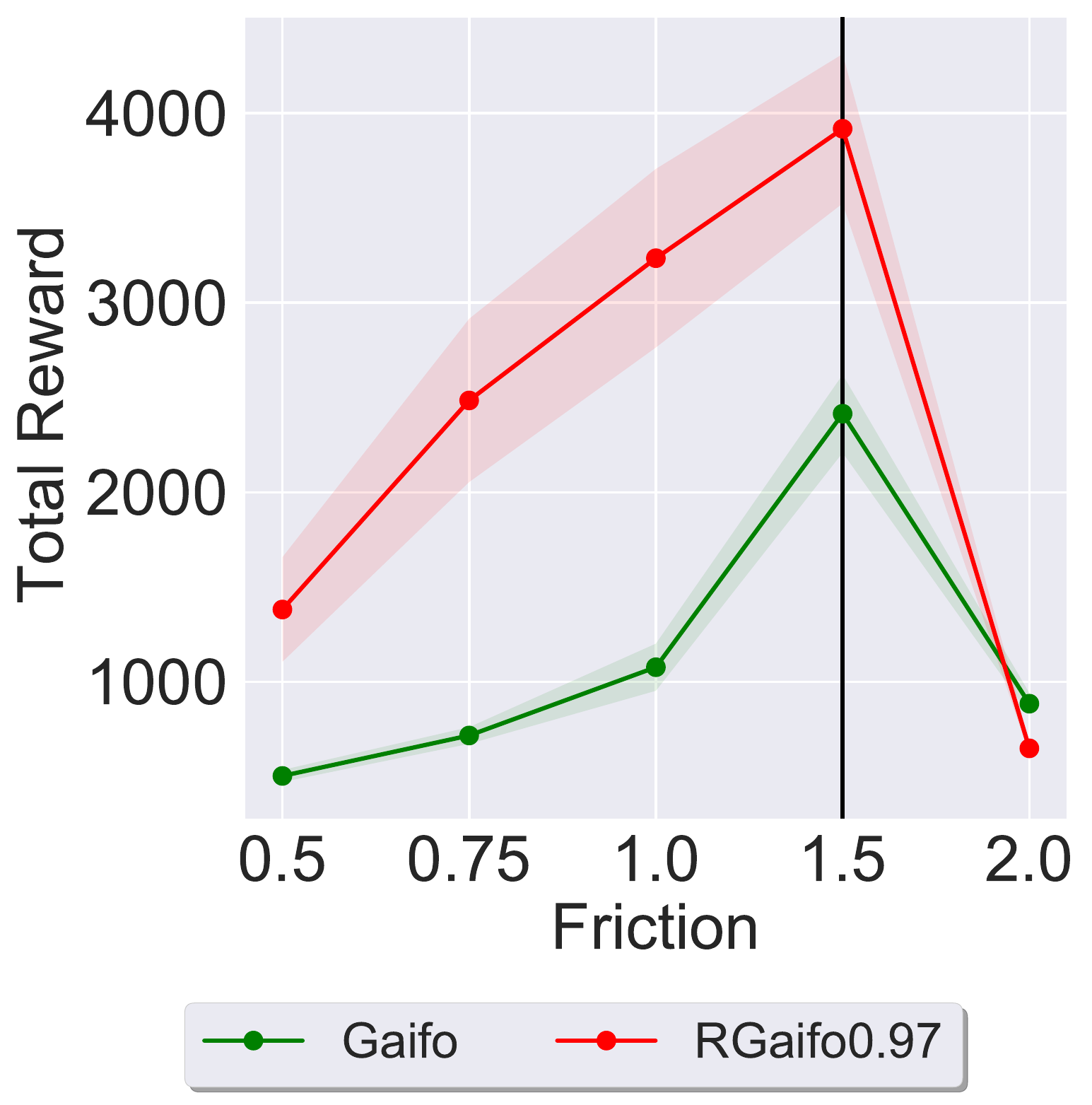}
     } &
\subfloat[Walker]{%
       \includegraphics[width=0.16\linewidth]{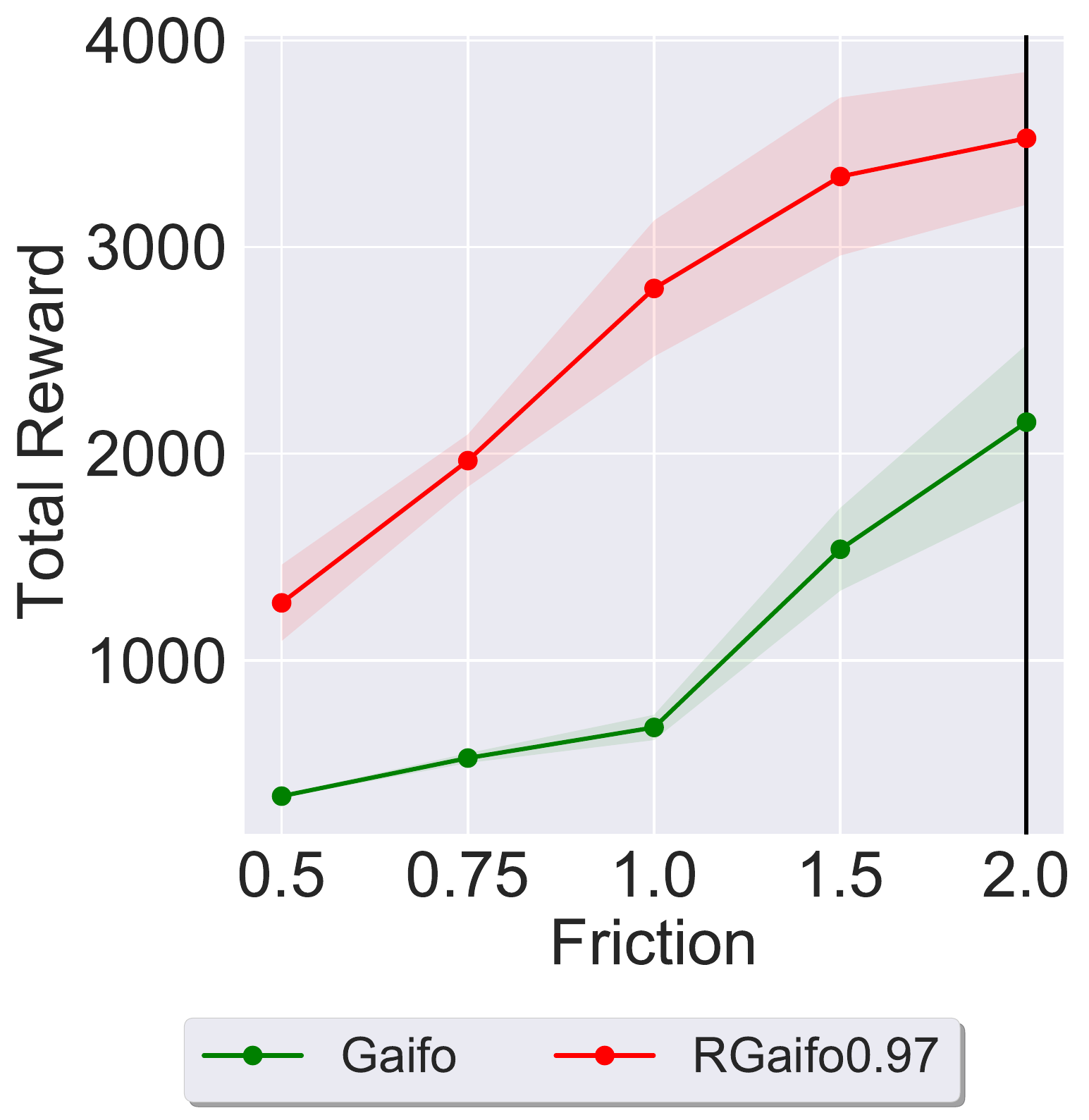}
     } \\
\subfloat[Hopper]{%
       \includegraphics[width=0.16\linewidth]{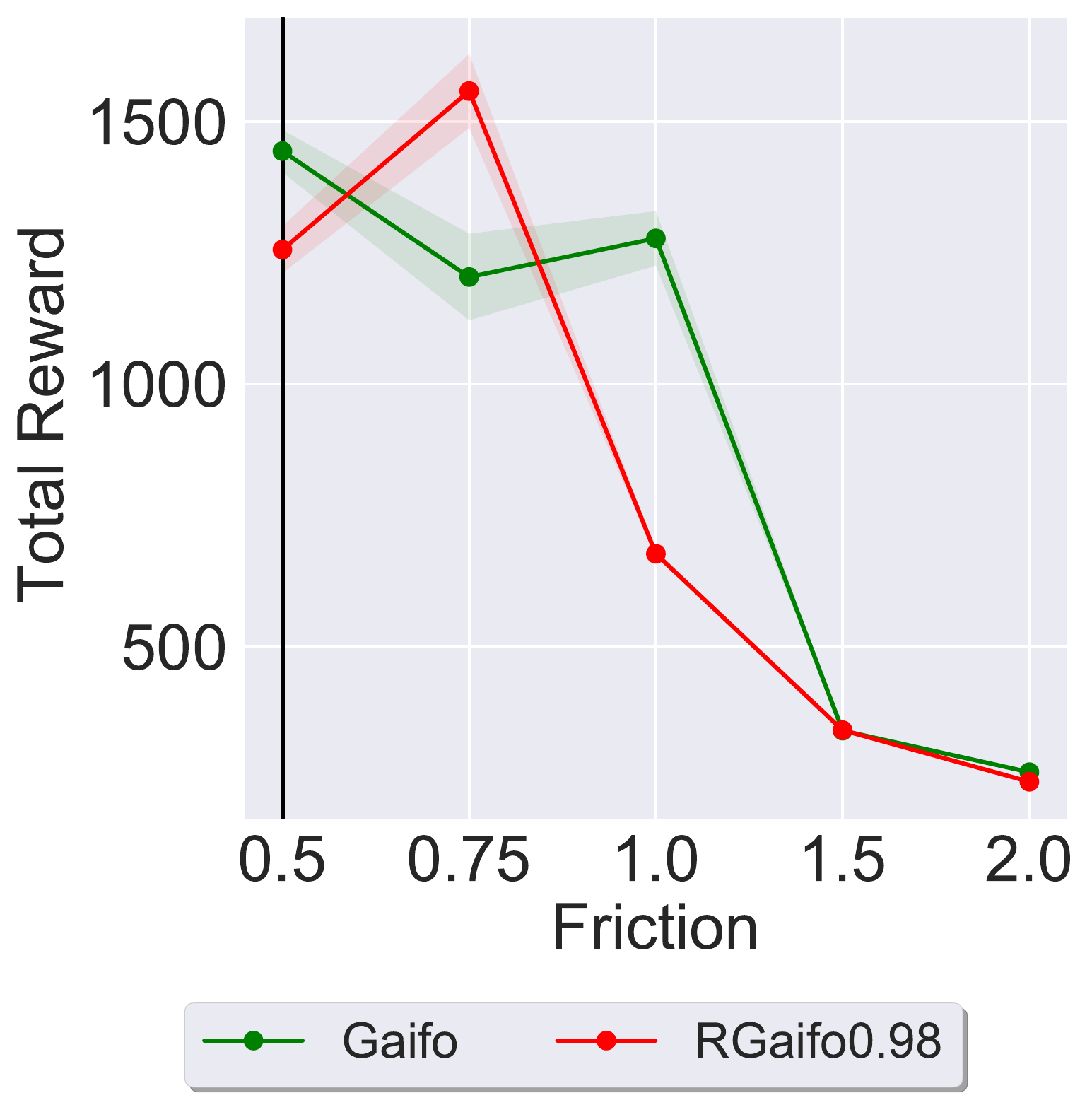}
     } &
\subfloat[Hopper]{%
       \includegraphics[width=0.16\linewidth]{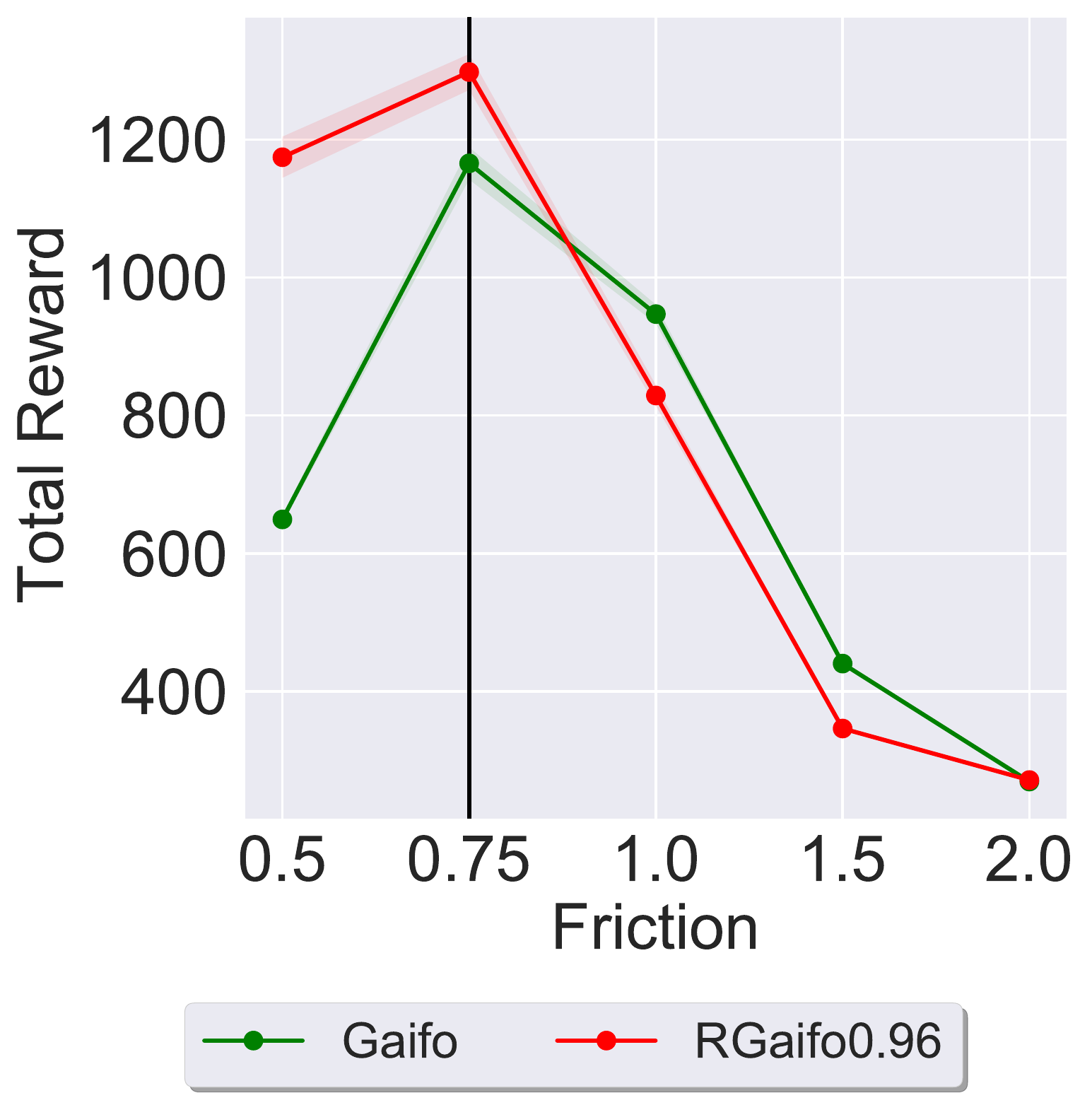}
     } &
\subfloat[Hopper]{%
       \includegraphics[width=0.16\linewidth]{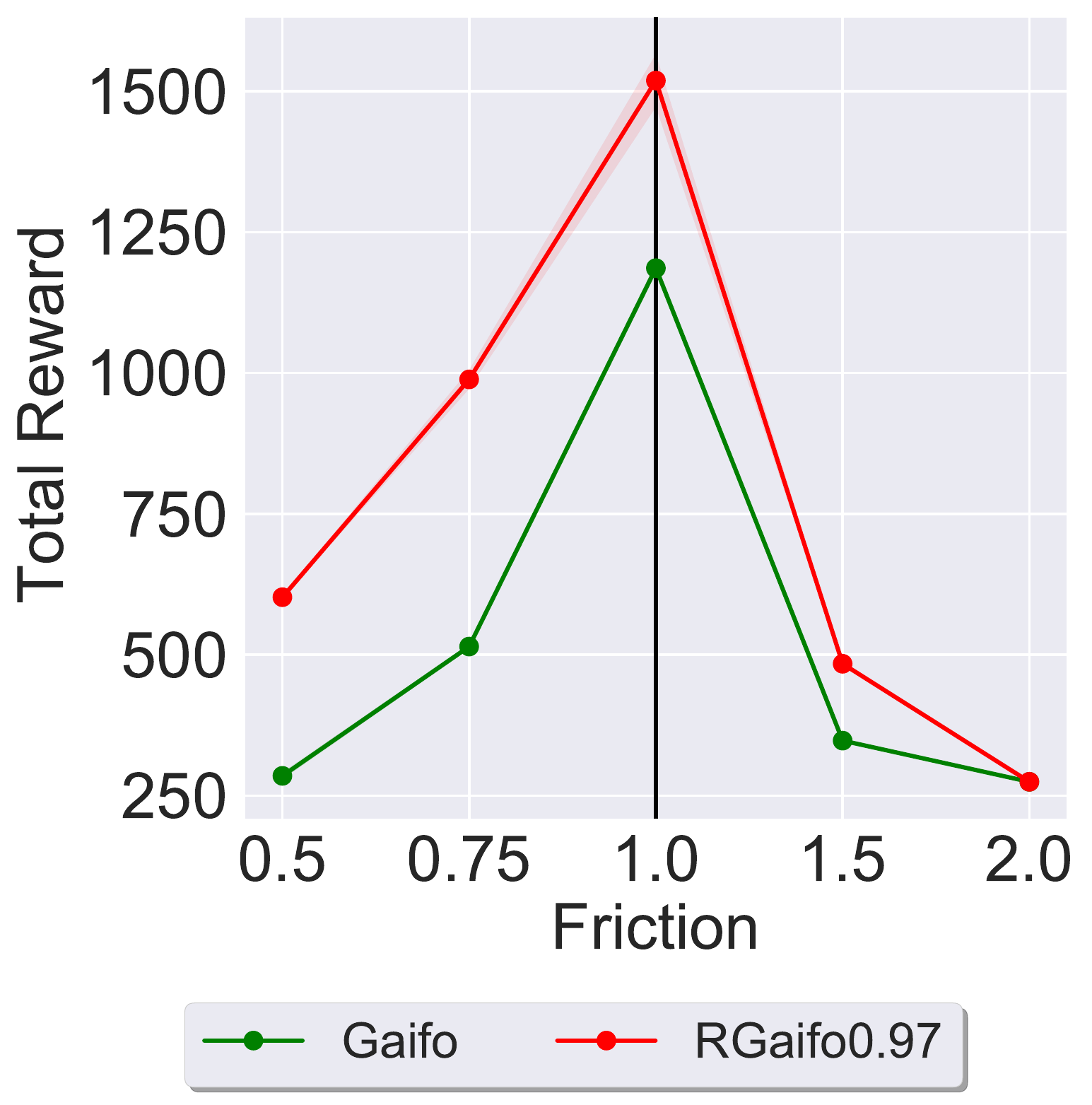}
     } &
\subfloat[Hopper]{%
       \includegraphics[width=0.16\linewidth]{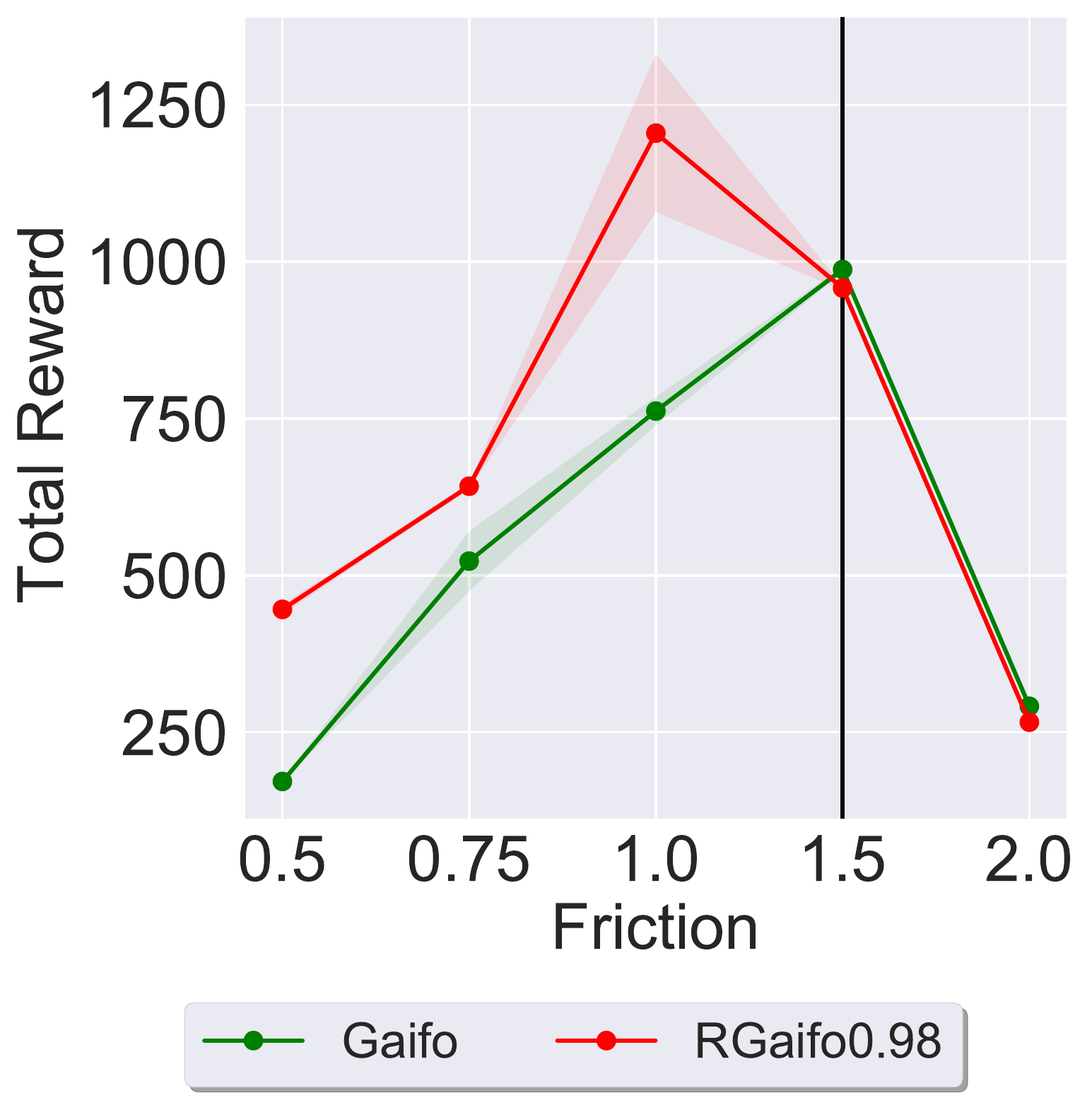}
     } &
\subfloat[Hopper]{%
       \includegraphics[width=0.16\linewidth]{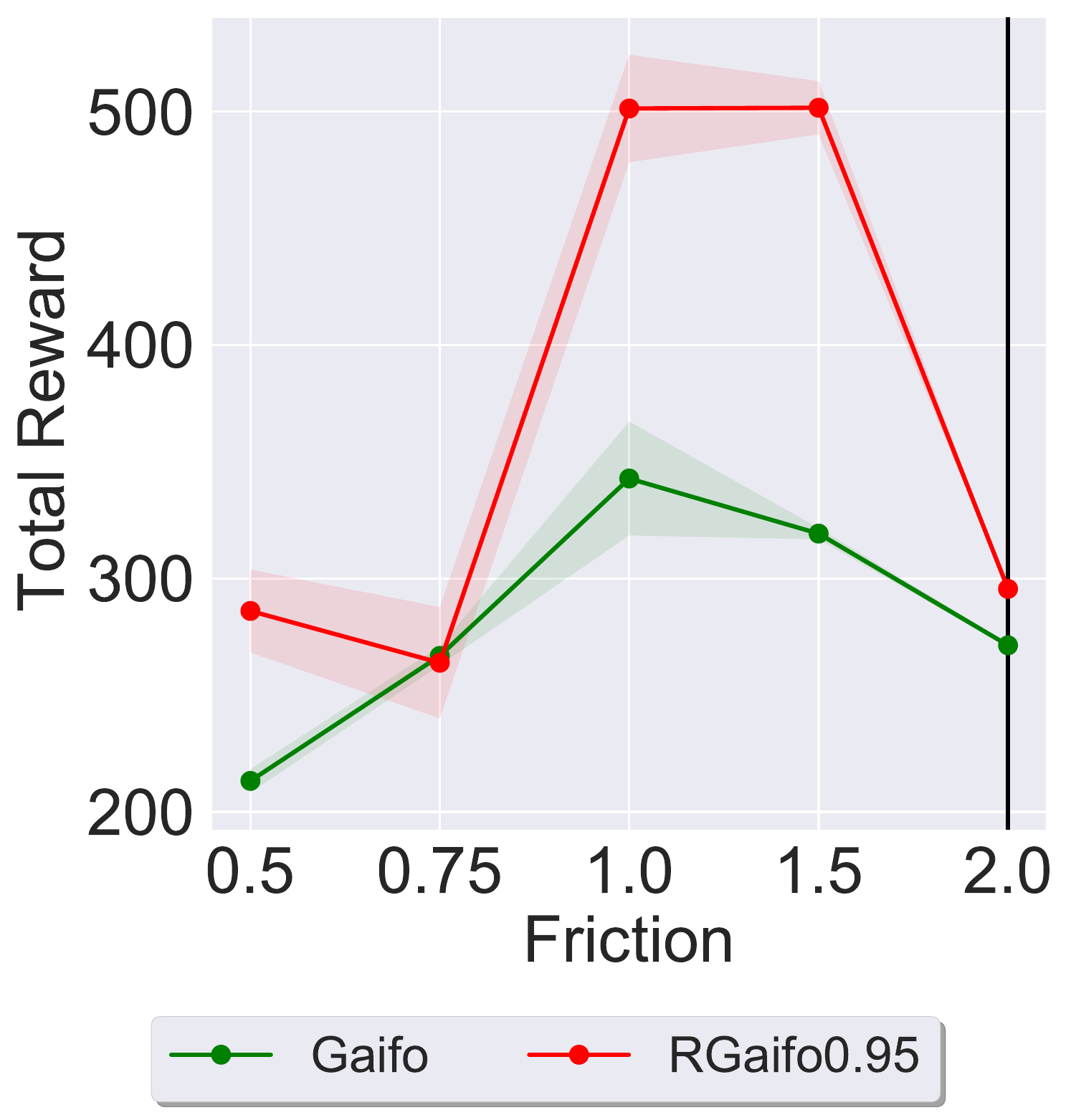}
     } \\

\end{tabular}
\caption{The average (over $3$ seeds) robust performance of Algorithm~\ref{alg:robust-gailfo} with different values of $\alpha$ for each MuJoCo task as reported in the legend of each plot. The expert environment $M^\mathrm{real}$, in which the demonstrations are collected, has relative friction $1.0$. In each plot, the black vertical line corresponds to the relative friction of the learner environment $M^\mathrm{sim}$ where we trained the policy with Algorithm~\ref{alg:robust-gailfo}. The x-axis denotes the relative friction of the test environment $M^\mathrm{test}$ in which the policies are evaluated. The policies are evaluated over $1e5$ steps truncating the last episode if it does not terminate.}
\label{fig:RobustnessFrictionFixedAlpha}
\end{figure*}

\begin{figure*}[!h] 
\centering
\begin{tabular}{ccccc}
\subfloat[HalfCheetah]{%
       \includegraphics[width=0.16\linewidth]{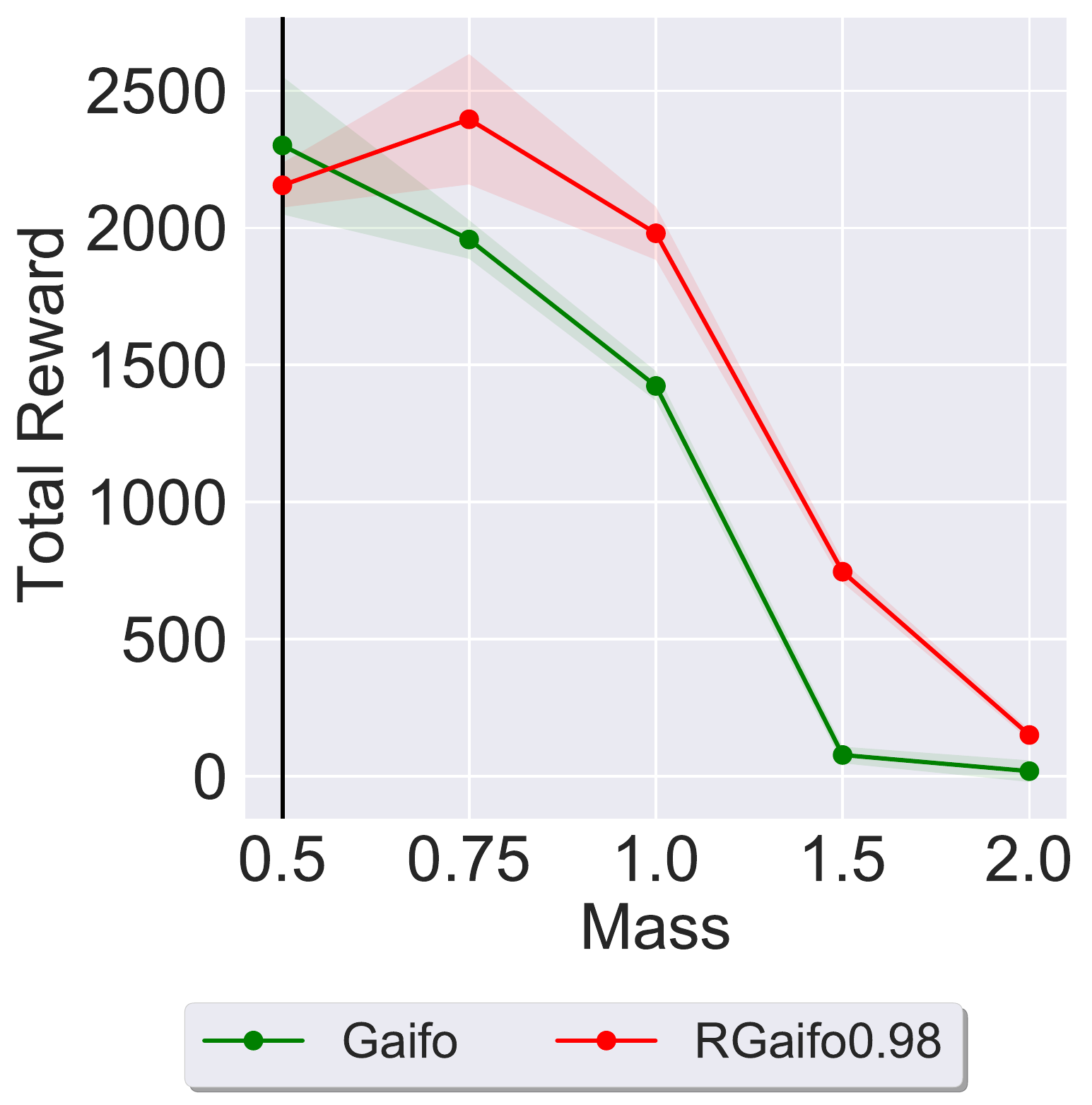}
     } &
\subfloat[HalfCheetah]{%
       \includegraphics[width=0.16\linewidth]{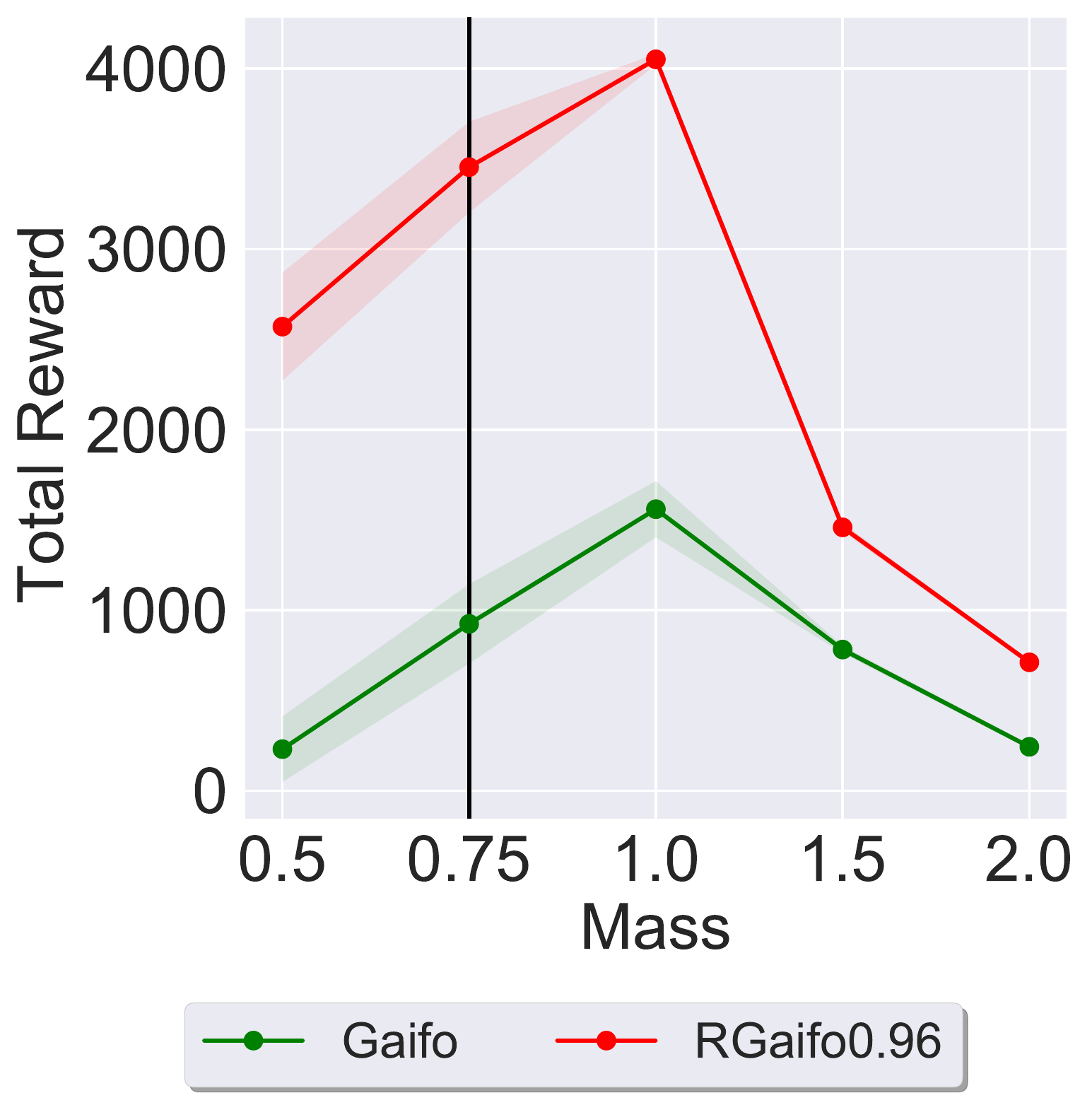}
     } &
\subfloat[HalfCheetah]{%
       \includegraphics[width=0.16\linewidth]{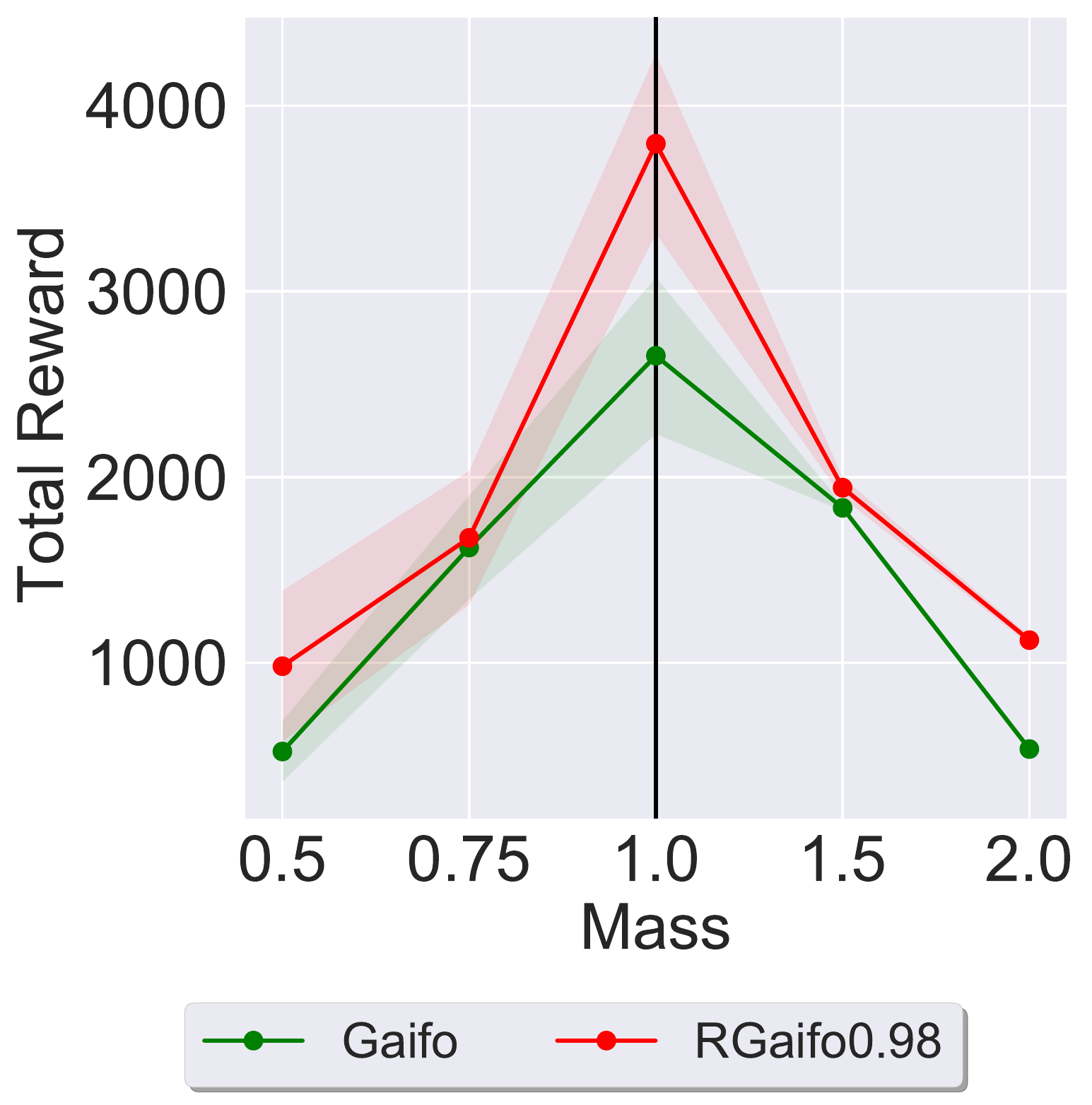}
     } &
\subfloat[HalfCheetah]{%
       \includegraphics[width=0.16\linewidth]{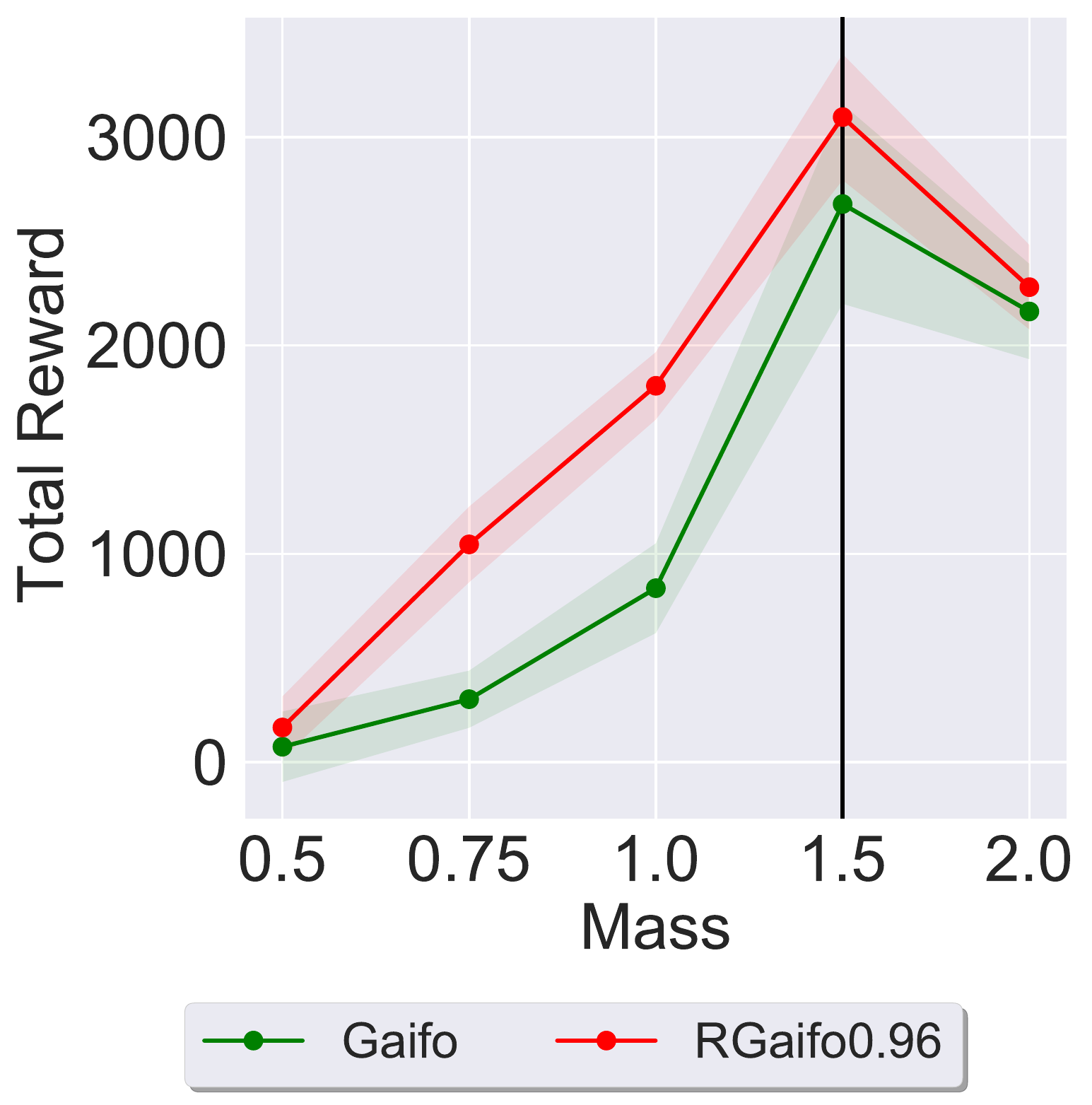}
     } &
\subfloat[HalfCheetah]{%
       \includegraphics[width=0.16\linewidth]{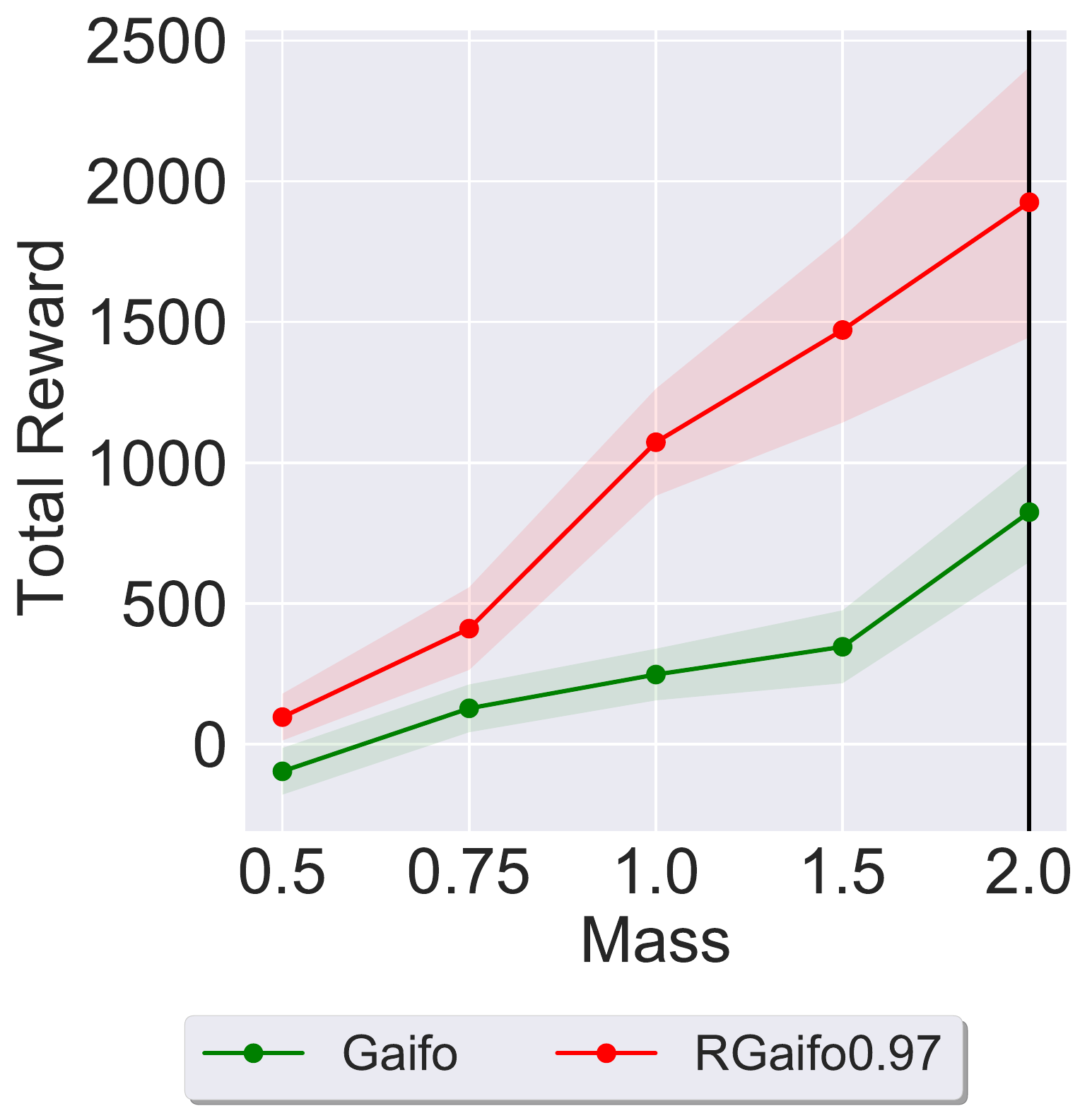}
     } \\
\subfloat[Walker]{%
       \includegraphics[width=0.16\linewidth]{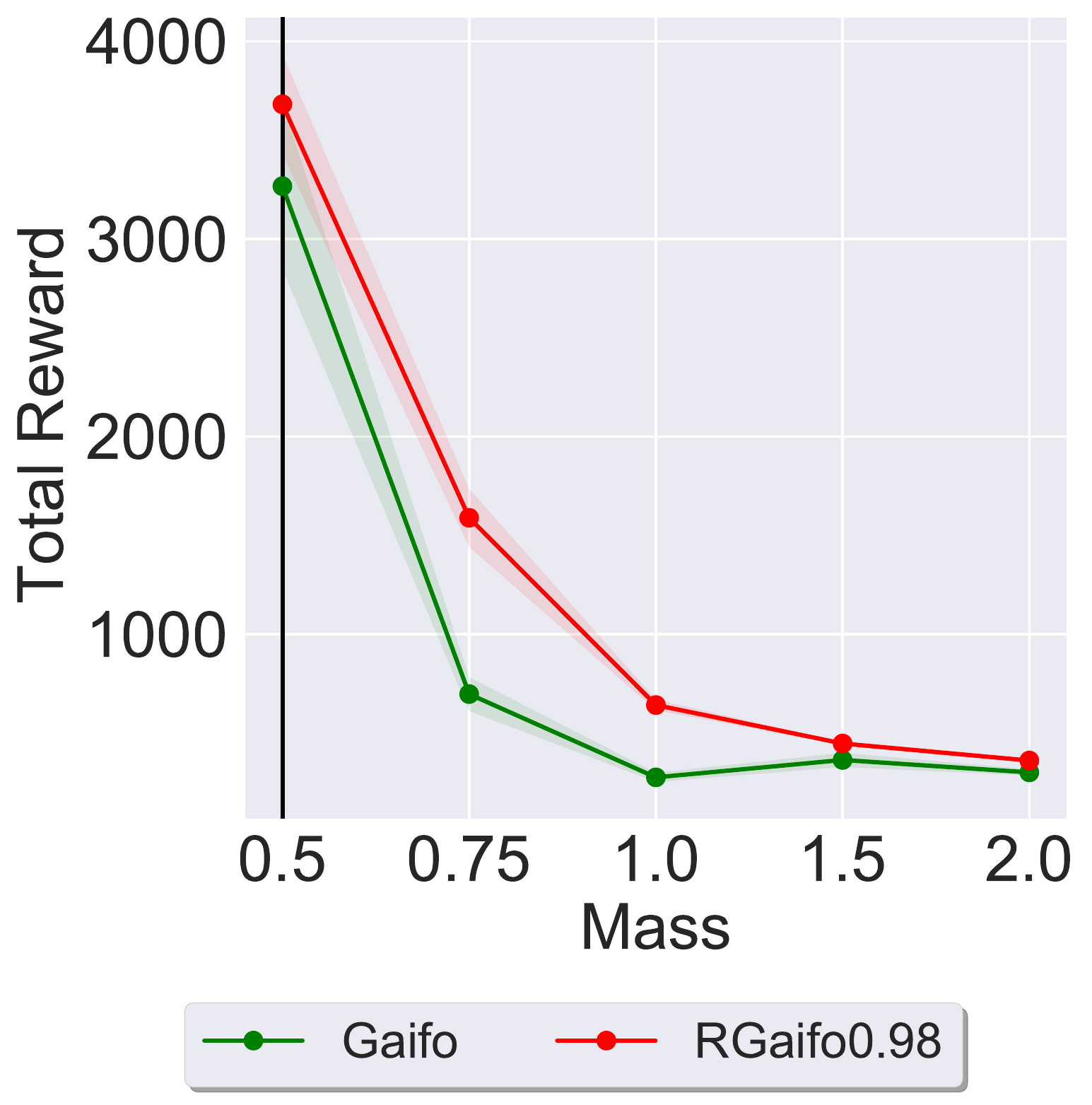}
     } &
\subfloat[Walker]{%
       \includegraphics[width=0.16\linewidth]{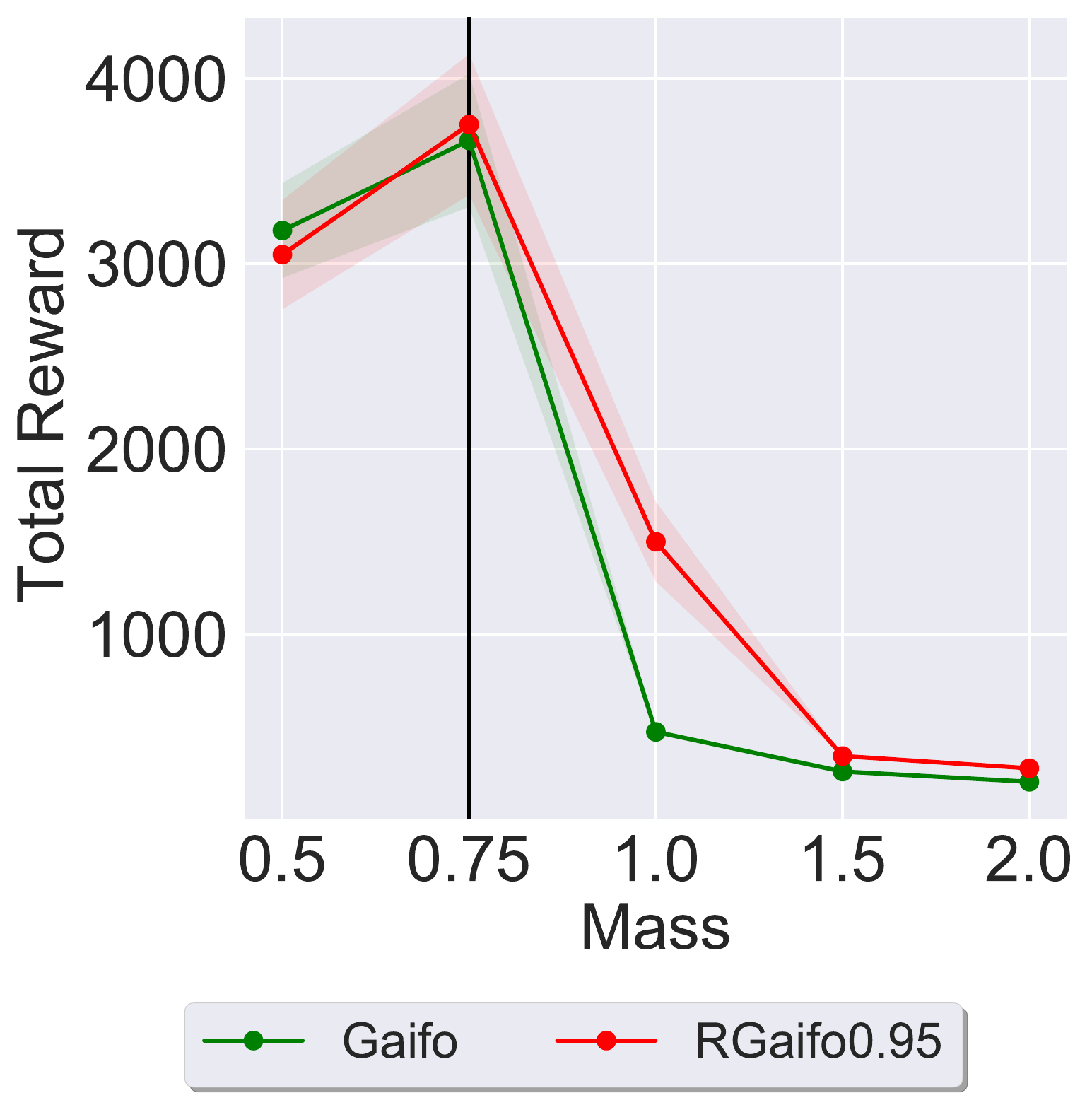}
     } &
\subfloat[Walker]{%
       \includegraphics[width=0.16\linewidth]{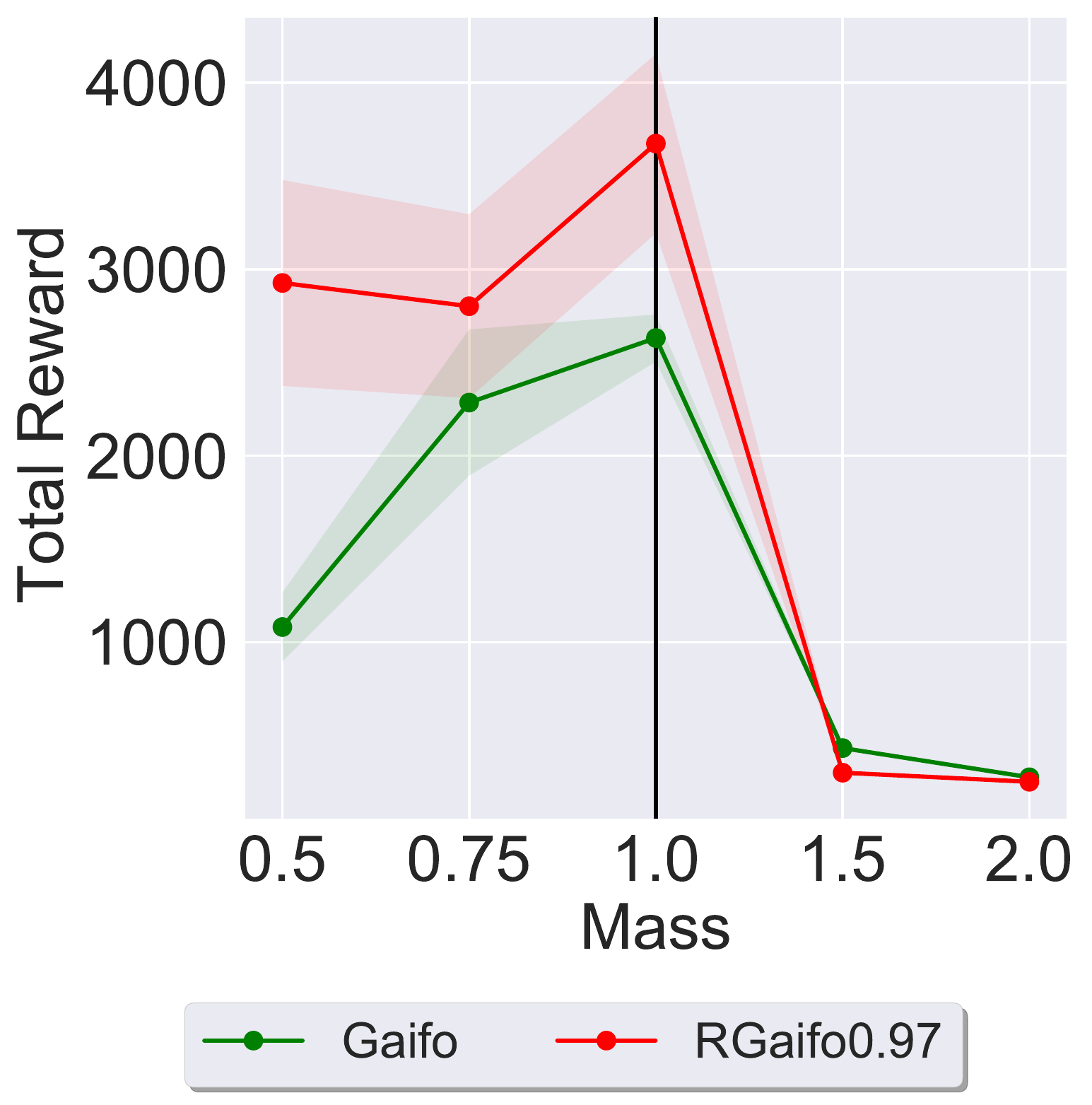}
     } &
\subfloat[Walker]{%
       \includegraphics[width=0.16\linewidth]{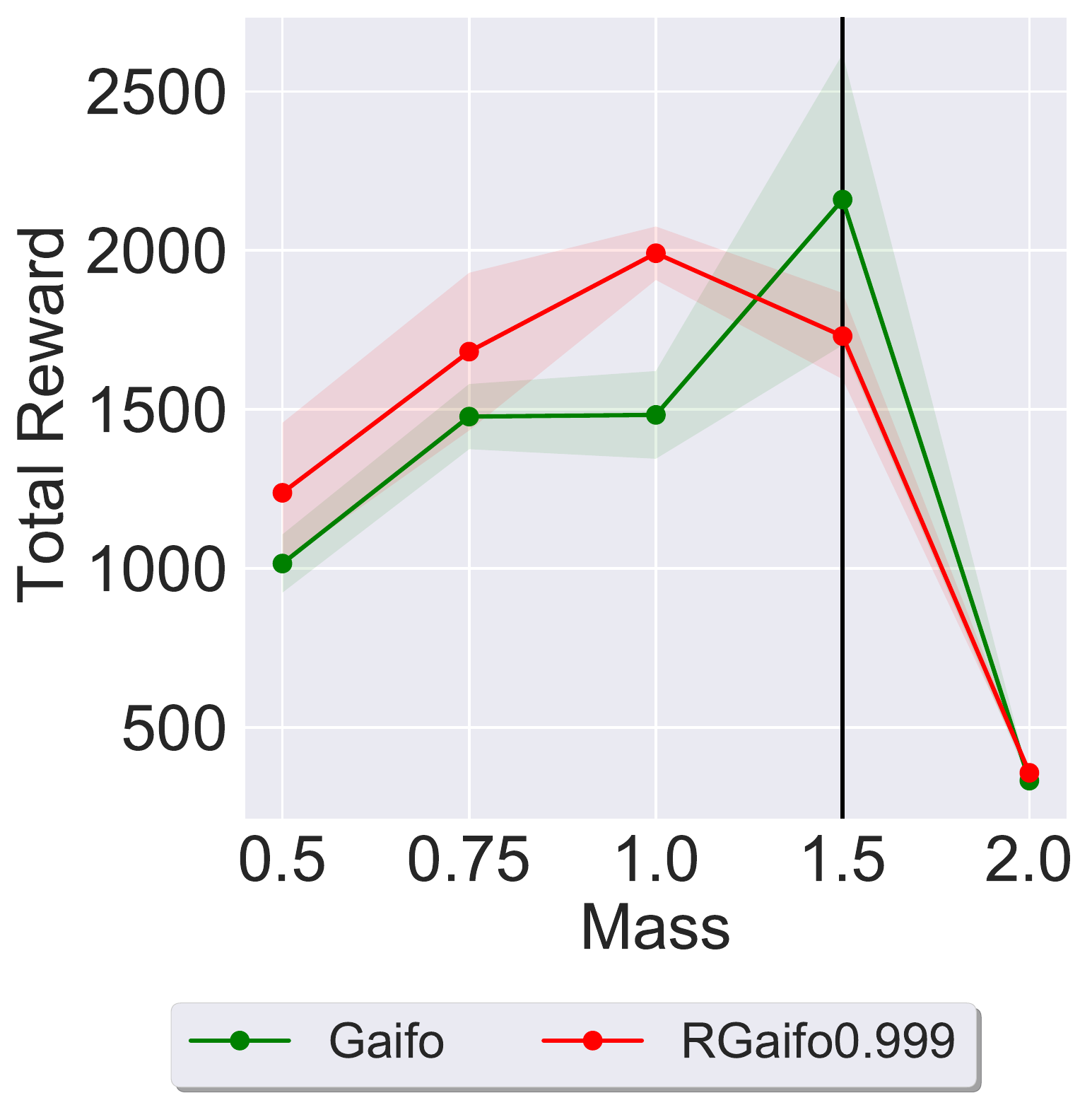}
     } &
\subfloat[Walker]{%
       \includegraphics[width=0.16\linewidth]{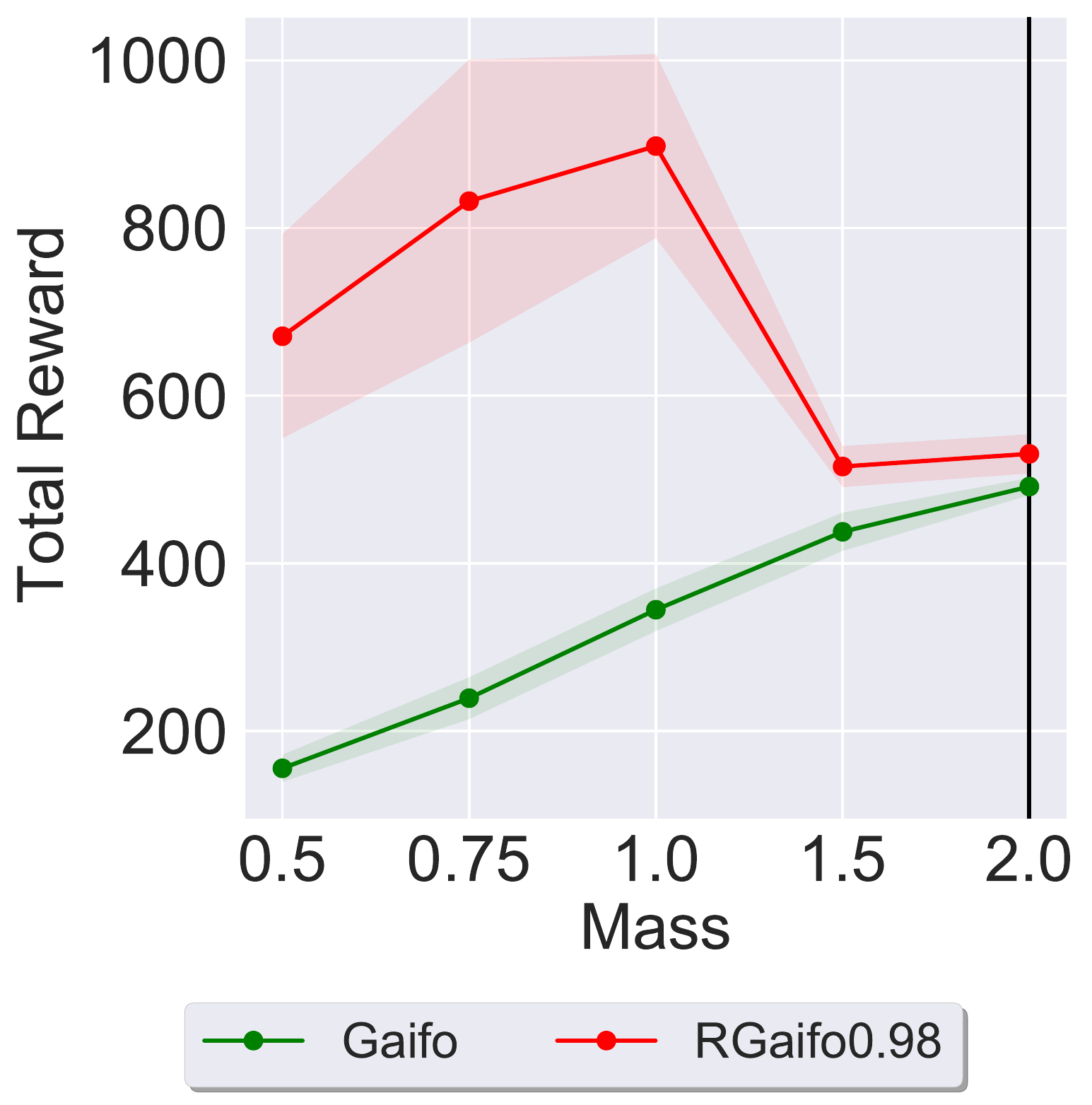}
     } \\
\subfloat[Hopper]{%
       \includegraphics[width=0.16\linewidth]{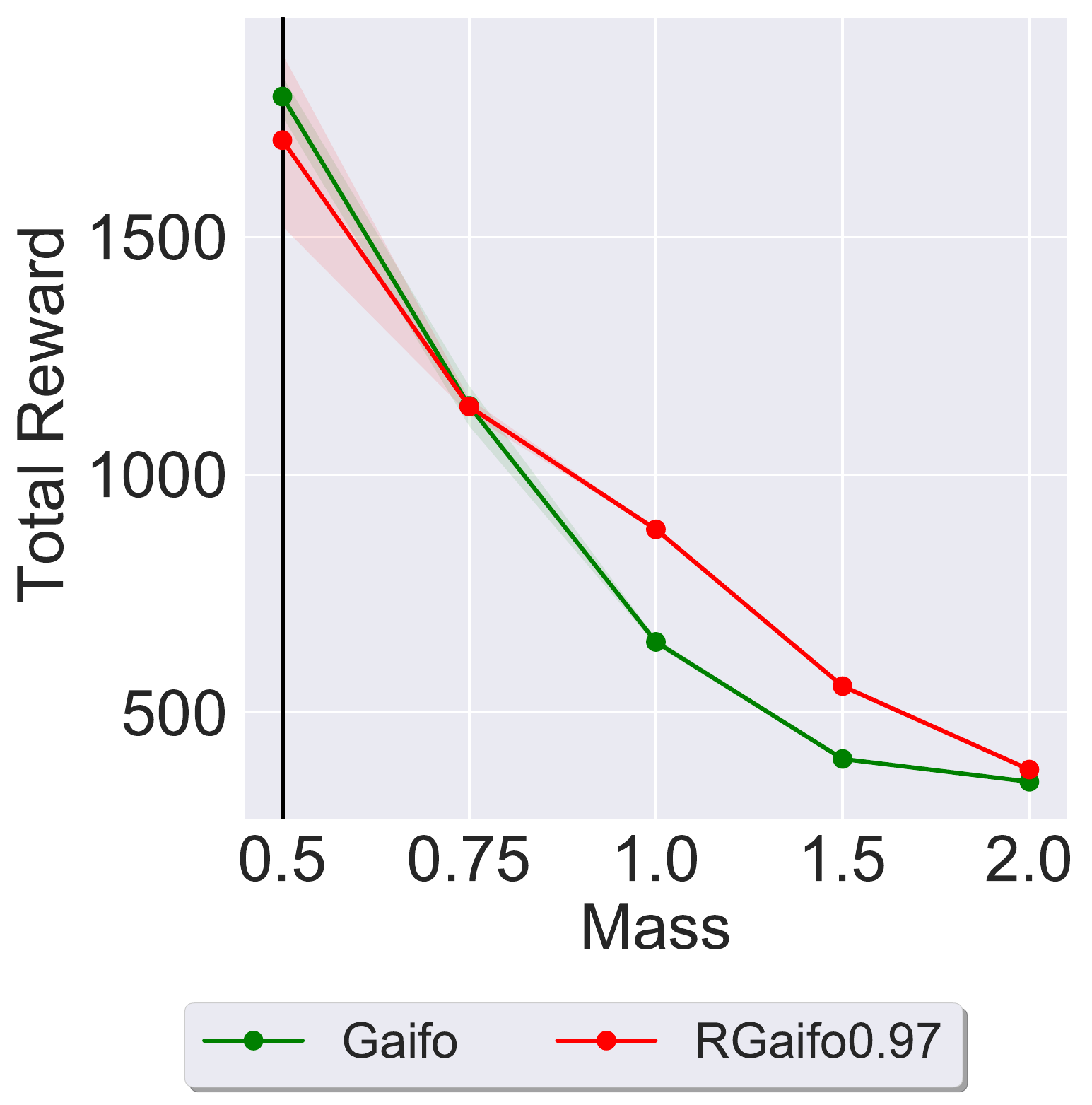}
     } &
\subfloat[Hopper]{%
       \includegraphics[width=0.16\linewidth]{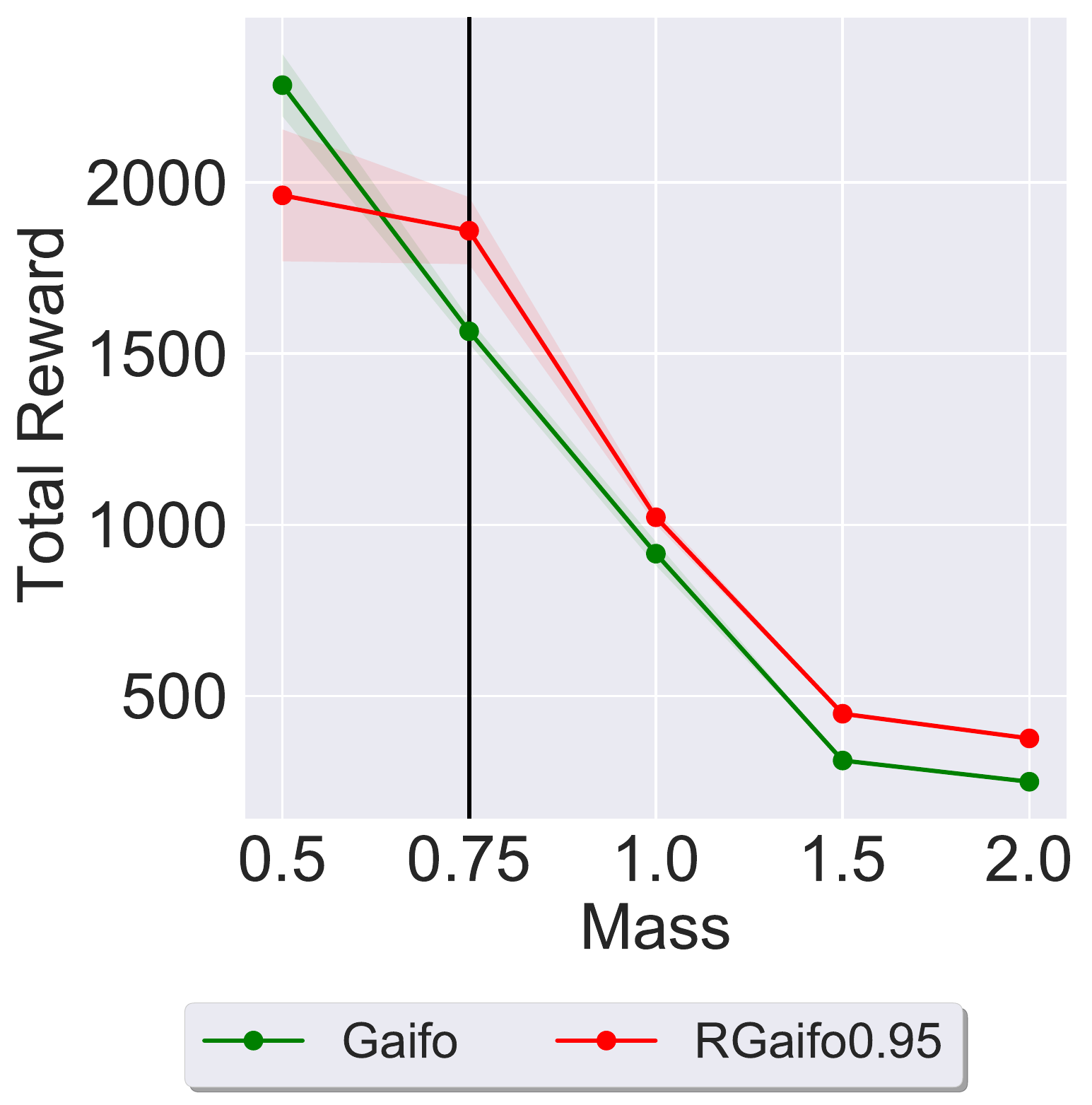}
     } &
\subfloat[Hopper]{%
       \includegraphics[width=0.16\linewidth]{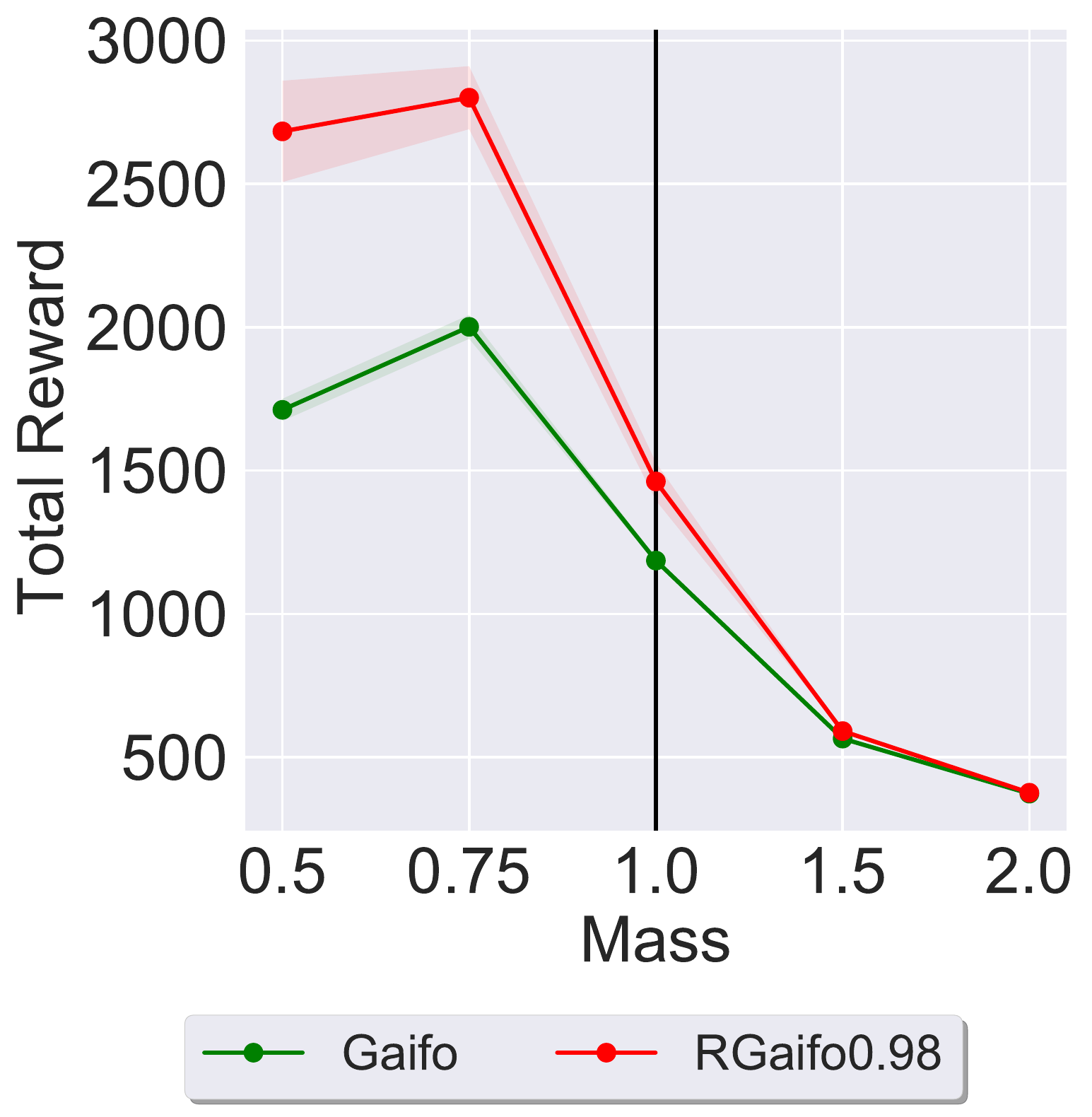}
     } &
\subfloat[Hopper]{%
       \includegraphics[width=0.16\linewidth]{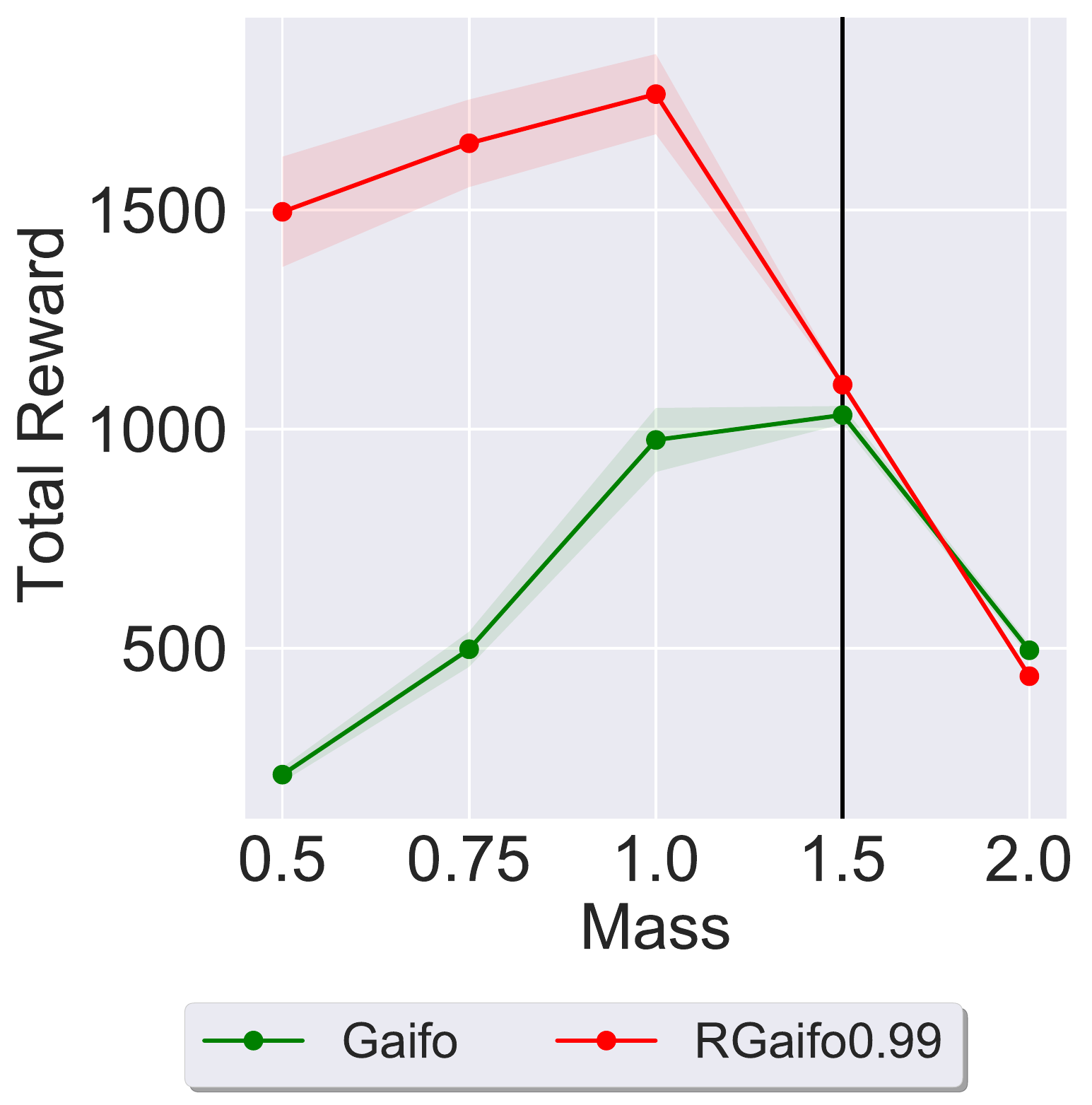}
     } &
\subfloat[Hopper]{%
       \includegraphics[width=0.16\linewidth]{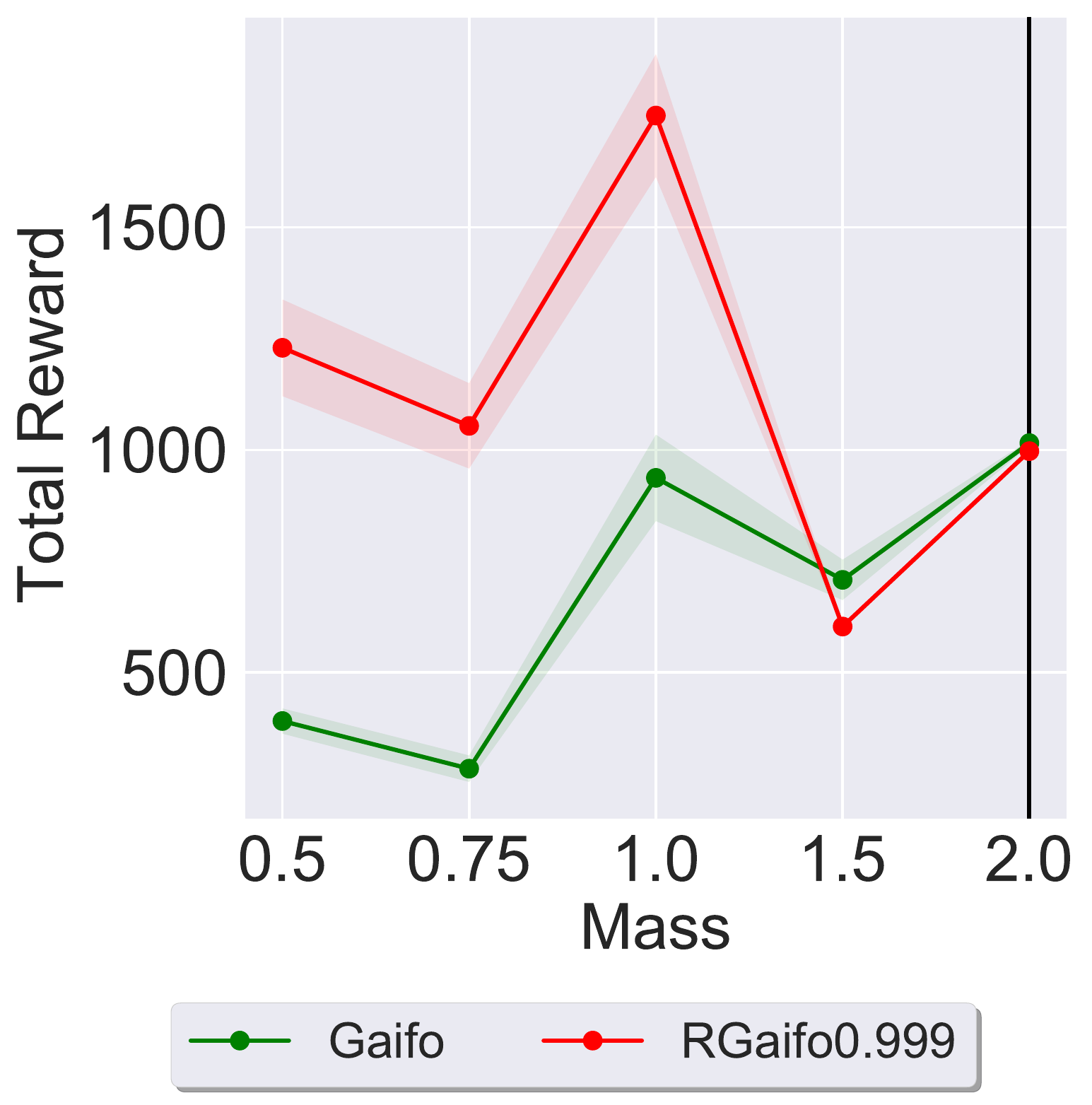}
     } \\
\subfloat[InvDoublePend]{%
       \includegraphics[width=0.16\linewidth]{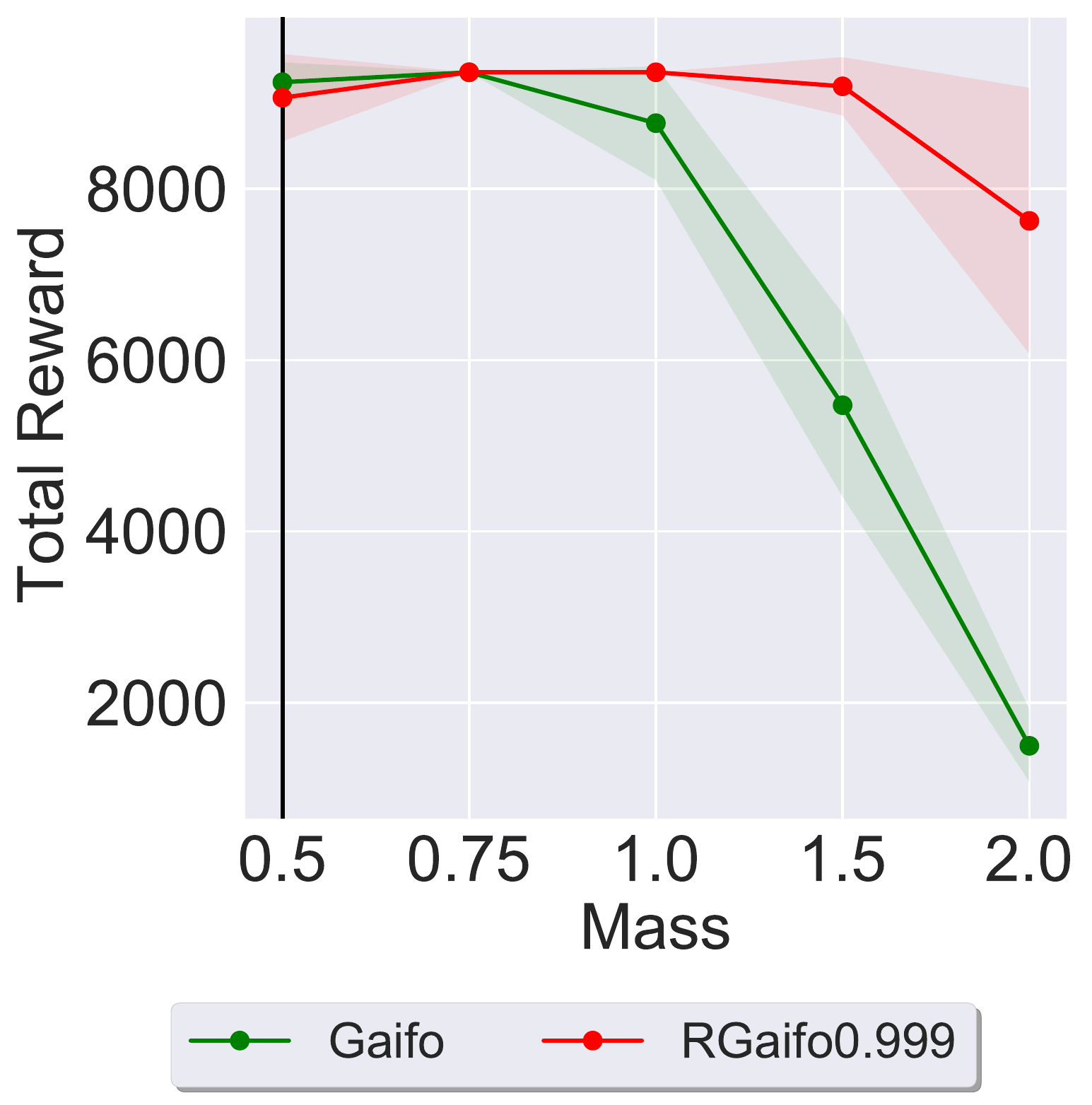}
     } &
\subfloat[InvDoublePend]{%
       \includegraphics[width=0.16\linewidth]{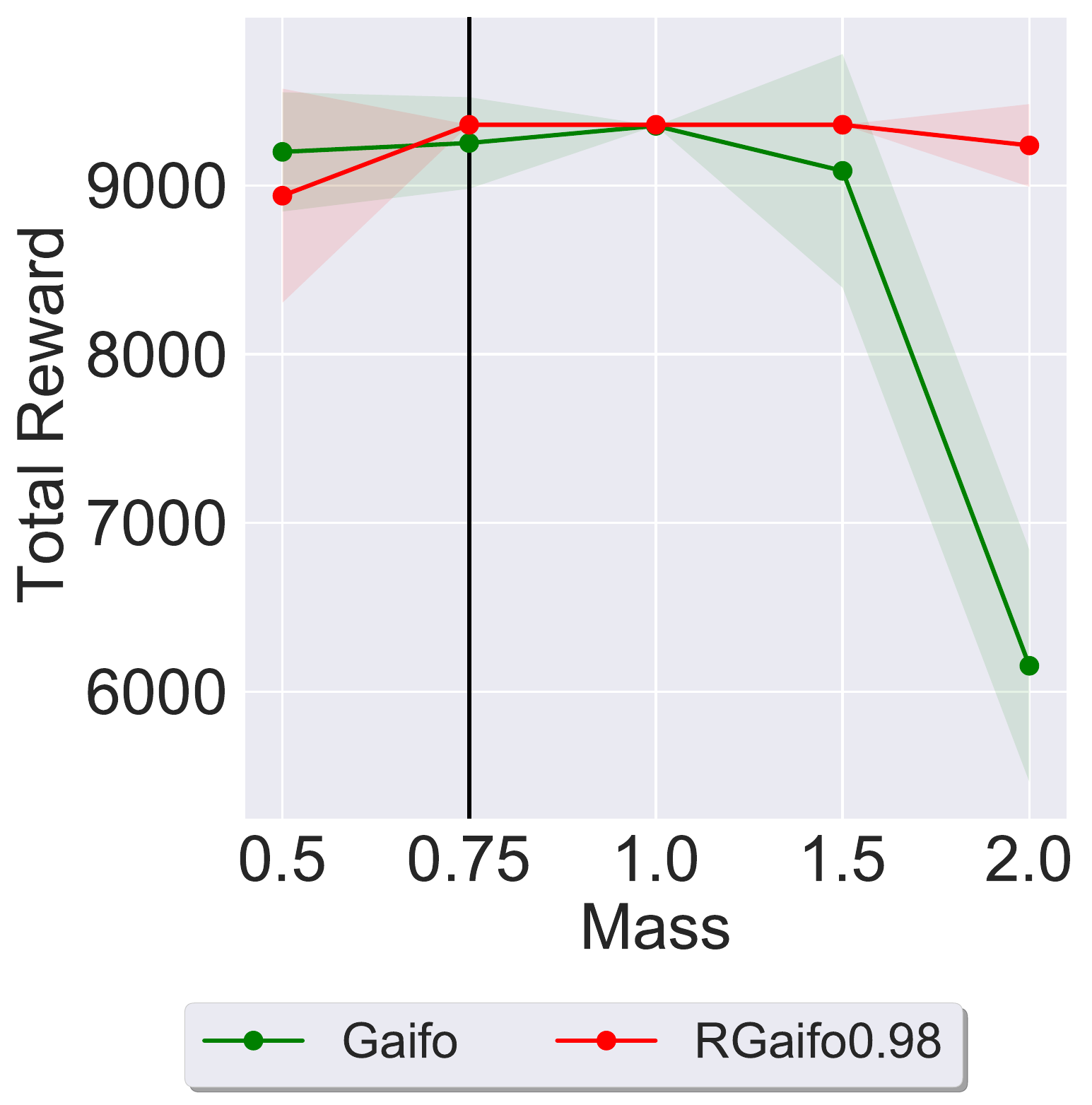}
     } &
\subfloat[InvDoublePend]{%
       \includegraphics[width=0.16\linewidth]{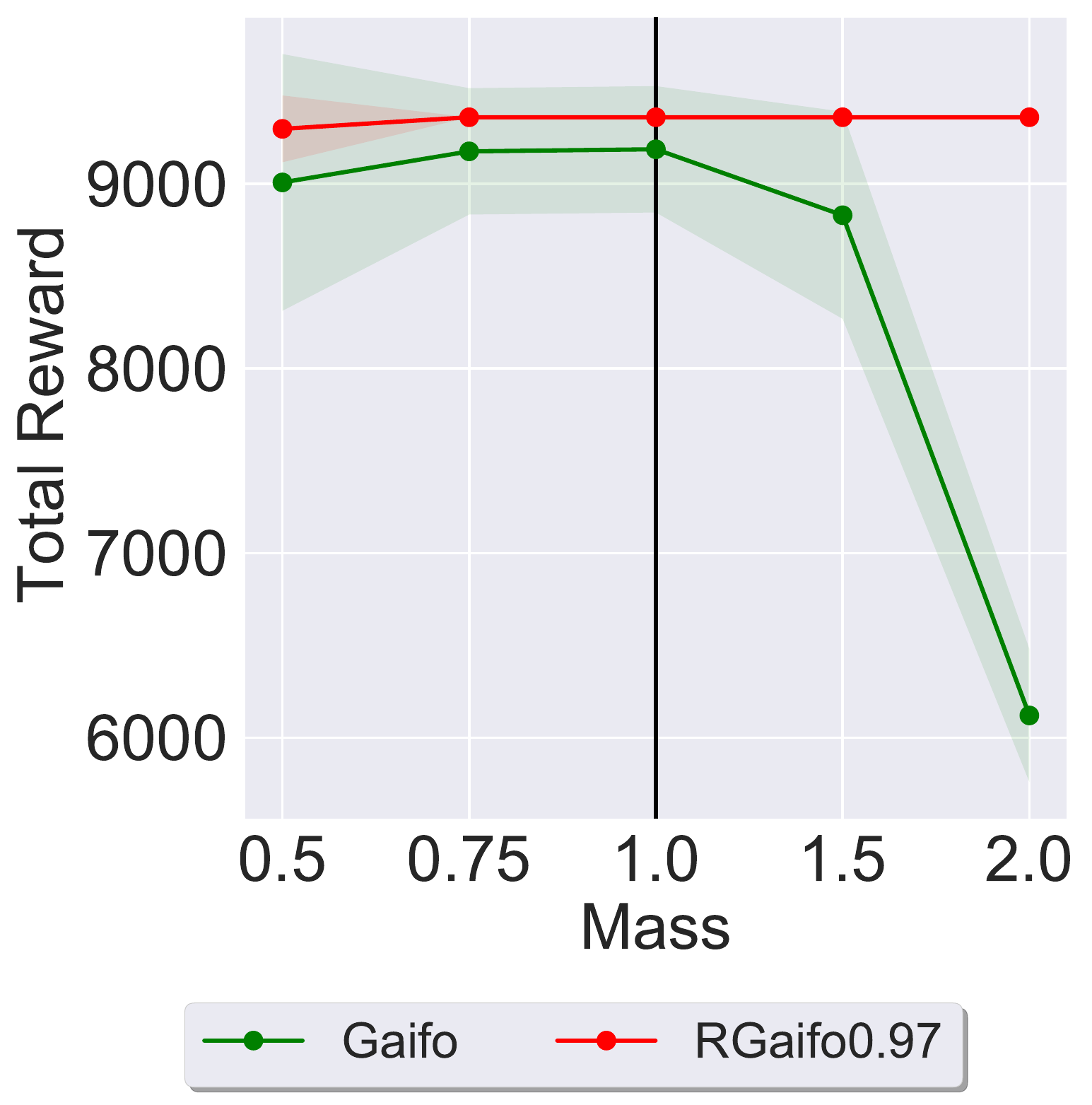}
     } &
\subfloat[InvDoublePend]{%
       \includegraphics[width=0.16\linewidth]{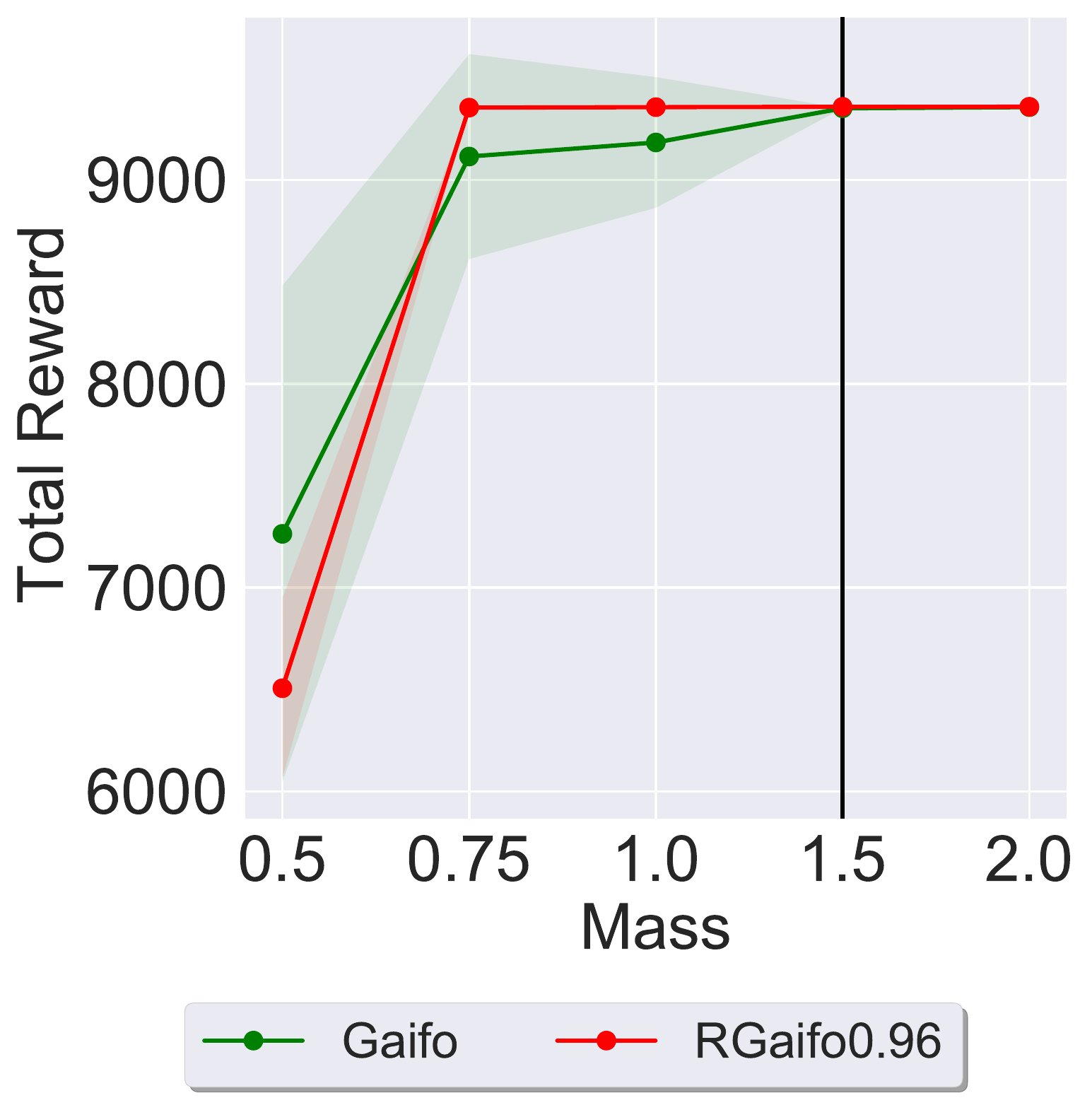}
     } &
\subfloat[InvDoublePend]{%
       \includegraphics[width=0.16\linewidth]{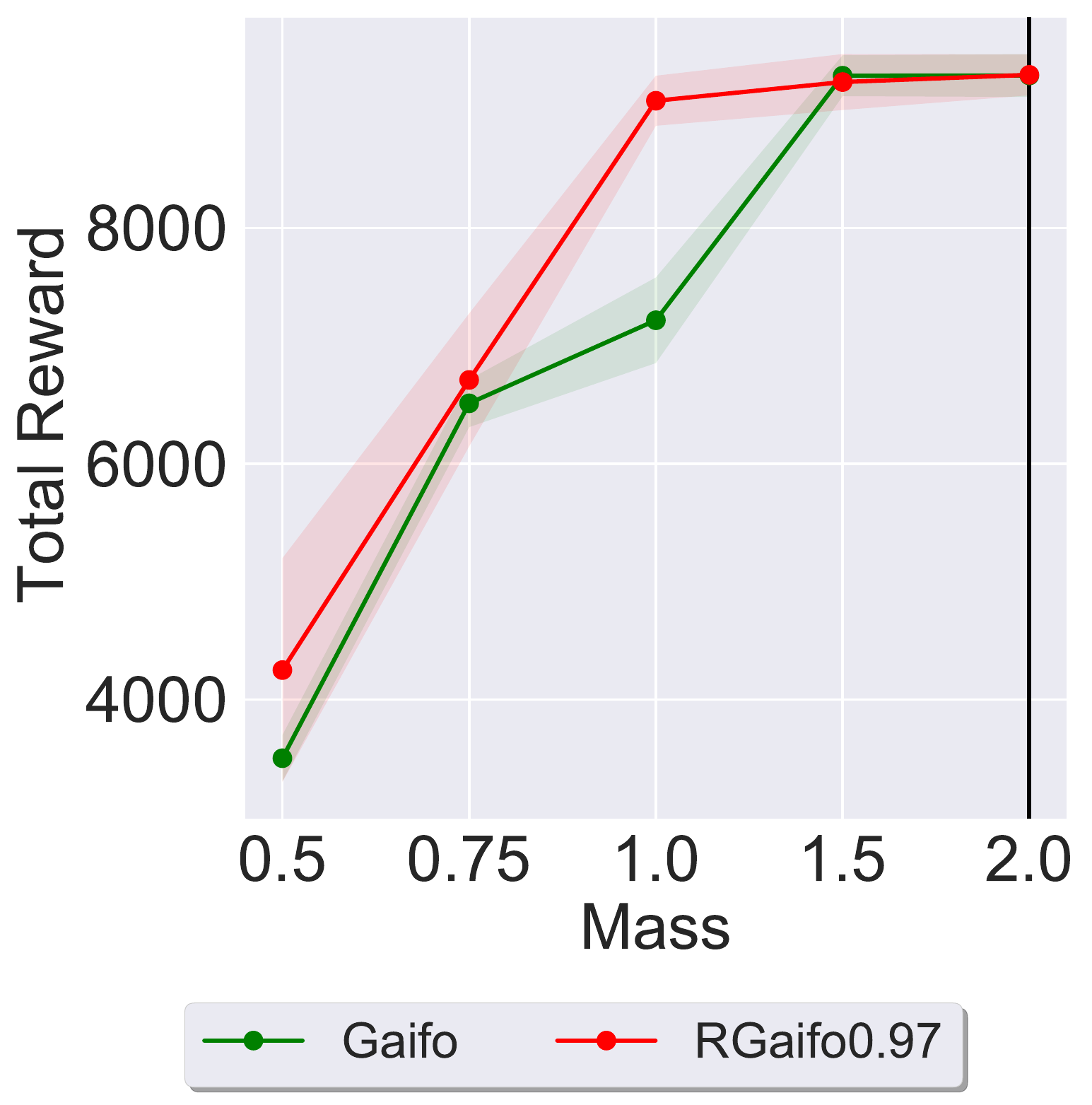}
     } \\
\subfloat[Swimmer]{%
       \includegraphics[width=0.16\linewidth]{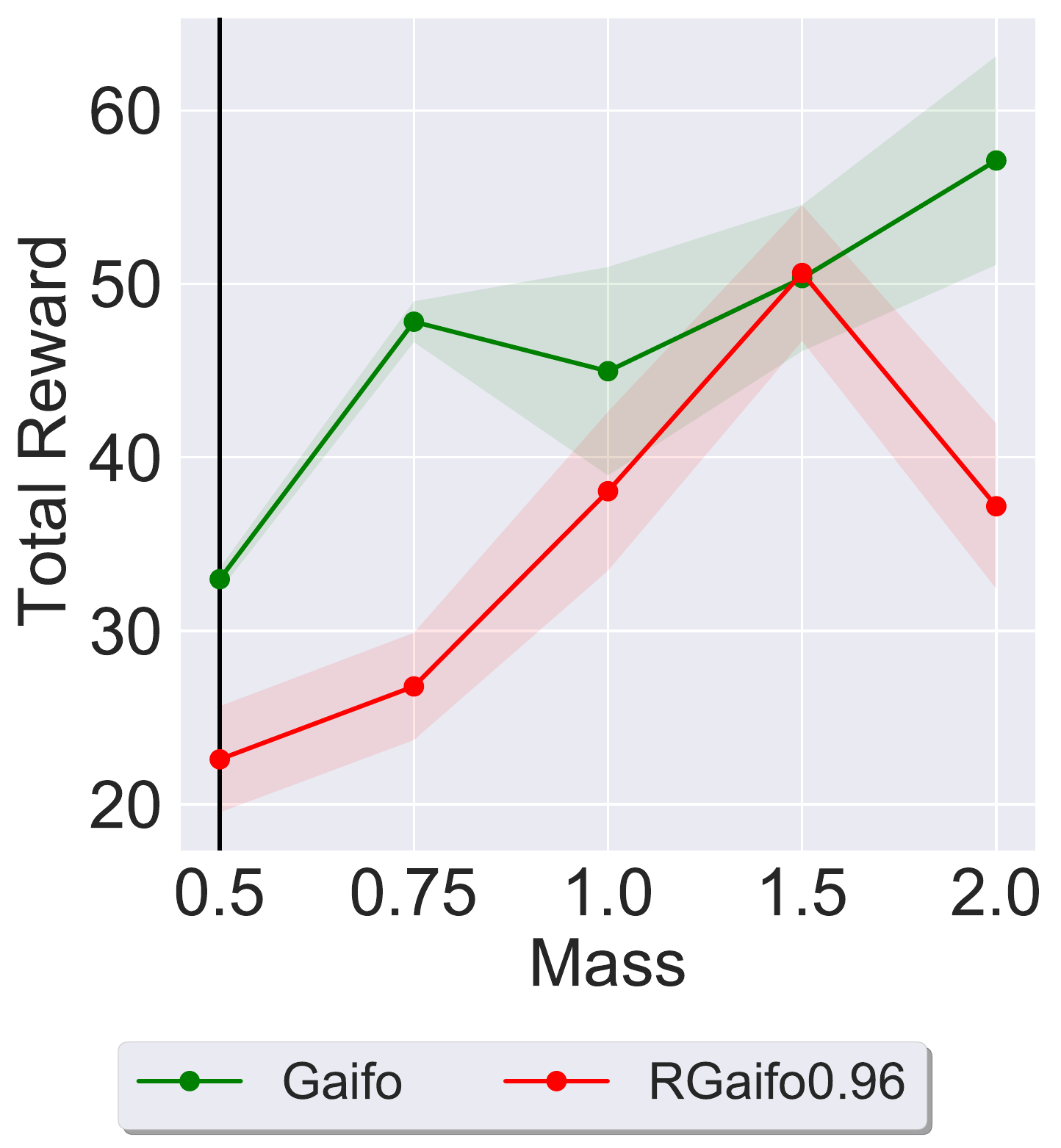}
     } &
\subfloat[Swimmer]{%
       \includegraphics[width=0.16\linewidth]{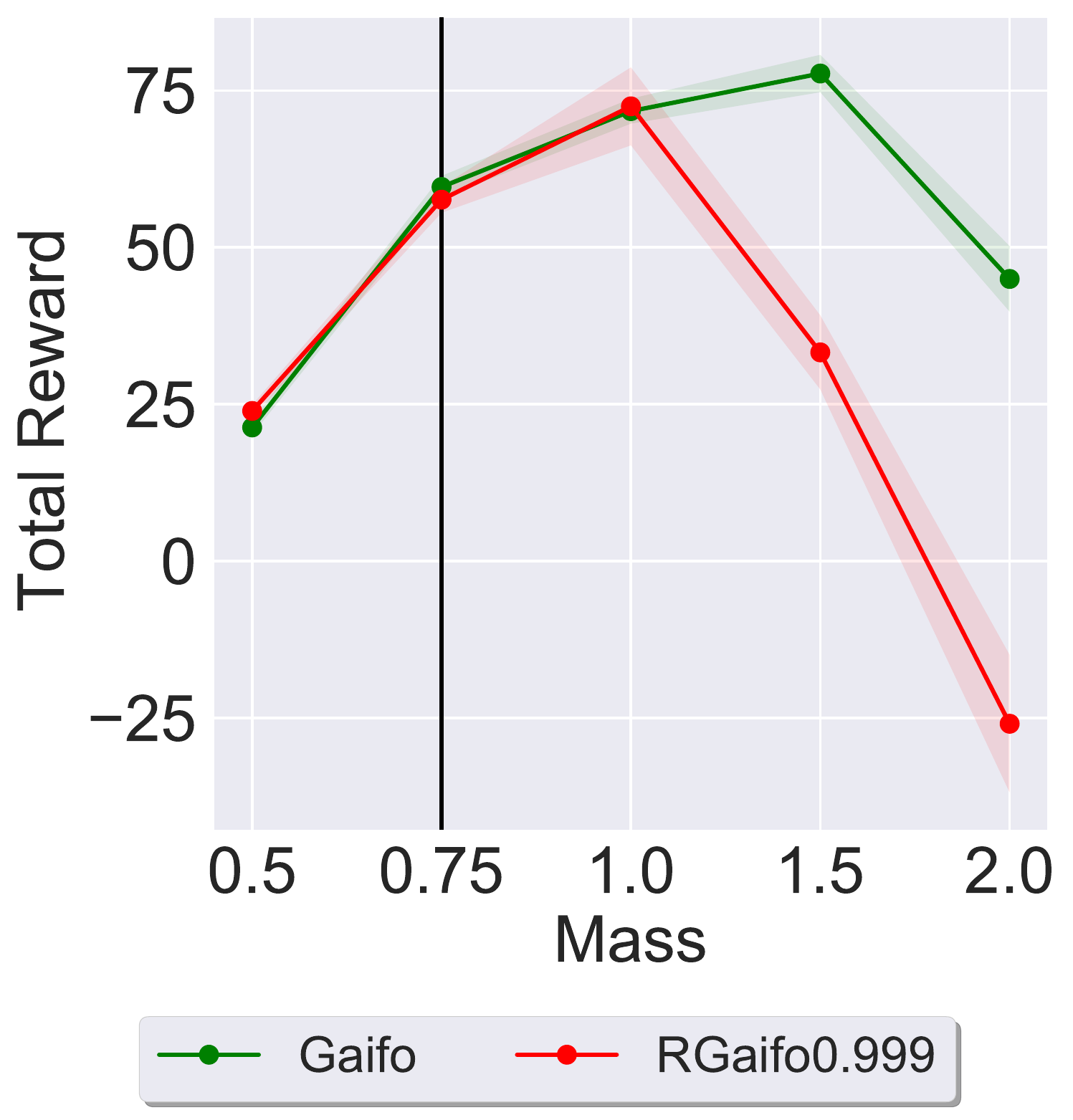}
     } &
\subfloat[Swimmer]{%
       \includegraphics[width=0.16\linewidth]{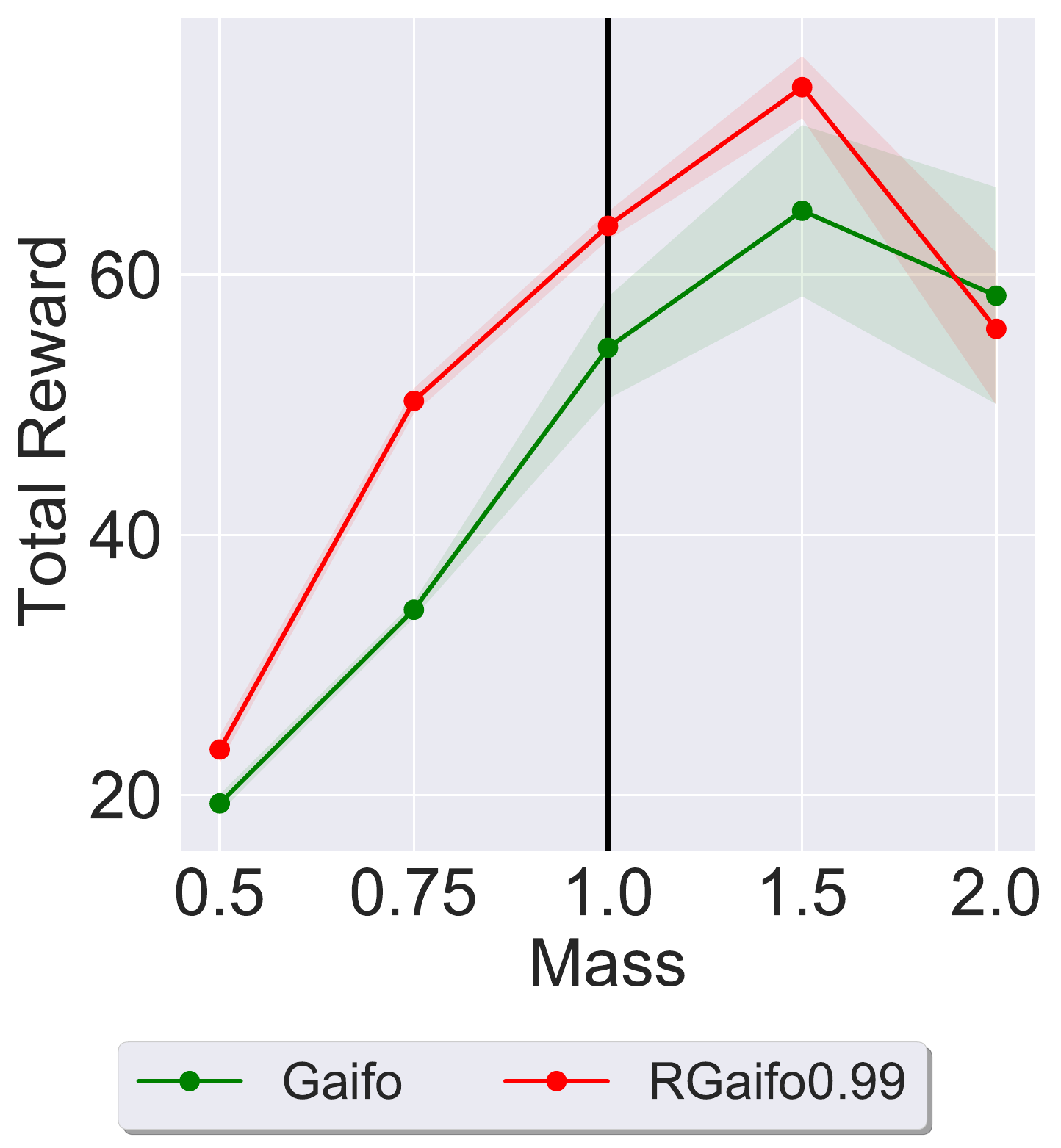}
     } &
\subfloat[Swimmer]{%
       \includegraphics[width=0.16\linewidth]{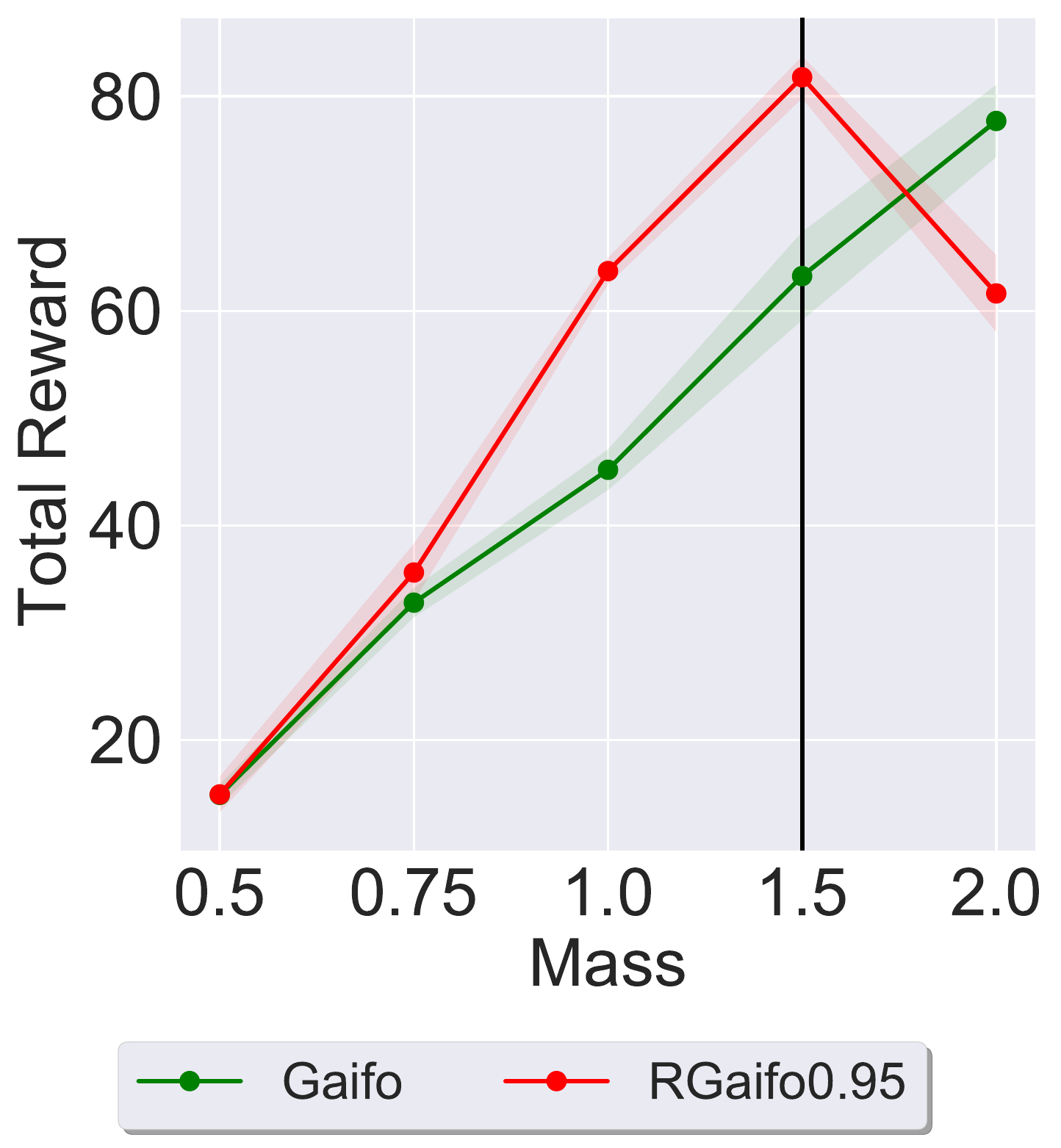}
     } &
\subfloat[Swimmer]{%
       \includegraphics[width=0.16\linewidth]{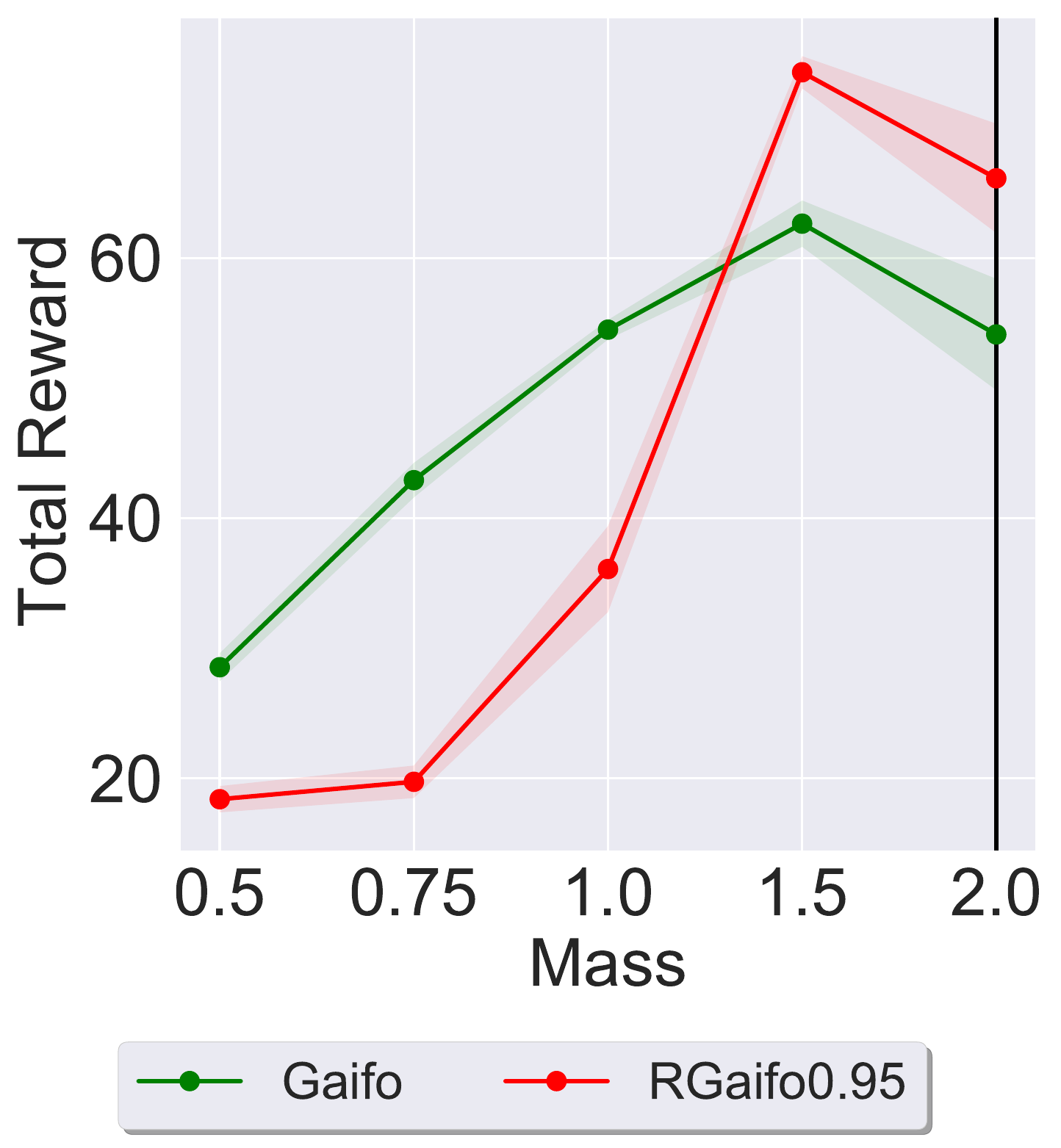}
     } \\

\end{tabular}
\caption{The average (over $3$ seeds) robust performance of Algorithm~\ref{alg:robust-gailfo} with different values of $\alpha$ for each MuJoCo task as reported in the legend of each plot. The expert environment $M^\mathrm{real}$, in which the demonstrations are collected, has relative mass $1.0$. In each plot, the black vertical line corresponds to the relative mass of the learner environment $M^\mathrm{sim}$ where we trained the policy with Algorithm~\ref{alg:robust-gailfo}. The x-axis denotes the relative mass of the test environment $M^\mathrm{test}$ in which the policies are evaluated. The policies are evaluated over $1e5$ steps truncating the last episode if it does not terminate.}
\label{fig:RobustnessMassFixedAlpha}
\end{figure*}

We compare the performance of our robust GAILfO algorithm with different values of $\alpha \in \bc{1.0, 0.999, 0.99, 0.98,0.97,0.96,0.95,0.90}$ against the standard GAILfO algorithm proposed in \citep{torabi2018generative}. To the best of our knowledge, GAILfO is the only large-scale imitation learning method that is applicable under the setting described in Section~\ref{sec:intro} (see Table~\ref{table:related-work}).

\subsection{Continuous Control Tasks on MuJoCo}

In this section, we evaluate the performance of our method on standard continuous control benchmarks available on OpenAI Gym~\citep{brockman2016openai} utilizing the MuJoCo environment~\citep{todorov2012mujoco}. Specifically, we benchmark on five tasks: Half-Cheetah, Walker, Hopper, Swimmer, and Inverted-Double-Pendulum. Details of these environments can be found in~\citep{brockman2016openai} and on the GitHub website.

The default configurations of the MuJoCo environment (provided in OpenAI Gym) is regarded as the real or deployment environment ($M^\mathrm{real}$), and the expert demonstrations are collected there. We do not assume any access to the expert MDP beyond this during the training phase. We construct the simulation or training environments ($M^\mathrm{sim}$) for the imitator by modifying some parameters independently: (i) the mass of the bot in $M^\mathrm{sim}$ is $\bc{0.5, 0.75, 1.0, 1.5, 2.0} \times$ the mass in $M^\mathrm{real}$, and (ii) the friction coefficient on all the joints of the bot in $M^\mathrm{sim}$ is $\bc{0.5, 0.75, 1.0, 1.5, 2.0} \times$ the coefficient in $M^\mathrm{real}$. 

We train an agent on each task by proximal policy optimization (PPO) algorithm~\citep{schulman2017proximal} using the rewards defined in the OpenAI Gym~\citep{brockman2016openai}. We use the resulting stochastic policy as the expert policy $\pi^E$. In all our experiments, 10 state-only expert demonstrations collected by the expert policy $\pi^E$ in the real environment $M^\mathrm{real}$ is given to the learner.

Our Algorithm~\ref{alg:robust-gailfo} implementation is based on the codebase from \url{https://github.com/Khrylx/PyTorch-RL}. We use a two-layer feedforward neural network structure of (128, 128, tanh) for both actors (agent and adversary) and discriminator. The actor or policy networks are trained by the proximal policy optimization (PPO) method. For training the discriminator $D$, we use Adam~\citep{kingma2014adam} with a learning rate of $1e-4$. For each environment-mismatch pair, we identified the best performing $\alpha$ parameter based on the ablation study reported in Appendix~\ref{app:transfer-performance}. The learner is trained in the simulator $M^\mathrm{sim}$ for $\approx$3M time steps. We run our experiments, for each environment, with 3 different seeds. We report the mean and standard error of the performance (cumulative true rewards) over 3 trials. The cumulative reward is normalized with ones earned by $\pi^E$ and a random policy so that 1.0 and 0.0 indicate the performance of $\pi^E$ and the random policy, respectively.

Figures~\ref{fig:TransferFrictionFixedAlpha},~and~\ref{fig:TransferMassFixedAlpha} plot the performance of the policy evaluated on the deployment environment ($M^\mathrm{real}$). The x-axis corresponds to the simulation environment ($M^\mathrm{sim}$) on which the policy is trained on. We observe that our robust GAILfO produces policies that can be successfully transferred to the $M^\mathrm{real}$ environment from $M^\mathrm{sim}$ compared to the standard GAILfO. 


Finally, we evaluate the robustness of the policies trained by our algorithm (with different dynamics mismatch) under different testing conditions. At test time, we evaluate the learned policies by changing the mass and friction values and estimating the cumulative rewards. As shown in Figures~\ref{fig:RobustnessFrictionFixedAlpha}~and~\ref{fig:RobustnessMassFixedAlpha}, our Algorithm~\ref{alg:robust-gailfo} outperforms the baseline in terms of robustness as well.

\subsection{Continuous Gridworld Tasks under Additive Transition Dynamics Mismatch}


\begin{figure}[!h] 
\centering
\includegraphics[width=0.2\textwidth]{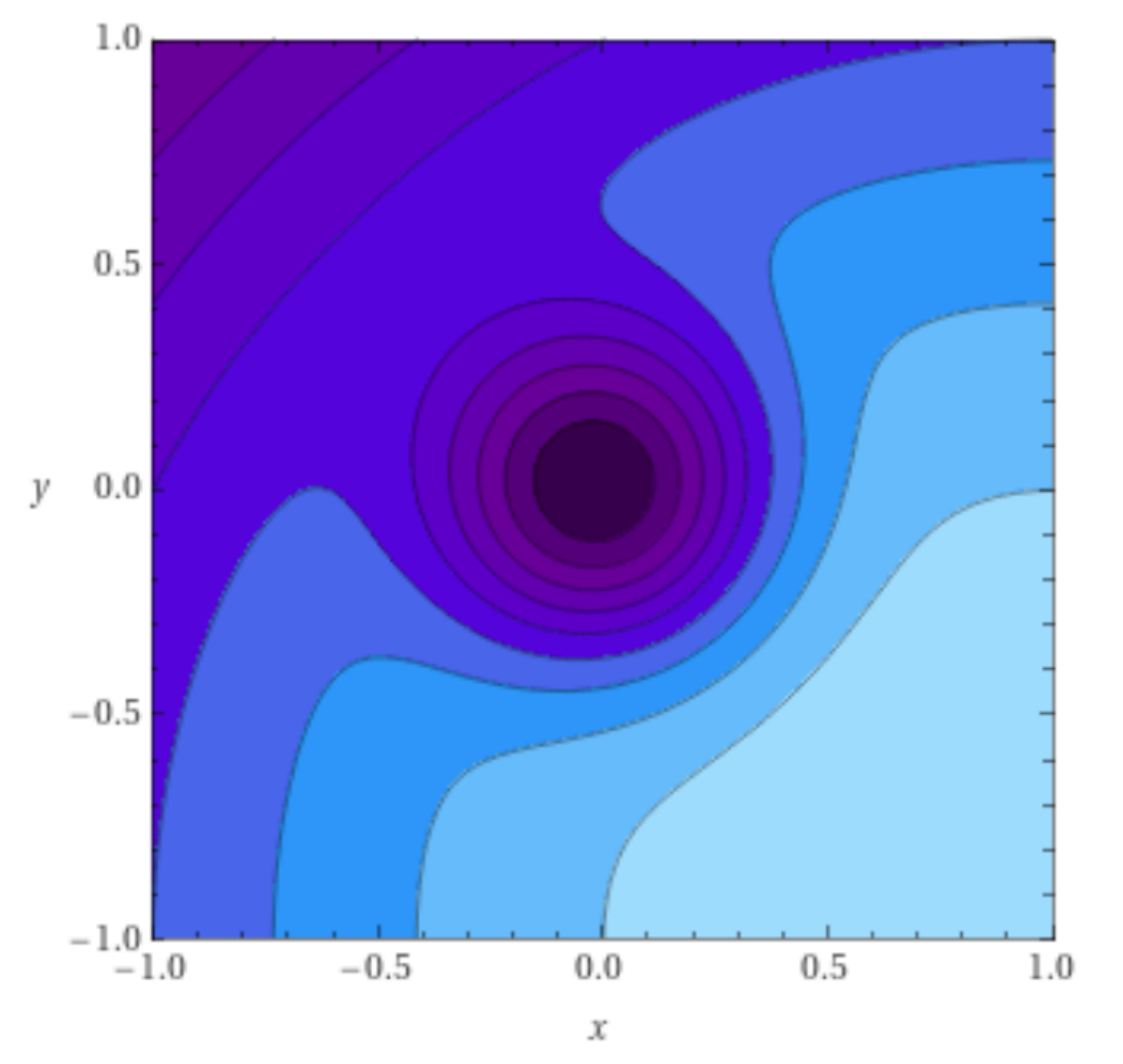}
\caption{The contour curves for the reward function of the 2D gridworld environment.}
\label{fig:env_continuous}
\end{figure}

\looseness-1In this section, we evaluate the effectiveness of our method on a continuous gridworld environment under a transition dynamics mismatch induced by additive noise. Specifically, we consider a 2D environment, where we denote the horizontal coordinate as $x \in [0,1]$ and vertical one as $y \in [0,1]$. The agent starts in the upper left corner, i.e., the coordinate $(0,1)$, and the episode ends when the agent reaches the lower right region defined by the indicator function $\mathbf{1}\{x\in [0.95, 1], y \in [-1, -0.95]\} $. The reward function is given by: $R(x,y) = -(x-1)^2 -(y+1)^2 -80 e^{-8(x^2 + y^2)} + 10 \cdot \mathbf{1}\{x\in [0.95, 1], y \in [-1, -0.95]\}$. Figure~\ref{fig:env_continuous} provides a graphical representation of the reward function. Note that the central region of the 2D environment represents a low reward area that should be avoided. The action space for the agent is given by $\mathcal{A} = [-0.5, 0.5]^2$, and the transition dynamics are given by: $s_{t+1} = s_t + \frac{a_t}{10}$ with probability (w.p.) $1 - \epsilon$, and $s_{t+1} = s_t - \frac{s_t}{10 \norm{s_t}_2}$ w.p. $\epsilon$. Thus, with probability $\epsilon$, the environment does not respond to the action taken by the agent, but it takes a step towards the low reward area centered at the origin, i.e., $- \frac{s_t}{10 \norm{s_t}_2}$. The agent should therefore pass far enough from the origin. The parameter $\epsilon$ can be varied to create a dynamic mismatch, e.g., higher $\epsilon$ corresponds to a more difficult environment.

We use three experts trained with $\epsilon = 0.0, \epsilon=0.05, \text{ and } \epsilon = 0.1$. The learners act in a different environment with the following values for $\epsilon$: $0.0, 0.05, 0.1, 0.15, 0.2$. Figure~\ref{fig:best_alpha_continuous} plots the performance of the trained learner policy evaluated on the expert environment. The x-axis corresponds to the learner environment on which the learner policy is trained. In general, we observe a behavior comparable to the MuJoCo experiments. We can often find an appropriate value for $\alpha$ such that Robust GAILfO learns to imitate under mismatch largely better than standard GAILfO. 

\begin{figure}[!h] 
\centering
\begin{tabular}{ccc}
\subfloat[Expert $\epsilon = 0.0$]{%
       \includegraphics[width=0.3\linewidth]{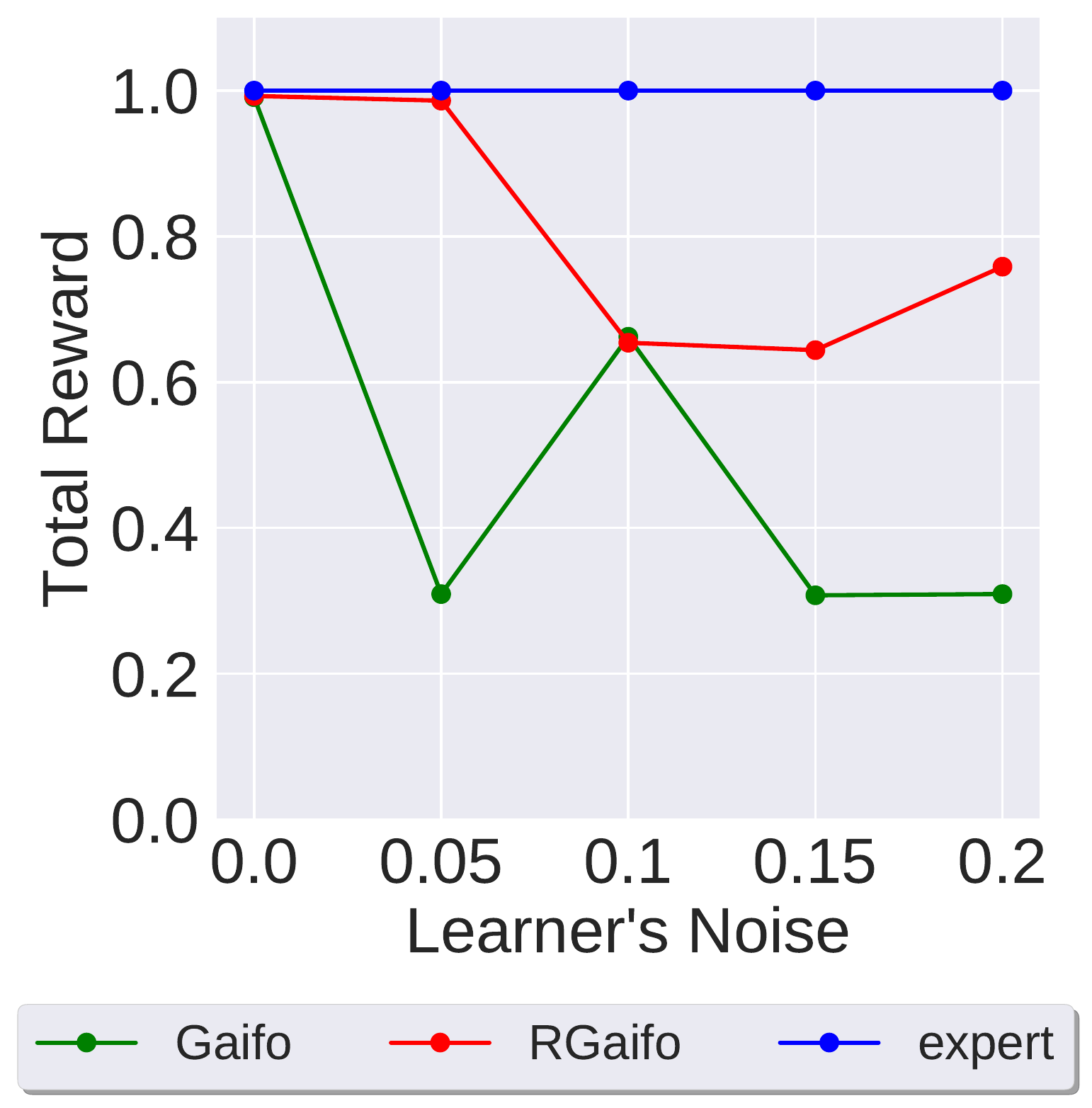}
     } &
\subfloat[Expert $\epsilon = 0.05$]{%
       \includegraphics[width=0.3\linewidth]{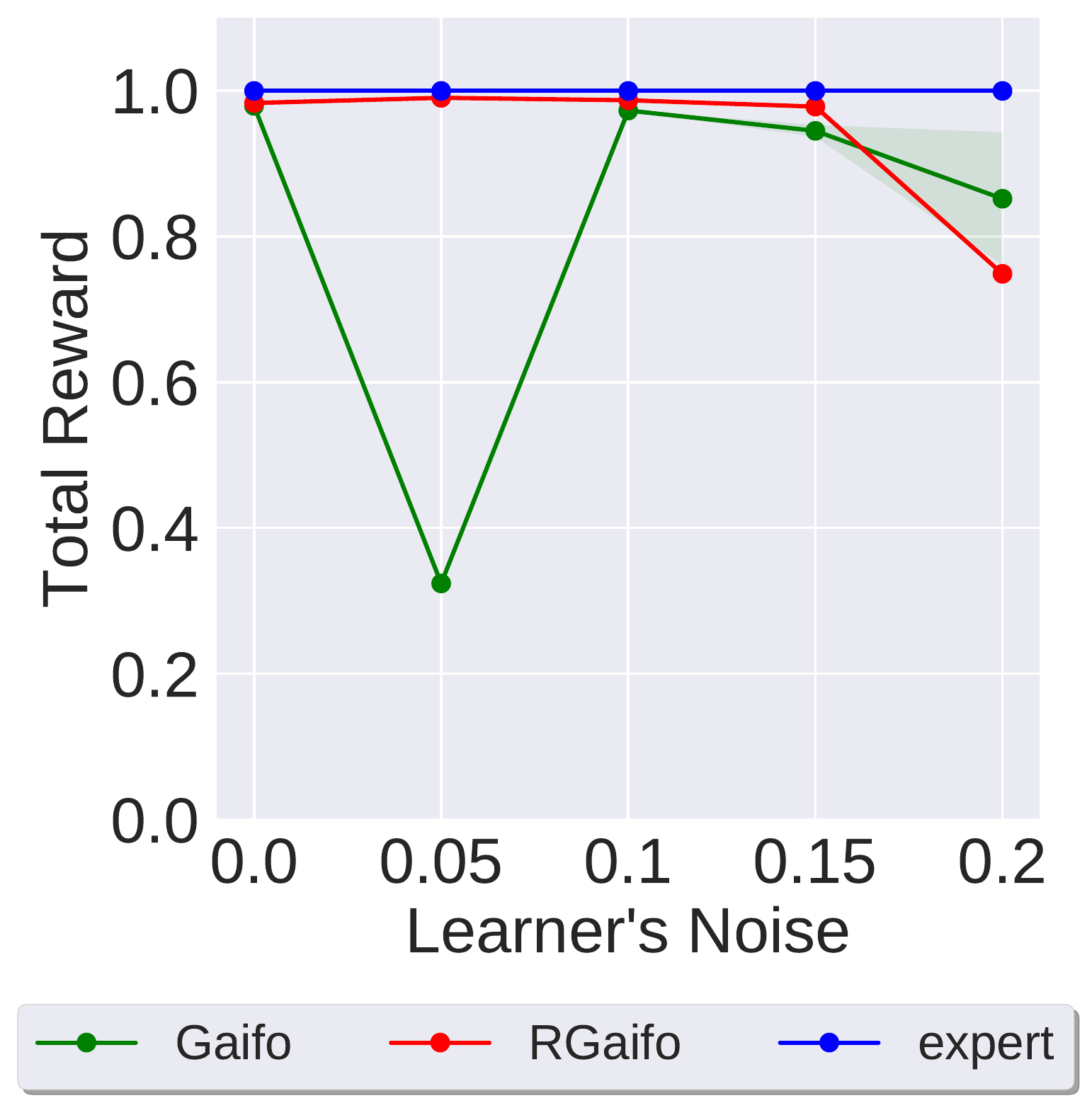}
     } &
\subfloat[Expert $\epsilon = 0.1$]{%
       \includegraphics[width=0.3\linewidth]{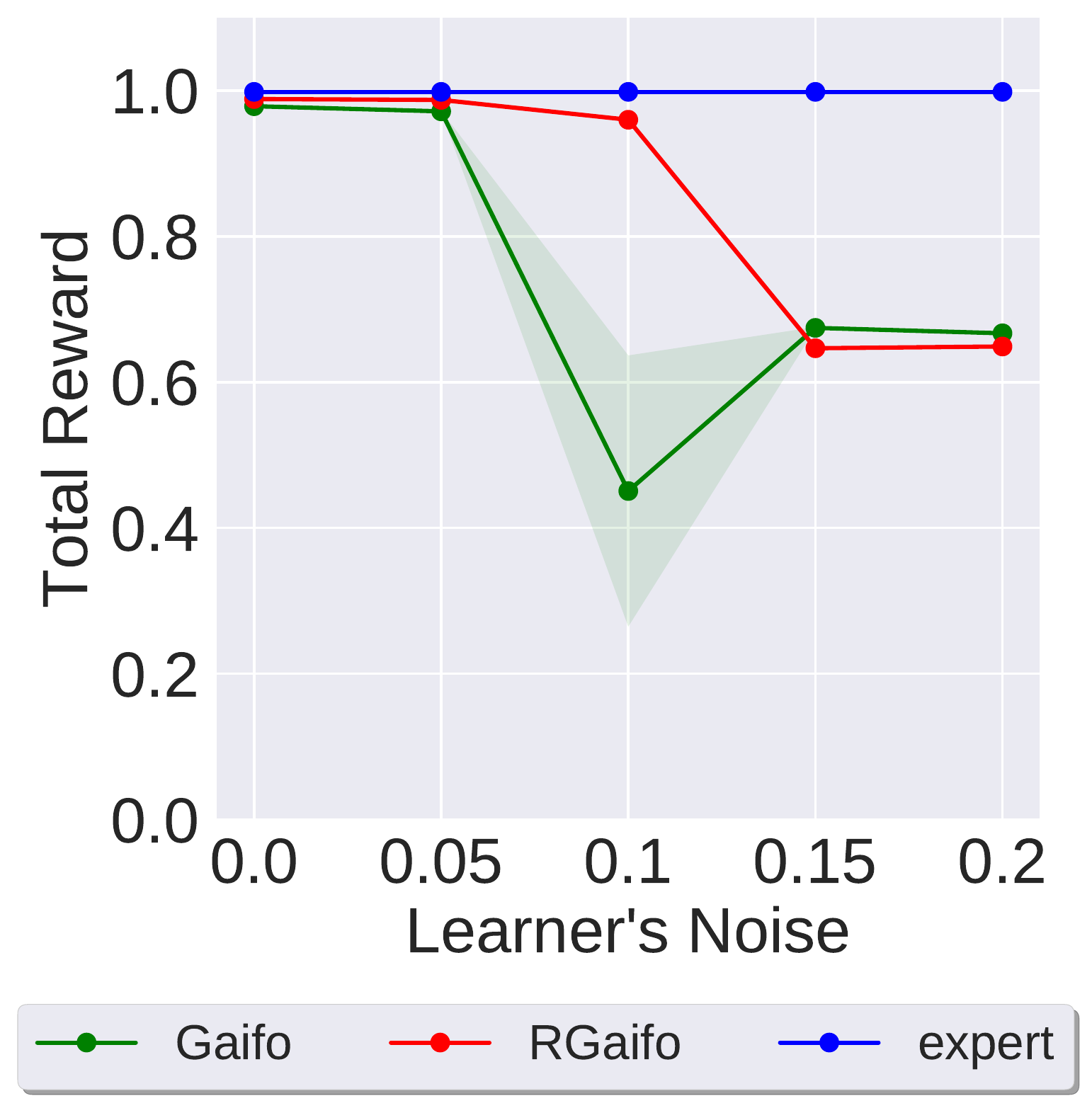}
     } \\
\end{tabular}
\caption{Average performance (over $3$ seeds) of Algorithm~\ref{alg:robust-gailfo} with different values of $\alpha$ for each mismatch (i.e., each point on the x-axis) in the environment shown in Figure~\ref{fig:env_continuous}. The $\alpha$ values are chosen based on the ablation study in Figure~\ref{fig:ablation_continuous} (see Appendix~\ref{app:transfer-performance-grid}). The x-axis denotes the $\epsilon$ value of the learner environment. The policies are evaluated over $1e5$ steps truncating the last episode if it does not terminate. In Appendix~\ref{app:new_seeds}, we verify that our strategy of choosing appropriate $\alpha$ value does not introduce maximization bias. \label{fig:best_alpha_continuous}}
\end{figure}

\subsection{Choice of $\alpha$}

We note that one has to carefully choose the value of $\alpha$ to avoid too conservative behavior (see Figure~\ref{fig:TransferFrictionAblation} in Appendix~\ref{app:transfer-performance}). In principle, given a rough estimate $\widehat T^E$ of the expert dynamics $T^E$, one could choose this value based on Eq.~\eqref{learner_unc_set}. However, the choice of suitable $\alpha$ value is also affected by the other design choices of the algorithm, e.g., how many iterations the player and adversary are updated in the inner loop, and function approximators used.

In order to estimate the accuracy of the simulator, we can execute a safe baseline policy in both the simulator and the real environment, collect trajectories or datasets, and compute an estimate of the transition-dynamics distance between them. We can also utilize the performance difference lemma from~\cite{even2003approximate} to obtain a lower bound on the transition dynamics mismatch based on the value function difference in the two environments.

Apart from the final evaluation, we also minimally access (in our experiments) the deployment environment for choosing the appropriate value for $\alpha$. Compared to training a policy in the deployment environment from scratch, accessing the deployment environment to choose $\alpha$ is sample-efficient. We only need to evaluate the final policies (trained in the simulation environment) once for each value of $\alpha$. When we already have a reasonable estimate of $\alpha$, we can also reduce these evaluations.  


\section{Conclusions}
\label{sec:conclusions}

In this work, we propose a robust LfO method to solve an offline imitation-learning problem, in which a few state-only expert demonstrations and a simulator with misspecified dynamics are given to the learner. Even though our Algorithm~\ref{alg:robust-gailfo} is not essentially different from the standard robust RL methods, the robust optimization problem formulation to derive our algorithm is important and novel in the IL context. Experiment results in continuous control tasks on MuJoCo show that our method clearly outperforms the standard GAILfO in terms of the transfer performance (with model misspecification) in the real environment, as well as the robust performance under varying testing conditions. 

Our algorithm falls under the category of zero-shot sim-to-real transfer~\citep{zhao2020sim} with expert demonstrations, making our method well suited for robotics applications. In principle, one can easily incorporate the two-player Markov game idea into any imitation learning algorithm and derive its robust version. This work can be considered a direction towards improving the sample efficiency of IL algorithms in terms of the number of environment interactions through robust training on a misspecified simulator.




\begin{acks}
Luca Viano has received financial support from the Enterprise for Society Center (E4S). Parameswaran Kamalaruban acknowledges support from The Alan Turing Institute. Craig Innes and Subramanian Ramamoorthy are supported by a grant from the UKRI Strategic Priorities Fund to the UKRI Research Node on Trustworthy Autonomous Systems Governance and Regulation (EP/V026607/1, 2020-2024). Adrian Weller acknowledges support from a Turing AI Fellowship under grant EP/V025379/1, EPSRC grant EP/V056522/1, The Alan Turing Institute, and the Leverhulme Trust via CFI.
\end{acks}



\clearpage

\bibliographystyle{ACM-Reference-Format} 
\balance
\bibliography{robust_lfo_2021}

\newpage
\appendix


\clearpage
\onecolumn

\section*{Code Repository}

\url{https://github.com/lviano/robust_gaifo}

\section{Details on the equivalence between Action Robust MDP and Robust MDP}

In the following we prove the last equality of Eq.~\eqref{equivalence_new}.
\begin{theorem}
Given the set 
\begin{equation*}
\mathcal{T}^{\alpha} ~:=~ \bc{ T : T(s^\prime|s,a) =  \alpha T^{\mathrm{sim}}(s^\prime|s,a) + (1 - \alpha)\bar{T}(s^\prime|s) , \, \bar{T}(s^\prime|s) = \sum_a \pi(a|s) T^{\mathrm{sim}}(s^\prime|s,a), \quad \forall \pi \in \Pi}
\end{equation*}
and a cost function depending only on states, i.e. $r: \mathcal{S}\times\mathcal{S}\rightarrow \mathbb{R}$, define $G_c = \sum^{\infty}_{t=0} \gamma^t r(s_t, s_{t+1})$.
Then, the following holds:
\begin{equation*}
\min_{\pi^{\mathrm{pl}} \in \Pi} \max_{T^\alpha \in \mathcal{T}^{\alpha}}  \E{G_c \bigm| \pi^{\mathrm{pl}}, P_0, T^\alpha} = \min_{\pi^{\mathrm{pl}} \in \Pi} \max_{\pi^{\mathrm{op}} \in \Pi} \E{G_c \bigm| \alpha \pi^{\mathrm{pl}} + (1 - \alpha) \pi^{\mathrm{op}}, M^{\mathrm{sim}}} .
\end{equation*}
In particular, the result in Eq.~\eqref{equivalence_new} follows from the choice: $ r(s_t, s_{t+1}) = c(s_t,s_{t+1}) + H^{\pi^{\mathrm{pl}}}(A|S=s_t)$.
\end{theorem}

\begin{proof}
Let us define $P^{\pi, T}(s_0, \dots, s_N) := P_0(s_0) \prod^{N-1}_{t=0}\sum_{a}\pi(a|s_t) T(s_{t+1}|s_t, a)$.
We need to show equality between the distributions $P^{\pi^{\mathrm{pl}}, \alpha T^{\mathrm{sim}} + (1 - \alpha) \bar{T}}$ and $P^{\alpha \pi^{\mathrm{pl}} + (1 - \alpha)\pi^{\mathrm{op}}, T^{\mathrm{sim}}}$.
Due to the Markov property, this is equivalent to show:
\begin{equation}
    \sum_{a}\pi^{\mathrm{pl}}(a|s_t) \bs{\alpha T^{\mathrm{sim}}(s_{t+1}|s_t, a) + (1 - \alpha) \bar{T}(s_{t+1}| s_t)} = \sum_{a} \bs{\alpha \pi^{\mathrm{pl}}(a|s_t) + (1 - \alpha)\pi^{\mathrm{op}}(a|s_t)}T^{\mathrm{sim}}(s_{t+1}|s_t, a) ,
\end{equation}
that implies:
\begin{equation*}
    \underbrace{\sum_{a}\pi^{\mathrm{pl}}(a|s_t)}_{=1} \bar{T}(s_{t+1}| s_t) = \sum_{a}\pi^{\mathrm{op}}(a|s_t)T^{\mathrm{sim}}(s_{t+1}|s_t, a) .
\end{equation*}
Hence, it follows that equality between $P^{\pi^{\mathrm{pl}}, \alpha T^{\mathrm{sim}} + (1 - \alpha) \bar{T}}$ and $P^{\alpha \pi^{\mathrm{pl}} + (1 - \alpha)\pi^{\mathrm{op}}, T^{\mathrm{sim}}}$ holds for:
\begin{equation*} \bar{T}(s_{t+1}| s_t) = \sum_{a}\pi^{\mathrm{op}}(a|s_t)T^{\mathrm{sim}}(s_{t+1}|s_t, a) ,
\end{equation*}
as we used in the definition of the set $\mathcal{T}^{\alpha}$.
\end{proof}

\section{Additional Details on Algorithm~\ref{alg:robust-gailfo}}
\label{app:robust-lfo-details}

By interpreting $R_w\br{s,s'} = - \log{D_w\br{s,s'}}$ as the reward function, we have (for a fixed $w$):
\begin{align*}
J \br{\theta,\phi} ~:=~& \Eee{\rho^{\pi^\mathrm{mix}_{\theta,\phi}}_{M^{\mathrm{sim}}}}{{R_w\br{s,s'}}} + \lambda H_{\rho^{\pi^\mathrm{mix}_{\theta,\phi}}_{M^{\mathrm{sim}}}}\br{\pi^\mathrm{pl}_\theta} ~=~ J_1 \br{\theta,\phi} + J_2 \br{\theta,\phi} ,
\end{align*}
where
\begin{align*}
J_1 \br{\theta,\phi} ~:=~& \E{\sum_{t} \gamma^t R_w(s_t, s_{t+1})\big\vert \pi_{\theta, \phi}^\mathrm{mix}, M^{\mathrm{sim}}} \\
J_2 \br{\theta,\phi} ~:=~& \lambda \E{\sum_{t} \gamma^t H^{\pi_{\theta}^{pl}}(A|S=s_t) \big\vert \pi_{\theta, \phi}^\mathrm{mix}, M^{\mathrm{sim}}} .
\end{align*}
By the policy gradient theorem, the derivatives of the first term w.r.t the player and the opponent policy parameters are given by:
\begin{align*}
\nabla_{\theta} J_1(\theta, \phi) ~=~& \sum_{s \in \mathcal{S}}\sum_{t}\gamma^t \P{S_t = s \mid \pi_{\theta, \phi}^\mathrm{mix}, M^{\mathrm{sim}}} \sum_{a} \nabla_{\theta} \pi_{\theta, \phi}^\mathrm{mix}(a|s) Q_{\pi_{\theta, \phi}^\mathrm{mix}}(s,a) \\
\nabla_{\phi} J_1(\theta, \phi) ~=~& \sum_{s \in \mathcal{S}}\sum_{t}\gamma^t \P{S_t = s \mid \pi_{\theta, \phi}^\mathrm{mix}, M^{\mathrm{sim}}} \sum_{a} \nabla_{\phi} \pi_{\theta, \phi}^\mathrm{mix}(a|s) Q_{\pi_{\theta, \phi}^\mathrm{mix}}(s,a) ,
\end{align*}
where 
\begin{align*}
Q_{\pi_{\theta, \phi}^\mathrm{mix}}(s,a) ~=~& \sum_{s^\prime} T^{\mathrm{sim}}(s^\prime \mid a, s) \br{R_w(s,s^\prime) + \gamma V_{\pi_{\theta, \phi}^\mathrm{mix}}(s^\prime)} \\
V_{\pi_{\theta, \phi}^\mathrm{mix}}(s) ~=~& \E{\sum_{t} \gamma^t R_w(s_t, s_{t+1}) \bigg\vert \pi_{\theta, \phi}^\mathrm{mix}, M^{\mathrm{sim}}, s_0 = s} .
\end{align*}
For the second term, we introduce the following quantities: 
\begin{align*}
Q_{\pi_{\theta, \phi}^\mathrm{mix}}^{\mathrm{log}}(s,a) ~=~& \sum_{s^\prime} T^{\mathrm{sim}}(s^\prime \mid s,a) \br{\lambda H^{\pi_{\theta}^{pl}}(A|S=s_t) + \gamma V_{\pi_{\theta, \phi}^\mathrm{mix}}^{\mathrm{log}}(s')} \\
V_{\pi_{\theta, \phi}^\mathrm{mix}}^{\mathrm{log}}(s) ~=~& \E{\sum_{t} \lambda \gamma^t H^{\pi_{\theta}^{pl}}(A|S=s_t)\bigg \vert \pi_{\theta, \phi}^\mathrm{mix}, M^{\mathrm{sim}}, s_0 = s}
\end{align*}
Then, we obtain the following derivatives of the second term:
\begin{align*}
\nabla_{\theta} J_2(\theta, \phi) ~=~& \sum_{s \in \mathcal{S}}\sum_{t}\gamma^t \P{S_t = s \mid \pi_{\theta, \phi}^\mathrm{mix}, M^{\mathrm{sim}}} \sum_{a} \nabla_{\theta} \pi_{\theta, \phi}^\mathrm{mix}(a|s) Q^{\mathrm{log}}_{\pi_{\theta, \phi}^\mathrm{mix}}(s,a) \\
\nabla_{\phi} J_2(\theta, \phi) ~=~& \sum_{s \in \mathcal{S}}\sum_{t}\gamma^t \P{S_t = s \mid \pi_{\theta, \phi}^\mathrm{mix}, M^{\mathrm{sim}}} \sum_{a} \nabla_{\phi} \pi_{\theta, \phi}^\mathrm{mix}(a|s) Q^{\mathrm{log}}_{\pi_{\theta, \phi}^\mathrm{mix}}(s,a) . 
\end{align*}
For a practical algorithm, we need to compute gradient estimates from a data-set of sampled trajectories $\mathcal{D} = \bc{\tau_i}_i$ with $\tau_i = (s^i_0, a^i_0, \dots, s^i_T, a^i_T)$. The gradient estimates are given by:
\begin{align*}
\widehat{\nabla}_{\theta} J_1(\theta, \phi) ~=~& \sum_{\tau_i \in \mathcal{D}} \sum_{t}\gamma^t \nabla_{\theta} \log \pi_{\theta, \phi}^\mathrm{mix}(a^i_t|s^i_t) \widehat{Q}_{\pi_{\theta, \phi}^\mathrm{mix}}(s^i_t,a^i_t) \\
\widehat{\nabla}_{\phi} J_1(\theta, \phi) ~=~& \sum_{\tau_i \in \mathcal{D}} \sum_{t}\gamma^t \nabla_{\phi} \log \pi_{\theta, \phi}^\mathrm{mix}(a^i_t|s^i_t) \widehat{Q}_{\pi_{\theta, \phi}^\mathrm{mix}}(s^i_t,a^i_t) \\
\widehat{\nabla}_{\theta} J_2(\theta, \phi) ~=~& \sum_{\tau_i \in \mathcal{D}} \sum_{t}\gamma^t \nabla_{\theta} \log \pi_{\theta, \phi}^\mathrm{mix}(a^i_t|s^i_t) \widehat{Q}^{\mathrm{log}}_{\pi_{\theta, \phi}^\mathrm{mix}}(s^i_t,a^i_t) \\
\widehat{\nabla}_{\phi} J_2(\theta, \phi) ~=~& \sum_{\tau_i \in \mathcal{D}} \sum_{t}\gamma^t \nabla_{\phi} \log \pi_{\theta, \phi}^\mathrm{mix}(a^i_t|s^i_t) \widehat{Q}^{\mathrm{log}}_{\pi_{\theta, \phi}^\mathrm{mix}}(s^i_t,a^i_t) ,
\end{align*}
where the estimator $\widehat{Q}_{\pi_{\theta, \phi}^\mathrm{mix}}(s^i_t, a^i_t)$ is the future return observed for the trajectory $i$ after time $t$, i.e., $\widehat{Q}_{\pi_{\theta, \phi}^\mathrm{mix}}(s^i_t, a^i_t) = \sum^{T}_{k=t+1} \gamma^{k - t - 1}R_w(s^i_k, s^i_{k+1})=G^i_t$. Similarly, for the entropy term we have $\widehat{Q}^{\mathrm{log}}_{\pi_{\theta, \phi}^\mathrm{mix}}(s^i_t, a^i_t) = \sum^{T}_{k=t+1} - \gamma^{k - t - 1} H^{\pi^{\mathrm{pl}}_{\theta}}(A|S=s^i_k) =G^{\mathrm{log},i}_t$. The trajectory sampling process is given in Algorithm~\ref{alg:collect_trajs}. 

\clearpage
 
\section{Transfer Performance: MuJoCo}
\label{app:transfer-performance}

We present the following results:
\begin{itemize}
\item The ablation study on the transfer performance of Algorithm~\ref{alg:robust-gailfo} with different values of $\alpha$ under the relative friction mismatches (see Figure~\ref{fig:TransferFrictionAblation}).
\item The ablation study on the transfer performance of Algorithm~\ref{alg:robust-gailfo} with different values of $\alpha$ under the relative mass mismatches (see Figure~\ref{fig:TransferMassAblation}).
\item The transfer performance of Algorithm~\ref{alg:robust-gailfo} with different (best) values of $\alpha$ for each relative friction mismatch of a task (see Figure~\ref{fig:TransferFrictionVarAlpha}). \item The transfer performance of Algorithm~\ref{alg:robust-gailfo} with different (best) values of $\alpha$ for each relative mass mismatch of a task (see Figure~\ref{fig:TransferMassVarAlpha}).
\end{itemize}

\begin{figure}[!h] 
\centering
\begin{tabular}{ccc}
\subfloat[HalfCheetah]{%
       \includegraphics[width=0.25\linewidth]{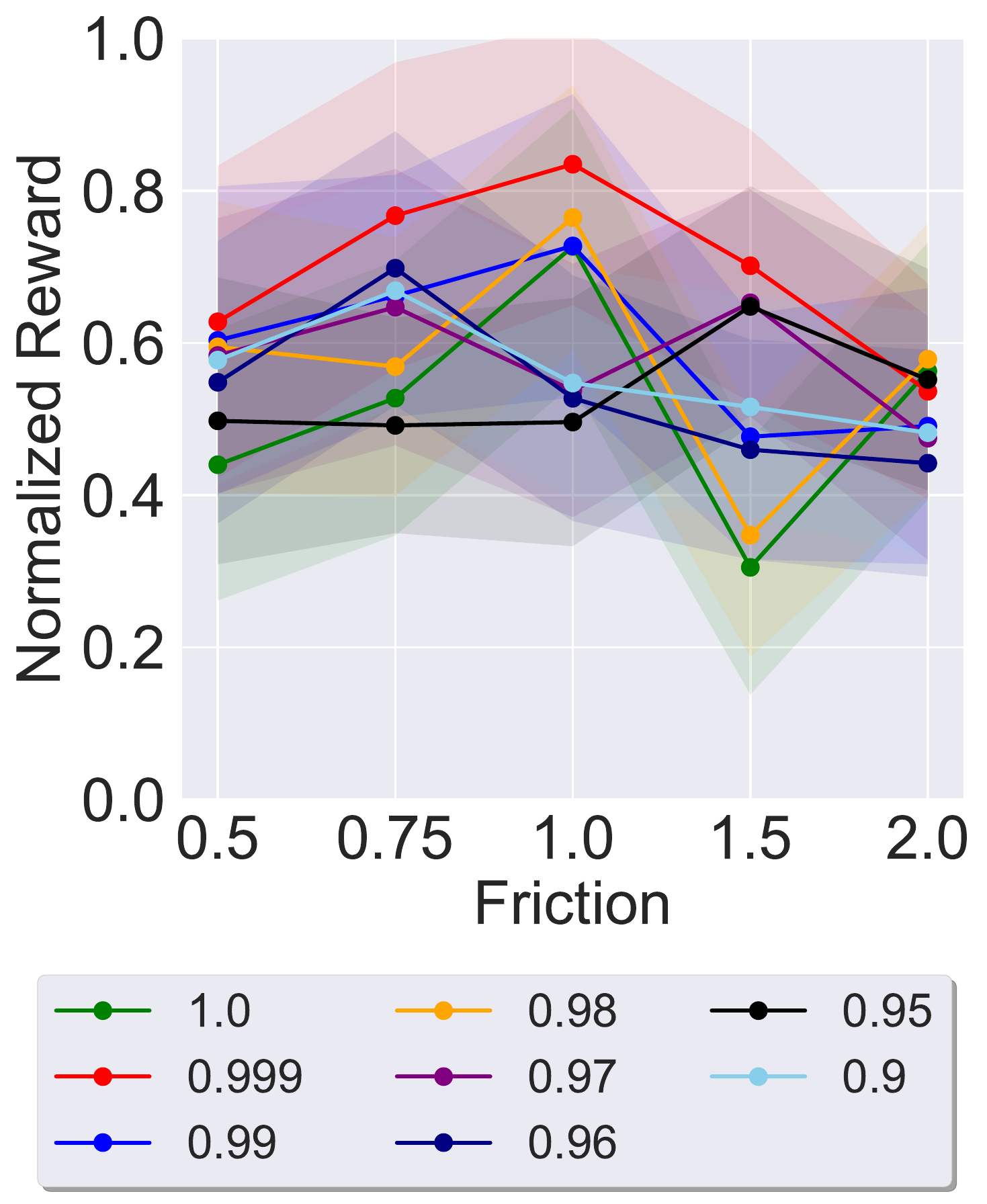}
     } &
\subfloat[Walker]{%
       \includegraphics[width=0.25\linewidth]{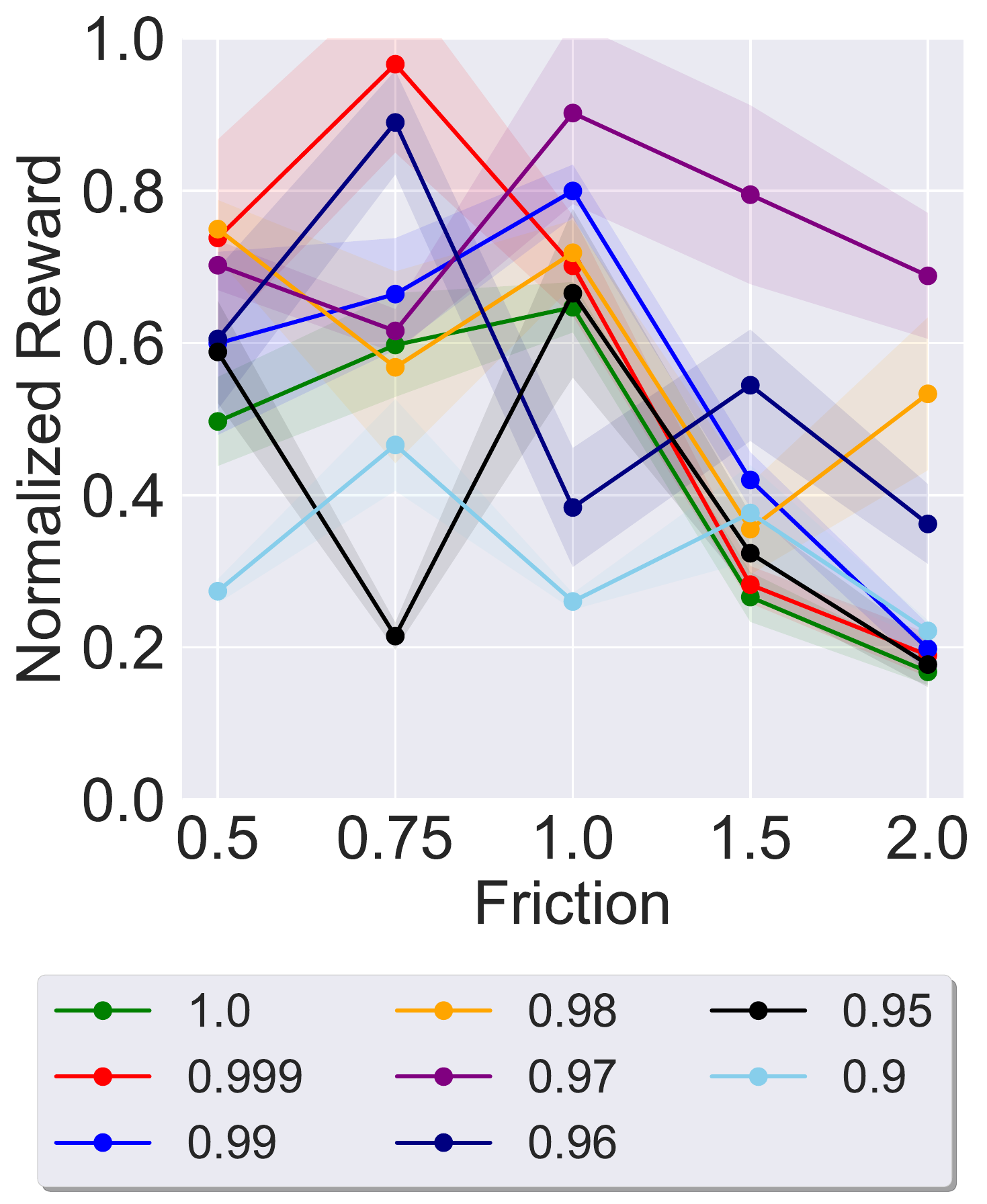}
     } &
\subfloat[Hopper]{%
       \includegraphics[width=0.25\linewidth]{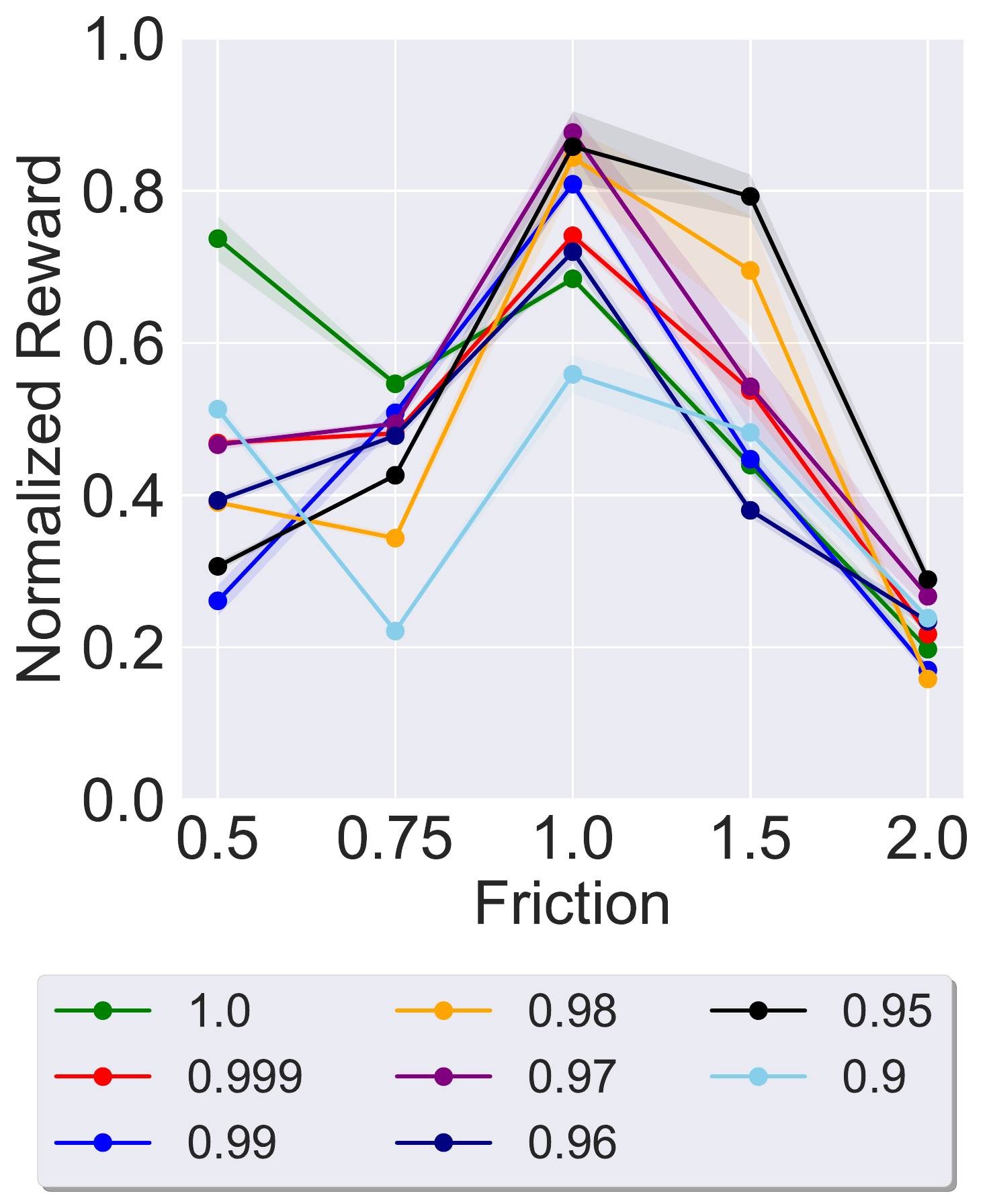}
     } \\
\end{tabular}
\caption{The average (over $3$ seeds) transfer performance of Algorithm~\ref{alg:robust-gailfo} with different values of $\alpha$. The ablation shown here is used to choose $\alpha$ in Figure~\ref{fig:TransferFrictionFixedAlpha}. The x-axis denotes the relative friction of the learner environment $M^\mathrm{sim}$. The policies are evaluated in $M^\mathrm{real}_{c^*}$ over $1e5$ steps truncating the last episode if it does not terminate. Note that robust-GAILfO with $\alpha = 1$ corresponds to GAILfO.}
\label{fig:TransferFrictionAblation}
\end{figure}

\begin{figure}[!h] 
\centering
\begin{tabular}{ccc}
\subfloat[HalfCheetah]{%
       \includegraphics[width=0.25\linewidth]{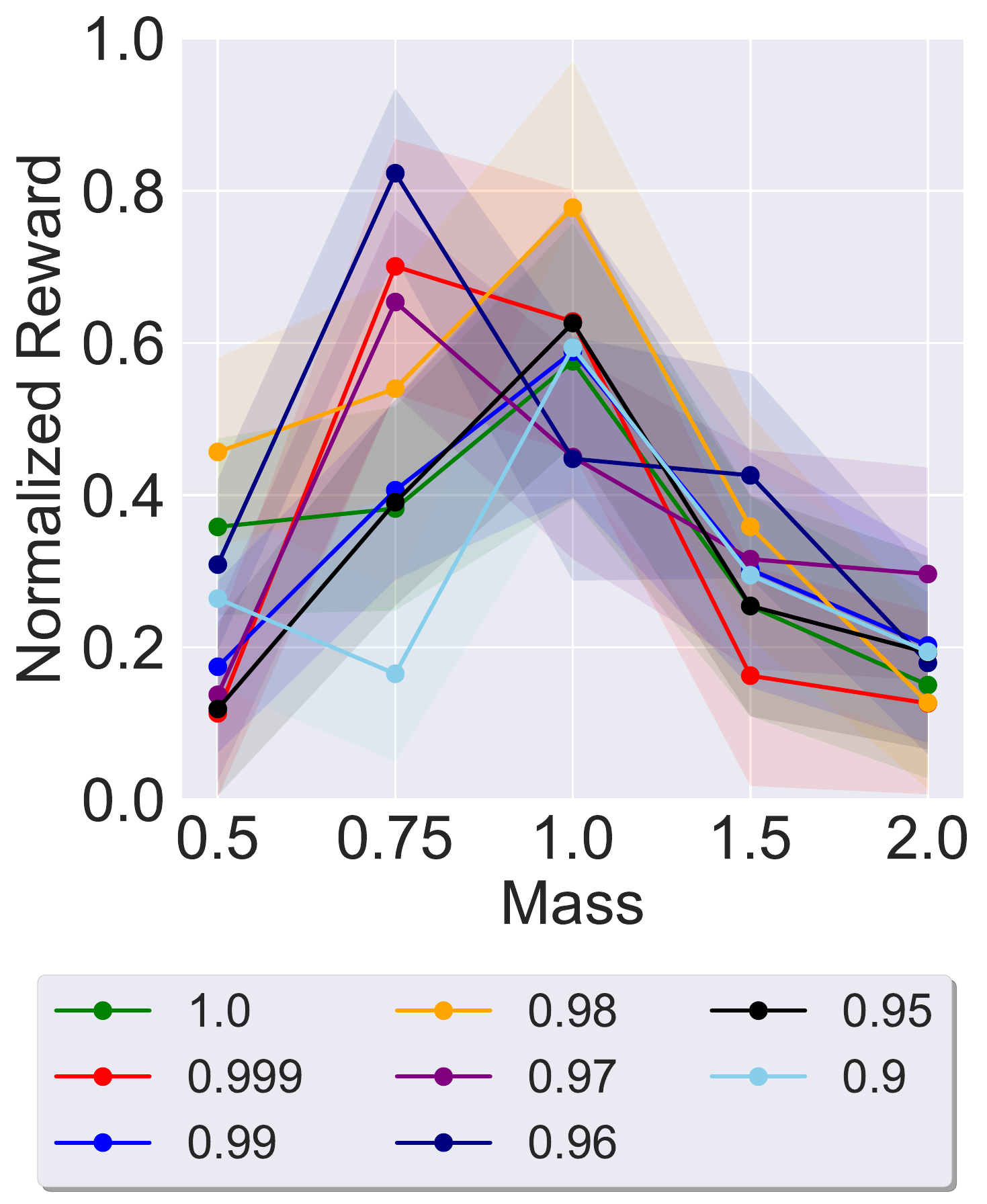}
     } &
\subfloat[Walker]{%
       \includegraphics[width=0.25\linewidth]{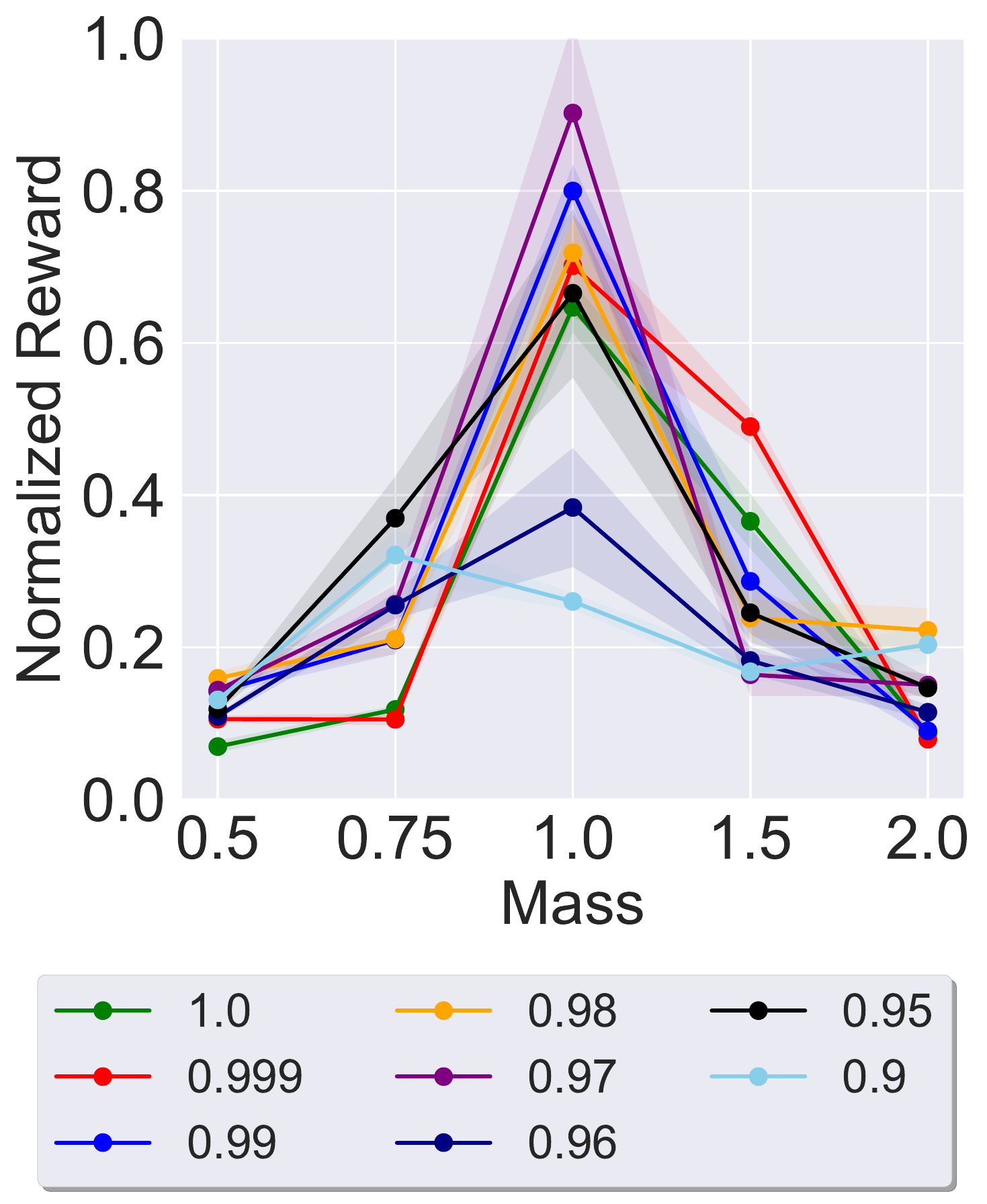}
     } &
\subfloat[Hopper]{%
       \includegraphics[width=0.25\linewidth]{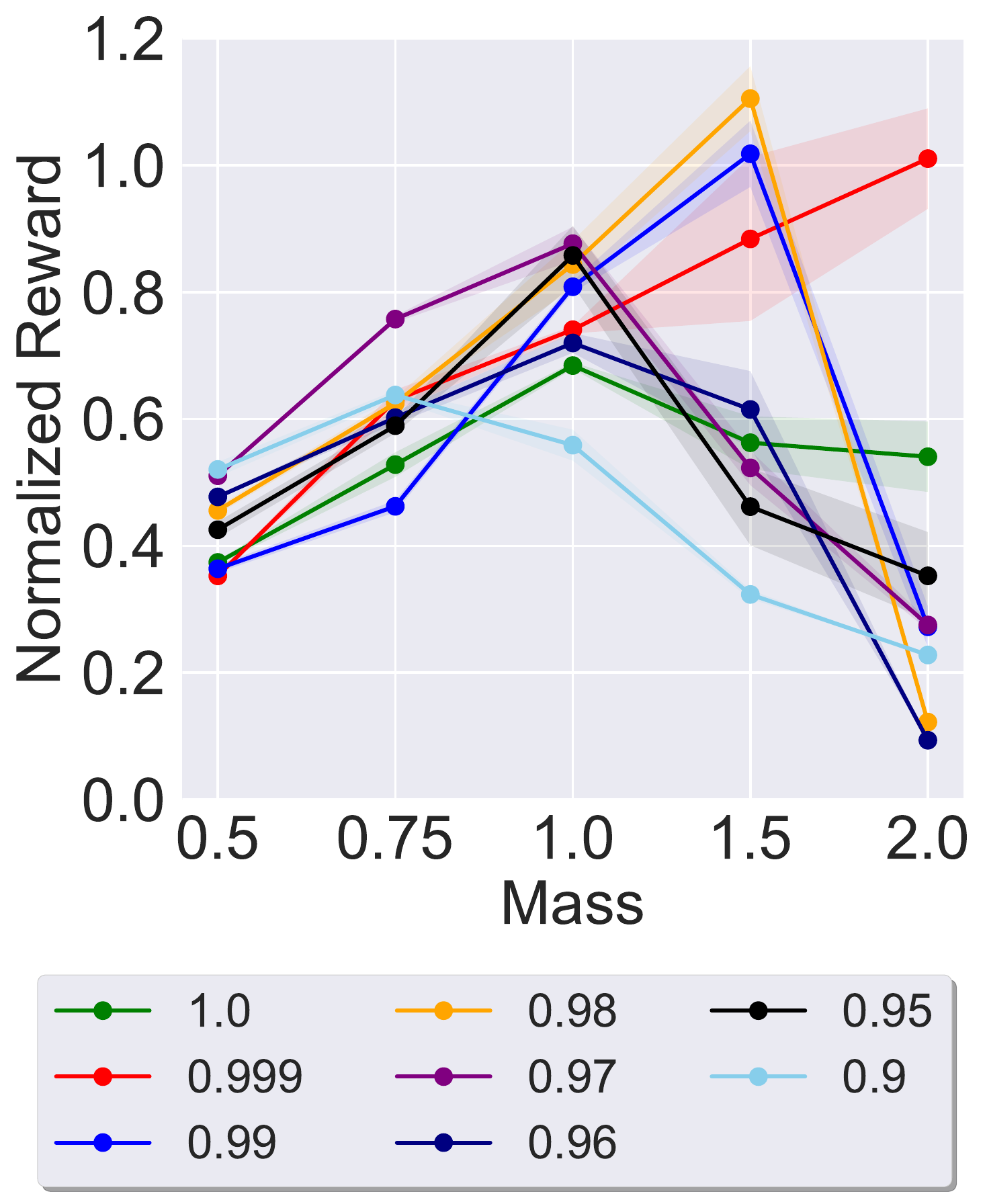}
     } \\
     \subfloat[InvDoublePend]{%
       \includegraphics[width=0.25\linewidth]{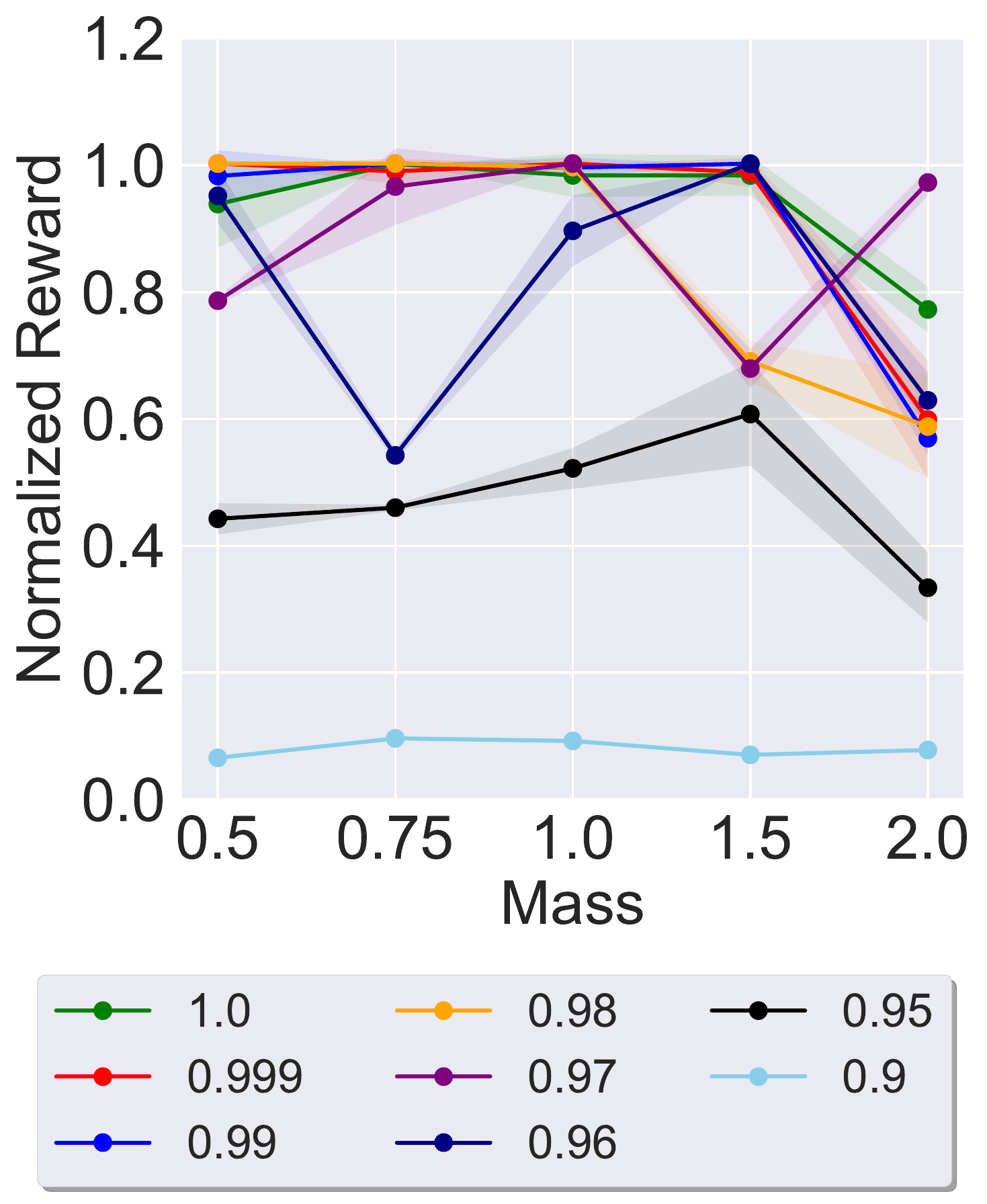}
     } &
\subfloat[Swimmer]{%
       \includegraphics[width=0.25\linewidth]{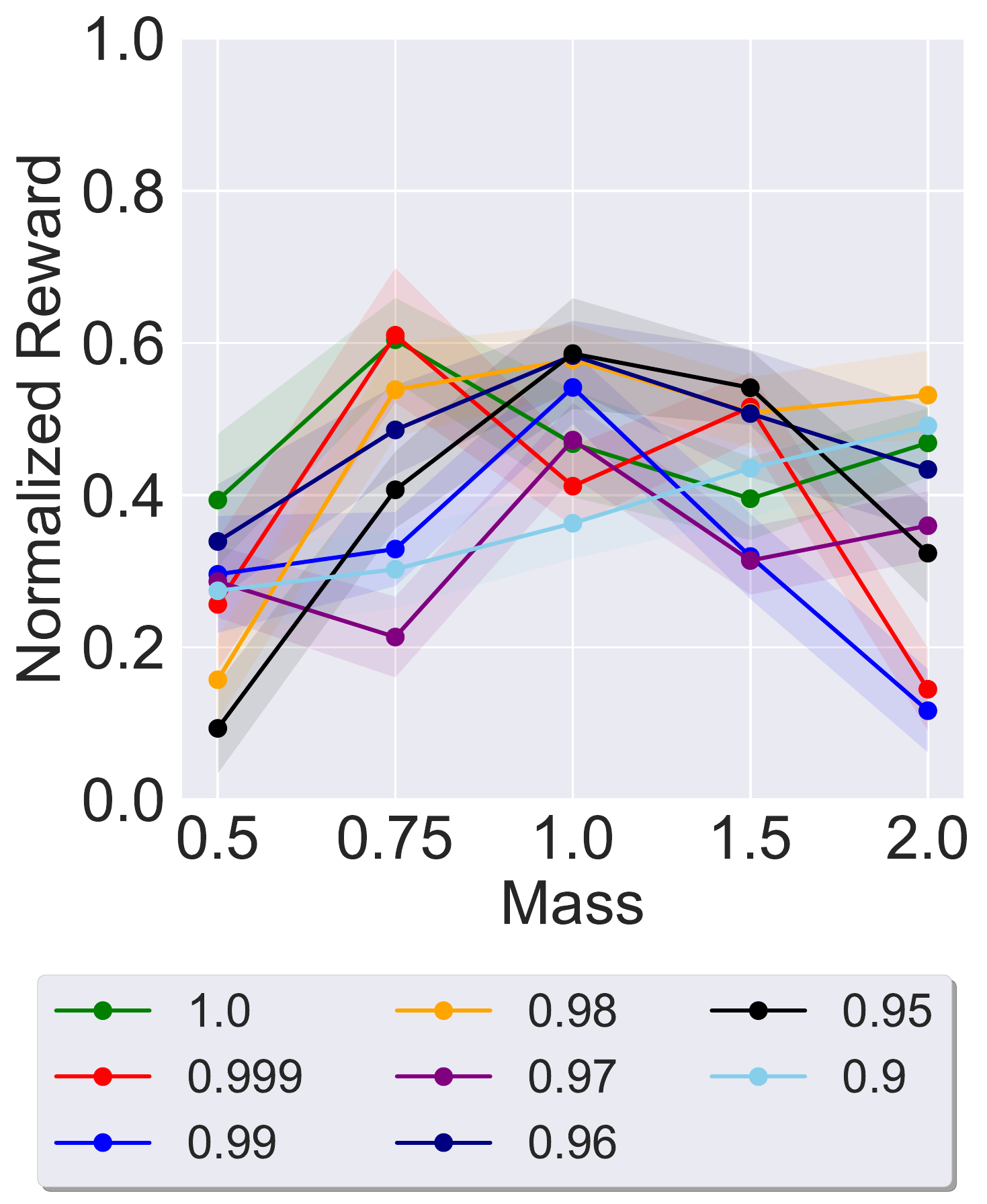}
     } &
     \\
\end{tabular}
\caption{The average (over $3$ seeds) transfer performance of Algorithm~\ref{alg:robust-gailfo} with different values of $\alpha$. The ablation shown here is used to choose $\alpha$ in Figure~\ref{fig:TransferMassFixedAlpha}. The x-axis denotes the relative mass of the learner environment $M^\mathrm{sim}$. The policies are evaluated in $M^\mathrm{real}_{c^*}$ over $1e5$ steps truncating the last episode if it does not terminate. Note that robust-GAILfO with $\alpha = 1$ corresponds to GAILfO.}
\label{fig:TransferMassAblation}
\end{figure}

\begin{table}[h]
\centering
\caption{Best value for $\alpha < 1$ chosen independently for each mismatch based on the ablation in Figure~\ref{fig:TransferFrictionAblation}. The performance of this configuration is reported by the red line in Figure~\ref{fig:TransferFrictionVarAlpha}. We add a $1$ in brackets when standard GAILfO outperforms the robust version. The value outside brackets denotes the best value found for the robust version.}
\begin{tabular}{l|lllll}
               & \multicolumn{5}{c}{Relative Friction}     \\
               & 0.5     & 0.75     & 1.0   & 1.5   & 2.0   \\ \hline
HalfCheetah & 0.999   & 0.999    & 0.999 & 0.999 & 0.999 \\
Walker    & 0.98    & 0.999    & 0.97  & 0.97  & 0.97  \\
Hopper     & 0.9 (1) & 0.99 (1) & 0.97  & 0.95  & 0.95 
\end{tabular}
\label{tab:best_alpha_friction_transfer}
\end{table}

\begin{table}[h]
\centering
\caption{Best value for $\alpha < 1$ chosen independently for each mismatch based on the ablation in Figure~\ref{fig:TransferMassAblation}. The performance of this configuration is reported by the red line in Figure~\ref{fig:TransferMassVarAlpha}. We add a $1$ in brackets when standard GAILfO outperforms the robust version. The value outside brackets denotes the best value found for the robust version.}
\begin{tabular}{l|lllll}
               & \multicolumn{5}{c}{Relative Mass}     \\
               & 0.5     & 0.75     & 1.0   & 1.5   & 2.0   \\ \hline
HalfCheetah & 0.96   & 0.97    & 0.98 & 0.96 & 0.97 \\
Walker    & 0.98    & 0.95    & 0.97  & 0.999  & 0.98  \\
Hopper      & 0.9 & 0.97 & 0.97  & 0.98  & 0.999 \\
InvDoublePendulum & 0.98 & 0.99 & 0.97 & 0.96 & 0.97 \\
Swimmer & 0.96 (1) & 0.999 (1) & 0.95 & 0.95 & 0.98 \\
\end{tabular}
\label{tab:best_alpha_mass_transfer}
\end{table}

\begin{figure}[!h] 
\centering
\begin{tabular}{ccc}
\subfloat[HalfCheetah]{%
       \includegraphics[width=0.25\linewidth]{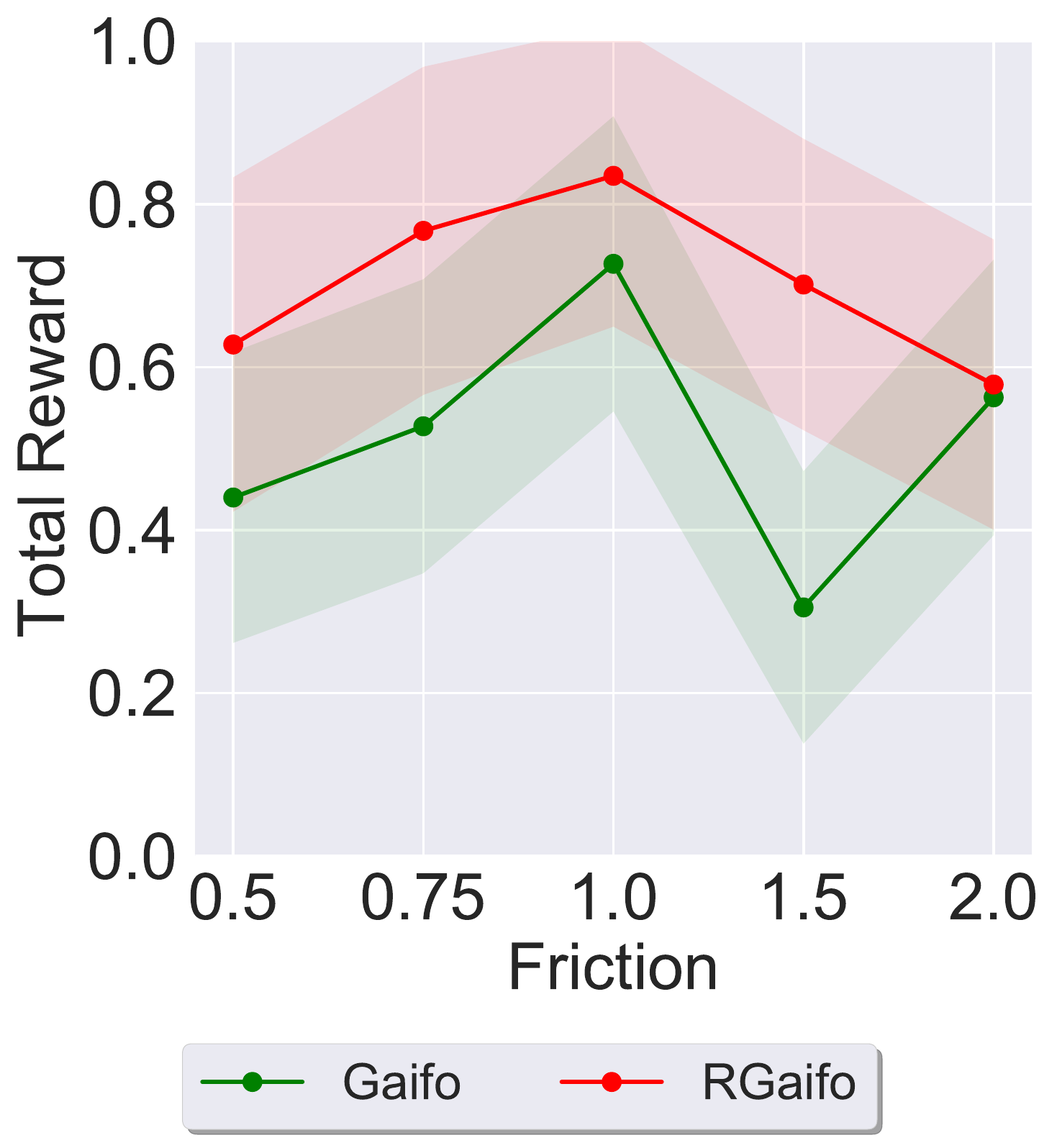}
     } &
\subfloat[Walker]{%
       \includegraphics[width=0.25\linewidth]{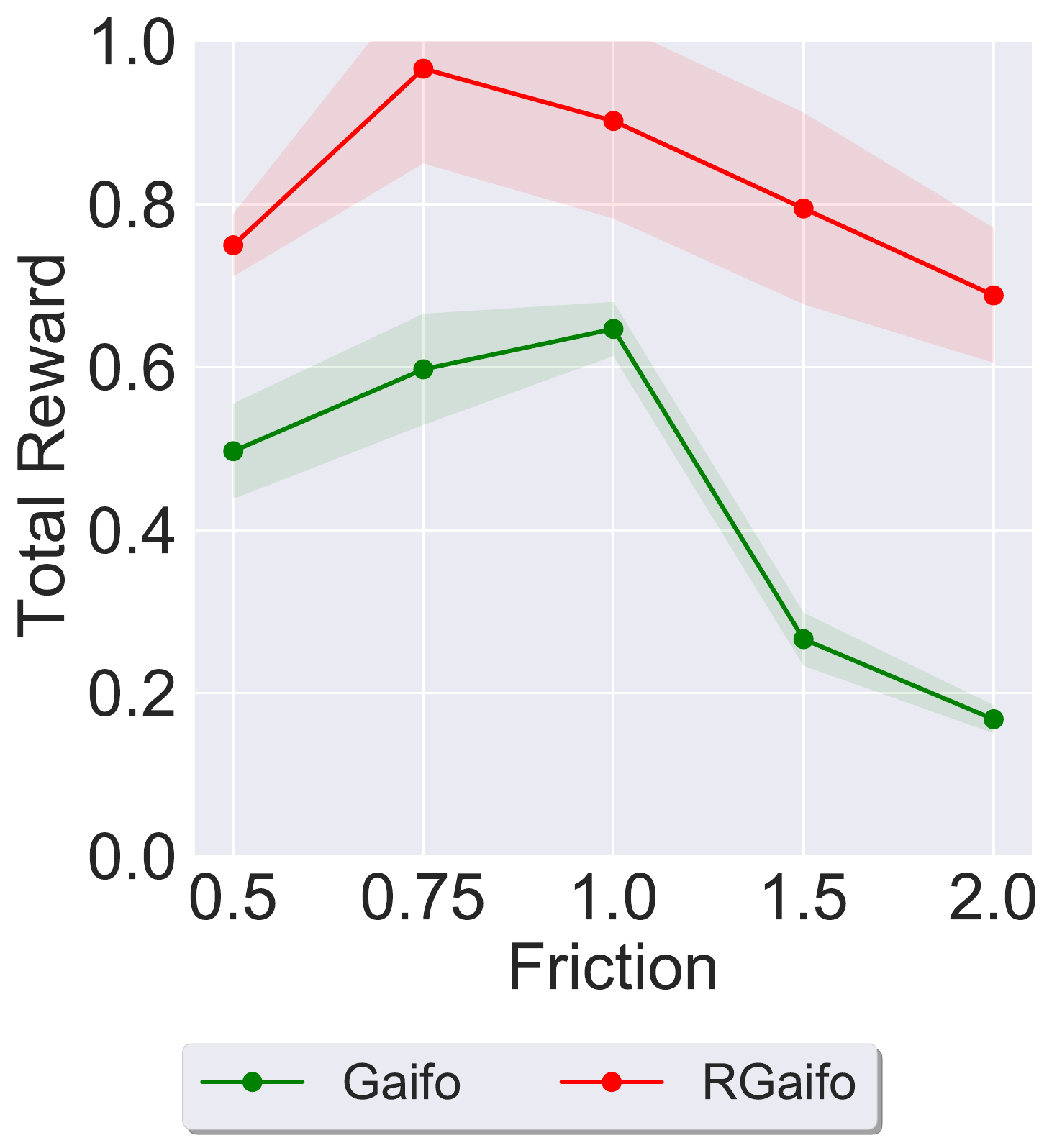}
     } &
\subfloat[Hopper]{%
       \includegraphics[width=0.25\linewidth]{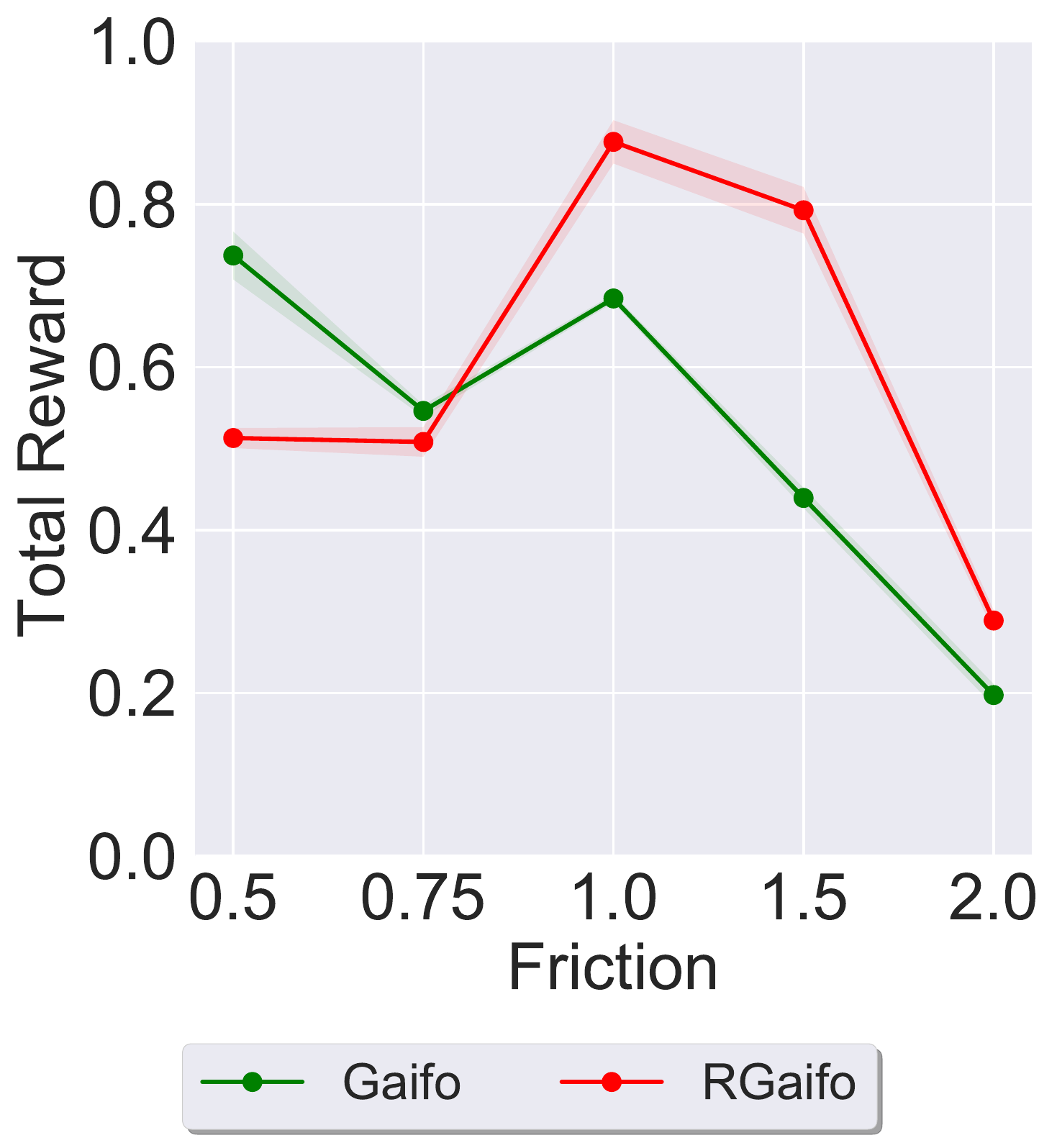}
     } \\
\end{tabular}
\caption{Average performance (over $3$ seeds ) of Algorithm~\ref{alg:robust-gailfo} with the value of $\alpha$ that is chosen indipendently for each mismatch (i.e. each point on the x-axis). The choice is made picking the best performing $\alpha$ for each mismatch in Figure~\ref{fig:TransferFrictionAblation}. The x-axis reports the relative friction of the learner environment. The policies are evaluated over $1e5$ steps truncating the last episode if it does not terminate.The values chosen for $\alpha$ are given in Table~\ref{tab:best_alpha_friction_transfer}.}
\label{fig:TransferFrictionVarAlpha}
\end{figure}

\begin{figure}[!h] 
\centering
\begin{tabular}{ccc}
\subfloat[HalfCheetah]{%
       \includegraphics[width=0.25\linewidth]{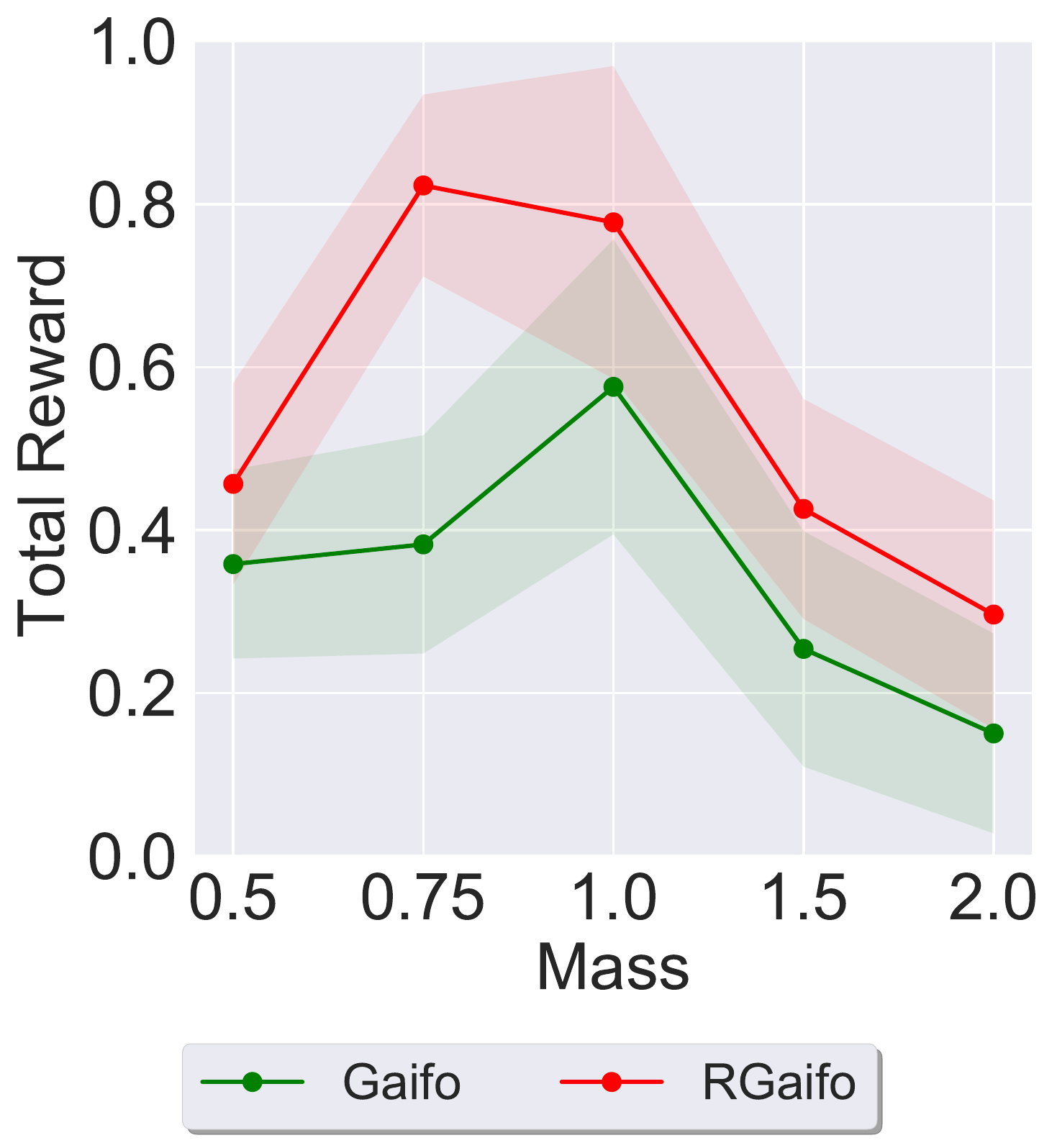}
     } &
\subfloat[Walker]{%
       \includegraphics[width=0.25\linewidth]{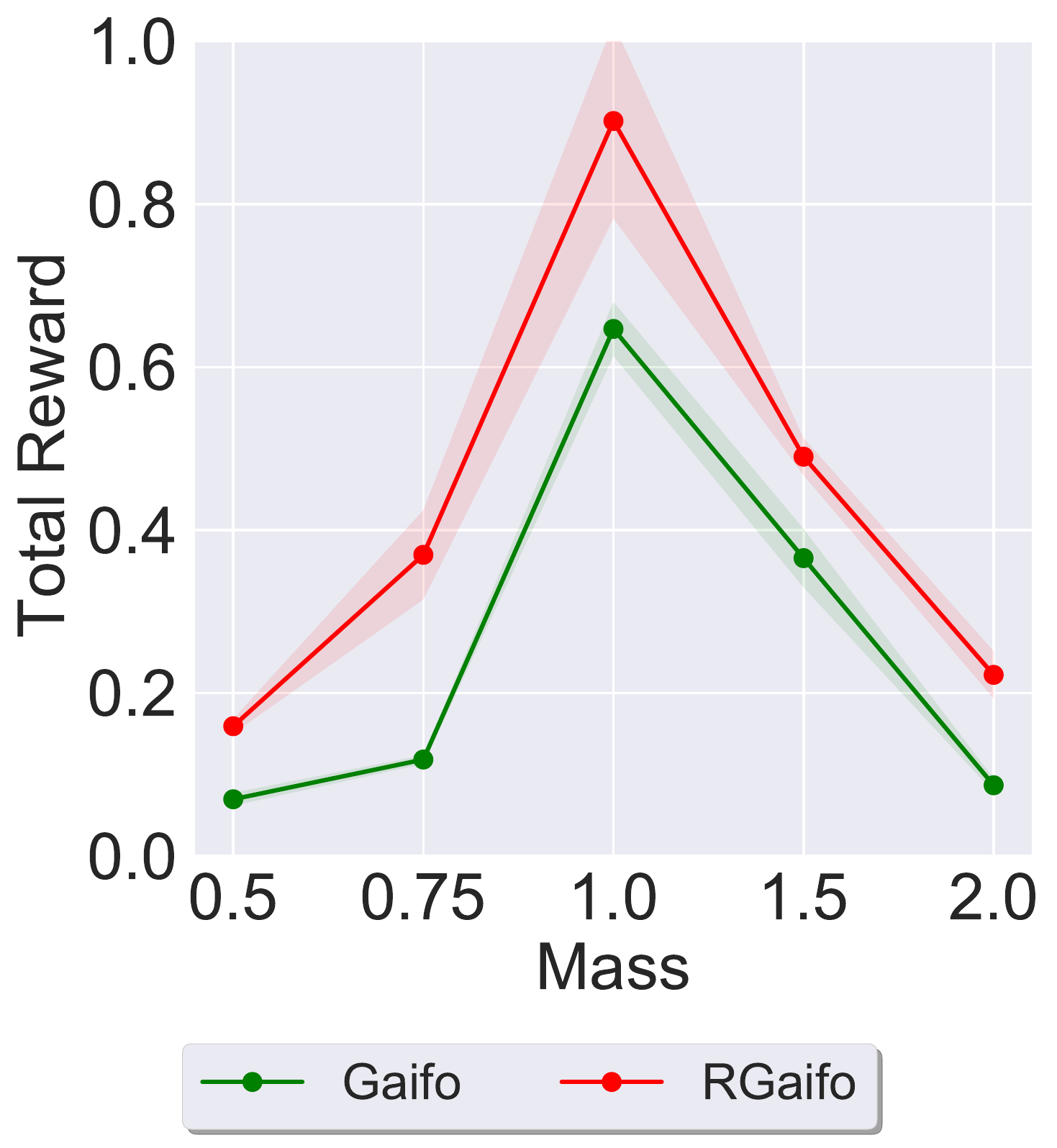}
     } &
\subfloat[Hopper]{%
       \includegraphics[width=0.25\linewidth]{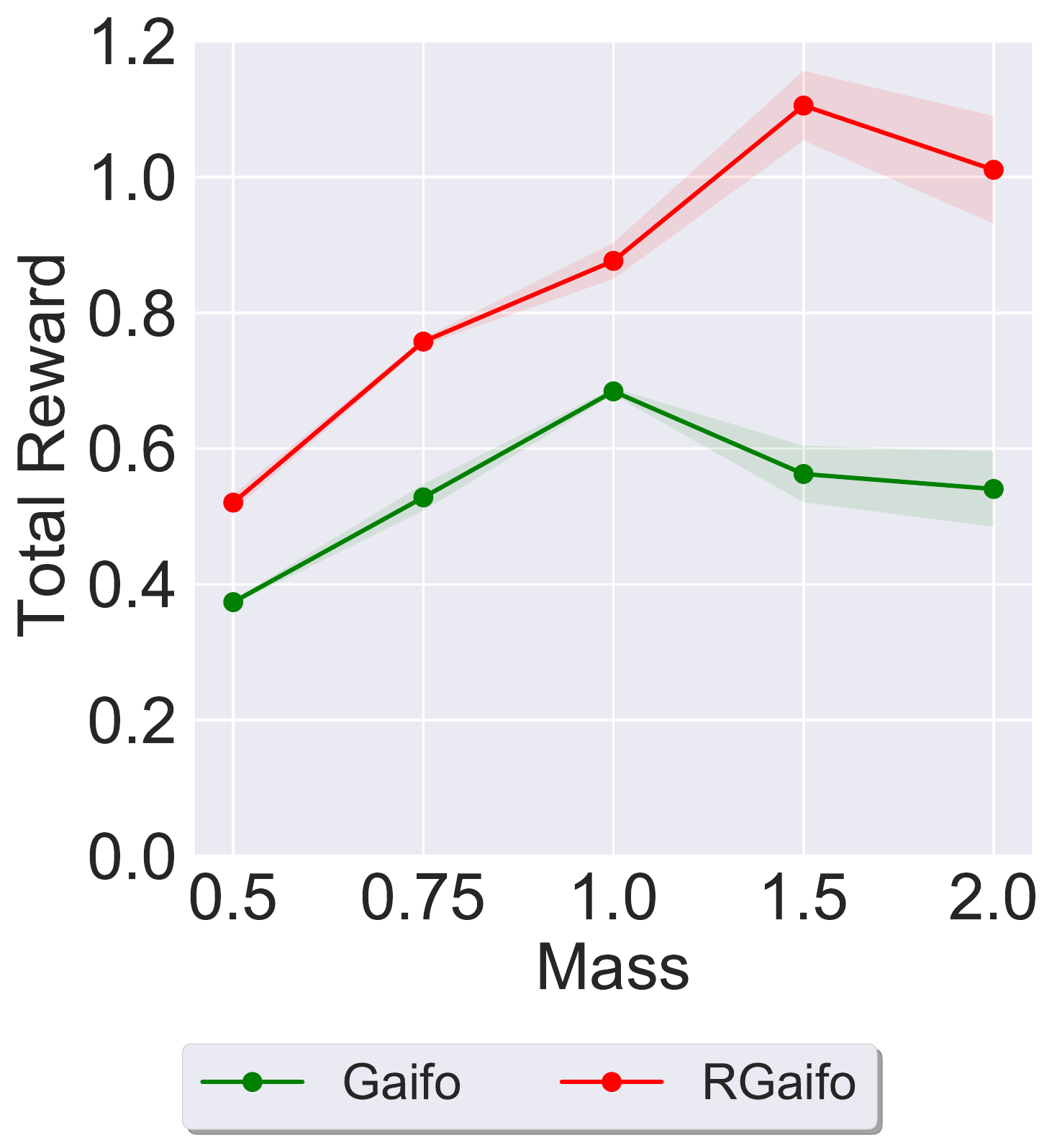}
     } \\
     \subfloat[InvDoublePend]{%
       \includegraphics[width=0.25\linewidth]{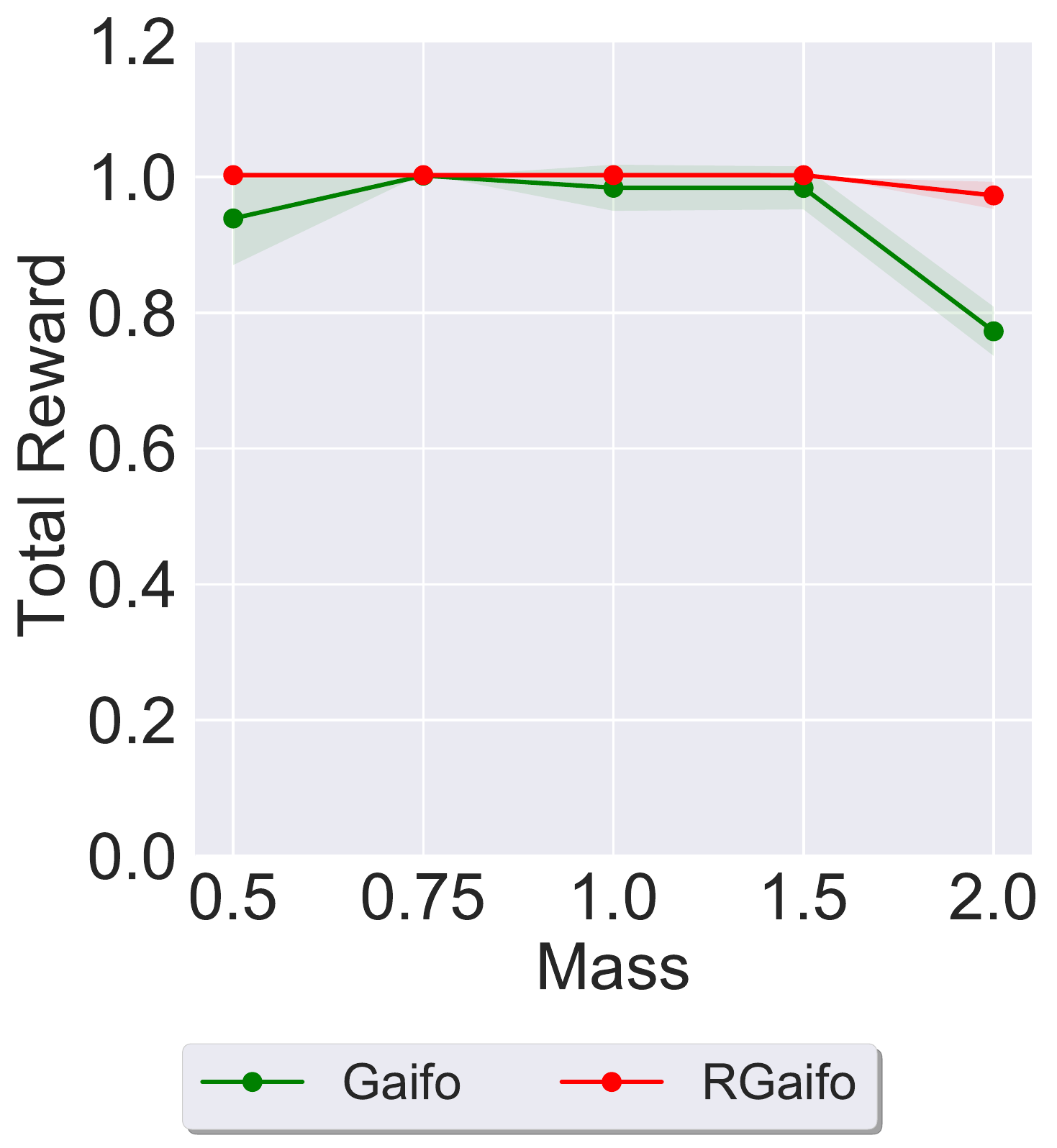}
     } &
\subfloat[Swimmer]{%
       \includegraphics[width=0.25\linewidth]{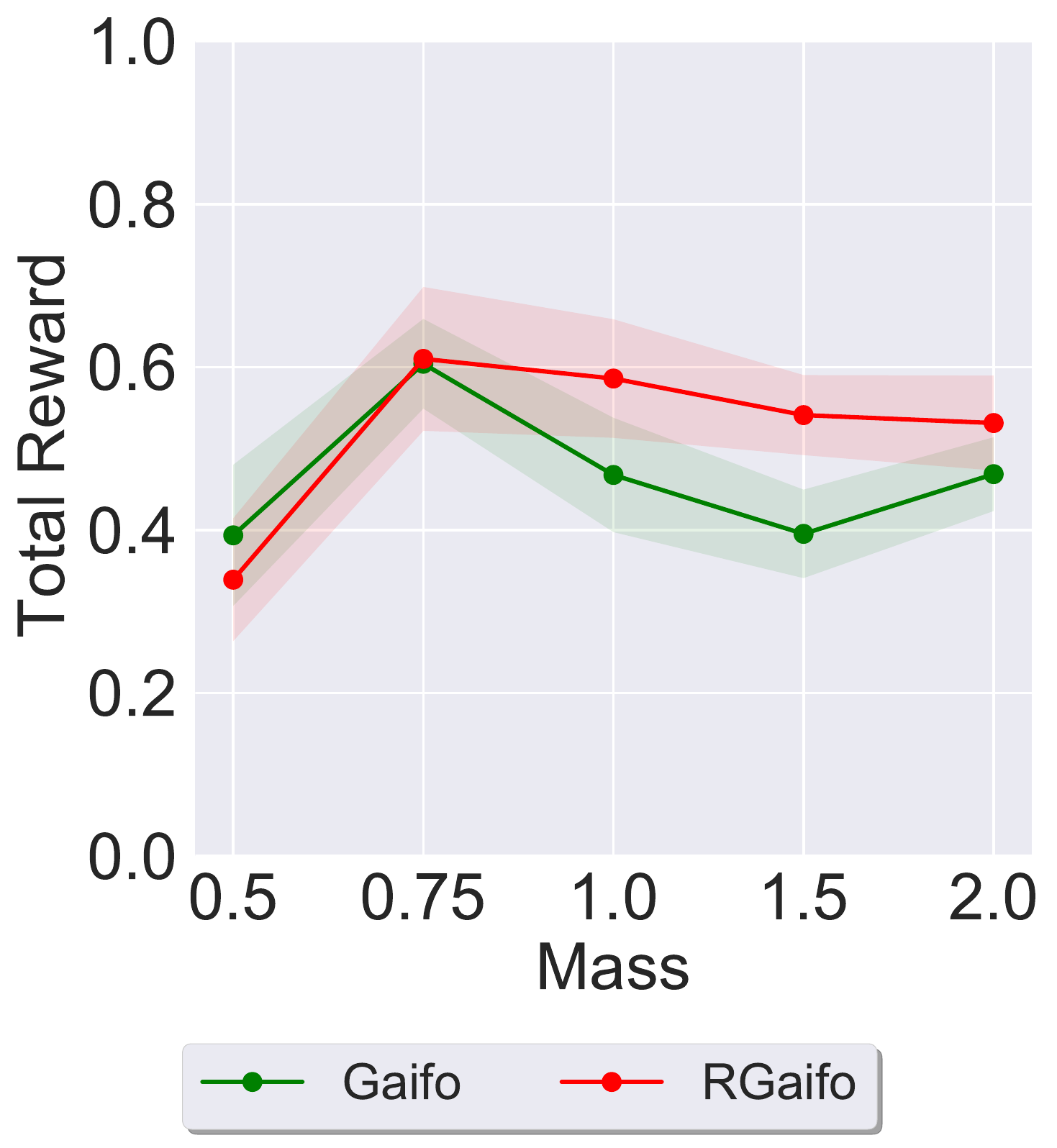}
     } &
     \\
\end{tabular}
\caption{Average performance (over $3$ seeds ) of Algorithm~\ref{alg:robust-gailfo} with the value of $\alpha$ that is chosen indipendently for each mismatch (i.e. each point on the x-axis). The choice is made picking the best performing $\alpha$ for each mismatch in Figure~\ref{fig:TransferMassAblation}. The x-axis reports the relative mass of the learner environment. The policies are evaluated over $1e5$ steps truncating the last episode if it does not terminate. The values chosen for $\alpha$ are given in Table~\ref{tab:best_alpha_mass_transfer}.}
\label{fig:TransferMassVarAlpha}
\end{figure}

\clearpage

\section{Robust Performance: MuJoCo}
\label{app:robust-results}

We present the following results:
\begin{itemize}
\item The ablation study on the robust performance of Algorithm~\ref{alg:robust-gailfo} with different values of $\alpha$ under the relative friction variations (see Figure~\ref{fig:RobustnessFrictionAblation}).
\item The ablation study on the robust performance of Algorithm~\ref{alg:robust-gailfo} with different values of $\alpha$ under the relative mass variations (see Figure~\ref{fig:RobustnessMassAblation}).
\end{itemize}

\begin{figure*}[!h] 
\centering
\begin{tabular}{ccccc}
\subfloat[HalfCheetah]{%
       \includegraphics[width=0.16\linewidth]{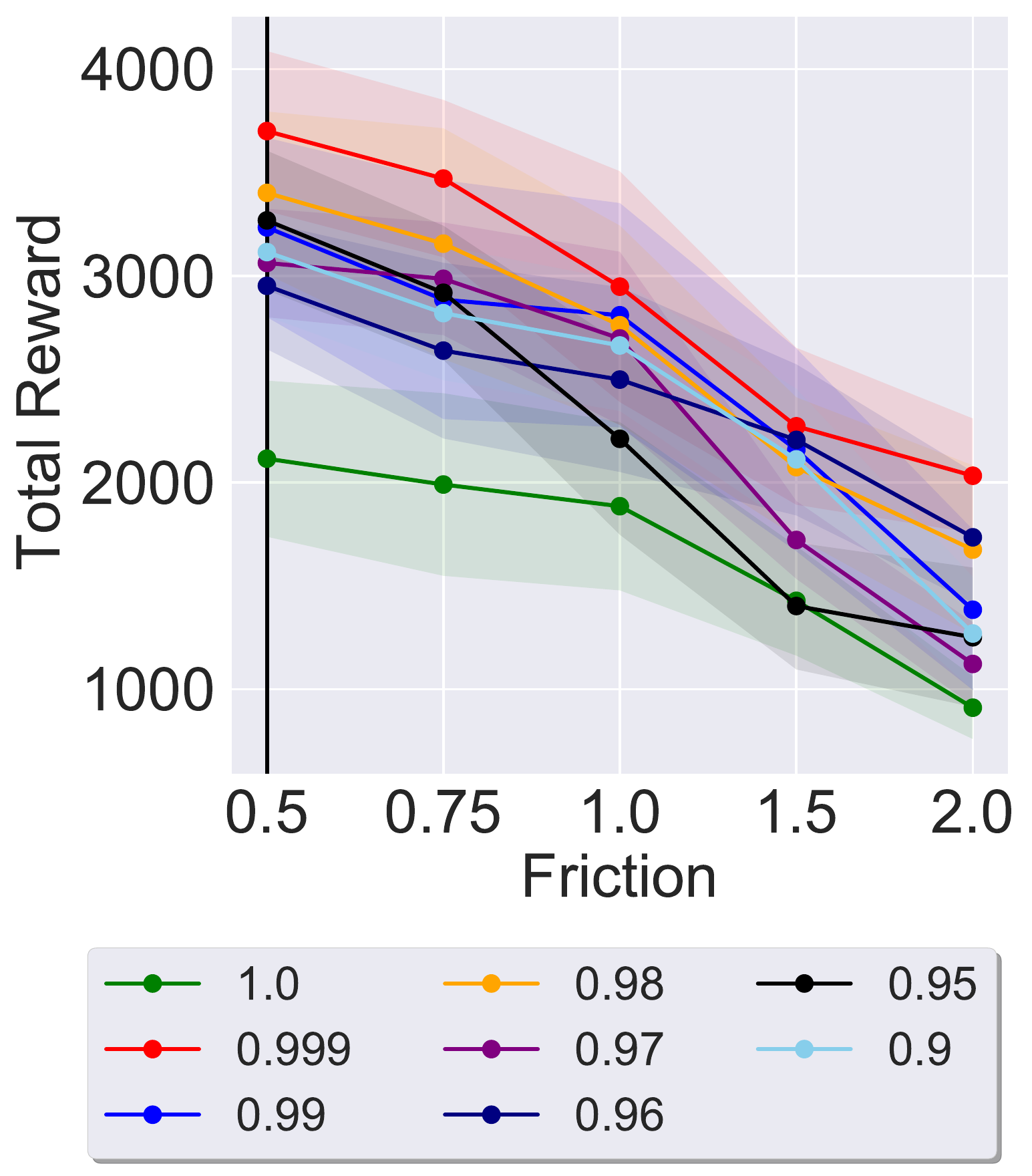}
     } &
\subfloat[HalfCheetah]{%
       \includegraphics[width=0.16\linewidth]{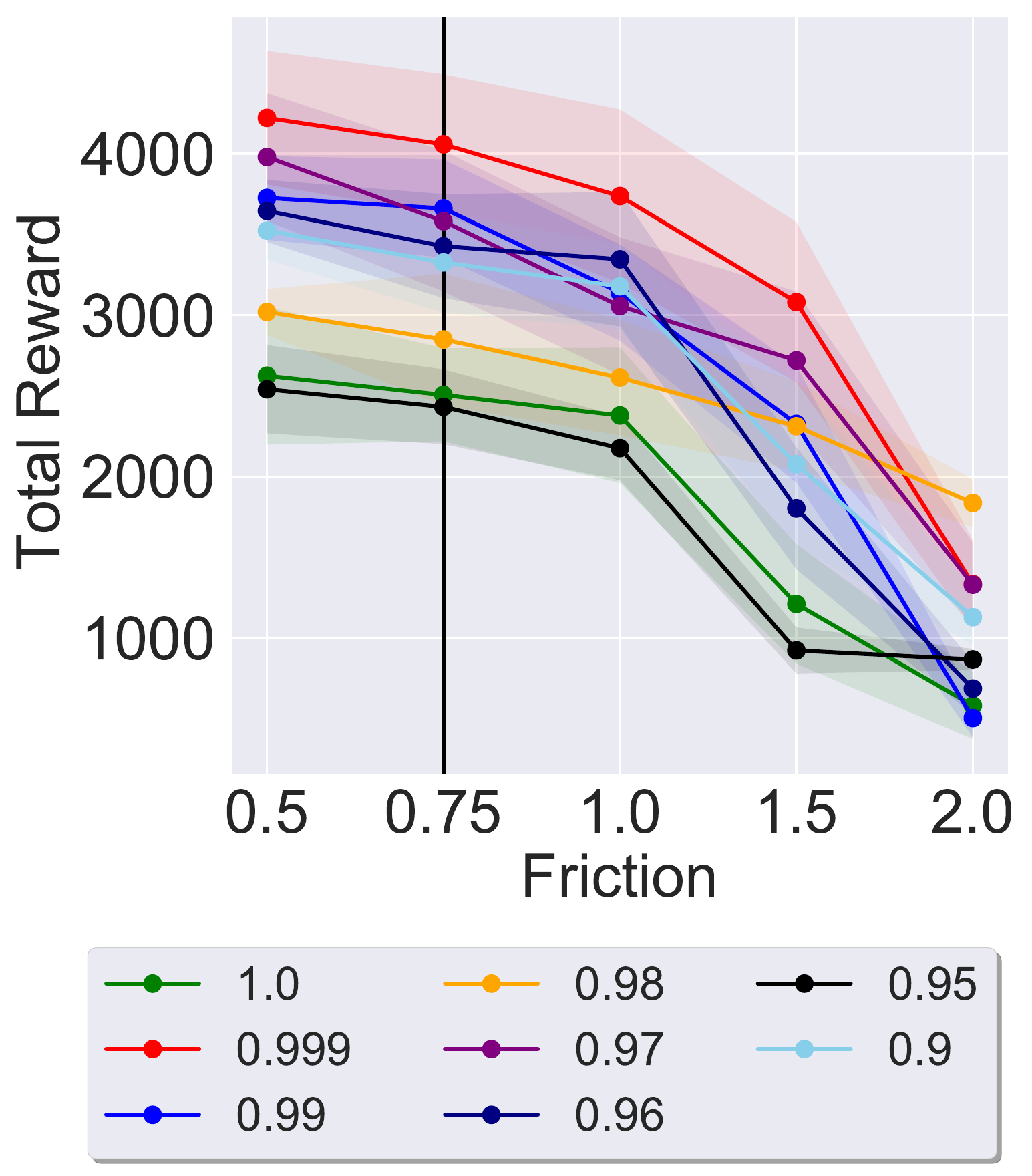}
     } &
\subfloat[HalfCheetah]{%
       \includegraphics[width=0.16\linewidth]{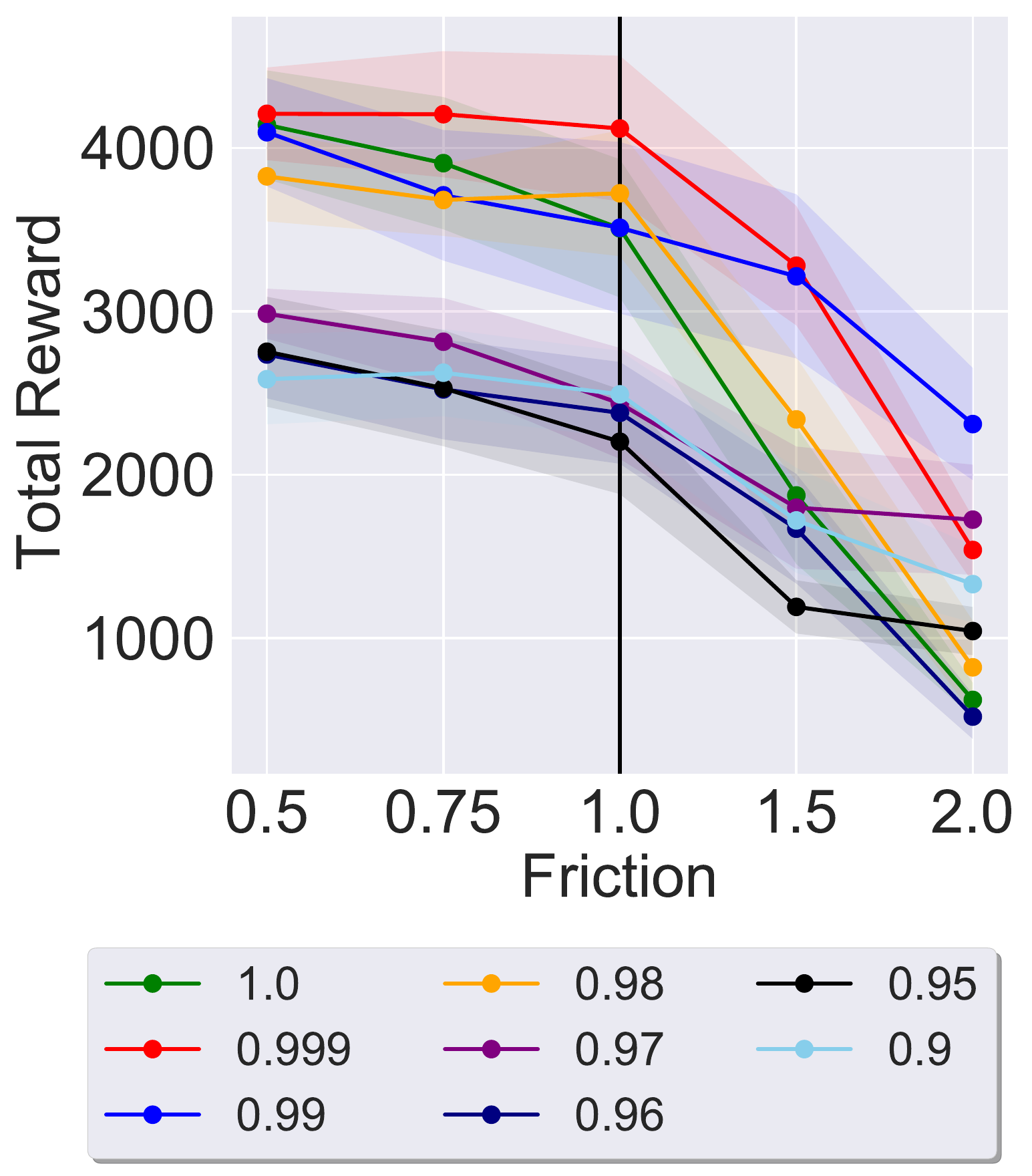}
     } &
\subfloat[HalfCheetah]{%
       \includegraphics[width=0.16\linewidth]{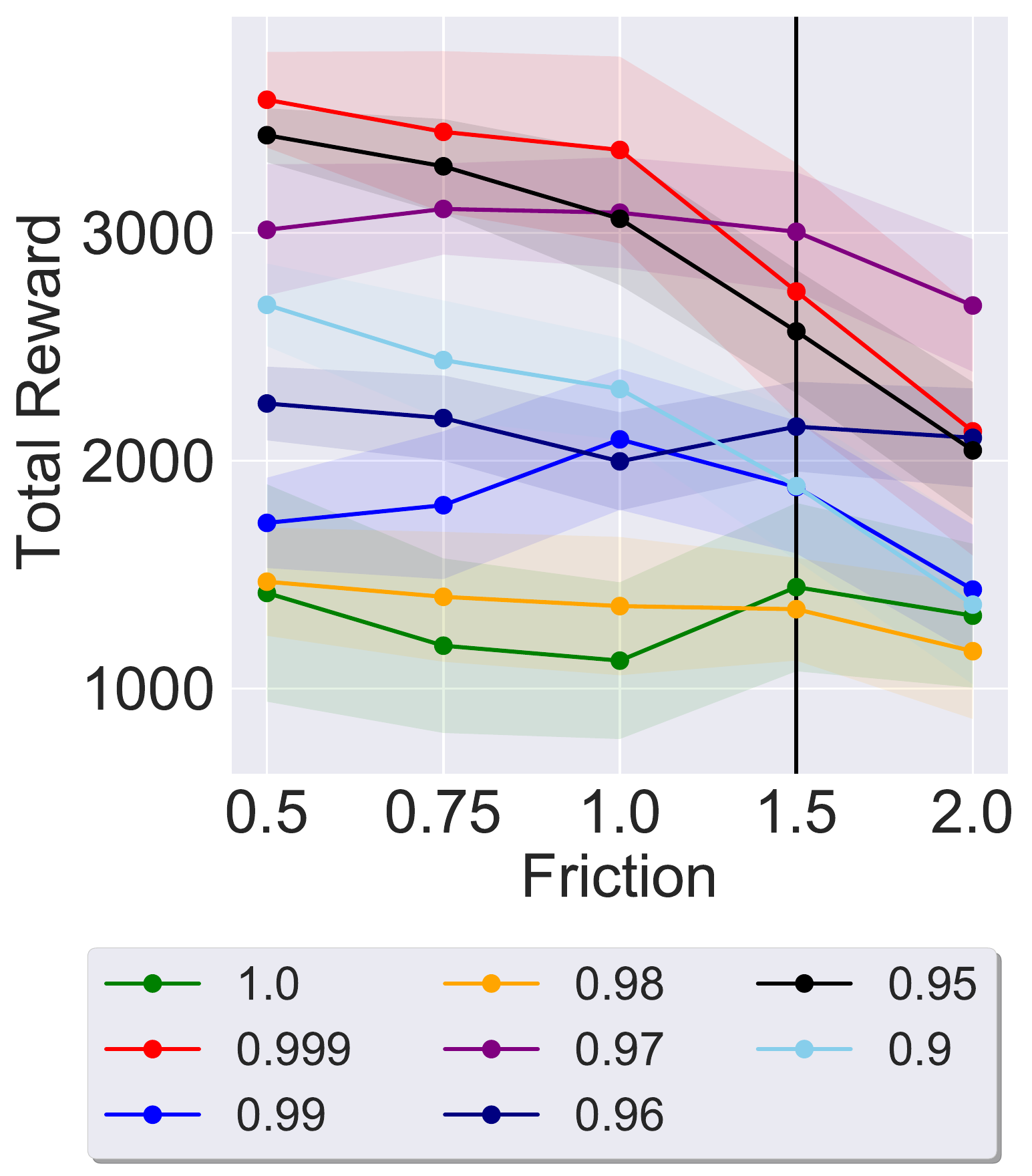}
     } &
\subfloat[HalfCheetah]{%
       \includegraphics[width=0.16\linewidth]{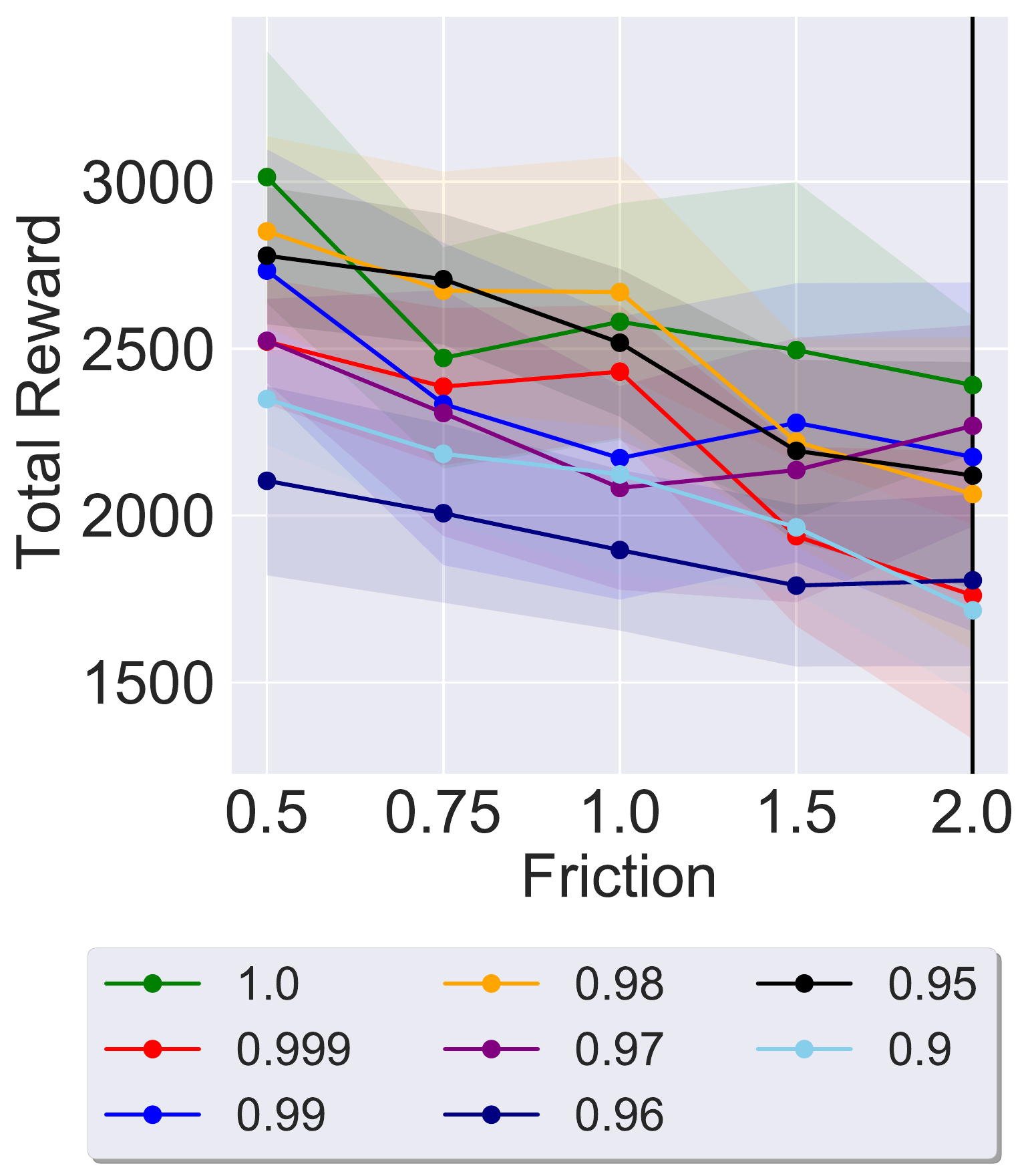}
     } \\
\subfloat[Walker]{%
       \includegraphics[width=0.16\linewidth]{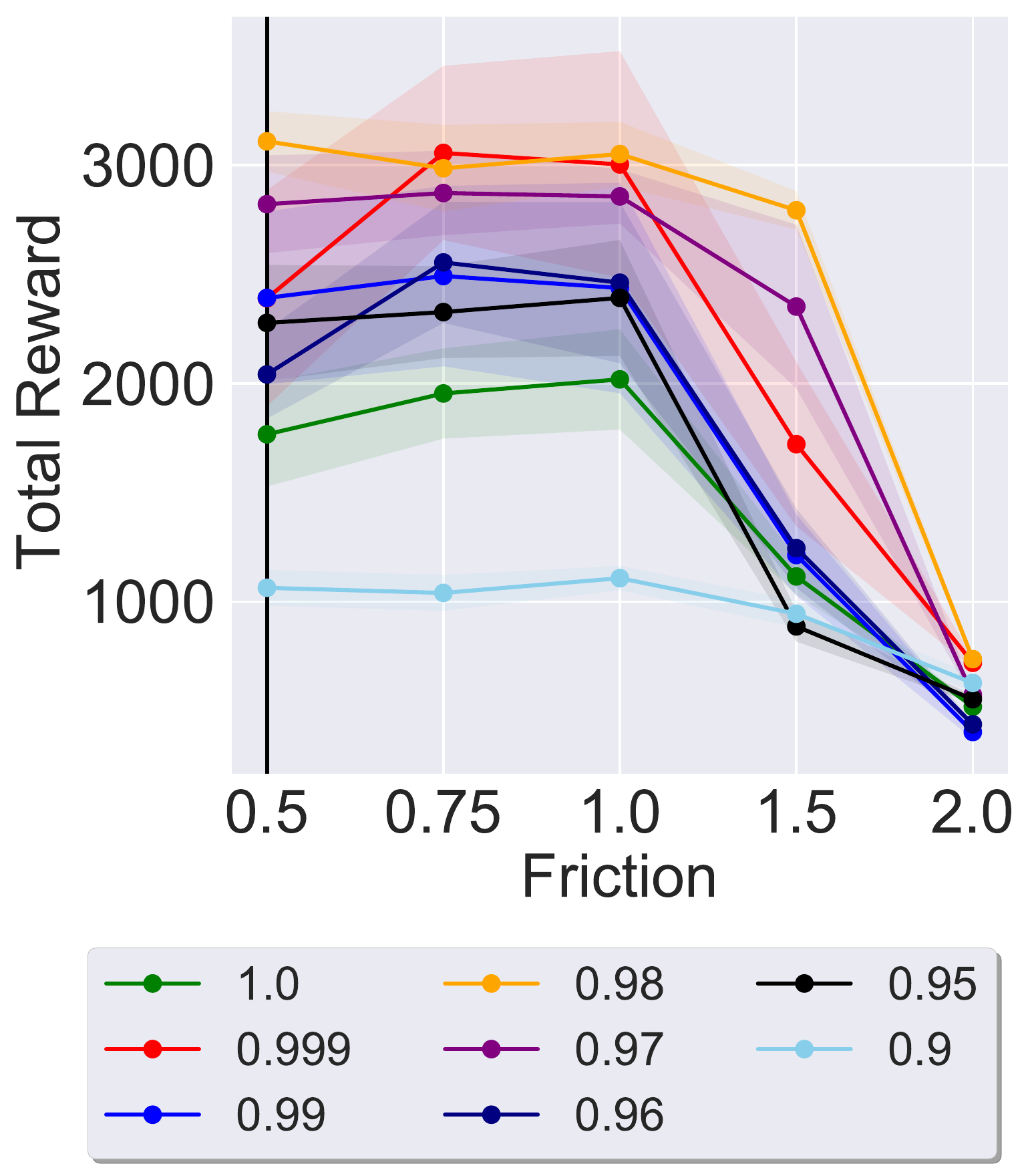}
     } &
\subfloat[Walker]{%
       \includegraphics[width=0.16\linewidth]{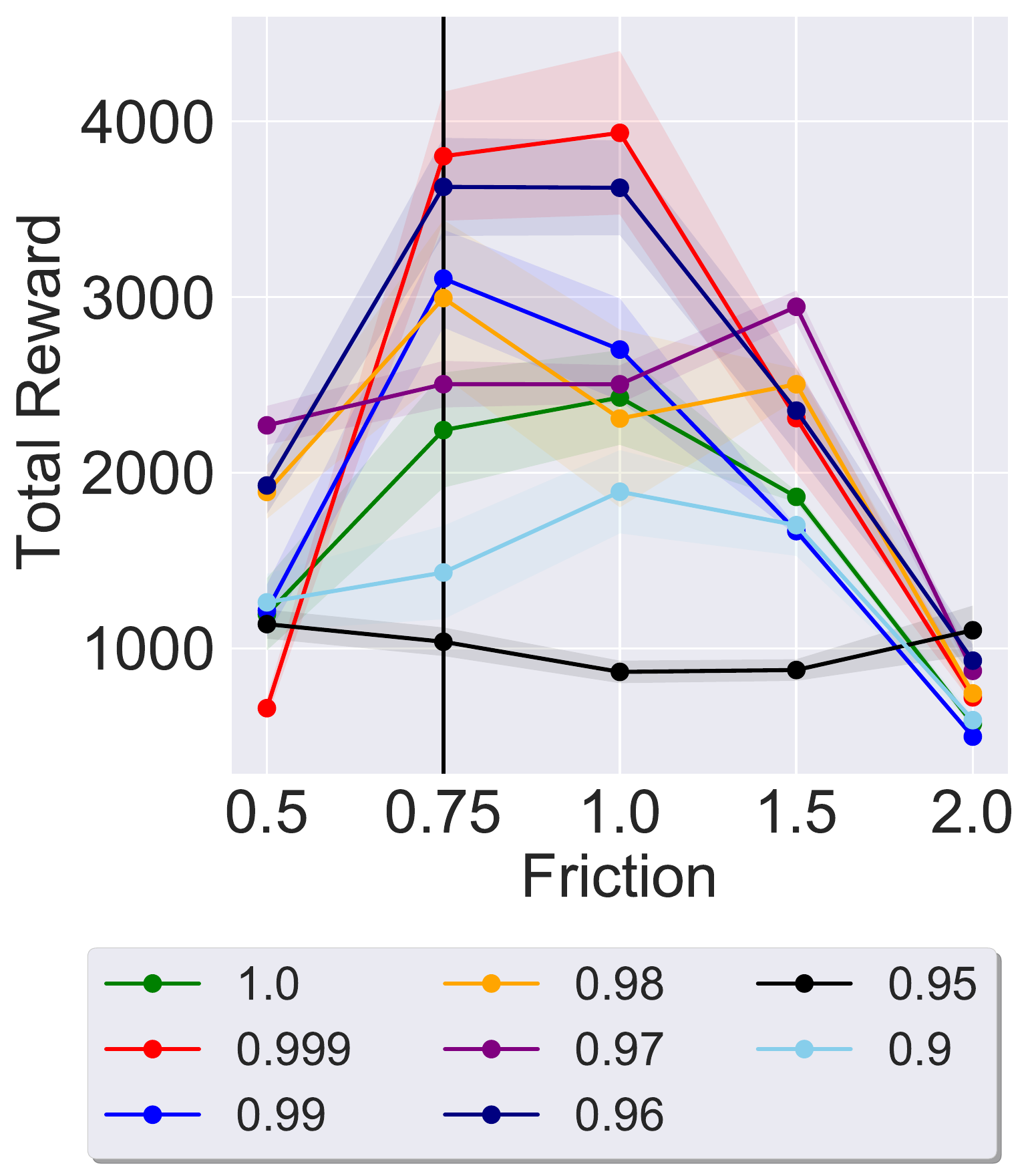}
     } &
\subfloat[Walker]{%
       \includegraphics[width=0.16\linewidth]{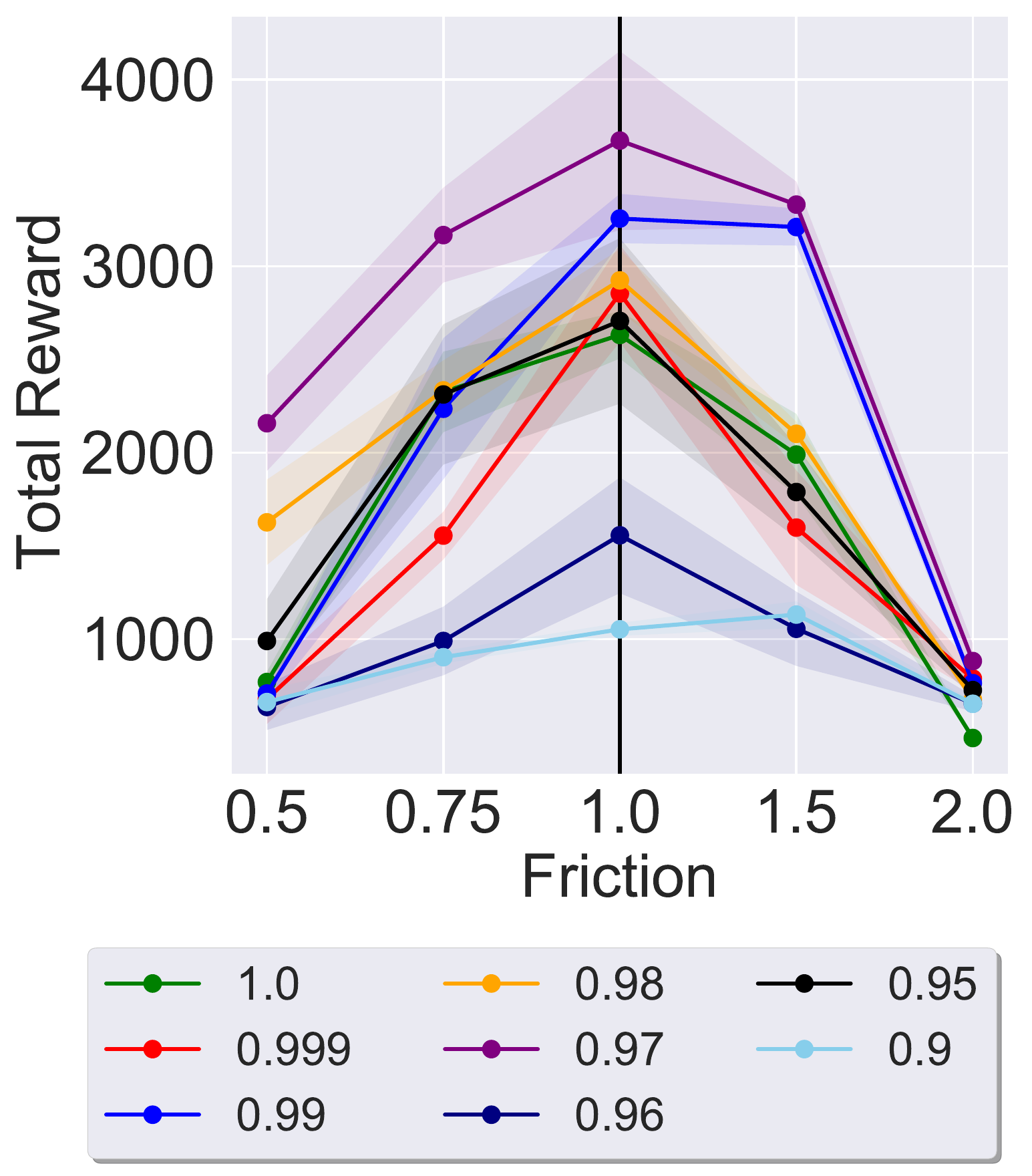}
     } &
\subfloat[Walker]{%
       \includegraphics[width=0.16\linewidth]{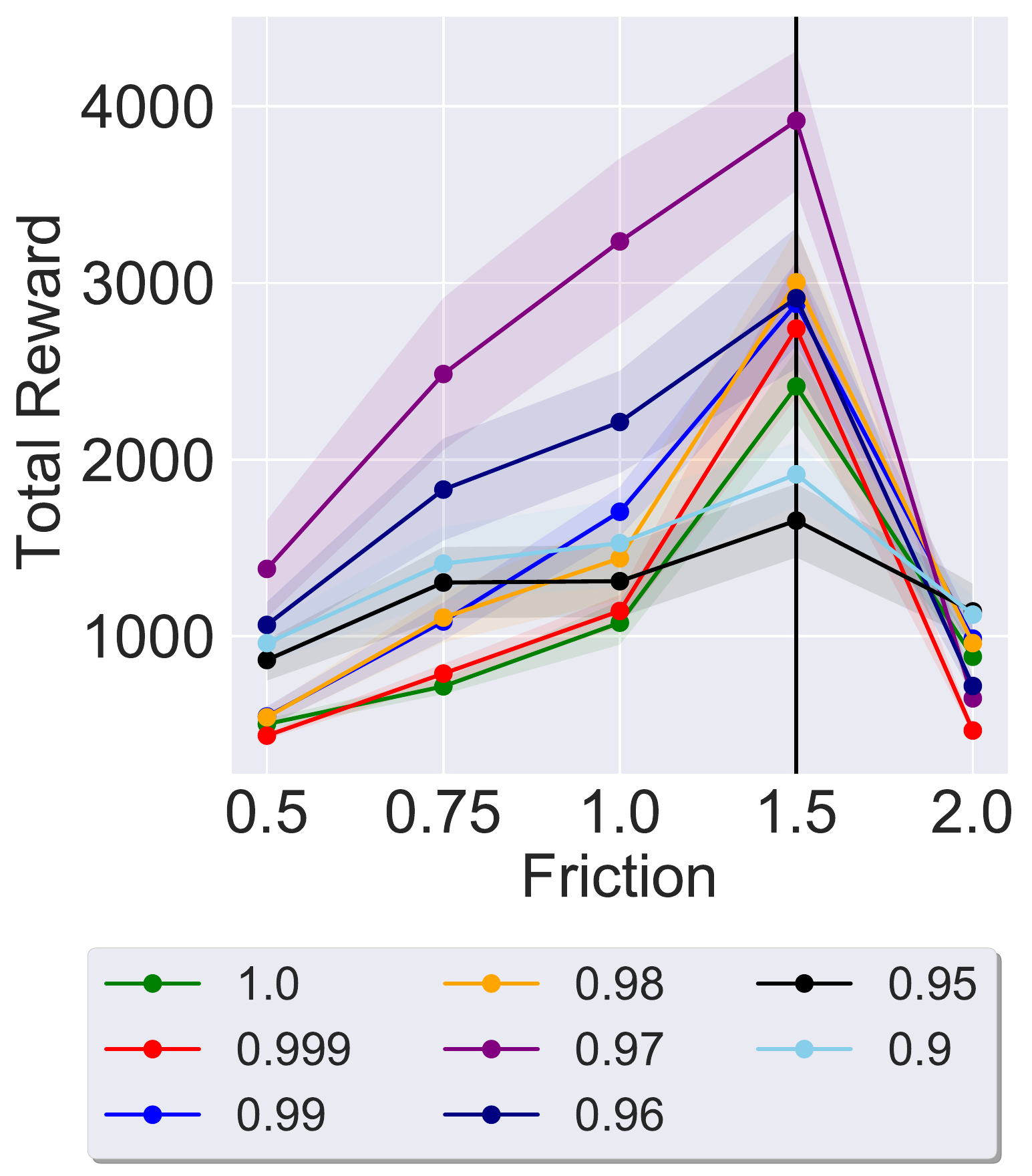}
     } &
\subfloat[Walker]{%
       \includegraphics[width=0.16\linewidth]{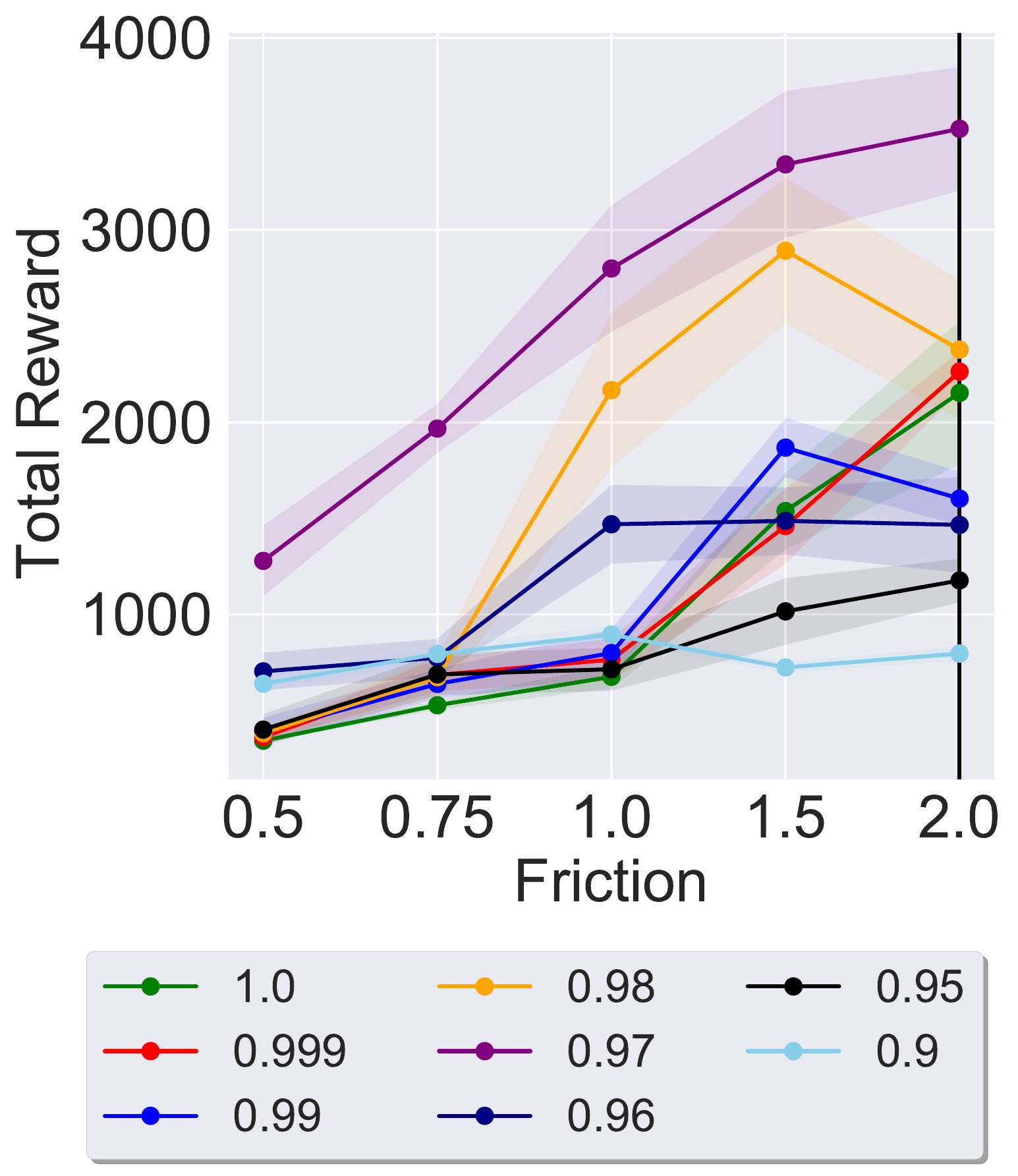}
     } \\
\subfloat[Hopper]{%
       \includegraphics[width=0.16\linewidth]{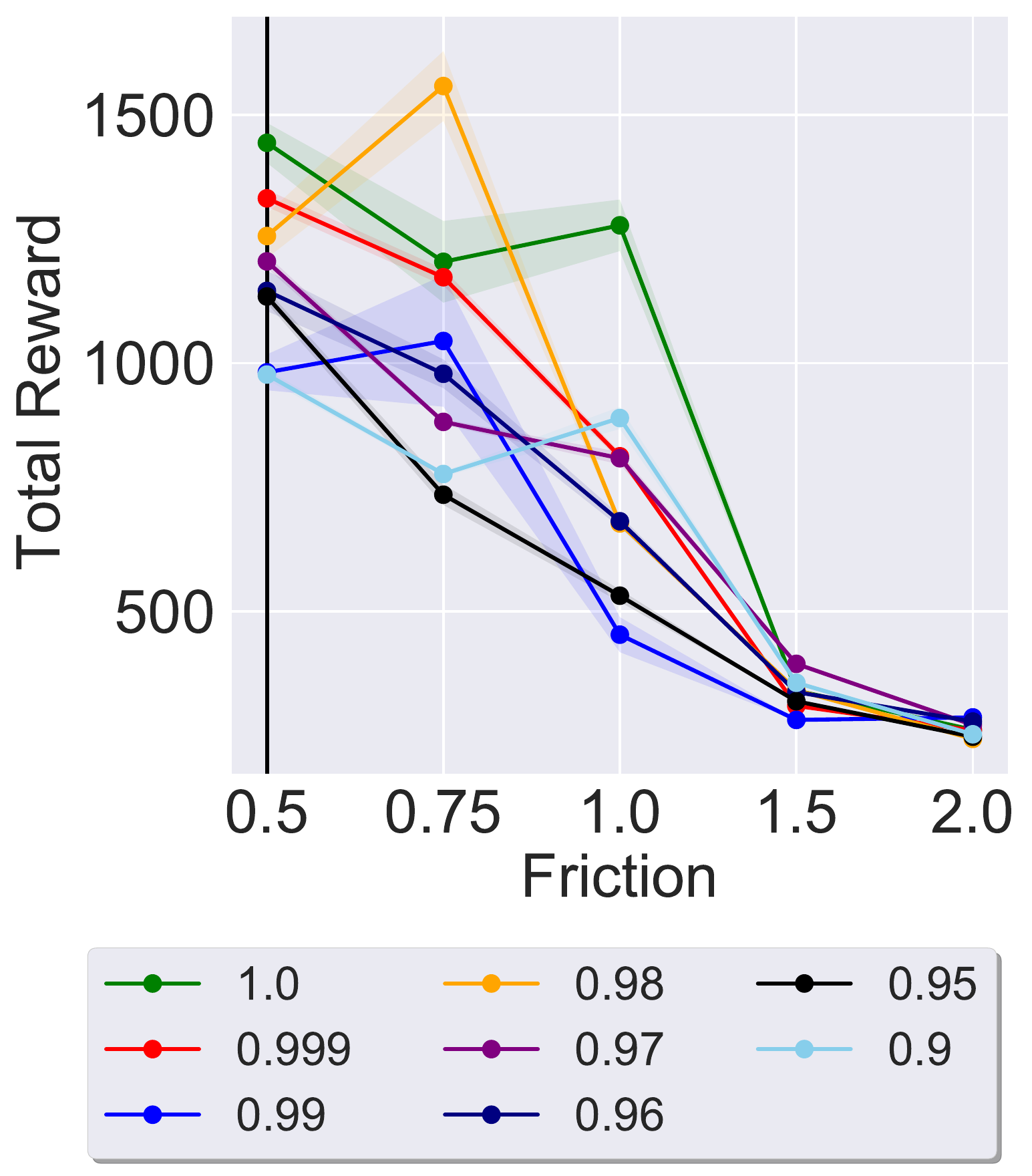}
     } &
\subfloat[Hopper]{%
       \includegraphics[width=0.16\linewidth]{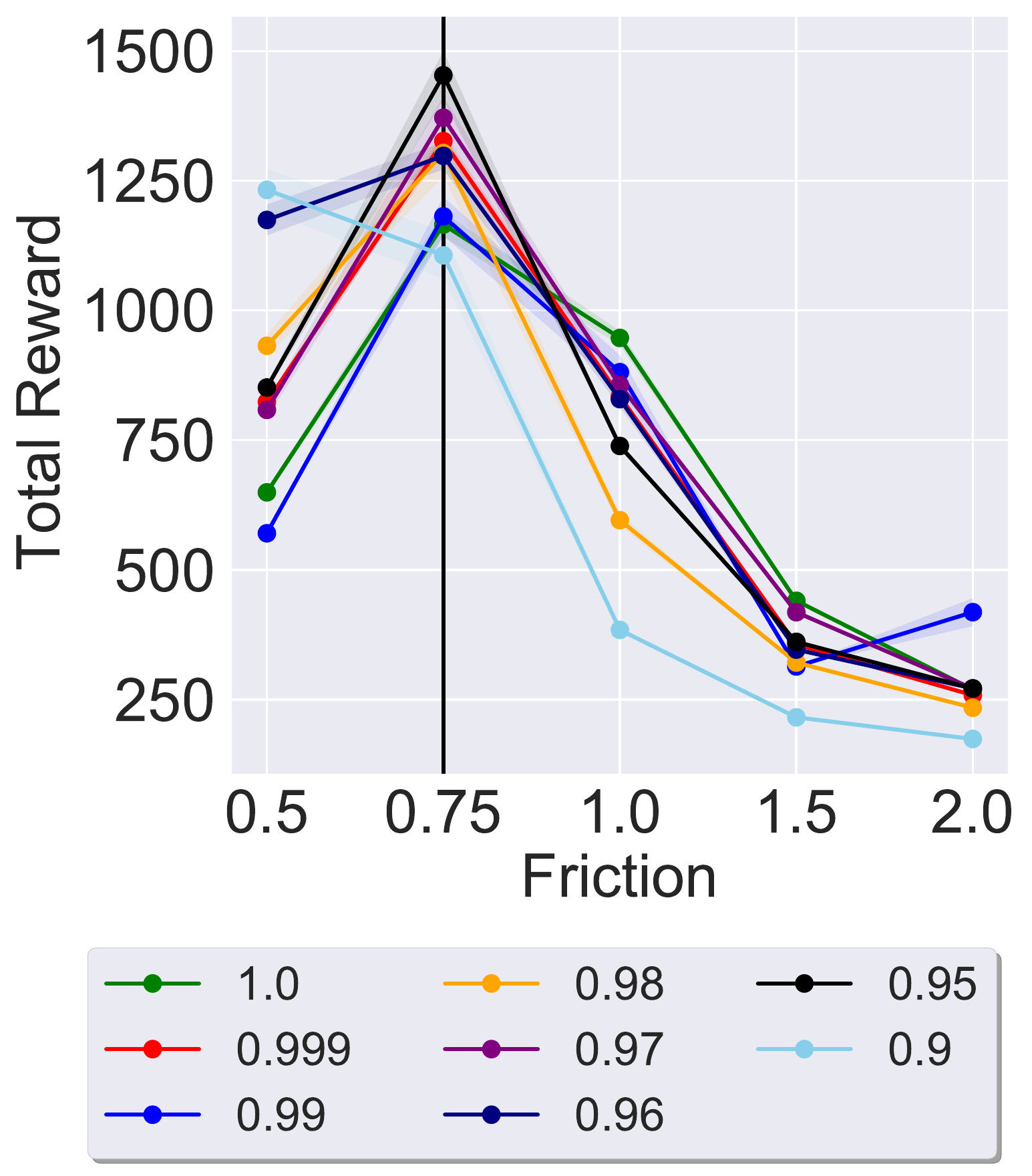}
     } &
\subfloat[Hopper]{%
       \includegraphics[width=0.16\linewidth]{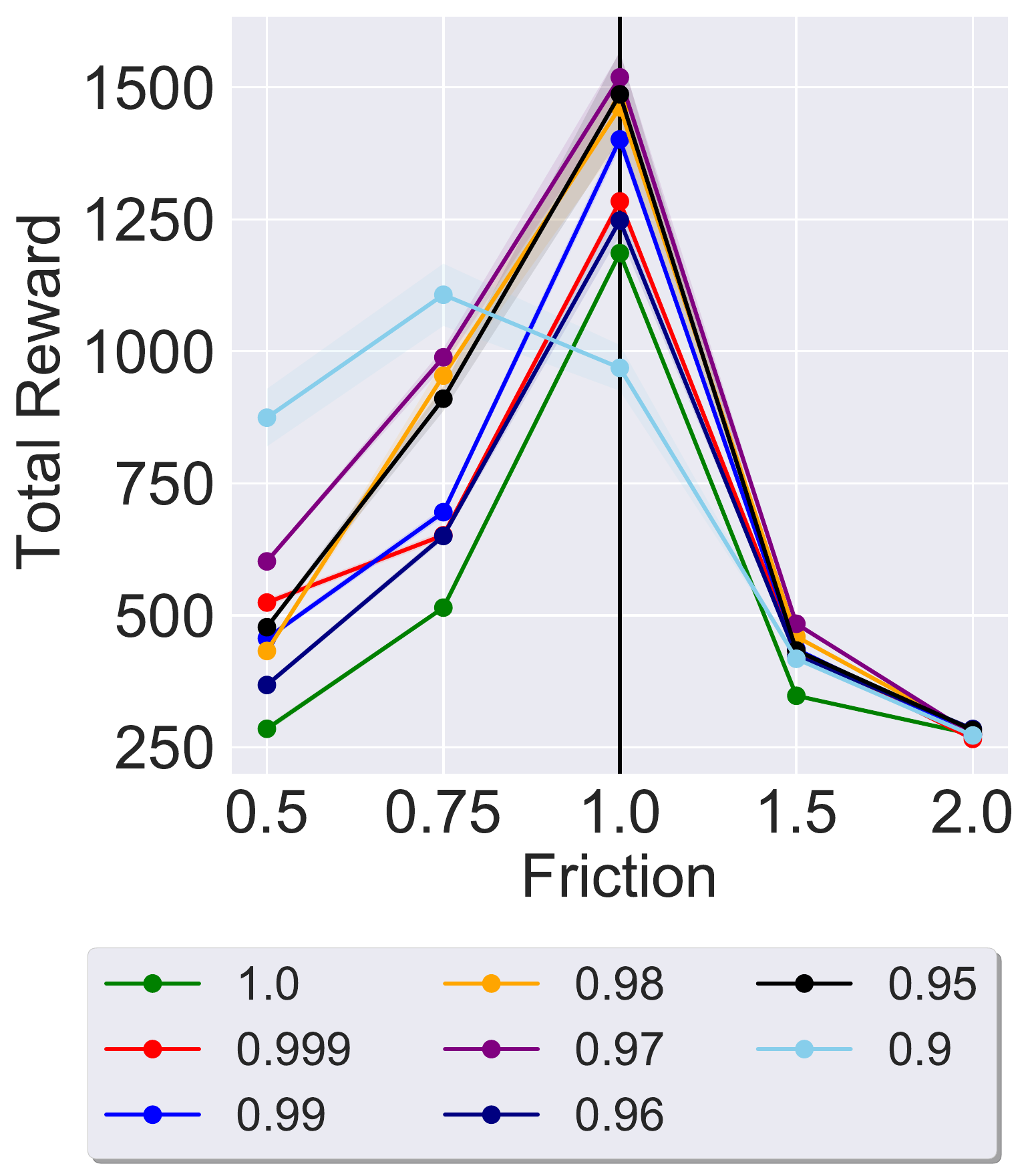}
     } &
\subfloat[Hopper]{%
       \includegraphics[width=0.16\linewidth]{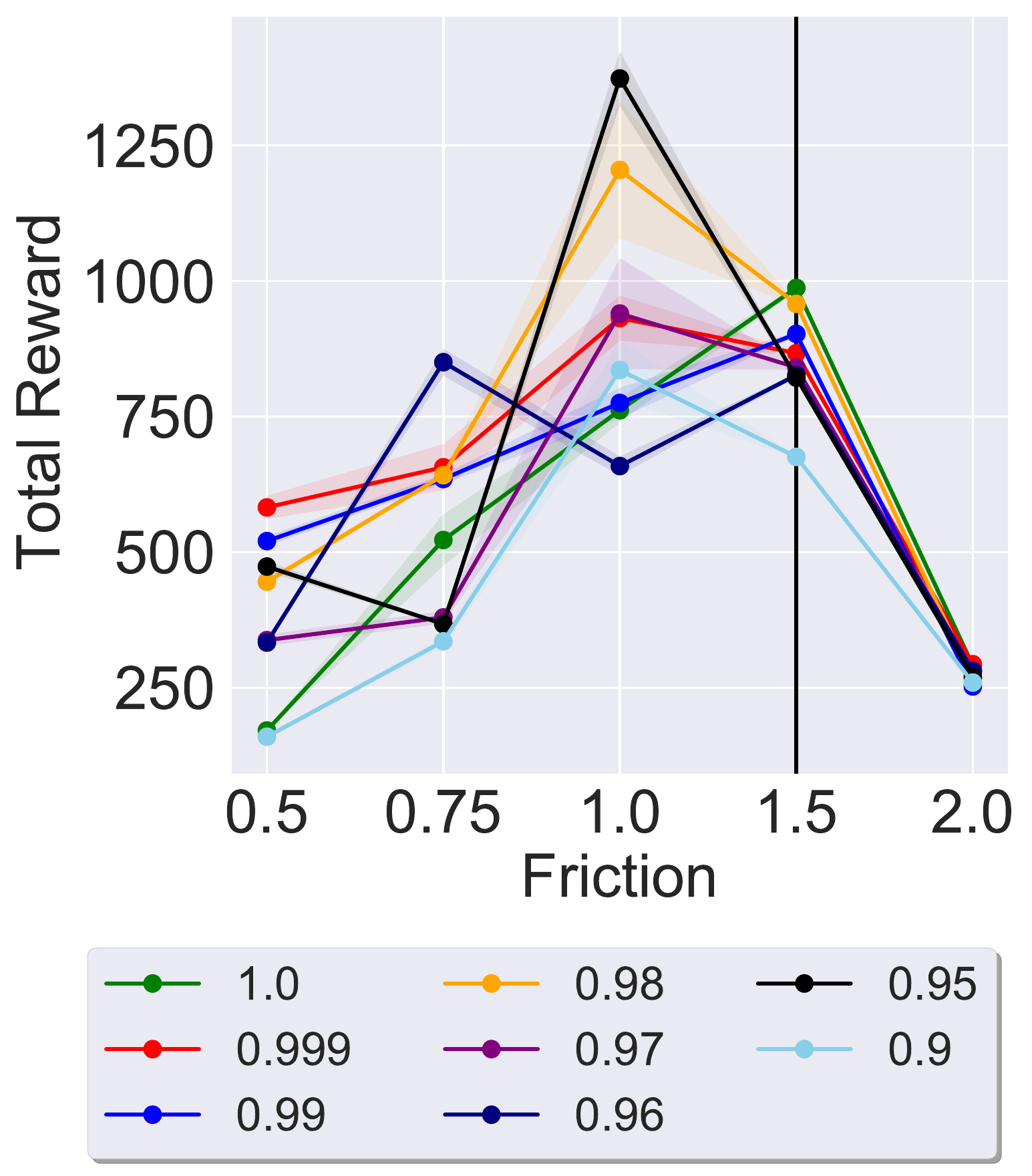}
     } &
\subfloat[Hopper]{%
       \includegraphics[width=0.16\linewidth]{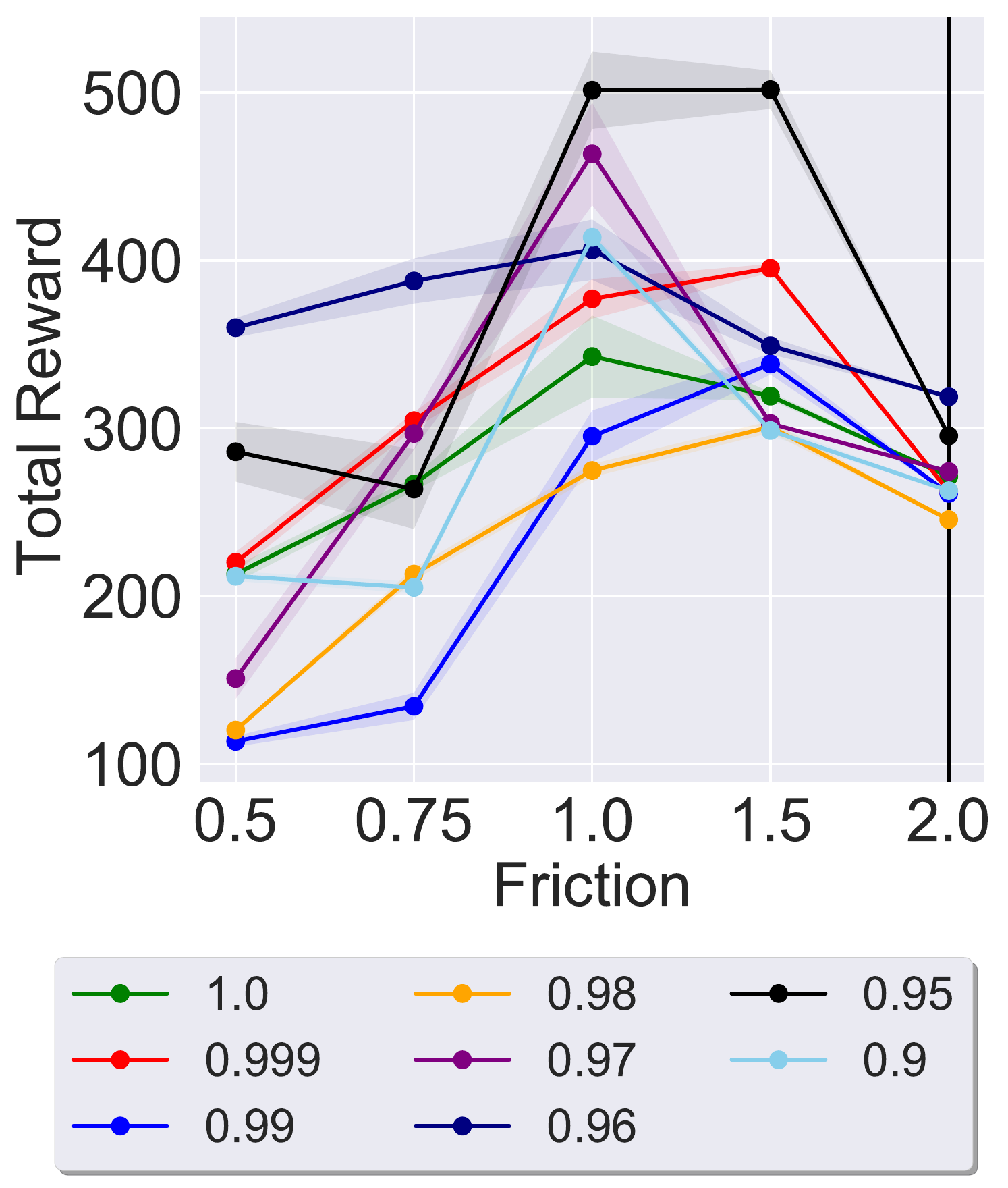}
     } \\
\end{tabular}
\caption{The average (over $3$ seeds) robust performance of Algorithm~\ref{alg:robust-gailfo} with different values of $\alpha$. The ablation shown here is used to choose $\alpha$ in Figure~\ref{fig:RobustnessFrictionFixedAlpha}. The expert environment $M^\mathrm{real}$, in which the demonstrations are collected, has relative friction $1.0$. In each plot, the black vertical line corresponds to the relative friction of the learner environment $M^\mathrm{sim}$ where we trained the policy with Algorithm~\ref{alg:robust-gailfo}. The x-axis denotes the relative friction of the test environment $M^\mathrm{test}$ in which the policies are evaluated. The policies are evaluated over $1e5$ steps truncating the last episode if it does not terminate. Note that robust-GAILfO with $\alpha = 1$ corresponds to GAILfO.}
\label{fig:RobustnessFrictionAblation}
\end{figure*}

\begin{figure*}[!h] 
\centering
\begin{tabular}{ccccc}
\subfloat[HalfCheetah]{%
       \includegraphics[width=0.16\linewidth]{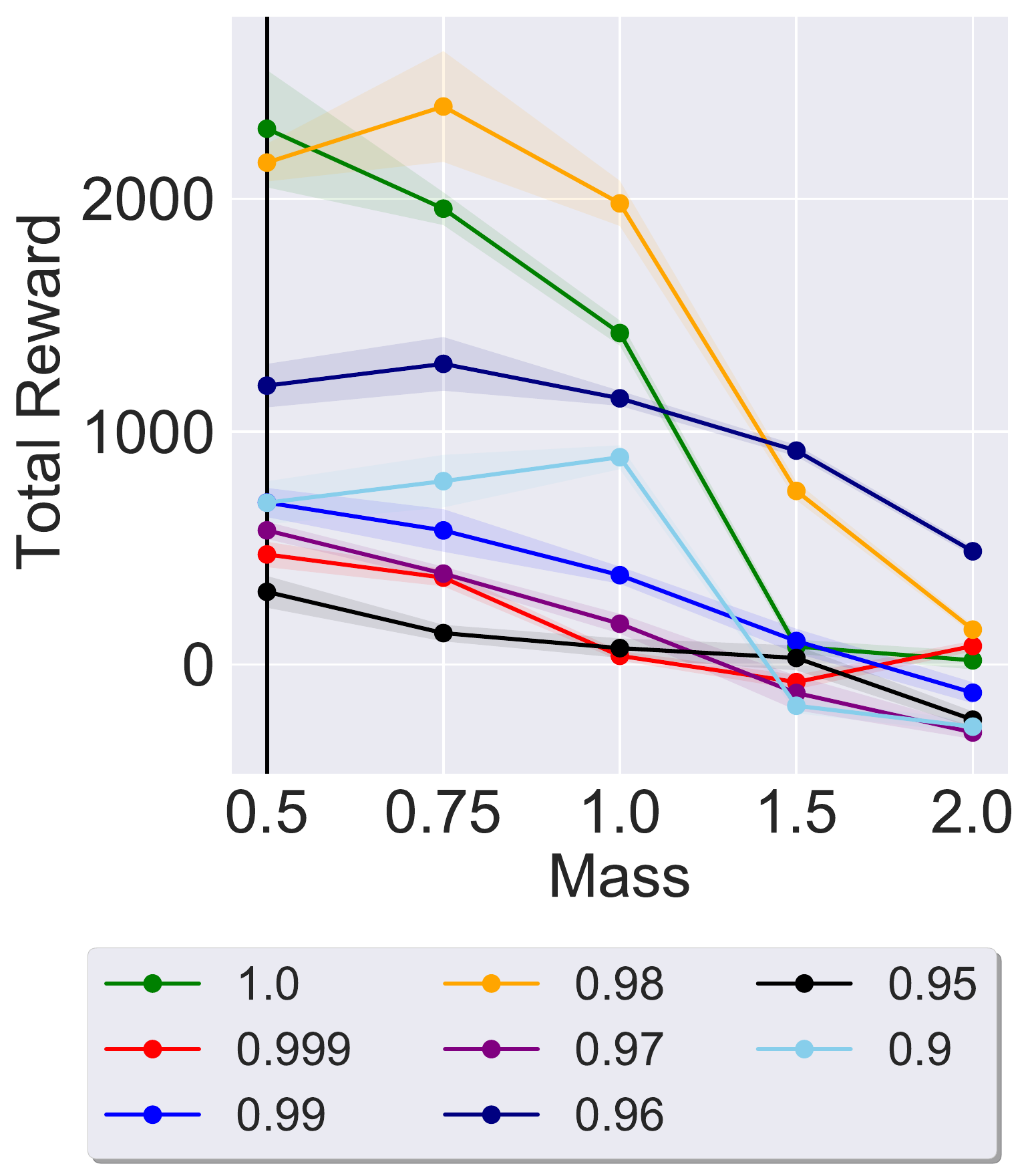}
     } &
\subfloat[HalfCheetah]{%
       \includegraphics[width=0.16\linewidth]{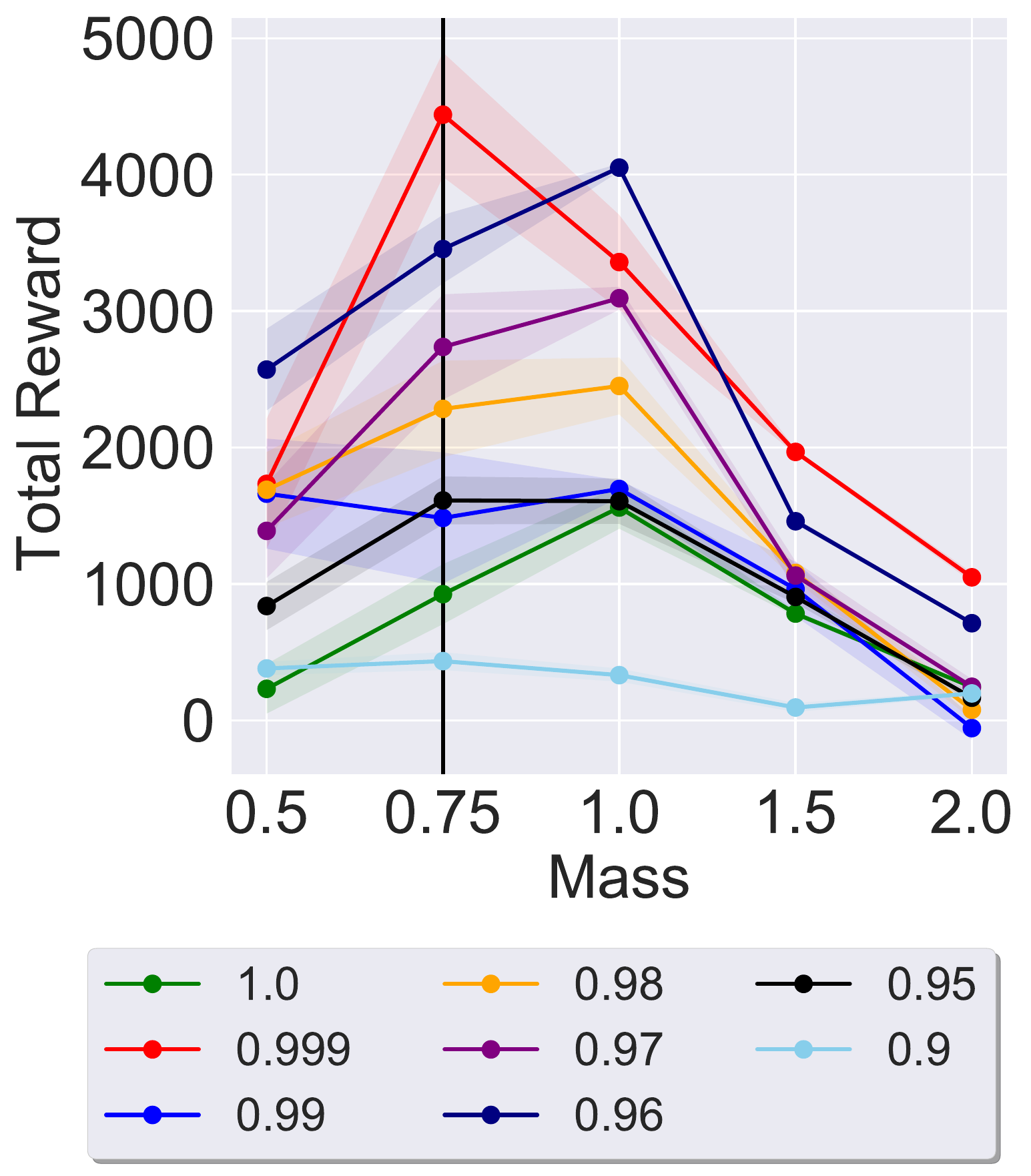}
     } &
\subfloat[HalfCheetah]{%
       \includegraphics[width=0.16\linewidth]{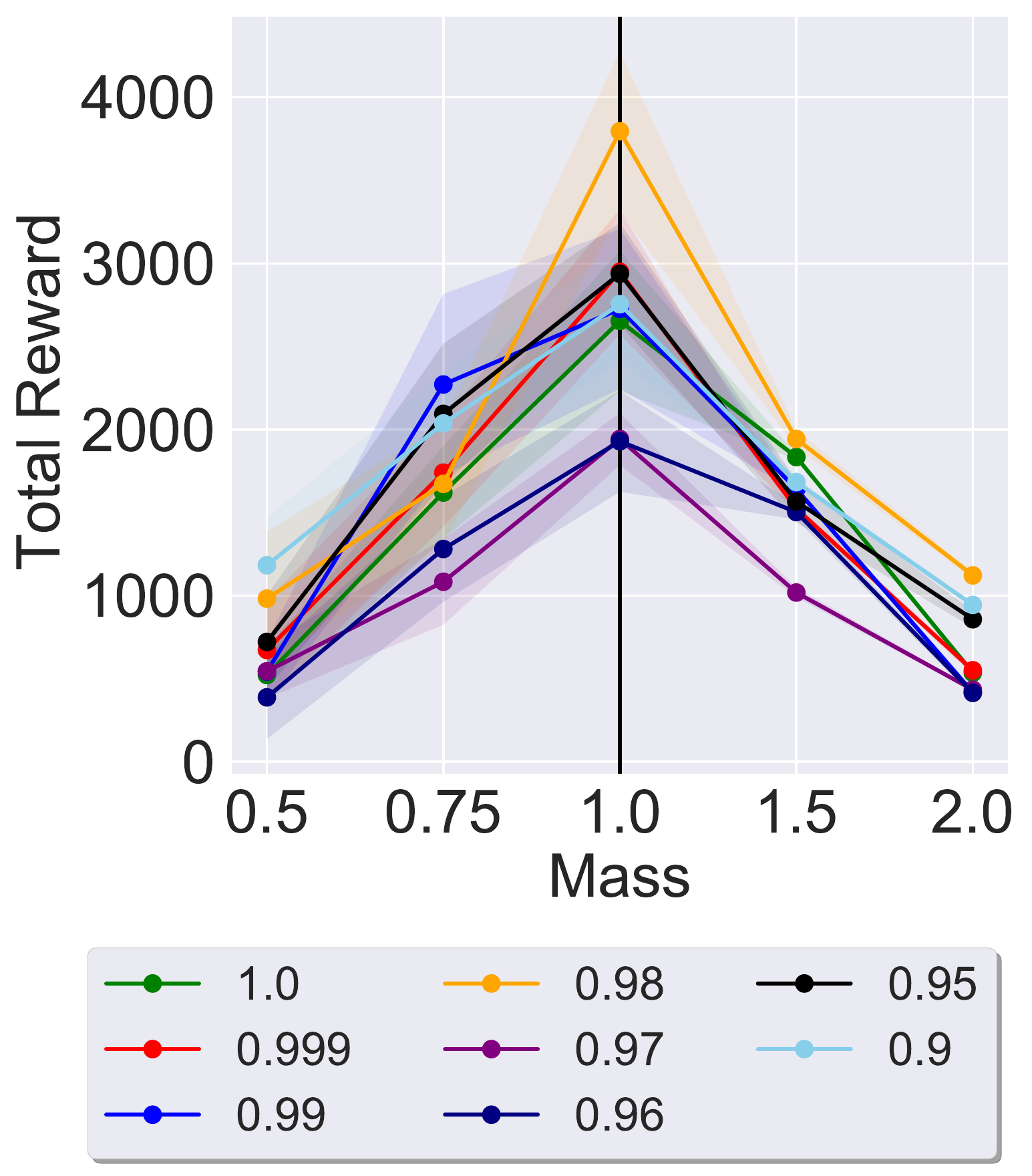}
     } &
\subfloat[HalfCheetah]{%
       \includegraphics[width=0.16\linewidth]{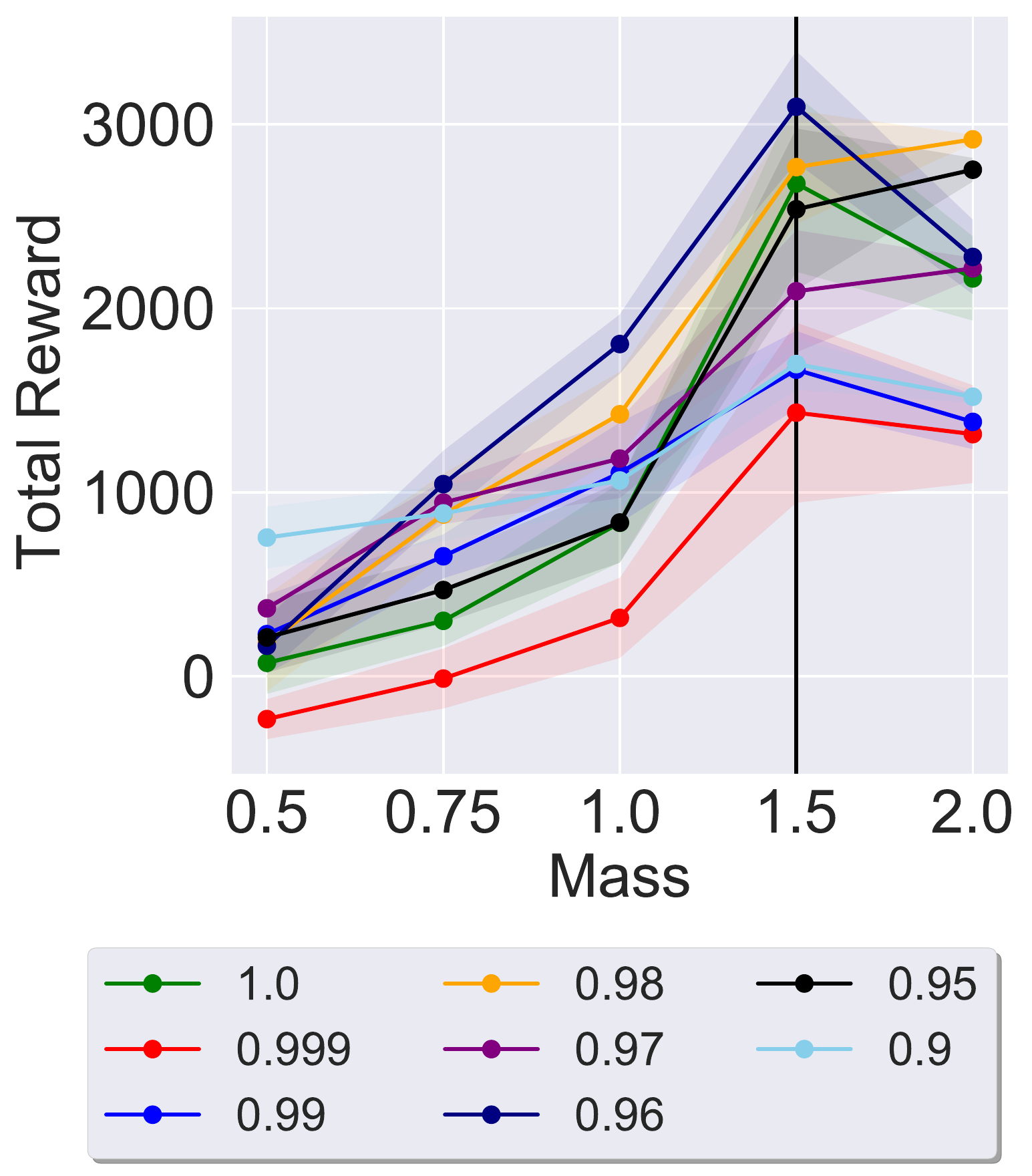}
     } &
\subfloat[HalfCheetah]{%
       \includegraphics[width=0.16\linewidth]{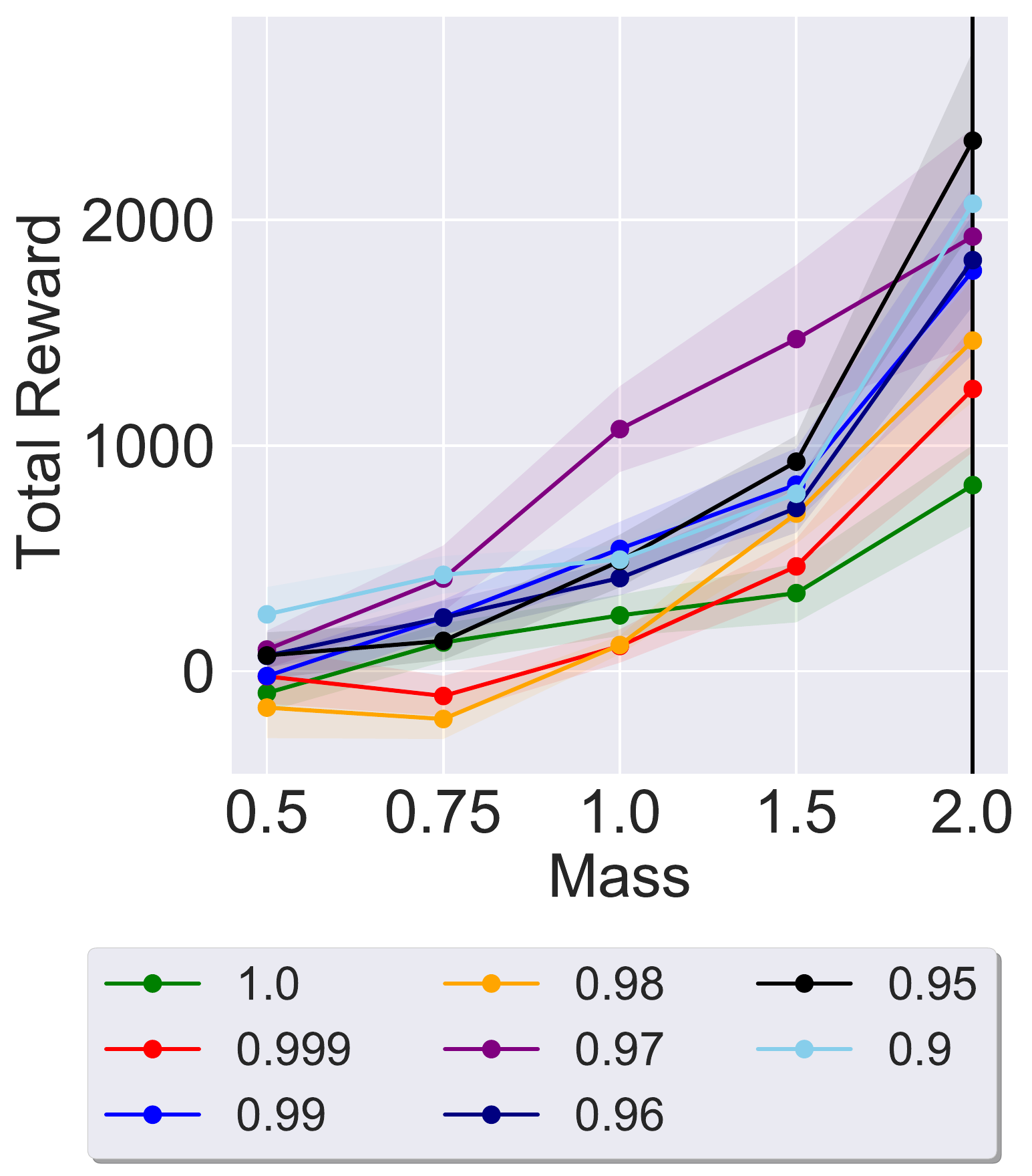}
     } \\
\subfloat[Walker]{%
       \includegraphics[width=0.16\linewidth]{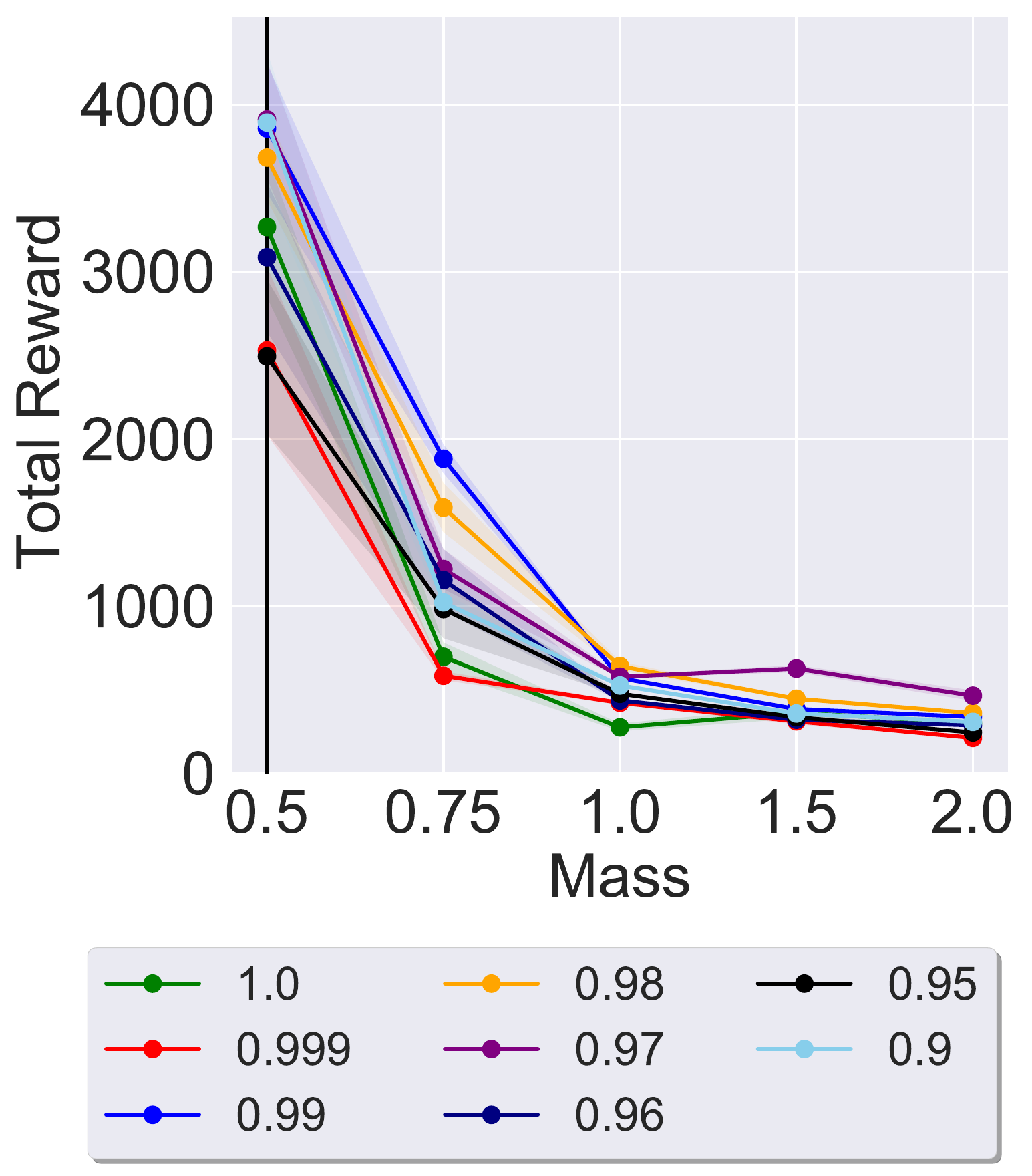}
     } &
\subfloat[Walker]{%
       \includegraphics[width=0.16\linewidth]{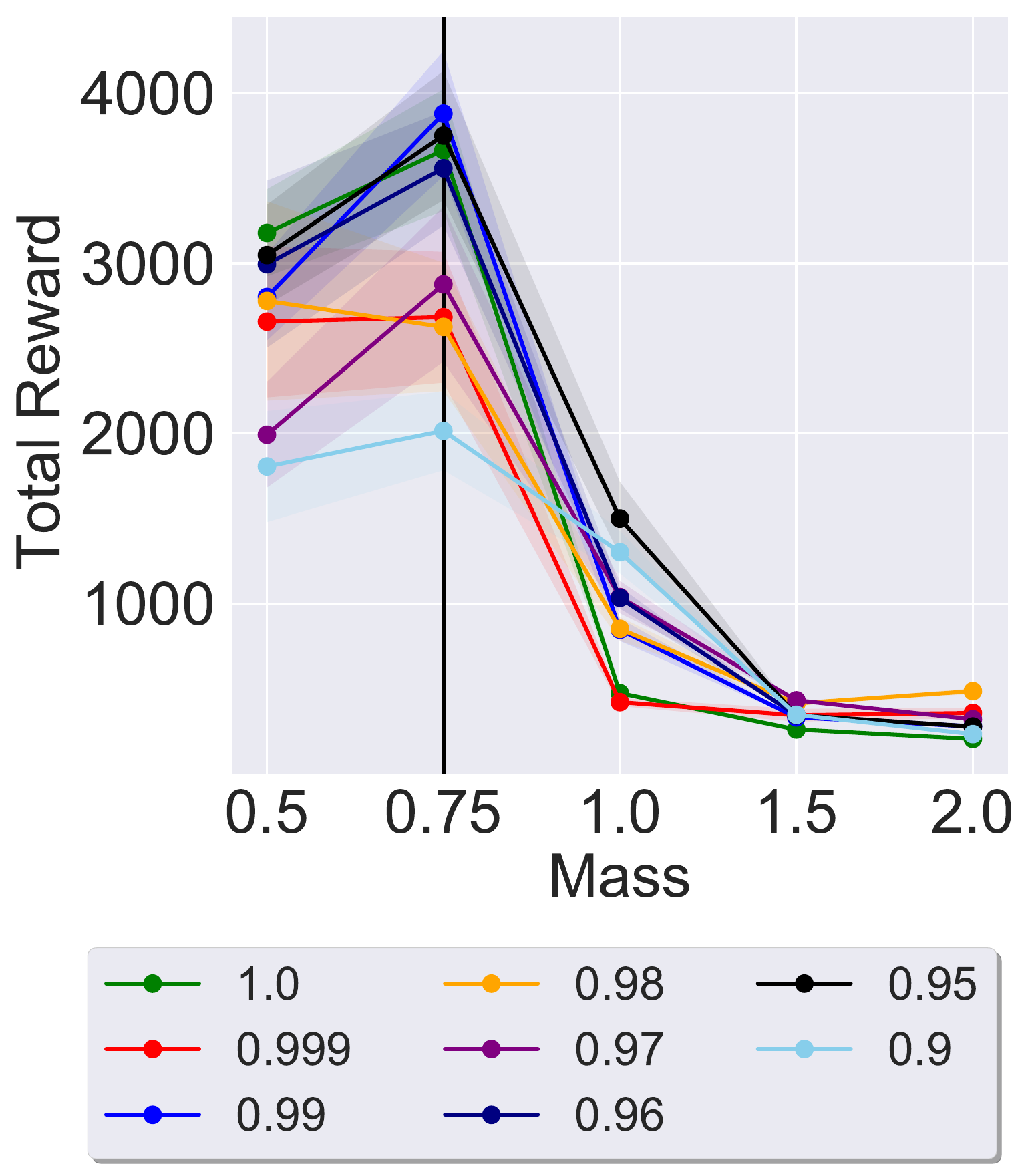}
     } &
\subfloat[Walker]{%
       \includegraphics[width=0.16\linewidth]{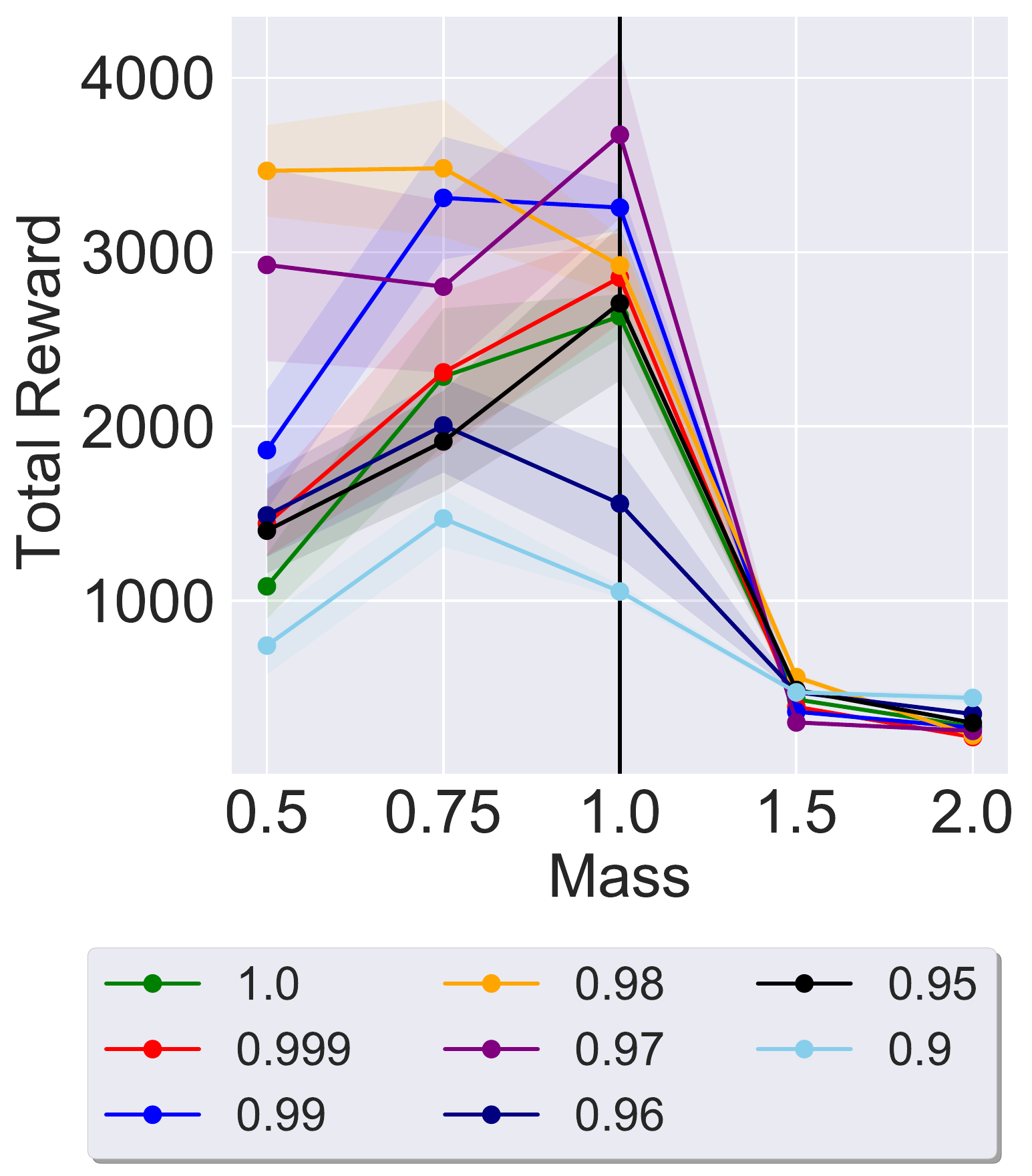}
     } &
\subfloat[Walker]{%
       \includegraphics[width=0.16\linewidth]{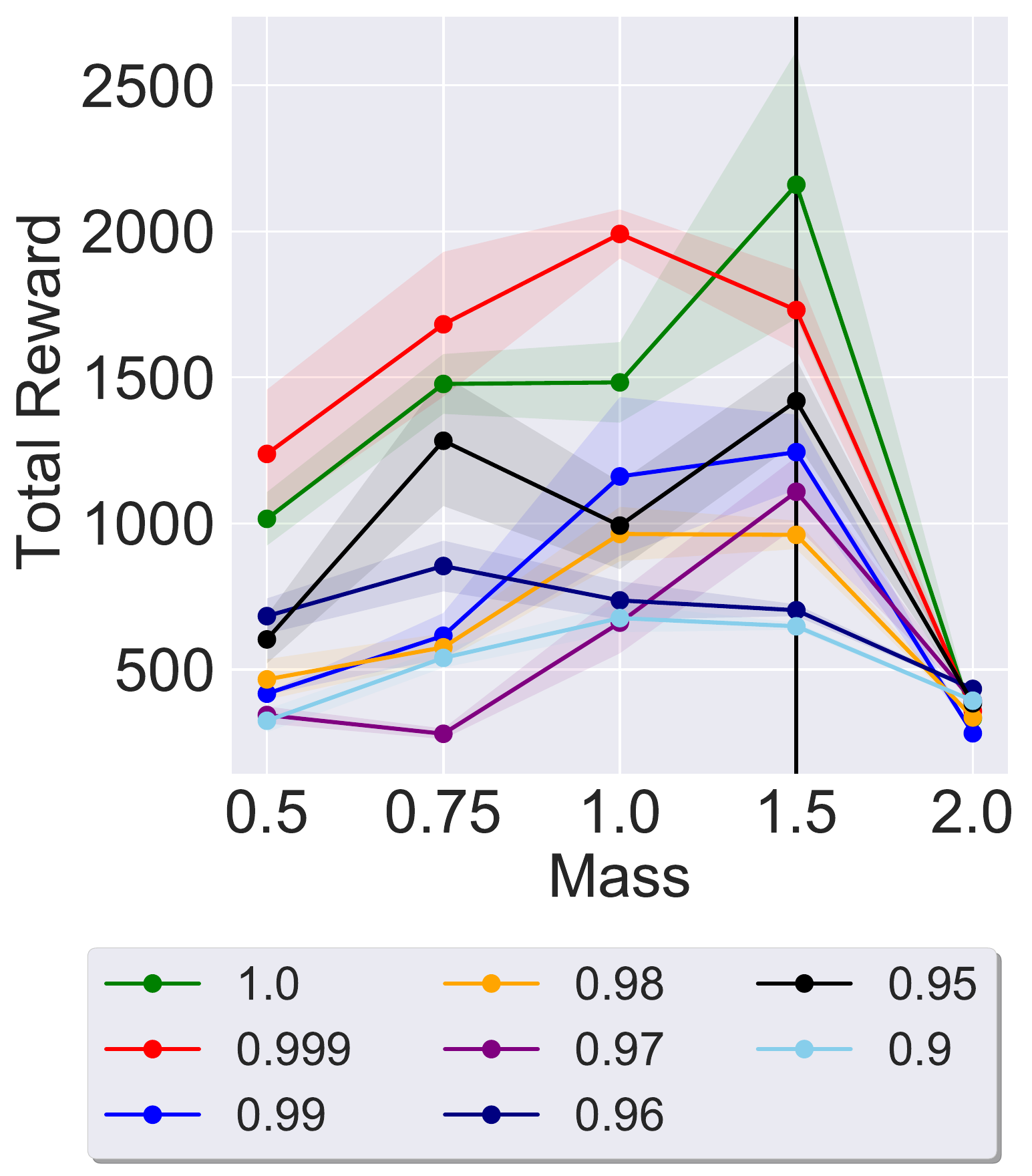}
     } &
\subfloat[Walker]{%
       \includegraphics[width=0.16\linewidth]{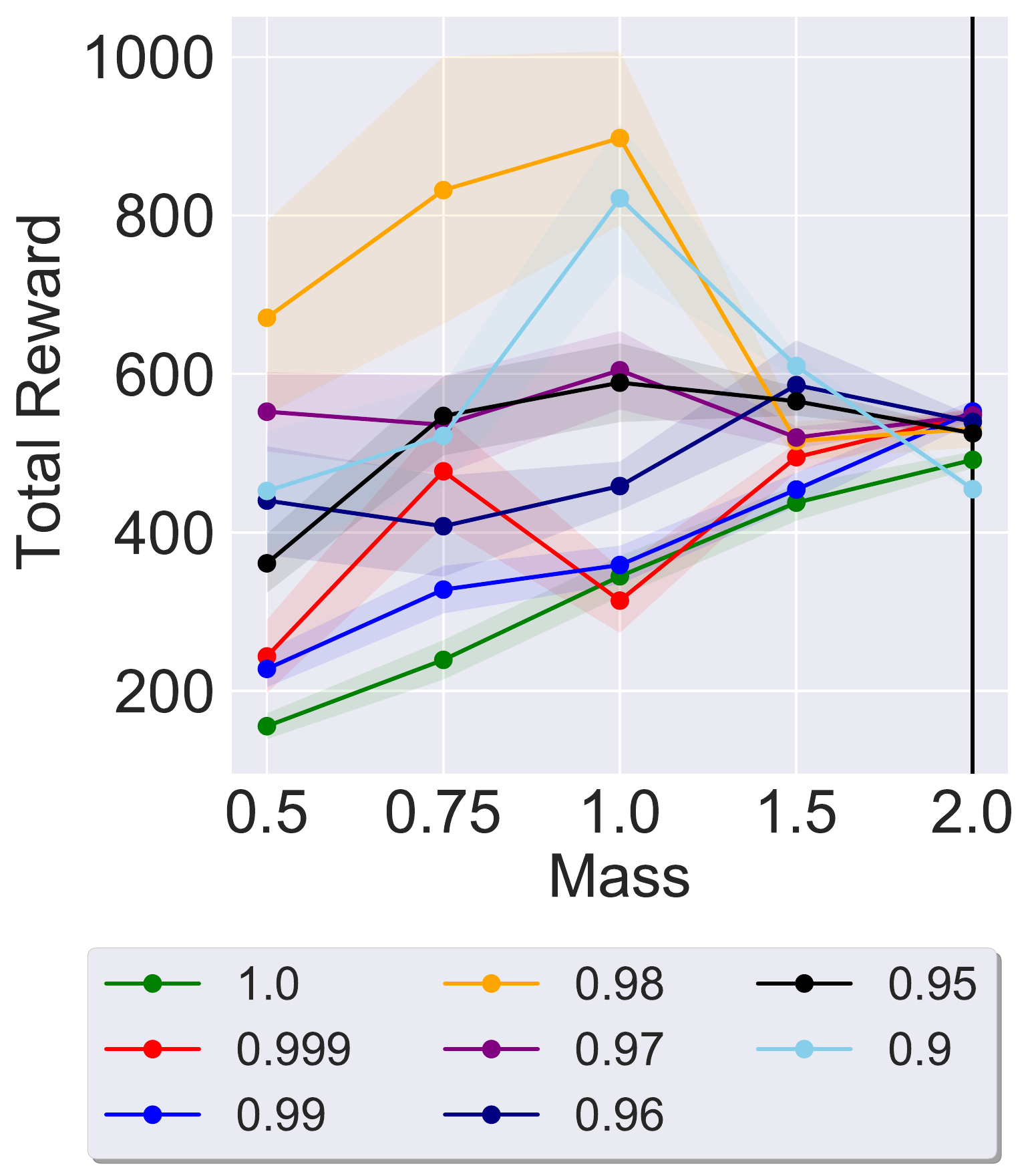}
     } \\
\subfloat[Hopper]{%
       \includegraphics[width=0.16\linewidth]{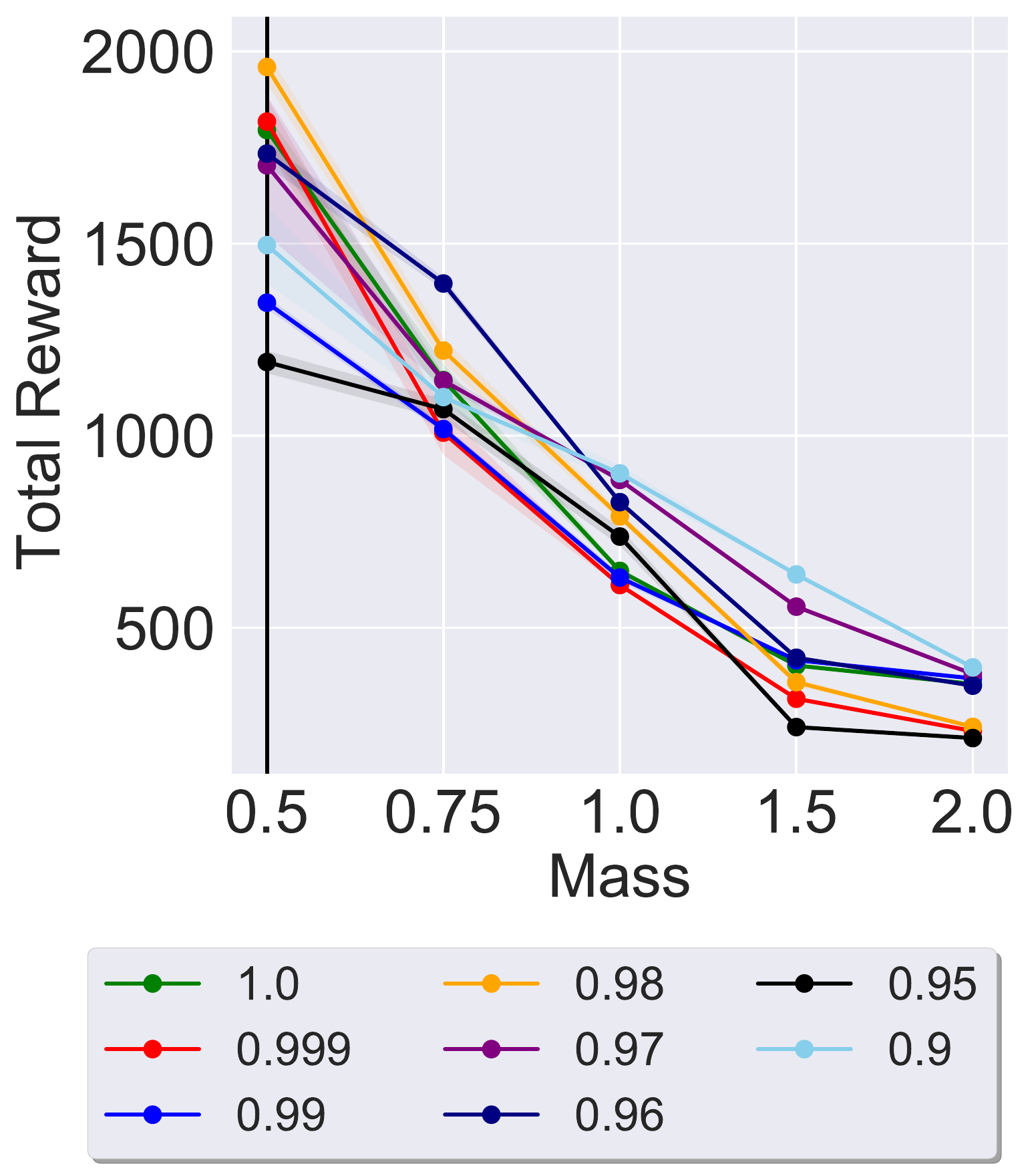}
     } &
\subfloat[Hopper]{%
       \includegraphics[width=0.16\linewidth]{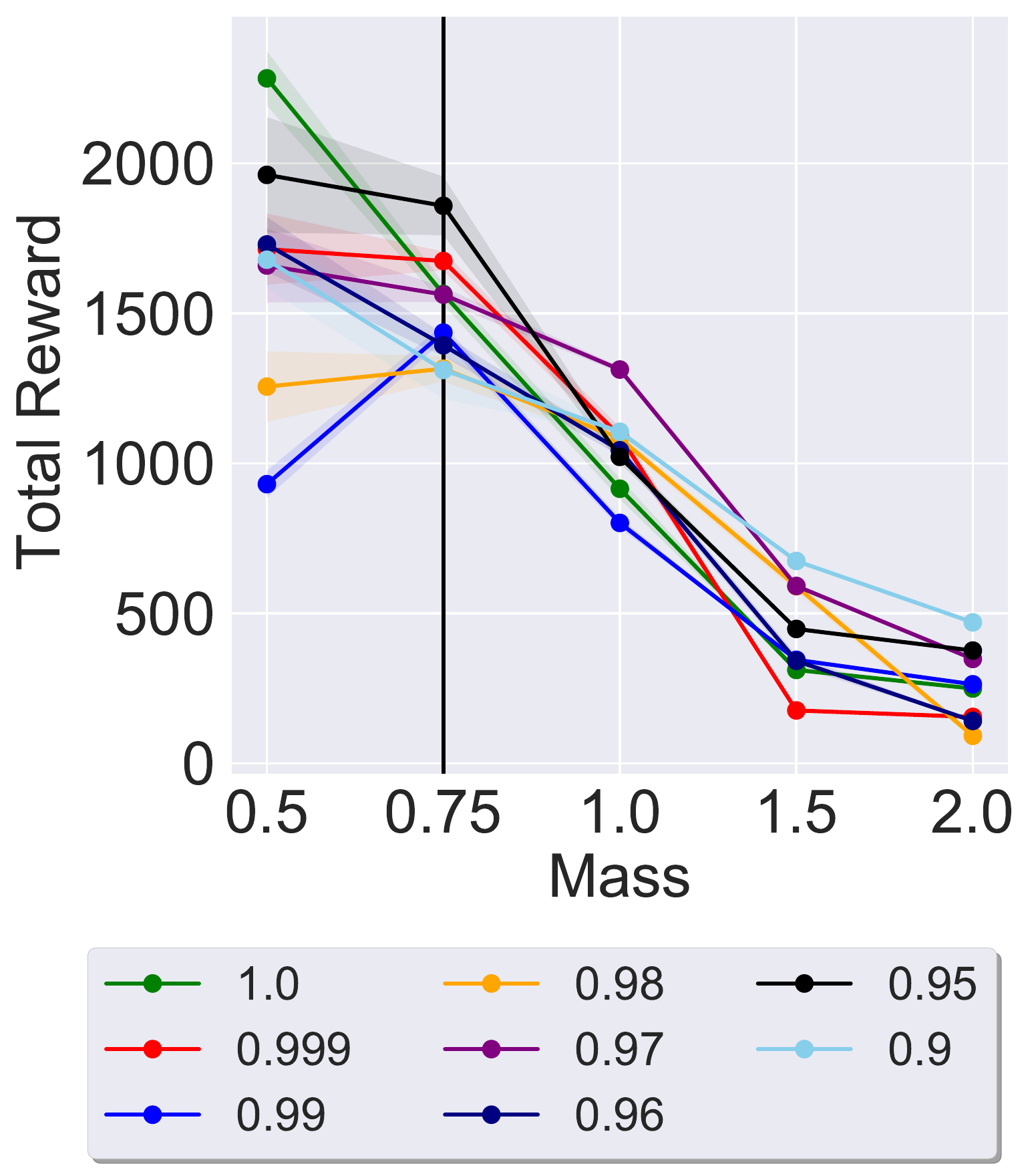}
     } &
\subfloat[Hopper]{%
       \includegraphics[width=0.16\linewidth]{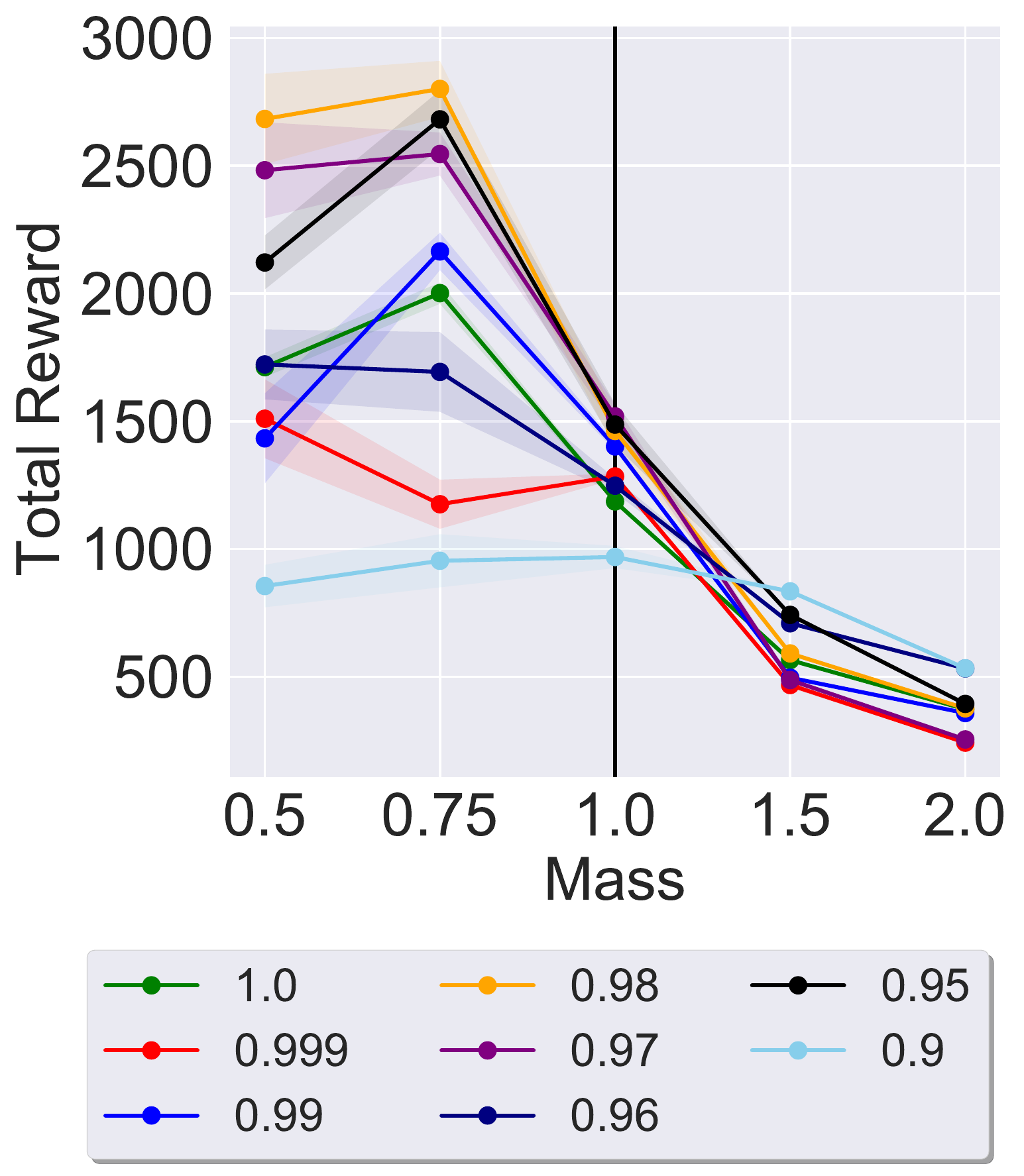}
     } &
\subfloat[Hopper]{%
       \includegraphics[width=0.16\linewidth]{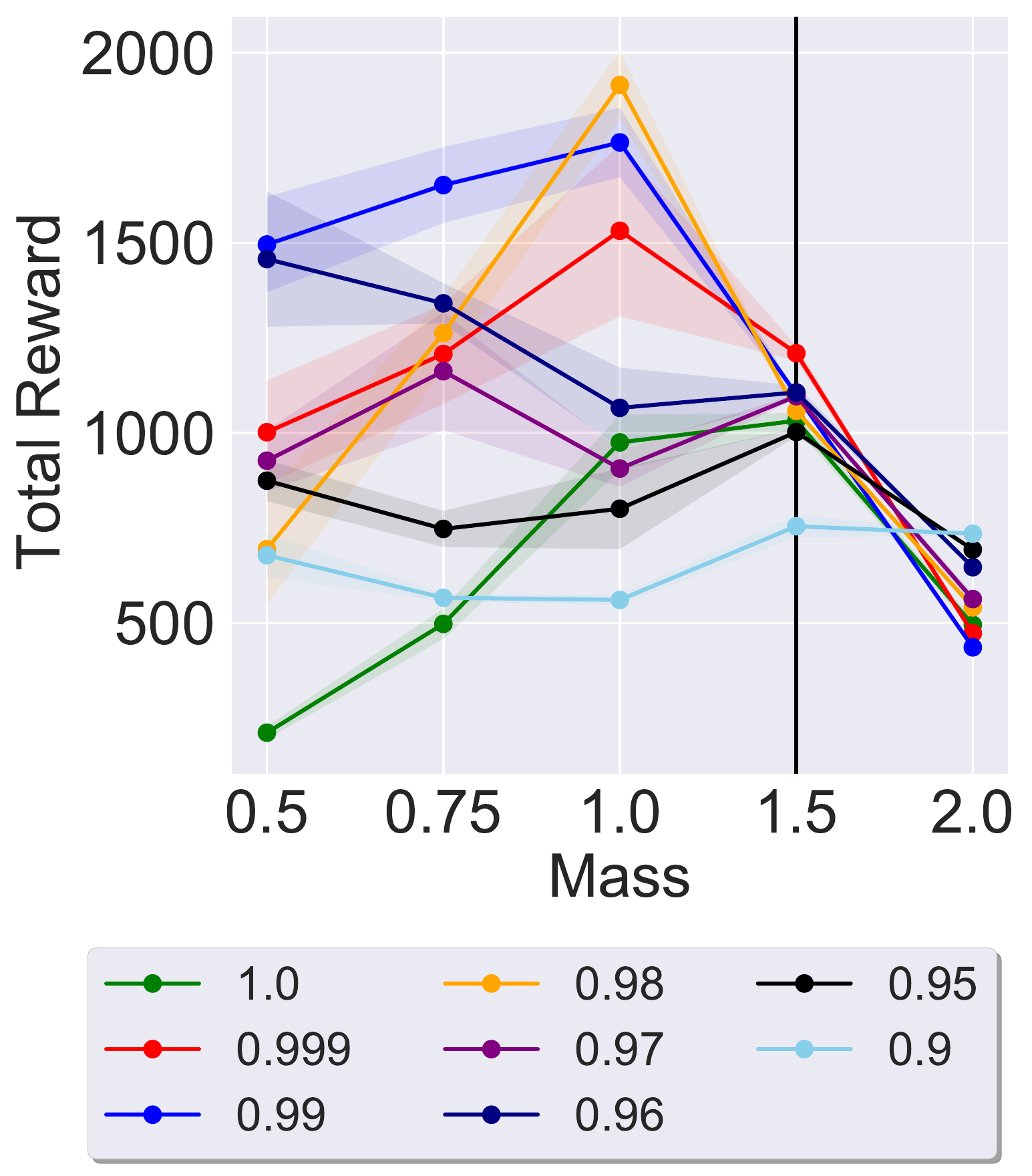}
     } &
\subfloat[Hopper]{%
       \includegraphics[width=0.16\linewidth]{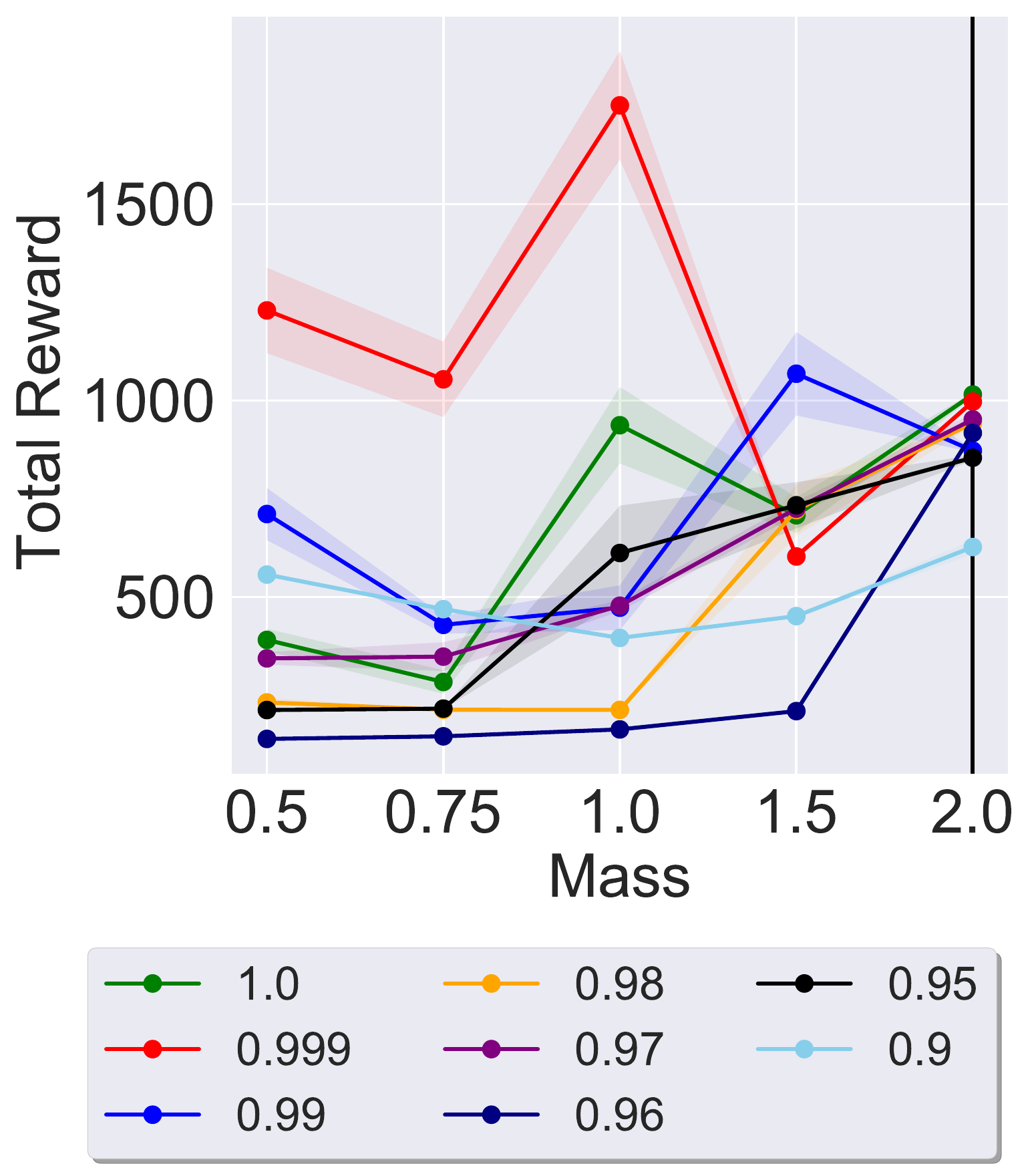}
     } \\
\subfloat[InvDoublePend]{%
       \includegraphics[width=0.16\linewidth]{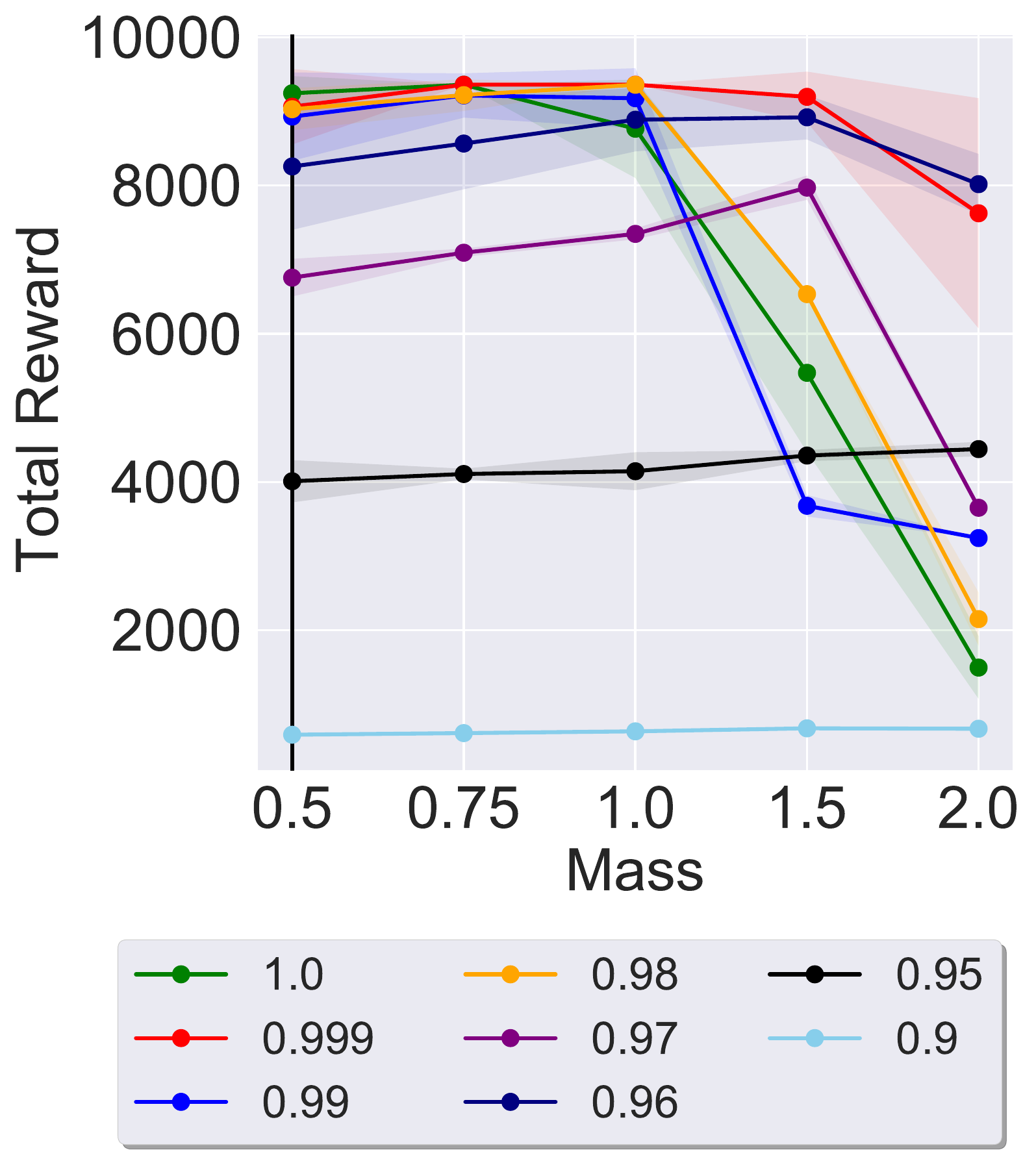}
     } &
\subfloat[InvDoublePend]{%
       \includegraphics[width=0.16\linewidth]{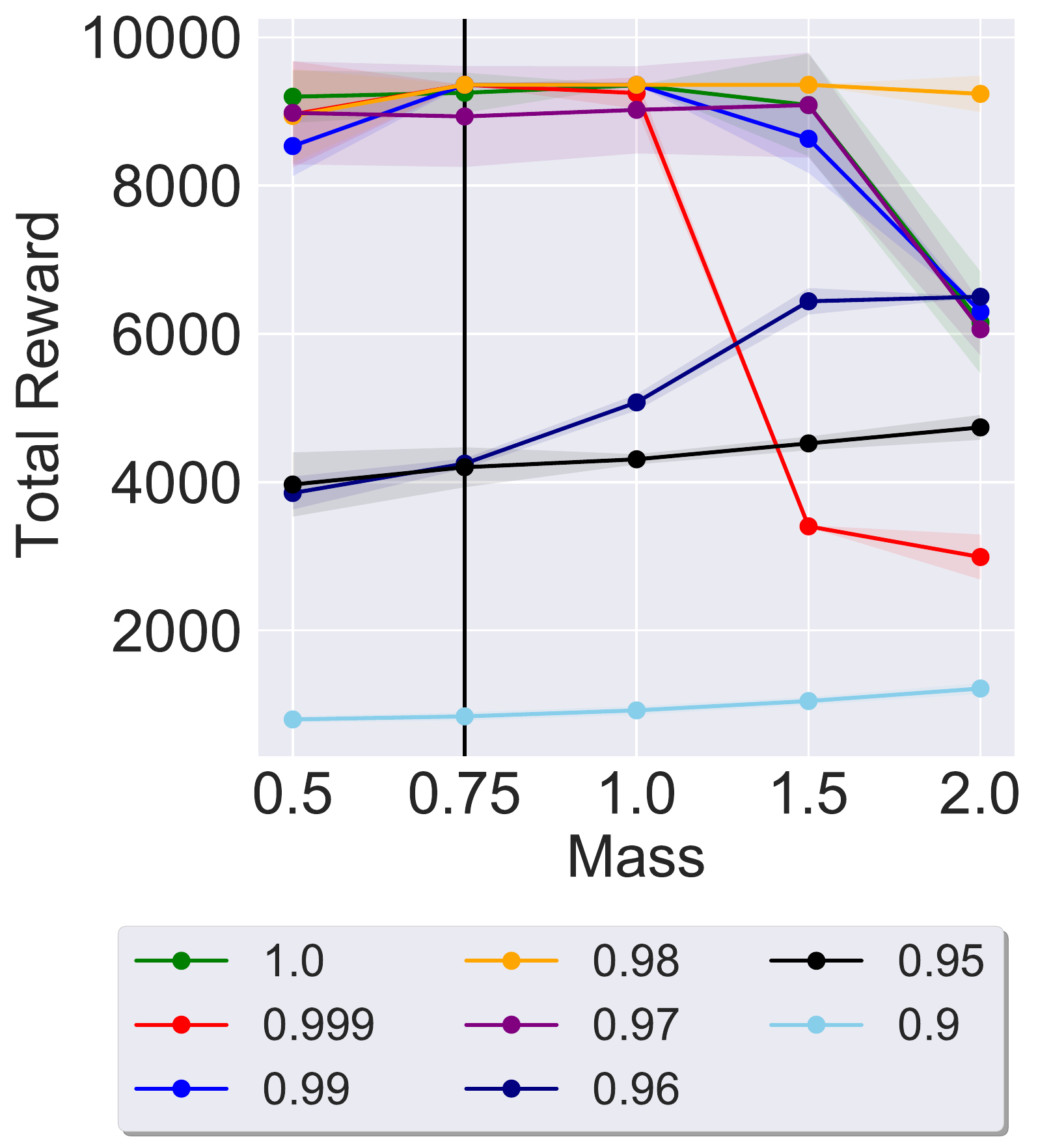}
     } &
\subfloat[InvDoublePend]{%
       \includegraphics[width=0.16\linewidth]{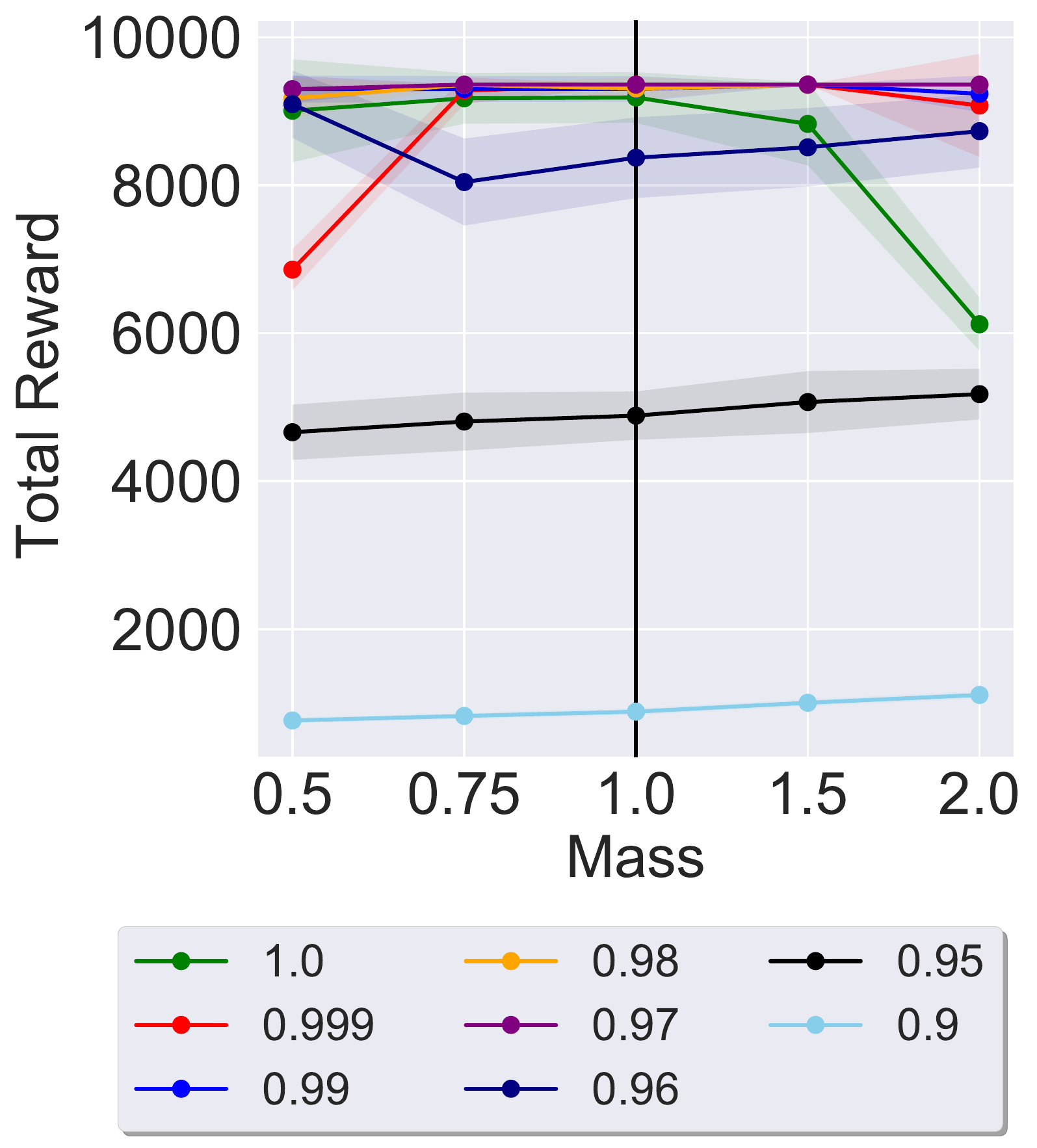}
     } &
\subfloat[InvDoublePend]{%
       \includegraphics[width=0.16\linewidth]{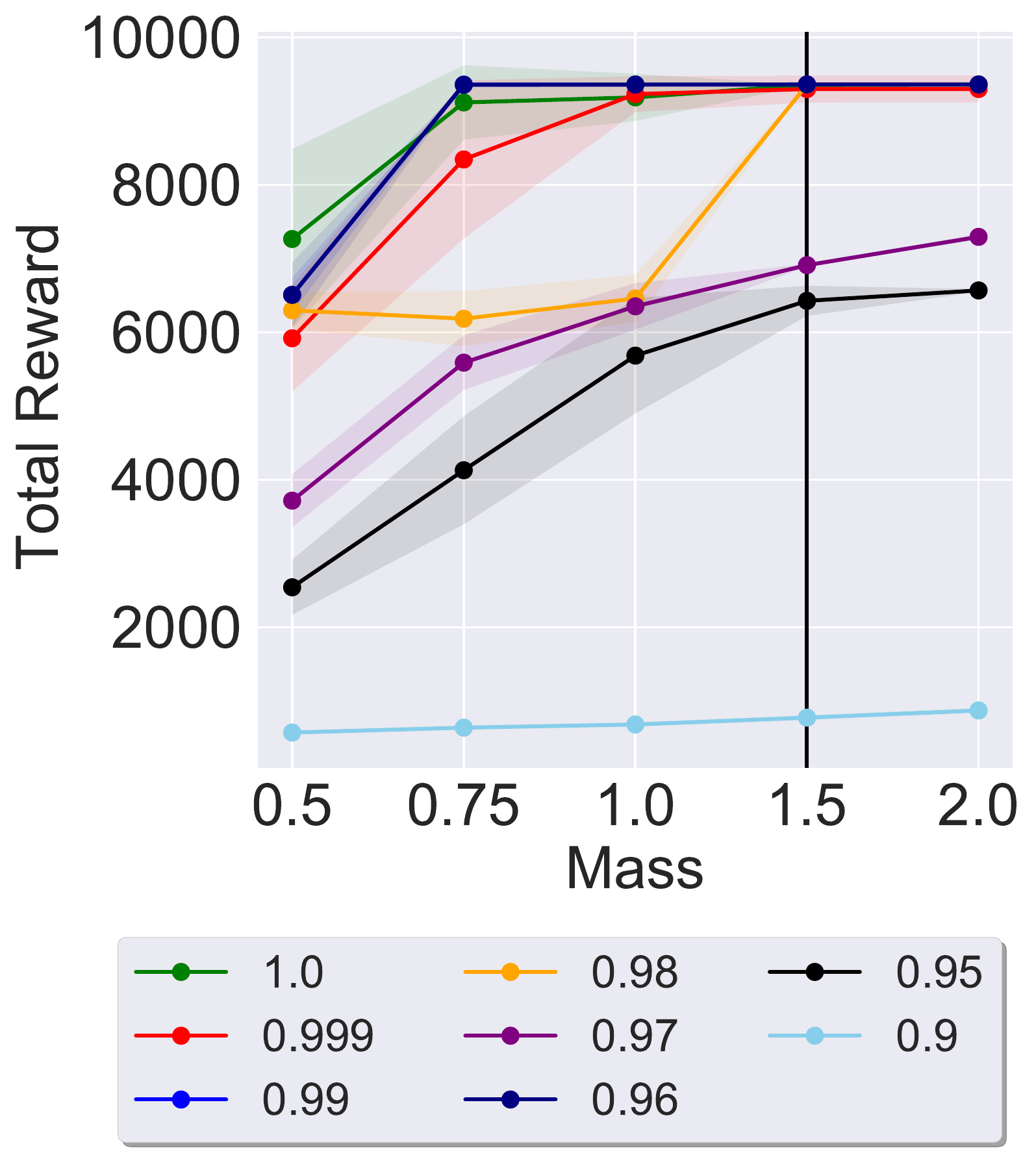}
     } &
\subfloat[InvDoublePend]{%
       \includegraphics[width=0.16\linewidth]{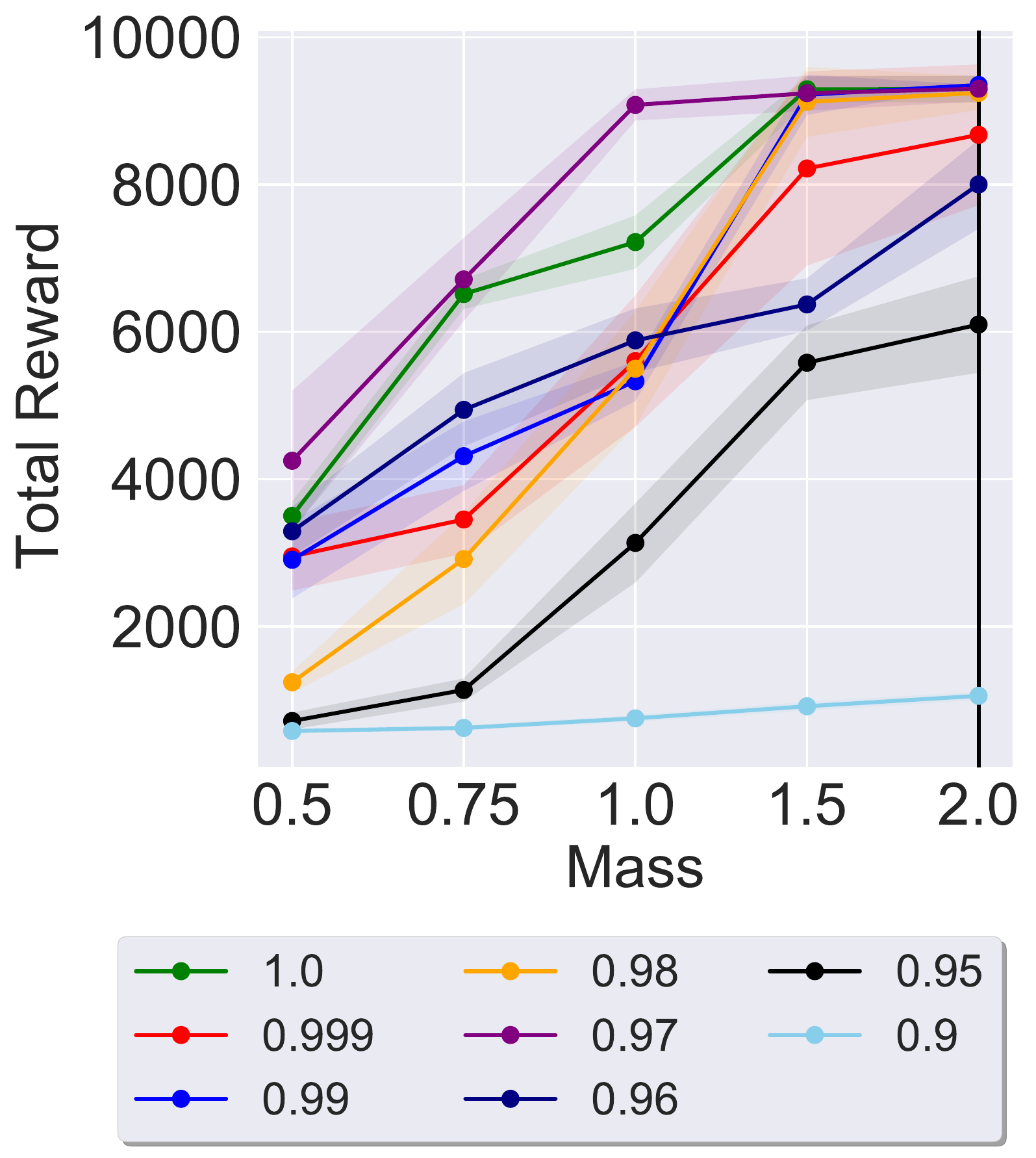}
     } \\
\subfloat[Swimmer]{%
       \includegraphics[width=0.16\linewidth]{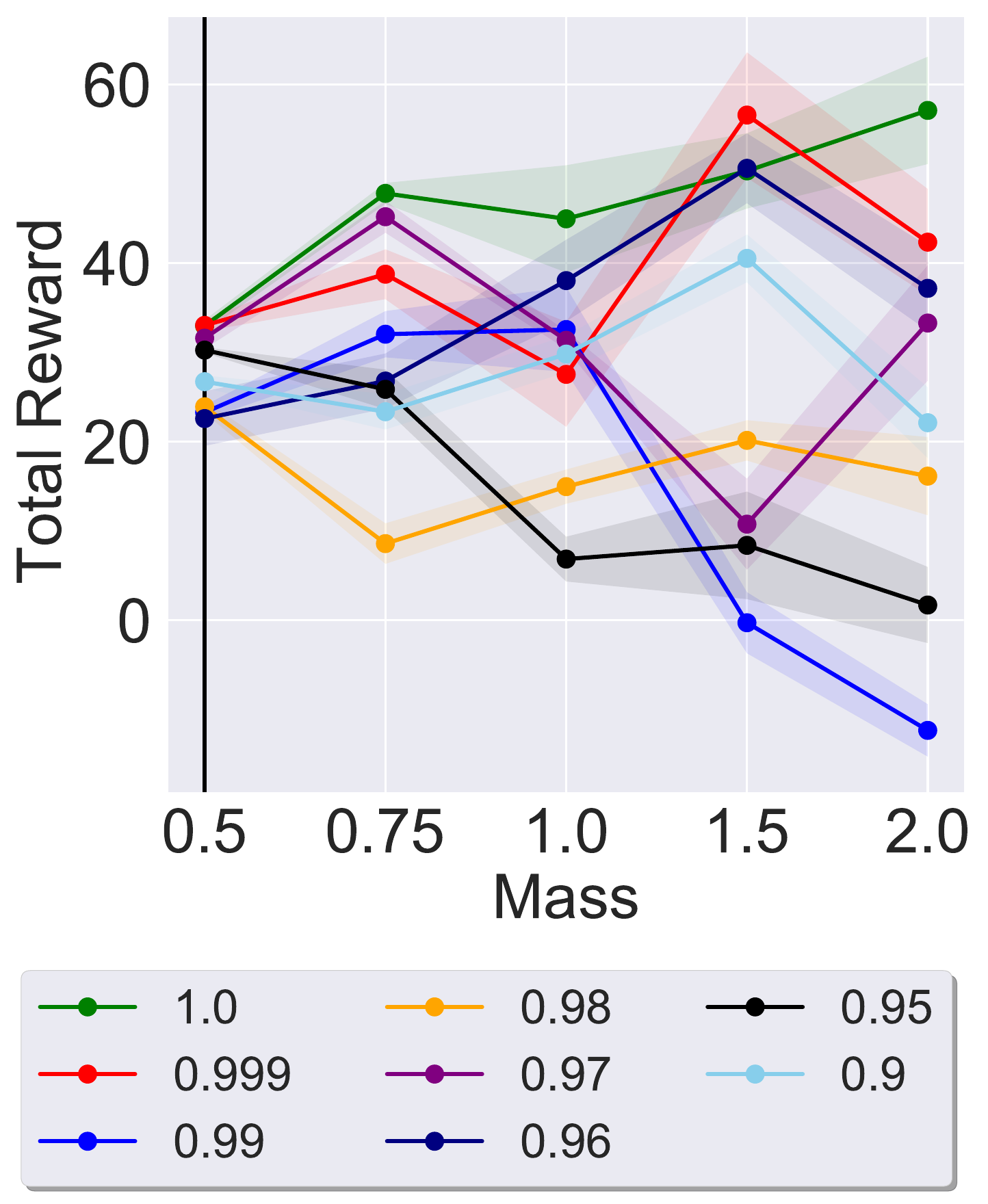}
     } &
\subfloat[Swimmer]{%
       \includegraphics[width=0.16\linewidth]{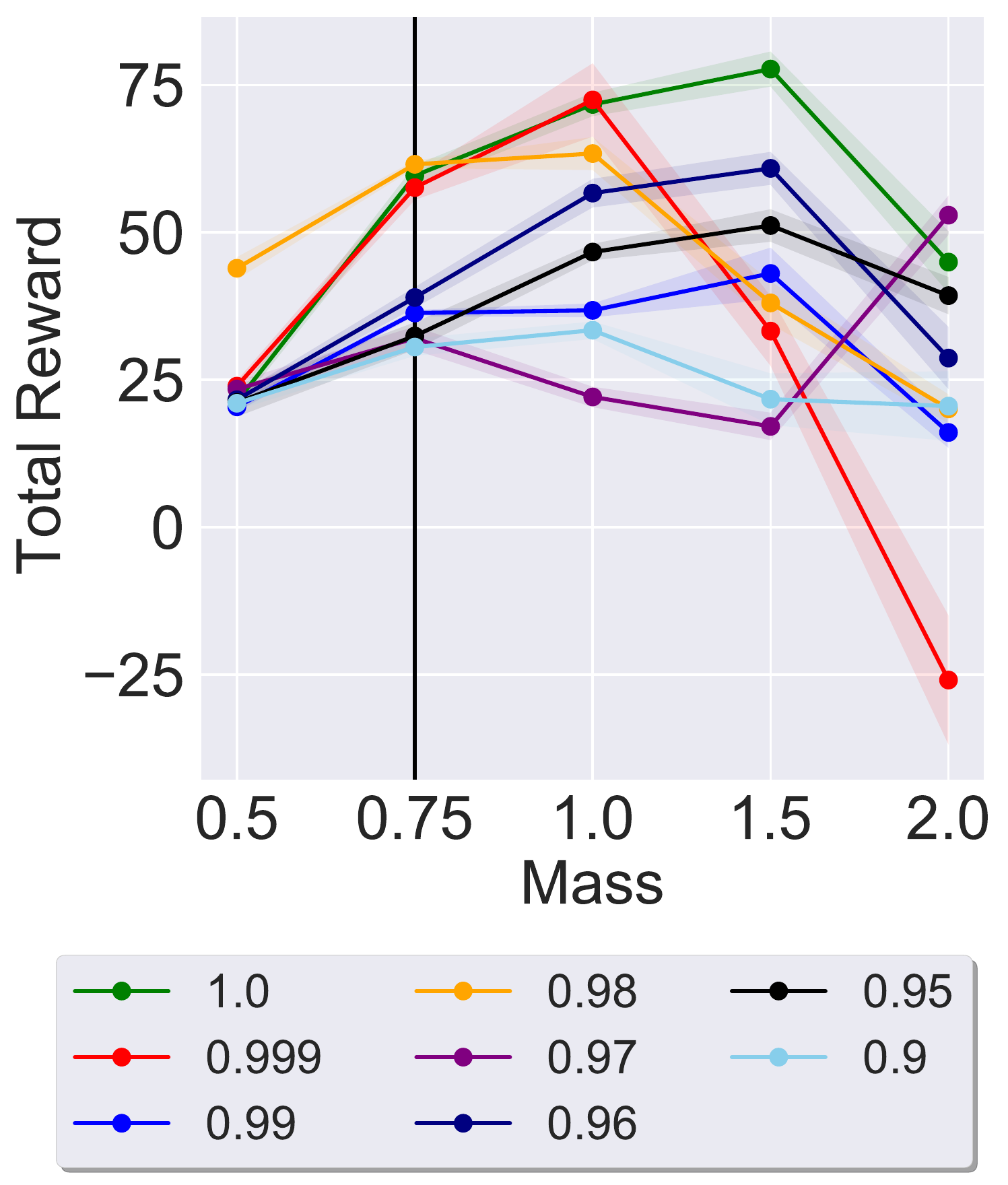}
     } &
\subfloat[Swimmer]{%
       \includegraphics[width=0.16\linewidth]{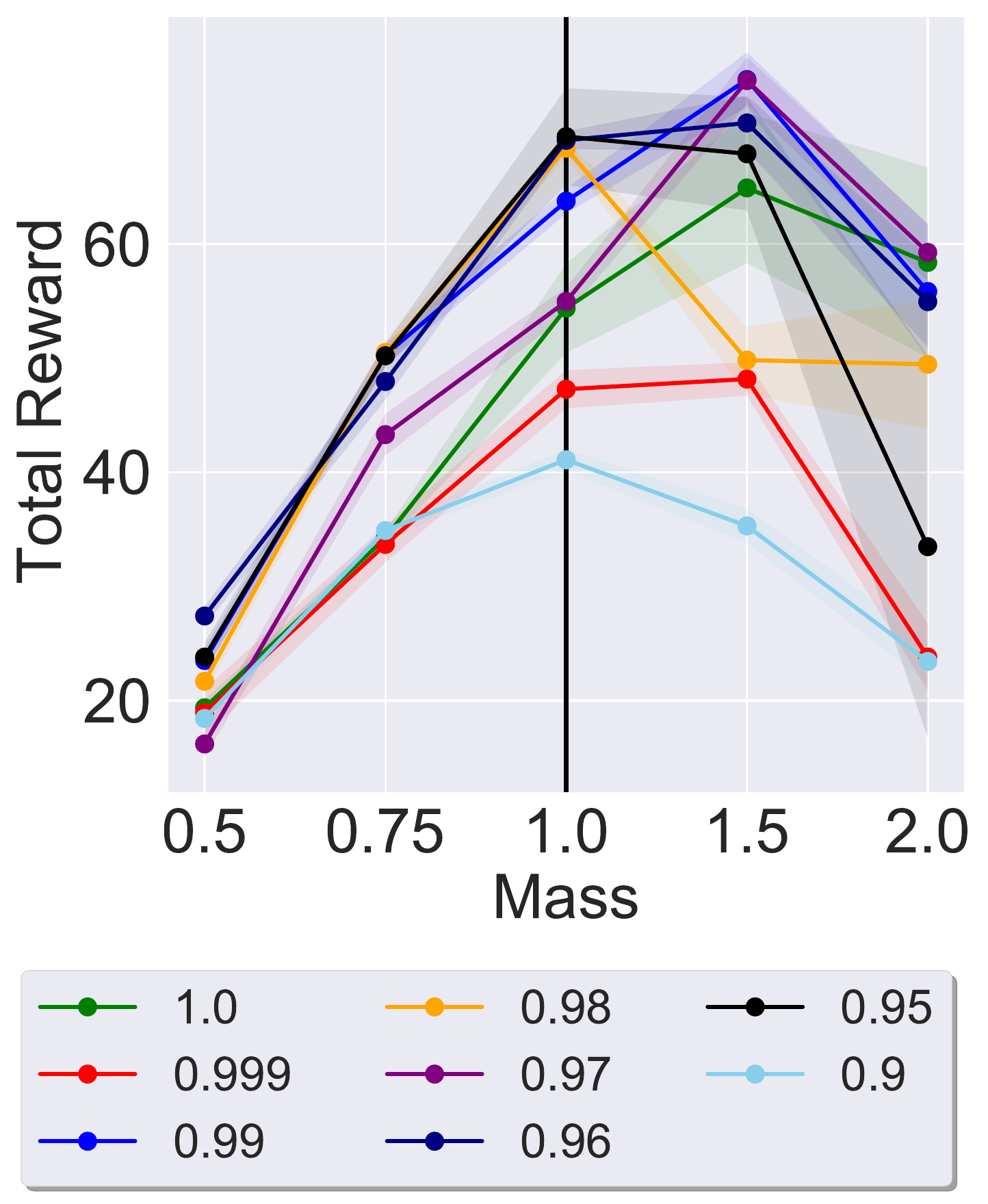}
     } &
\subfloat[Swimmer]{%
       \includegraphics[width=0.16\linewidth]{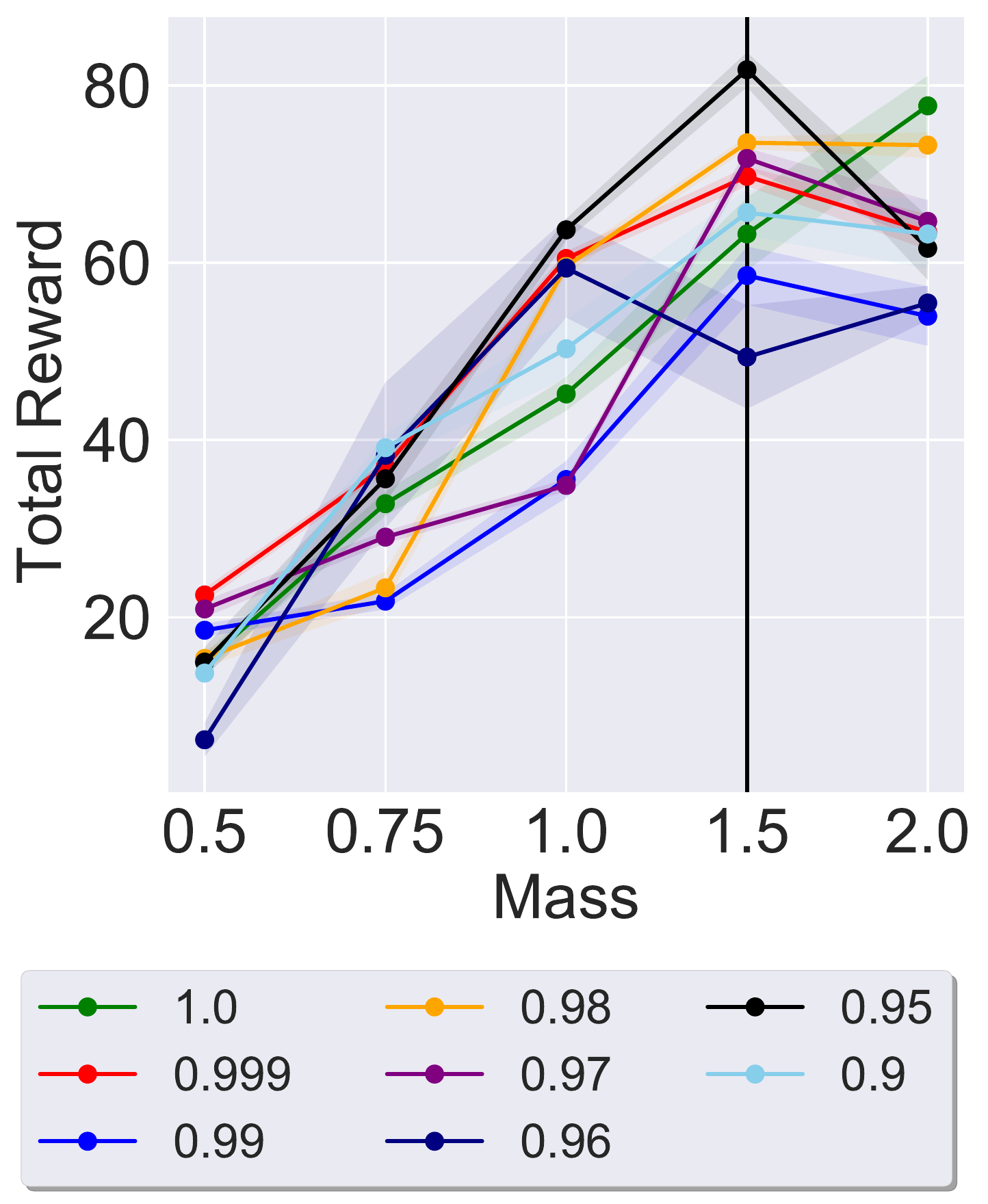}
     } &
\subfloat[Swimmer]{%
       \includegraphics[width=0.16\linewidth]{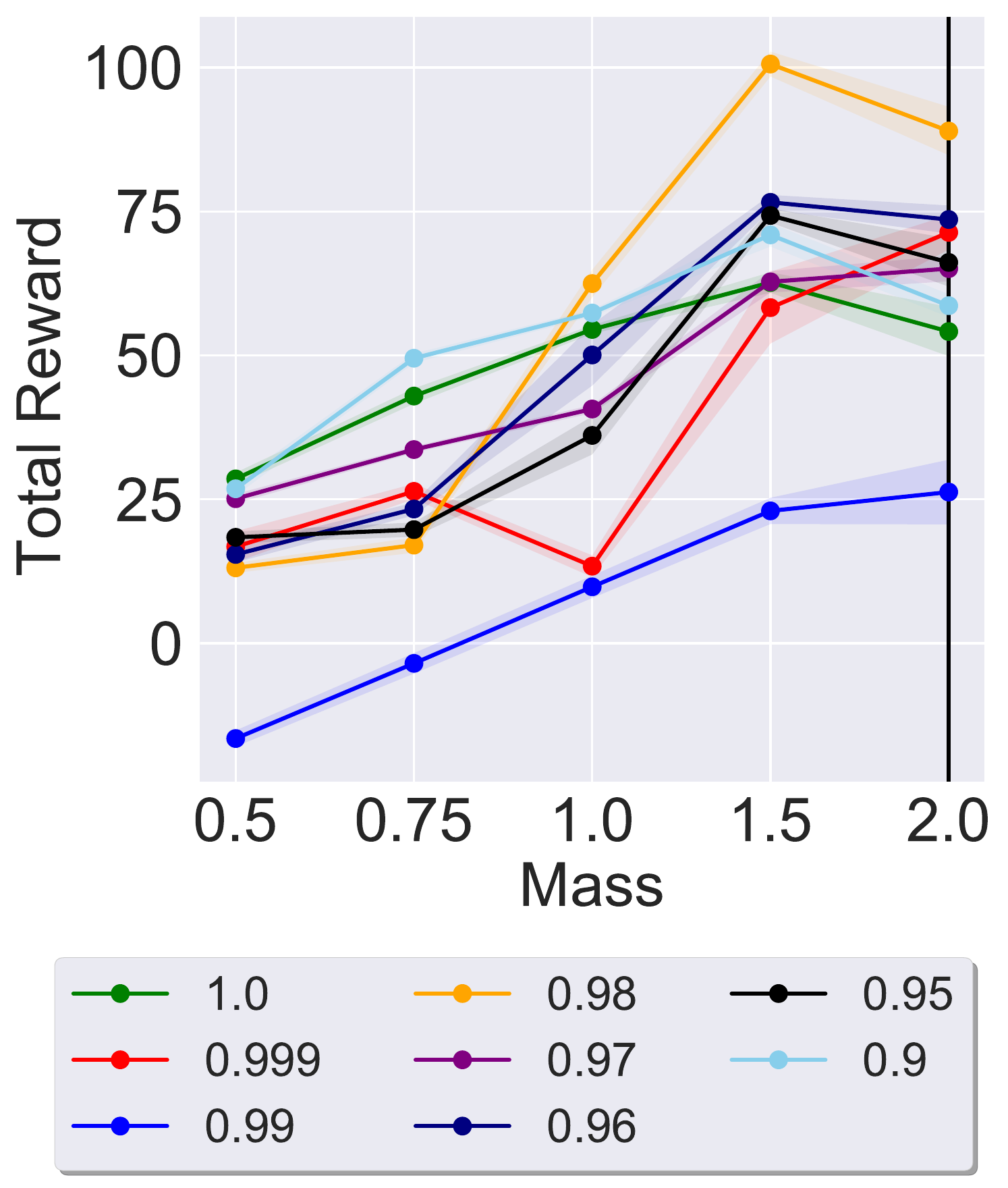}
     } \\

\end{tabular}
\caption{The average (over $3$ seeds) robust performance of Algorithm~\ref{alg:robust-gailfo} with different values of $\alpha$. The ablation shown here is used to choose $\alpha$ in Figure~\ref{fig:RobustnessMassFixedAlpha}. The expert environment $M^\mathrm{real}$, in which the demonstrations are collected, has relative mass $1.0$. In each plot, the black vertical line corresponds to the relative mass of the learner environment $M^\mathrm{sim}$ where we trained the policy with Algorithm~\ref{alg:robust-gailfo}. The x-axis denotes the relative mass of the test environment $M^\mathrm{test}$ in which the policies are evaluated. The policies are evaluated over $1e5$ steps truncating the last episode if it does not terminate. Note that robust-GAILfO with $\alpha = 1$ corresponds to GAILfO.}
\label{fig:RobustnessMassAblation}
\end{figure*}

\clearpage

\section{Transfer Performance: Continuous Gridworld}
\label{app:transfer-performance-grid}

\begin{figure}[!h] 
\centering
\begin{tabular}{ccc}
\subfloat[Expert $\epsilon = 0.0$]{%
       \includegraphics[width=0.3\linewidth]{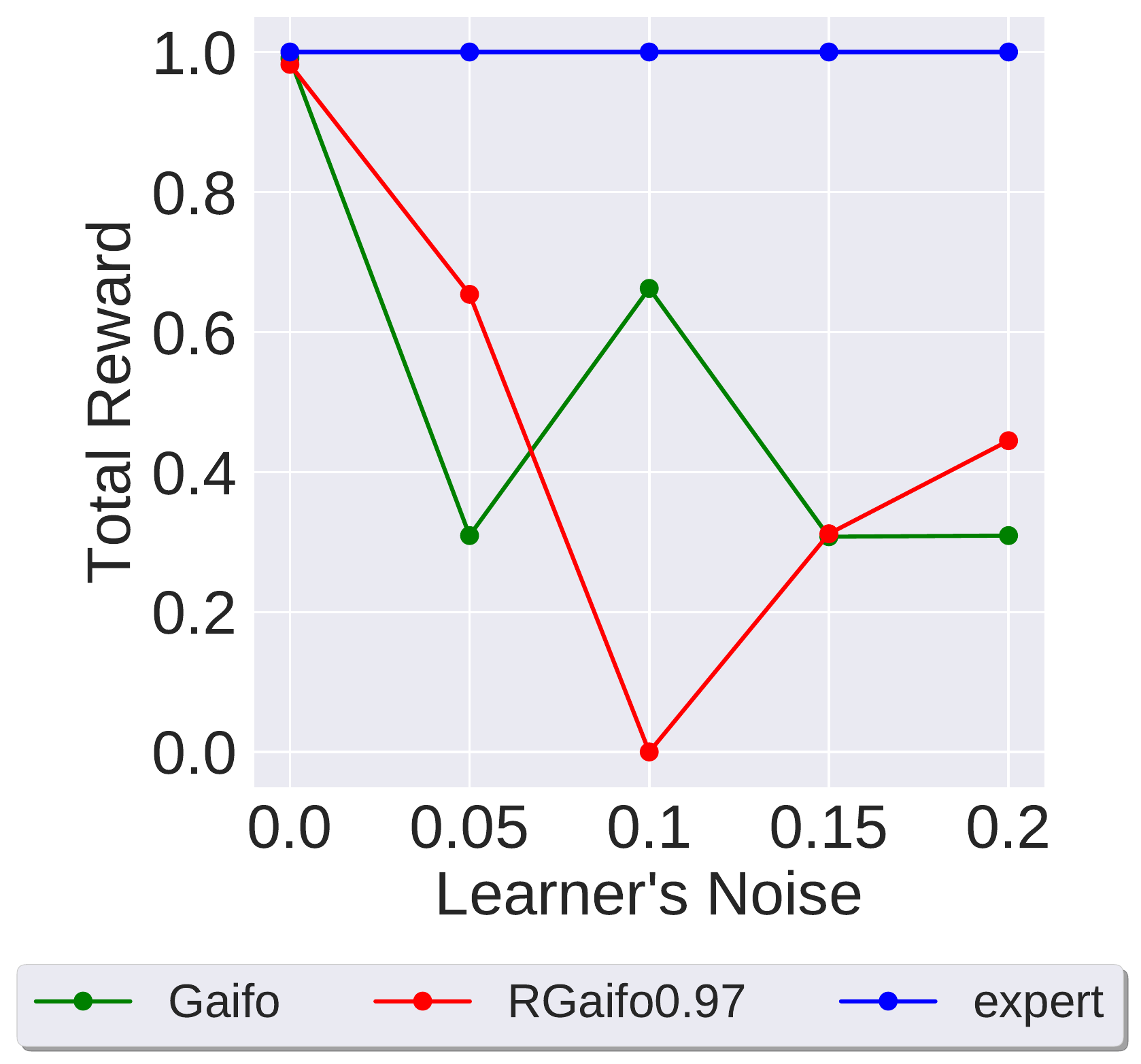}
     } &
\subfloat[Expert $\epsilon = 0.05$]{%
       \includegraphics[width=0.3\linewidth]{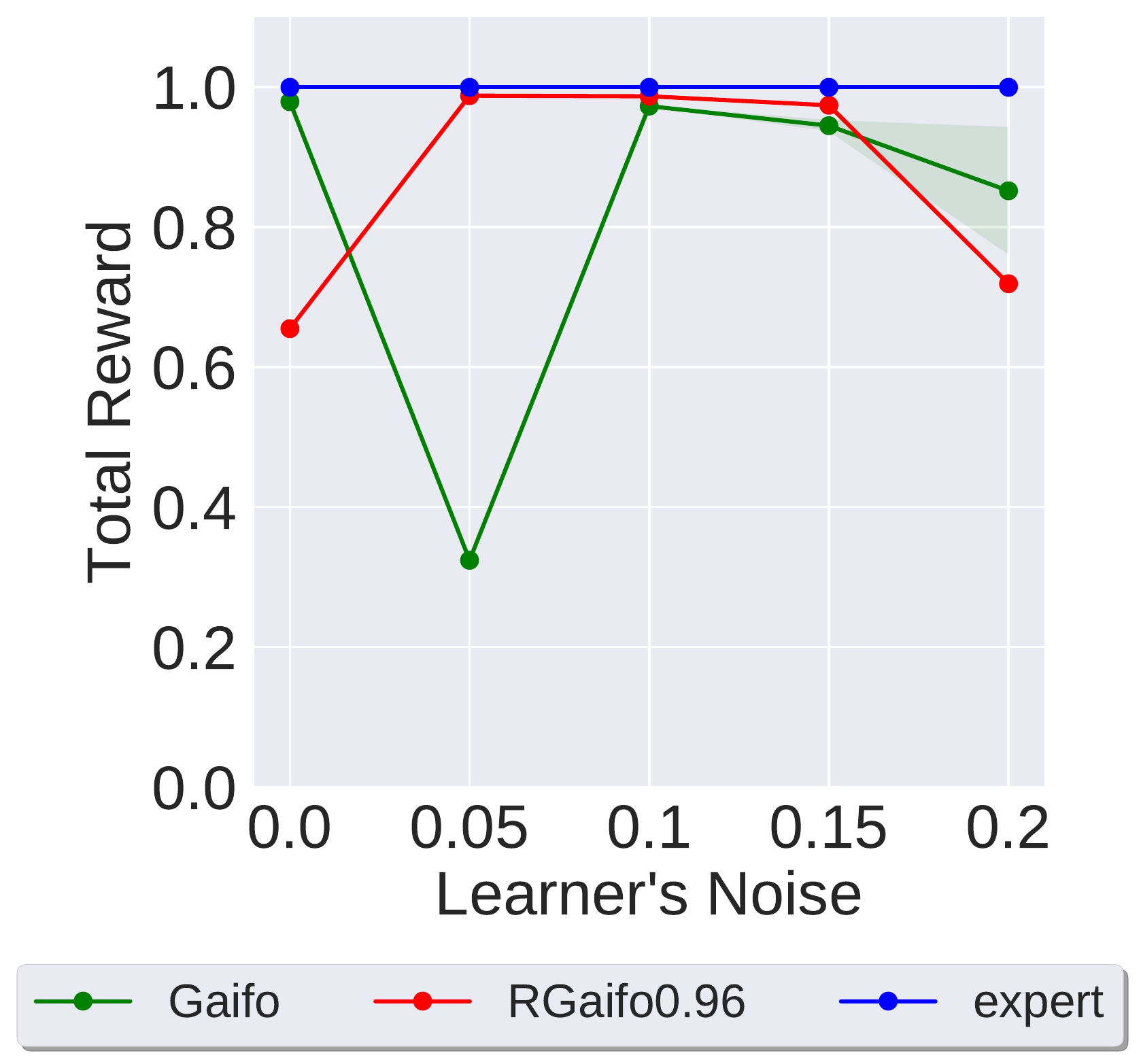}
     } &
\subfloat[Expert $\epsilon = 0.1$]{%
       \includegraphics[width=0.3\linewidth]{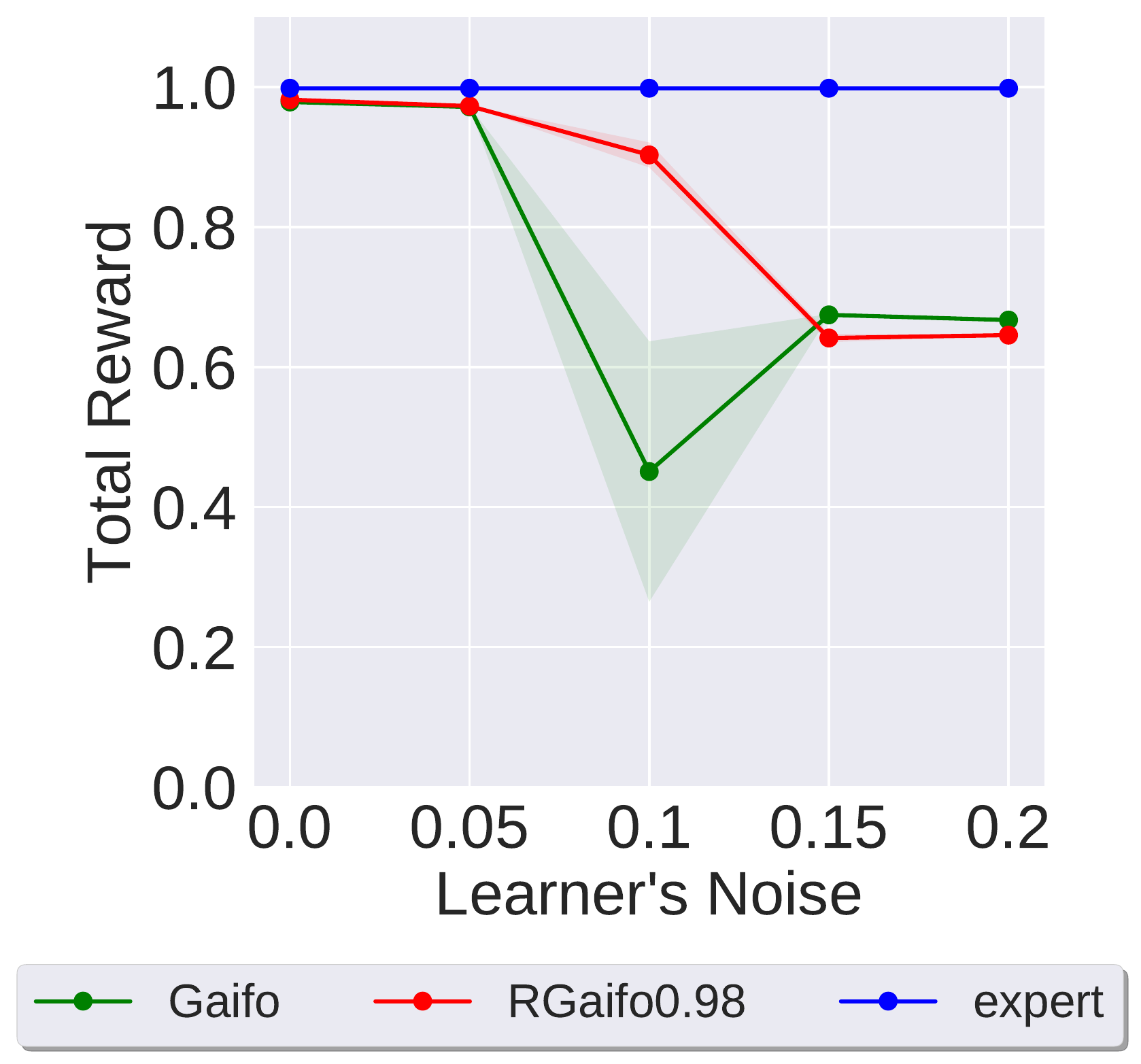}
     } \\
\end{tabular}
\caption{The average (over $3$ seeds) transfer performance of Algorithm~\ref{alg:robust-gailfo} with fixed value of $\alpha$ for each mismatch (i.e., each point on the x-axis) in the environment shown in Figure~\ref{fig:env_continuous}. The x-axis denotes the $\epsilon$ value of the learner environment. The policies are evaluated over $1e5$ steps truncating the last episode if it does not terminate. }
\label{fig:single_best_alpha_continuous}
\end{figure}

\begin{figure}[!h] 
\centering
\begin{tabular}{ccc}
\subfloat[Expert $\epsilon = 0.0$]{%
       \includegraphics[width=0.3\linewidth]{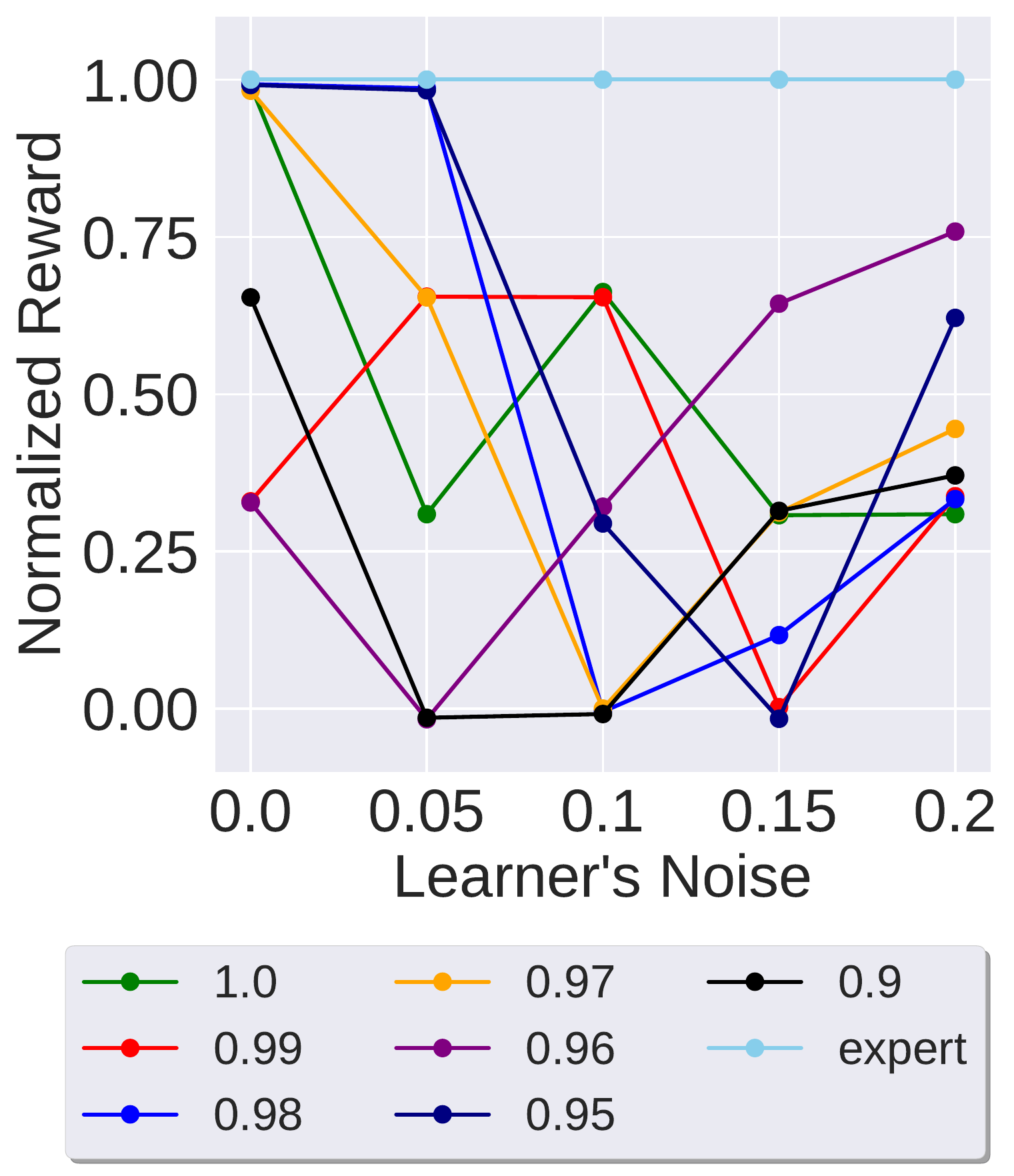}
     } &
\subfloat[Expert $\epsilon = 0.05$]{%
       \includegraphics[width=0.3\linewidth]{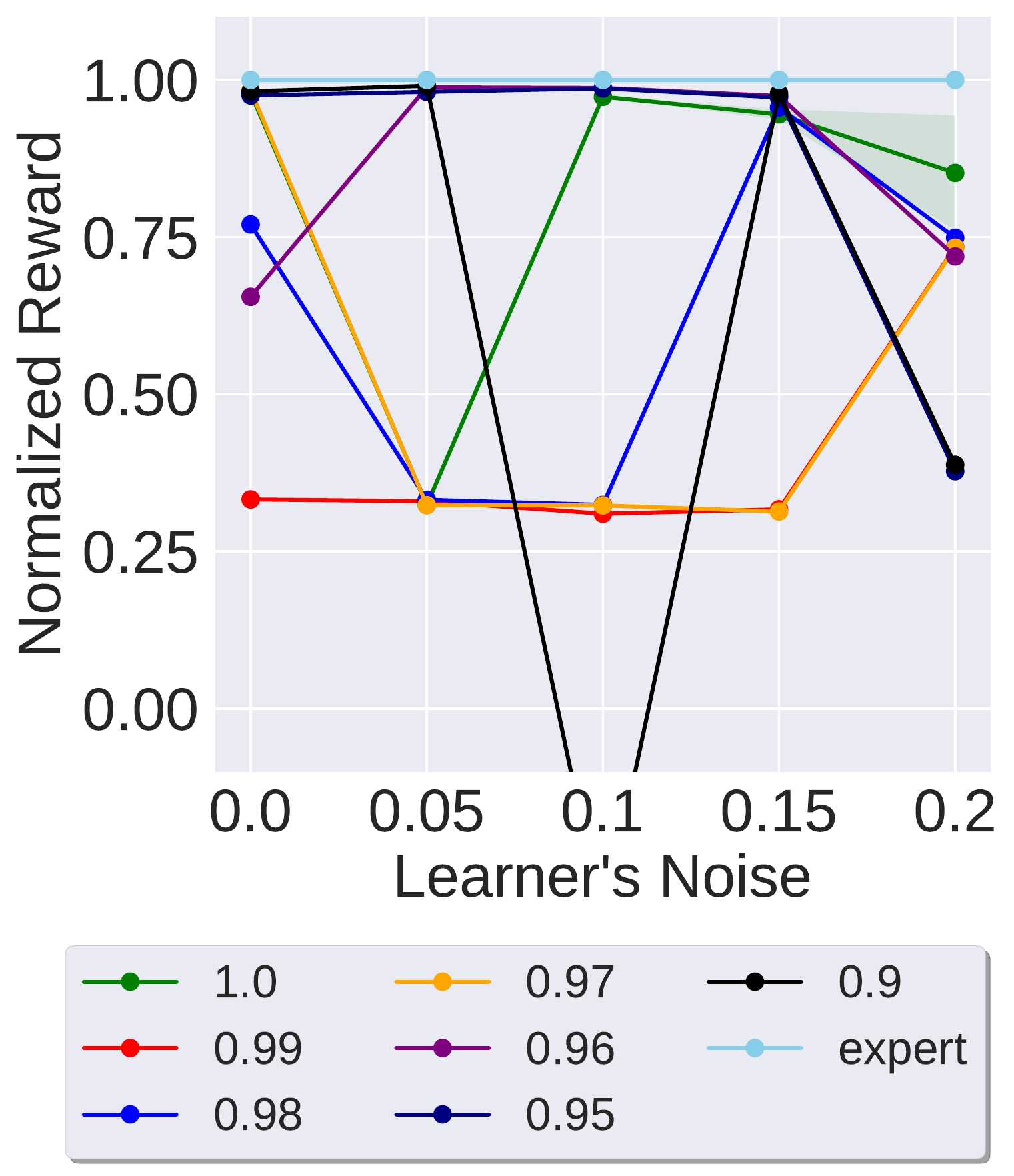}
     } &
\subfloat[Expert $\epsilon = 0.1$]{%
       \includegraphics[width=0.3\linewidth]{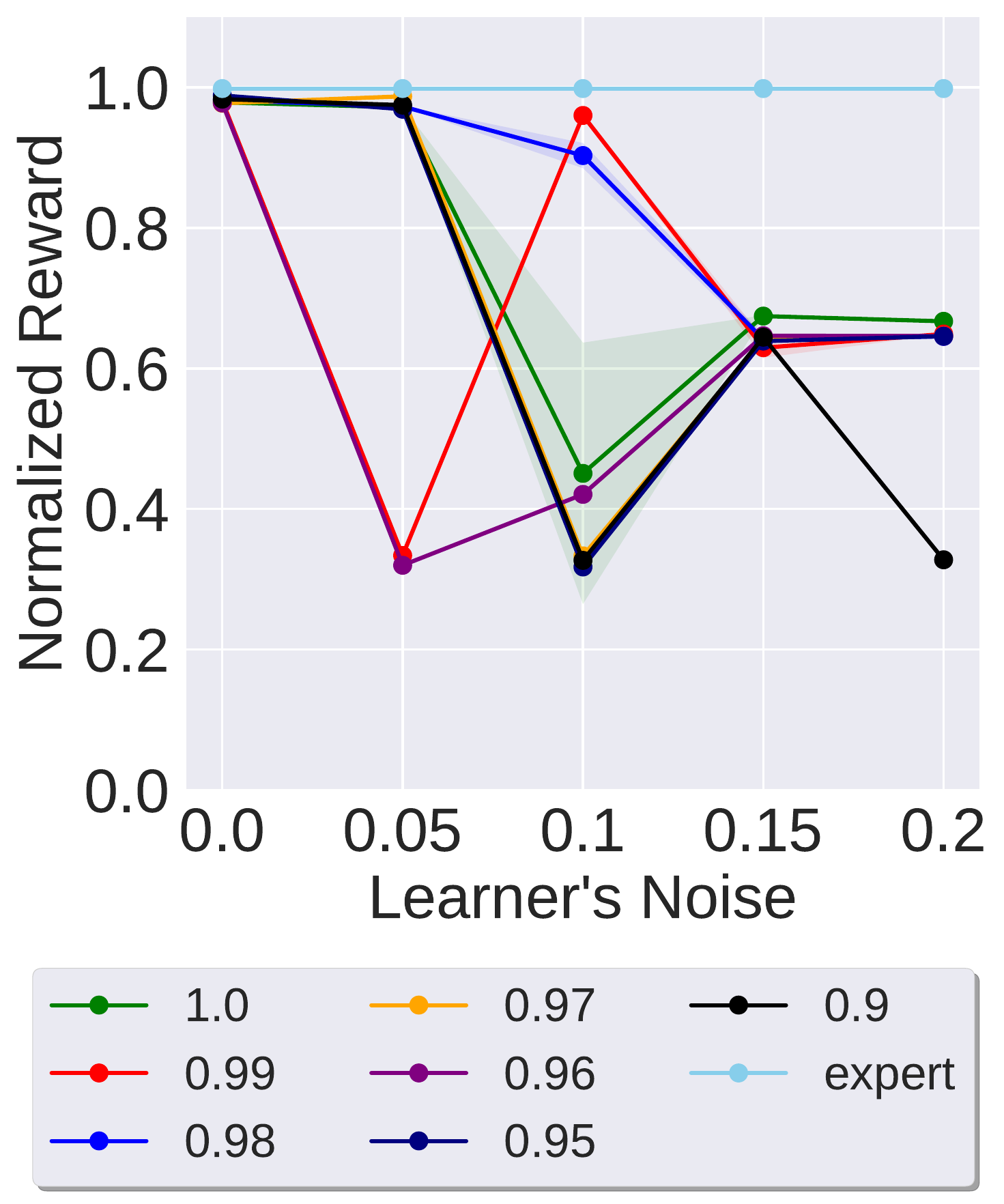}
     } \\
\end{tabular}
\caption{The average (over $3$ seeds) transfer performance of Algorithm~\ref{alg:robust-gailfo} with different values of $\alpha$ for each mismatch (i.e., each point on the x-axis) in the environment shown in Figure~\ref{fig:env_continuous}. The x-axis denotes the $\epsilon$ value of the learner environment. The policies are evaluated over $1e5$ steps truncating the last episode if it does not terminate. The ablation shown here is used to choose $\alpha$ in Figure~\ref{fig:single_best_alpha_continuous}. Note that robust-GAILfO with $\alpha = 1$ corresponds to GAILfO.}
 \label{fig:ablation_continuous}
\end{figure}

\section{Additional Experiments on Choice of $\alpha$}
\label{app:new_seeds}
In this section, we aim to understand whether our strategy of choosing suitable $\alpha$ value introduces maximization bias. For example, in Figure~\ref{fig:ablation_continuous}, the best performing $\alpha$ is chosen, and its performance curve (w.r.t. the original seeds used for training) is presented in Figure~\ref{fig:single_best_alpha_continuous}. To avoid this bias, for the chosen best performing $\alpha$ in Figure~\ref{fig:ablation_continuous}, we conduct a new set of runs with a new set of seeds. The new results presented in Figure~\ref{fig:seed_consistency} suggest that our $\alpha$ selection process does not introduce maximization bias. 


\begin{figure}[!h] 
\centering
\begin{tabular}{cc}
\subfloat[Expert $\epsilon = 0.0$; Original seeds]{%
       \includegraphics[width=0.3\linewidth]{plot/0_0_SingleBest.pdf}
     } &
\subfloat[Expert $\epsilon = 0.0$; New seeds]{%
       \includegraphics[width=0.3\linewidth]{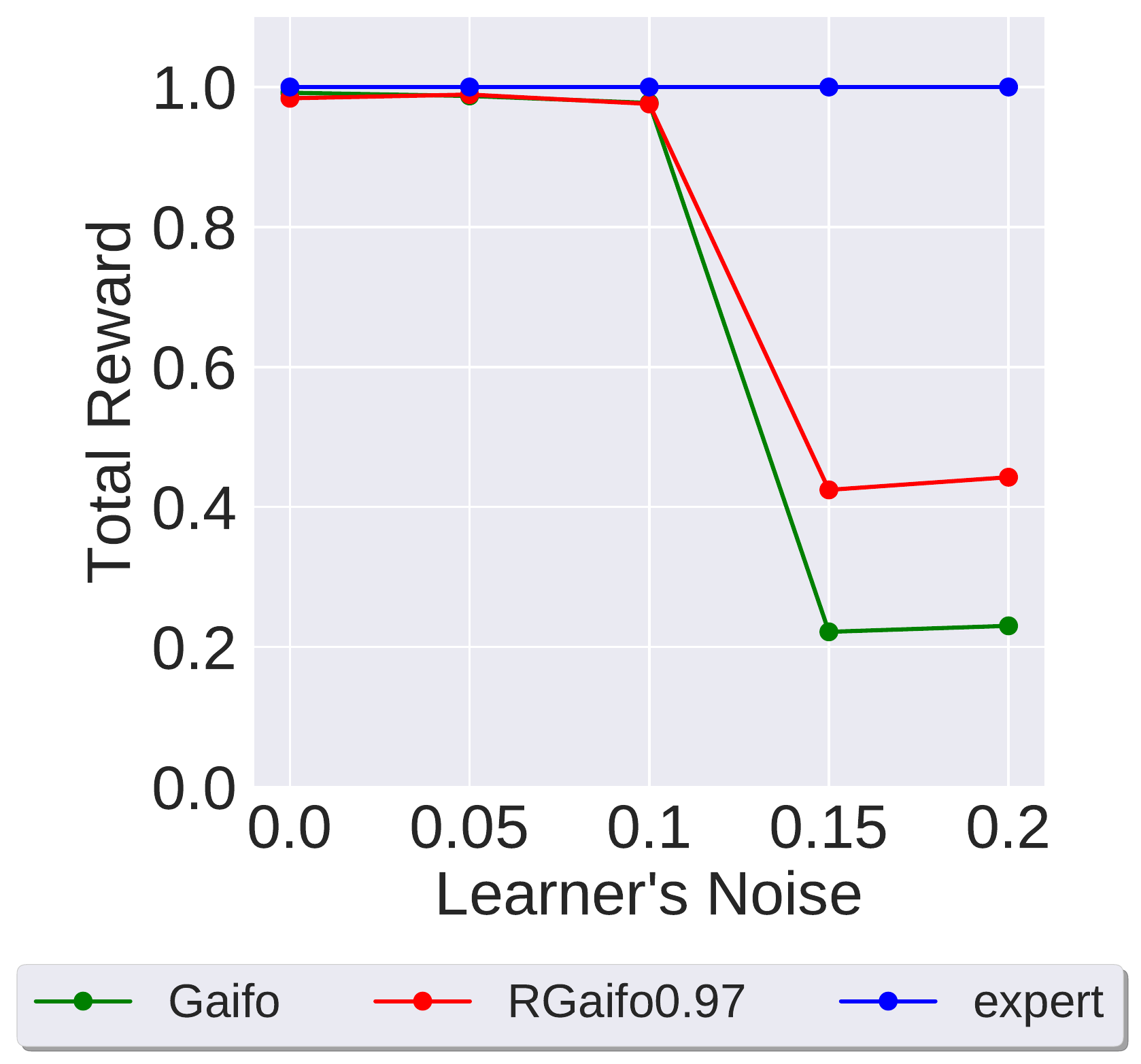}
     } \\
\subfloat[Expert $\epsilon = 0.05$; Original seeds]{%
       \includegraphics[width=0.3\linewidth]{plot/0_05_SingleBest.pdf}
     } &
\subfloat[Expert $\epsilon = 0.05$; New seeds]{%
       \includegraphics[width=0.3\linewidth]{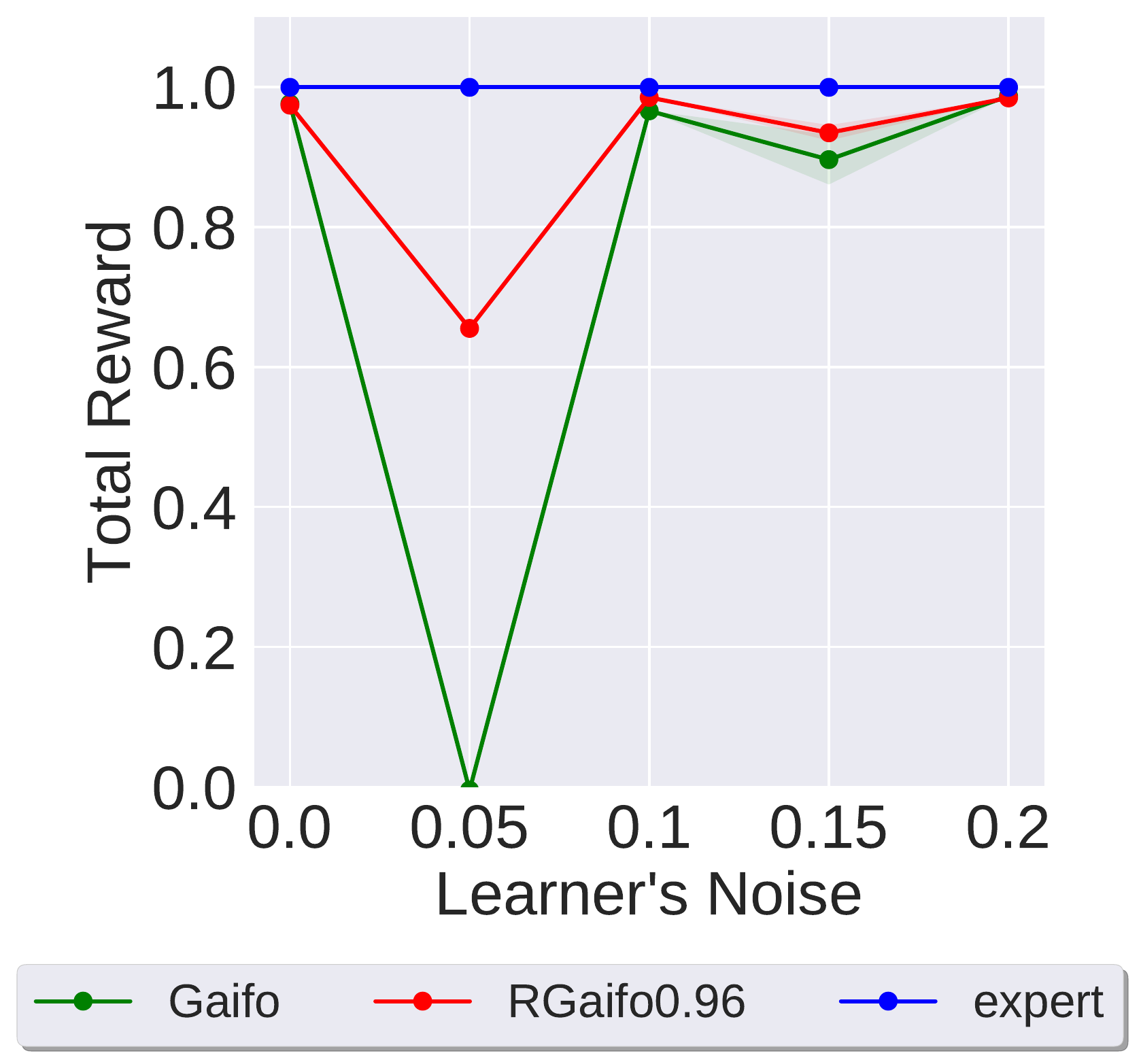}
     } \\
\subfloat[Expert $\epsilon = 0.1$; Original seeds]{%
       \includegraphics[width=0.3\linewidth]{plot/0_1_SingleBest.pdf}
     } &
\subfloat[Expert $\epsilon = 0.1$; New seeds]{%
       \includegraphics[width=0.3\linewidth]{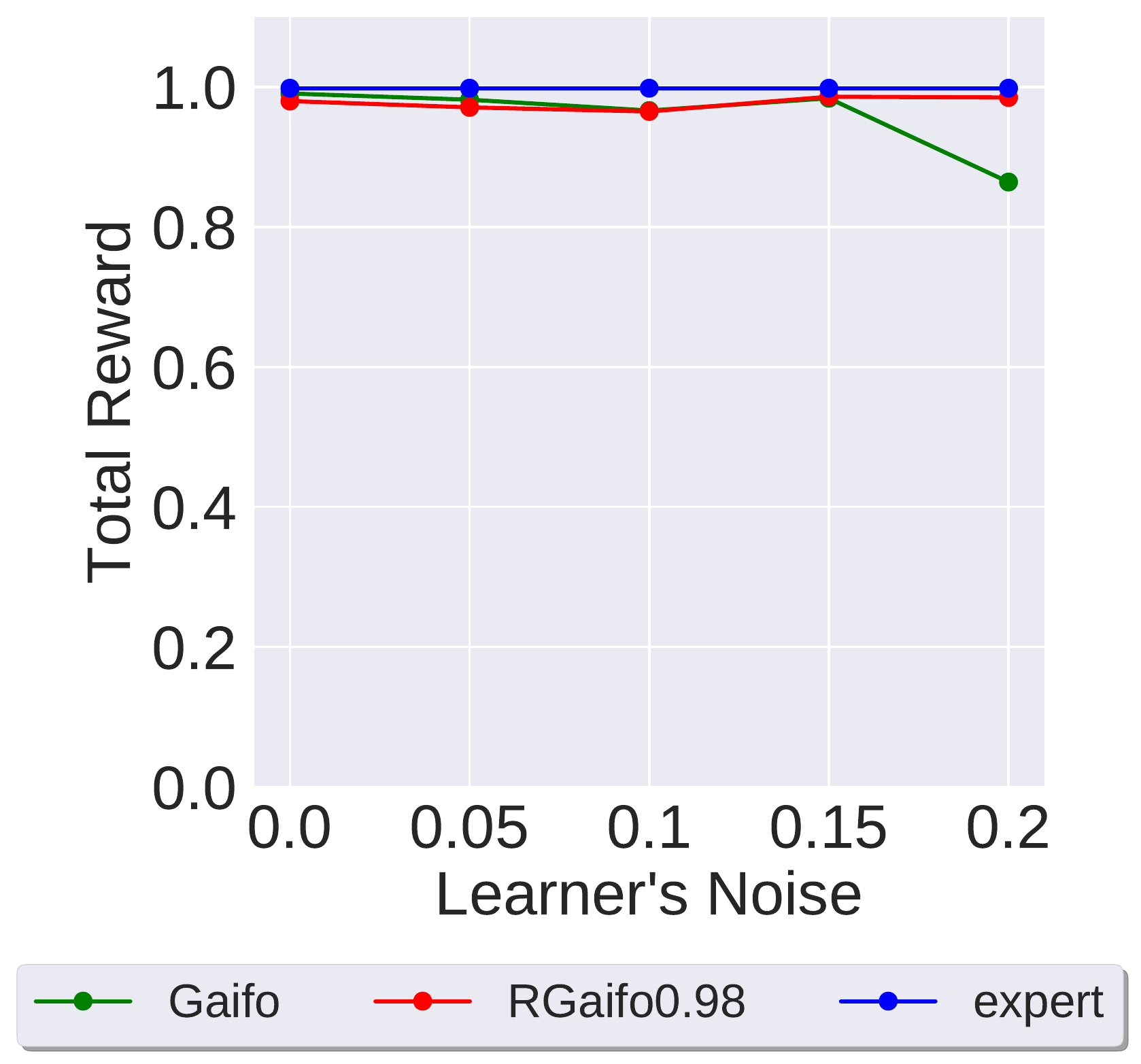}
     } \\
\end{tabular}

\caption{Experiments for understanding whether our strategy of choosing suitable $\alpha$ value introduces maximization bias.}
\label{fig:seed_consistency}
\end{figure}


\end{document}